\documentclass{article}

\usepackage[nonatbib,final]{neurips_2022}

\usepackage[square, numbers]{natbib}
\usepackage{enumitem}

\usepackage[utf8]{inputenc} 
\usepackage[T1]{fontenc}    
\usepackage{hyperref}       
\usepackage{url}            
\usepackage{booktabs}       
\usepackage{amsfonts}       
\usepackage{nicefrac}       
\usepackage{microtype}      
\usepackage{xcolor}         
\usepackage[pdftex]{graphicx}
\usepackage{subfigure}
\usepackage{bm}
\usepackage{amsmath}
\usepackage{amssymb}
\usepackage{amsthm}
\usepackage{float}
\usepackage{wrapfig}
\newtheorem{theorem}{Theorem}

\newcommand{\sha}{\Lambda}

\newcommand{\smallconst}{\kappa}
\newcommand{\sphere}{\mathbb{B}}
\newcommand{\bH}{\mathbb{H}}

\newcommand{\bx}{\mathbf{x}}
\newcommand{\bX}{\mathbf{X}}
\newcommand{\bY}{{\bm Y}}
\newcommand{\bM}{{\bm M}}
\newcommand{\bD}{{\bm D}}
\newcommand{\bF}{{\bm F}}
\newcommand{\bR}{{\bm R}}
\newcommand{\bA}{{\bm A}}
\newcommand{\bv}{{\bm v}}
\newcommand{\bW}{{\bm W}}
\newcommand{\be}{{\bm e}}

\allowdisplaybreaks[4]

\newtheorem{Assumption}{Assumption}[section]
\newtheorem{proposition}{Proposition}[section]
\newtheorem{lemma}{Lemma}[section]
\newtheorem{Corollary}{Corollary}[section]
\newtheorem{observation}{Observation}[section]

\title{Analyzing Sharpness along GD Trajectory: Progressive Sharpening and Edge of Stability}

\author{%
  Zhouzi Li\thanks{Contributed equally, listed in alphabetical order.} \\
  IIIS, Tsinghua University\\
  \texttt{zhouzi188763@gmail.com}
  \\
   \And
  Zixuan Wang$^{*}$\\
    IIIS, Tsinghua University\\
    \texttt{wangzx2019012326@gmail.com} \\
    \AND
    Jian Li\thanks{The authors are supported in part  by the National Natural Science Foundation of China Grant 62161146004, Turing AI Institute of Nanjing and Xi'an Institute for Interdisciplinary Information Core Technology.} \\
  IIIS, Tsinghua University \\
  \texttt{lapordge@gmail.com}
}

\begin{document}

\maketitle

\begin{abstract}
Recent findings demonstrate that modern neural networks trained by full-batch gradient descent typically enter a regime called Edge of Stability (EOS). In this regime, the sharpness, i.e., the maximum Hessian eigenvalue, first increases to the value 2/(step size) (the progressive sharpening phase) and then oscillates around this value (the EOS phase). 
This paper aims to analyze the GD dynamics and the sharpness along the optimization trajectory.
Our analysis naturally divides the GD trajectory into four phases depending on the change in the sharpness value. We empirically identify the norm of output layer weight as an interesting indicator of the sharpness dynamics. Based on this empirical observation, we attempt to theoretically and empirically explain the dynamics of various key quantities that lead to the change of the sharpness in each phase of EOS. Moreover, based on certain assumptions, we provide a theoretical proof of the sharpness behavior in the EOS regime in two-layer fully-connected linear neural networks. We also discuss some other empirical findings and the limitation of our theoretical results.
\end{abstract}

\section{Introduction}

Deep learning has achieved great success in a variety of machine learning applications, and gradient-based algorithms are the prevailing optimization methods for training deep neural networks. However, mathematically understanding the behavior of the optimization methods for deep learning is highly challenging, due to non-convexity, over-parameterization, and complicated architectures. 
In particular, some recent empirical findings in deep networks contradict the traditional understandings of gradient methods. For example, \citet{wu2018sgd} observed that the solution found by gradient descent has sharpness approximately equal to $2/\eta$ instead of just being smaller than $2/\eta$. Also, \citet{jastrzebski2020break} observed that there is a break-even point in the SGD trajectory, and after this point, there is a regularization effect on the loss curvature.

One recent well-known example is the phenomenon called ``Edge of Stability" (EOS) (\citet{cohen2021gradient}). Based on the classical optimization theory, 
the learning rate $\eta$ of gradient-based method should be smaller than $2/\lambda$ so that the loss can decrease, where $\lambda$ is the largest eigenvalue of the Hessian of the objective, also called ``sharpness'' in the literature. Otherwise, the loss diverges (even for simple quadratic functions). However, the empirical findings in \citet{cohen2021gradient} show that under various network settings, the EOS phenomena typically occurs along the gradient descent trajectory: (1) the sharpness first increases until it reaches $2/\eta$ (called ``progressive sharpening'') (2) the sharpness starts hovering around $2/\eta$ (the EOS regime) and (3) the loss non-monotonically decreases without diverging.

Although (1) seems to be consistent with the traditional beliefs about optimization, a rigorous mathematical explanation for it is still open. Moreover, phenomena (2) and (3) are more mysterious because they violate the $\eta<2/\lambda$ ``rule'' in traditional optimization theory, yet the training loss does not completely diverge. Instead, the loss may oscillate but still decrease 
in the long run, while the sharpness seems to be restrained from further increasing. 

In this paper, we aim to provide a theoretical and empirical explanation for the mystery of EOS. Towards the goal, we focus on the dynamics of these key quantities when EOS happens and attempt to find out the main driving force to explain these phenomena along the gradient descent trajectory from both theoretical and empirical perspectives.



\subsection{Our Contributions}
Our contributions can be summarized as follows.

  (Section~\ref{Four-phase def}) We analyze the typical sharpness behavior along the gradient descent trajectory when EOS happens, and propose a four-phase division of GD trajectory, based on the dynamics of some key quantities such as the loss and the sharpness, for further understanding this phenomenon. 
    
    (Section~\ref{subsection A norm}) We empirically identify the weight norm of the output layer as an effective indicator of the sharpness dynamics. We show that analyzing the dynamics of this surrogate can qualitatively explain the dynamics of sharpness. 
    By assuming this relation, 
    together with some additional simplifying assumptions and approximations, we can explain the dynamics of the sharpness, the loss, and the output layer norm in each phase of EOS (Section~\ref{subsection 4-phase proof}). 
    In this context, we also offers an interesting explanation for the non-monotonic loss decrement (also observed in \citet{cohen2021gradient,xing2018walk}) (Section~\ref{subsection loss decrement}). 
    
 (Section~\ref{theoretical analysis}) Following similar ideas, we provide a more rigorous proof for the progressive sharpening and EOS phenomena in a two-layer fully-connected linear neural network setting based on certain assumptions. The assumptions made here are either weaker or arguably less restrictive.

\subsection{Related work}
\textbf{The structure of Hessian} The Hessian matrix carries the second order information of the loss landscape. Several prior works have empirically found that the spectrum of Hessian has several ``outliers'' and a continuous ``bulk'' (\citet{sagun2016eigenvalues, sagun2017empirical, papyan2018full, papyan2019measurements}).
Typically, each outlier corresponds to one class in multi-class classification. As we consider the binary classification setting, there is typically one outlier (i.e., the largest eigenvalue) that is much larger than other eigenvalues. It is consistent with our Assumption~\ref{eigenspectrum assumption}.
The Gauss-Newton decomposition of the Hessian was used in several prior works
(\citet{martens2016second, bottou2018optimization, papyan2018full, papyan2019measurements}).
\citet{papyan2018full} empirically showed that the outliers of Hessian can be attributed to a ``G component'', which is also known as Fisher Information Matrix (FIM) in \citet{karakida2019normalization, karakida2019pathological}. Also, \citet{wu2020dissecting} analyzed the leading Hessian eigenspace by approximating the Hessian with Kronecker factorization and theoretically proved the outliers structure under some random setting assumption. 

\noindent\textbf{Neural Tangent Kernel } A recent line of work studied the learning of over-parameterized neural networks in the so-called. “neural tangent kernel (NTK) regime or the lazy training regime (\citet{jacot2018neural,lee2019wide, du2018gradient,du2019gradient,arora2019fine,chizat2019lazy}). A main result in this regime is that if the neural network is wide enough, gradient flow can find the global optimal empirical minimizer very close to the initialization. Moreover, the Hessian does not change much in the NTK regime. Our findings go beyond NTK setting to analyze the change of sharpness.

\noindent\textbf{Edge of Stability regime } The Edge of Stability phenomena was first formalized by \citet{cohen2021gradient}.  Similar phenomena were also identified in \citet{jastrzebski2020break} as the existence of the ``break-even'' point on SGD trajectory after which loss curvature gets regularized. \citet{xing2018walk} observed that gradient descent eventually enters a regime where the iterates oscillate on the leading curvature direction and the loss drops non-monotonically. 
Recently \citet{ahn2022understanding} studied the non-monotonic decreasing behavior of GD which they called unstable convergence, and discussed the possible causes of this phenomenon. \citet{ma2022multiscale} 
proposed a special subquadratic landscape property and proved that EOS occurs based on this assumption. 
\citet{arora2022understanding} studied the implicit bias on the sharpness of deterministic gradient descent in the EOS regime. They proved in some specific settings with a varying learning rate (called normalized GD) or with a modified loss $\sqrt{L}$, gradient descent enters EOS and further reduces sharpness. 
They mainly focus on the analysis near the manifold of minimum
loss, but our analysis also applies to the early stage of the training when the loss is not close to the minimum. 
In particular, our analysis provides an explanation of non-monotonic loss decrease that cannot be explained by their theory.
Another difference is that they consider $\sqrt{L}$ (for constant learning rate) where $L$ is a fairly general MSE loss independent of
any neural network structure, while our analysis is strongly tied with 
the MSE loss of a neural network. 
Very recently, \citet{lyu2022understanding} explained how GD enters EOS
for normalized loss (e.g., neural networks with normalization layers), and analyzed the sharpness reduction effect along the training trajectory. The notion of sharpness in their work is somewhat different due to normalization. In particular, they consider the so-called {\em spherical sharpness}, that is the sharpness of the normalized weight vector. They also mainly studied the regime where the parameter is close to the manifold of minimum
loss as in \cite{arora2022understanding} and proved that GD approximately tracks a continuous sharpness-reduction flow.
\citet{lewkowycz2020large} proposed a similar regime called ``catapult phase'' where loss does not diverge even if the largest Hessian eigenvalue is larger than $2/\eta$. Our work mainly considers training in this regime and assumes that the training is not in the ``divergent phase'' in \citet{lewkowycz2020large}. 
Compared with \citet{lewkowycz2020large}, we provide a more detailed analysis in more general settings along gradient descent trajectory. 

\section{Preliminaries}\label{2}

\textbf{Notations:}
We denote the training dataset as $\{\mathbf{x}_i, y_i\}_{i=1}^n\subset \mathbb R^d \times \{1, -1\}$ and the neural network as $f: \mathbb R^d\times \mathbb R^p\rightarrow \mathbb R.$ The network $f(\bm\theta, \mathbf{x})$ maps the input $\mathbf{x} \in \mathbb R^d$ and parameter $\bm\theta \in \mathbb R^p$ to an output in $\mathbb R$. 
In this paper, we mainly consider the case of binary classification with mean square error (MSE) loss $ \ell(z, y) = (z-y)^2$.

Denote the input matrix as $\mathbf{X} = (\mathbf{x}_1, \mathbf{x}_2, ..., \mathbf{x}_n)\in \mathbb{R}^{d\times n}$ and the label vector as 
$\bm{Y} = (y_1, y_2,..., y_n)\in \mathbb R^n$. 
We let $\bm{F}(t) = (f(\bm \theta(t),\mathbf{x}_1), f(\bm \theta(t),\mathbf{x}_2),...,f(\bm \theta(t),\mathbf{x}_n))\in \mathbb{R}^n$ and $\bm{D}(t) = \bm{F}(t)-\bm{Y}$ be the (output) prediction vector, and the residual vector respectively at time $t$.  
The training objective is:
 $
 \mathcal L(f(\bm\theta)) = \frac{1}{n}\sum_{i=1}^n \ell(f(\bm\theta, \mathbf{x}_i),y_i)
 =\frac{1}{n}\sum_{i=1}^n (f(\bm\theta, \mathbf{x}_i),y_i)^2.
 $

\noindent\textbf{Hessian, Fisher information matrix and NTK:}
In this part, we apply previous works to show that the largest eigenvalue of Hessian is almost the same as the largest eigenvalue of NTK. We use the latter as the definition of \textbf{the sharpness} in this paper. Further details can be found in Appendix~\ref{appendix: prelim}.

As shown in \citet{papyan2019measurements, martens2016second, bottou2018optimization}, the Hessian can be
decomposed into two components, where the term known as ``Gauss-Newton matrix'', G-term or Fisher information matrix (FIM), dominates the second term in terms of the largest eigenvalue.
Meanwhile, \citet{karakida2019pathological} pointed out the duality between the FIM and a Gram matrix $\bM$, defined as  
$
    \bm{M} = \frac{2}{n}\frac{\partial \bF(\bm\theta)}{\partial {\bm\theta}} \frac{\partial \bF(\bm\theta)}{\partial {\bm\theta}}^\top
$.
It is also known as the neural tangent kernel NTK
 (\citet{karakida2019pathological,karakida2019normalization}), which has been studied extensively in recent years (see e.g., \cite{jacot2018neural},\cite{du2018gradient},\cite{arora2019fine},\cite{chizat2019lazy}).
Note that in this paper, we do not assume the training is in NTK regime,
in which the Hessian does not change much during training.
It is not hard to see that $\bM$ and FIM share the same non-zero eigenvalues: 
if $\bm G \bm u = \lambda \bm u$ for some eigenvector $\bm u\in \mathbb R^p$, $\bM\frac{\partial \bF(\bm\theta)}{\partial {\bm\theta}} \bm u = \frac{\partial \bF(\bm\theta)}{\partial {\bm\theta}}\bm G  \bm u = \lambda \frac{\partial \bF(\bm\theta)}{\partial {\bm\theta}}\bm u$,
i.e., $\lambda$ is also an eigenvalue of $\bM$.


In this paper, 
we use $\bm\theta(t)$ to denote the parameter at iteration $t$ (or time $t$) and the sharpness at time $t$ as $\sha(t)=\sha(\bm\theta(t))$.
We similarly define $\bM(t), \bF(t), \bD(t),\mathcal{L}(t)$.

Here we show the gradient flow dynamics of the residual vector $\bm D(t)$:
\begin{equation}
    \frac{\mathrm d\bm{D}(t)}{\mathrm dt} = \frac{\partial \bD(t)}{\partial {\bm\theta}}\frac{\mathrm d\bm{\theta}(t)}{\mathrm d t}  = - \frac{\partial \bF(t)}{\partial {\bm\theta}}\frac{\partial \mathcal L(t)}{\partial{{\bm\theta}}} = -\frac{2}{n}  \frac{\partial \bF(t)}{\partial {\bm\theta}} \frac{\partial \bF(t)}{\partial {\bm\theta}}^\top\bm{D}(t) = - \bm{M}(t)\bm{D}(t)
    \label{residual dynamics}
\end{equation}


\section{A Four-phase Analysis of GD Dynamics}
\label{Section 3}
In this section, we study the dynamics of gradient descent
and the change of sharpness along the optimization trajectory.
We divide the whole training process into four phases,
occurring repeatedly in the EOS regime. In Section \ref{Four-phase def}, we introduce the four phases. In Section~\ref{subsection A norm}, we show empirically that 
the change of the norm of the output layer weight vector almost coincides with the change of the sharpness. In Section~\ref{subsection 4-phase proof}, using this observation, we attempt to explain the dynamics of each phase and provide a mathematical explanation for the changes in the sharpness. 
In Section~\ref{subsection loss decrement}, we explain why the loss decreases but non-monotonically. 
We admit that a completely rigorous theoretical explanation is still beyond our reach and much of our argument is based on various simplifying assumptions and is somewhat heuristic at some points. Due to space limits, we defer all the proofs in this section to Appendix \ref{appendix: proofs in sec3}.

\subsection{A Four-phase Division}
\label{Four-phase def}
To further understand the properties along the trajectory when EOS happens, we study the behaviors of the loss and the sharpness during the training process. 
As illustrated in Figure 1, we train a shallow neural network by gradient descent on a subset of 1,000 samples from CIFAR-10 (\citet{krizhevsky2009learning}), using the MSE loss as the objective. 
Notice that the sharpness keeps increasing while the loss decreases until the sharpness reaches $2/\eta$. Then the sharpness begins to oscillate around $2/\eta$ while the loss decreases non-monotonically. This is a typical sharpness behavior in the EOS regime, and consistent with the experiments in \cite{cohen2021gradient}.

We divide the training process into four phases according to the evolution of the loss, the sharpness, and their correlation, as shown in Figure 1.
The four phases happen cyclically along the training trajectory.
We first briefly describe the properties of each phase and explain the dynamics in more detail in Section~\ref{3.3}.


\textbf{Phase I:} Sharpness $\sha < 2/\eta$.
In this stage, all the eigenvalues of Gram matrix $\bm{M}$ are below the threshold $2/\eta$. 
In particular, using standard initialization, the training typically starts from this phase, and during this phase 
the loss keeps decreasing and the sharpness keeps growing along the trajectory.
This initial phase is called {\em progressive sharpening (PS)} in prior work \citet{cohen2021gradient}. 
Empirically, the behavior of GD trajectory (as well as the loss and the sharpness) is very similar to that of gradient flow, until the sharpness reaches $2/\eta$ 
(this phenomena is also observed in \citet{cohen2021gradient}. See Figure 5 or Appendix J.1 in their paper).
We note that GD may come back to this phase from Phase IV later.

\textbf{Phase II:} Sharpness $\sha > 2/\eta$. 
In this phase, the sharpness exceeds $2/\eta$ and may keep increasing.
We will show shortly that the fact that $\sha > 2/\eta$ 
causes $|\bm{D}^\top\bm{v}_1|$ (where $\bm{v}_1$ the first eigenvector of $\bM$) to increase exponentially (Lemma~\ref{divergence proposition}).
This would quickly lead $\|\bm{D}\|$ to exceed $\|\bY\|$
in a few iterations,
which leads the sharpness to start decreasing by Proposition \ref{prop:DlargerthanY}, hence the training process enters Phase III.

\textbf{Phase III:} Sharpness $\sha > 2/\eta$ yet begins to 
gradually drop. 
Before $\sha$ drops below $2/\eta$,
Lemma~\ref{divergence proposition} still holds, so
$| \bm{D}^\top \bm{v}_1|$ keeps increasing.
Proposition~\ref{prop:DlargerthanY} still holds and thus
the sharpness keeps decreasing until it is below
$2/\eta$, at which point we enter Phase IV.
A distinctive feature of this phase is that the loss may increase
due to the exponential increase of 
$| \bm{D}^\top \bm{v}_1|$.

\textbf{Phase IV:} Sharpness $\sha < 2/\eta$. 
When the sharpness is below $2/\eta$, $| \bm{D}^\top \bm{v}_1|$ begins to decrease quickly, leading the loss to decrease quickly.
At the same time, the sharpness keeps oscillating and gradually decreasing for some iterations. This lasts until the loss decrease to a level that is around its value right before Phase III. The sharpness is still below $2/\eta$ and 
our training process gets back to Phase I. 

\begin{figure}[ht]
    \vspace{-0.5cm}
    \centering
    \subfigure[Sharpness and Loss]{
    \includegraphics[width=0.44\textwidth]{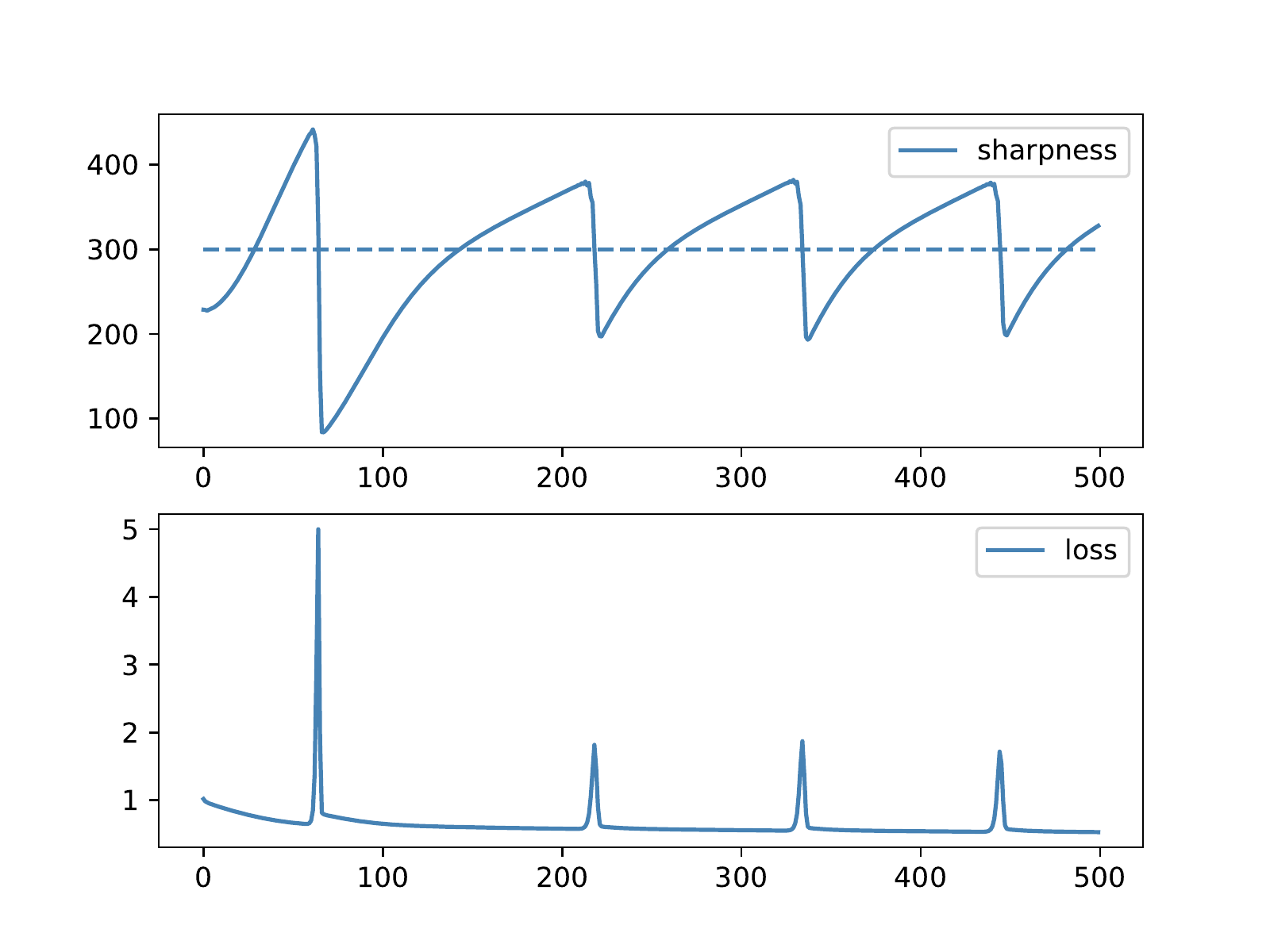}
    }
    \subfigure[Phase Division Illustration]{
    \label{Illustration}
    \includegraphics[width=0.44\textwidth]{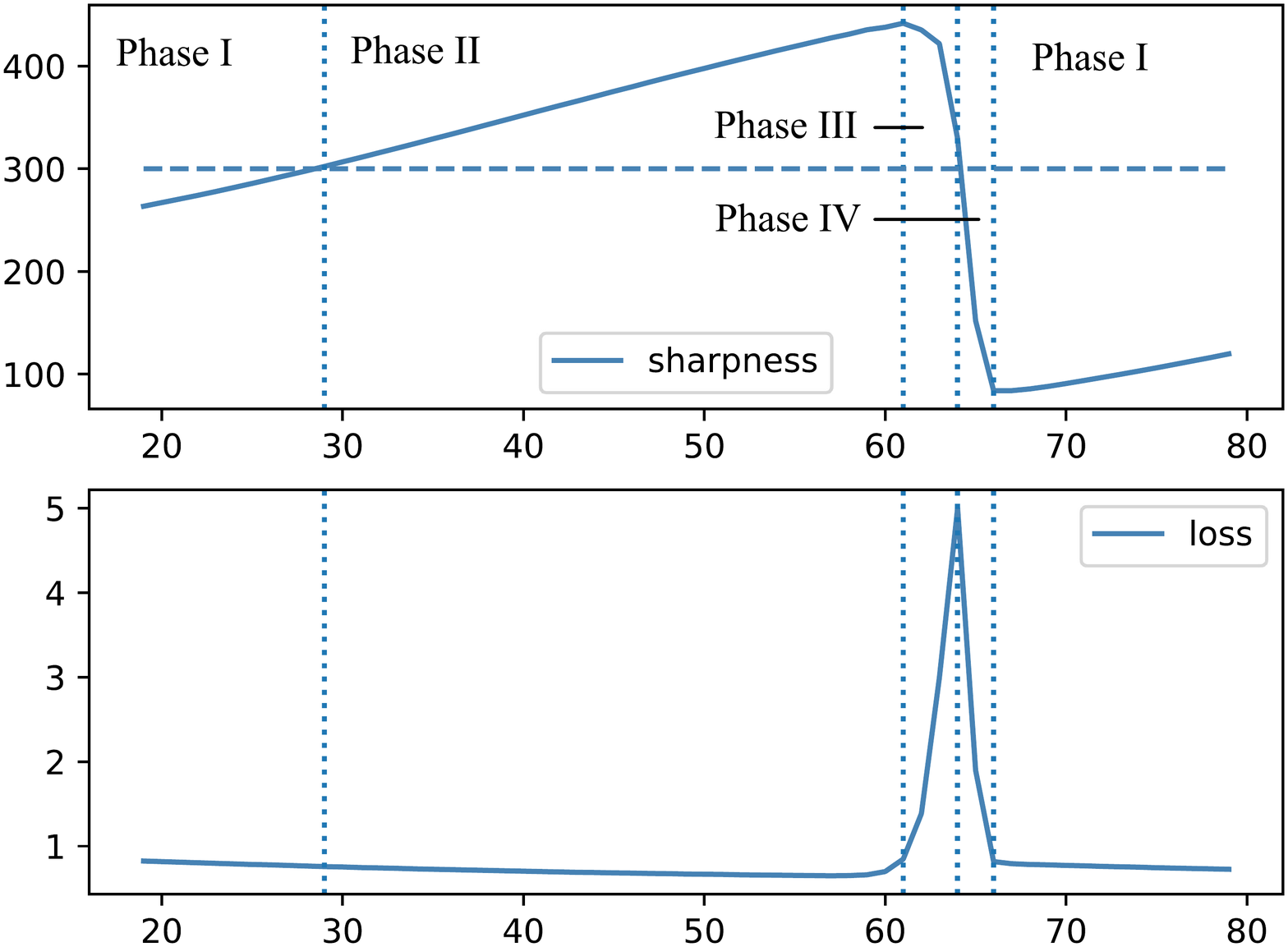}
    }
    \caption{In Figure (a), we show the sharpness and the loss when training on a two-hidden-layer linear activated network by gradient descent. Shortly after the sharpness crosses $2/\eta$, it drops quickly back to some value below $2/\eta$.  In Figure (b) we illustrate how the four phases are divided according to the evolution of the sharpness.}
    \label{Diagram & illustration}
    \vspace{-0.3cm}
\end{figure}

\subsection{The Norm of the Output Layer Weight}\label{3.2}
\label{subsection A norm}
It is difficult to rigorously analyze the dynamics of the sharpness $\sha(t)$.
In this subsection, we make an interesting observation, that 
the change of the norm of
the output layer of the network (usually a fully-connected linear layer)
is consistent with the change of the sharpness most of the time.

In particular, for a general neural network $f(\bx) = \bm A^\top \bm h(\bm W, \bx)$, where $\bm A \in \mathbb R^{m}$ is the output layer weight and the feature extractor $\bm h:\mathbb{R}^p\times \mathbb{R}^d\rightarrow \mathbb R^m$ outputs a $m$-dimensional feature vector
($\bm h$ corresponds to all but the last layers). 
$\bm W\in \mathbb{R}^p$ is the parameter vector of the extractor $\bm h$.

Note that
$\bm M = (\frac{\partial \bm{F}}{\partial \bm{\theta}})^\top (\frac{\partial \bm{F}}{\partial \bm{\theta}})$ can be decomposed as follows:
$$
\bm{M} = \left(\frac{\partial \bm{F}}{\partial \bm{\theta}}\right) \left(\frac{\partial \bm{F}}{\partial \bm{\theta}}\right)^\top 
= \left(\frac{\partial \bm{F}}{\partial \bm{A}}\right) \left(\frac{\partial \bm{F}}{\partial \bm{A}}\right)^\top + \left(\frac{\partial \bm{F}}{\partial \bm{W}}\right) \left(\frac{\partial \bm{F}}{\partial \bm{W}}\right)^\top := \bm {M_A} +\bm {M_W}. 
$$ 
where the $(i,j)-$entry of $\bm{M_W}$ is $ (\bm{M_W})_{ij} = \left\langle \frac{\partial f(\bx_i)}{\partial \bm W},\frac{\partial f(\bx_j)}{\partial \bm W}\right\rangle = \bm A^\top \frac{\partial \bm h(\bm W,\bx_i)}{\partial \bm W}\frac{\partial \bm h(\bm W,\bx_j)}{\partial \bm W}^\top\bm A.$

In this expression, intuitively $\|\bm A\|$ should be positively related to $\|\bM_{\bm W}\|$. We empirically observe that the part $\bm M_{\bm A} = (\frac{\partial \bm{F}}{\partial \bm{A}}) (\frac{\partial \bm{F}}{\partial \bm{A}})^\top$ has a much smaller spectral norm compared to the whole Gram matrix $\bm M$ (see Figure~\ref{sha assumption fig} and Appendix~\ref{appendixD}), which means $\|\bm M_{\bm W}\|$ dominates $\|\bm M_{\bm A}\|$.
Therefore, $\|\bm{A}\|$
should be positively correlated with $\|\bm{M}\|$. 

The benefit of analyzing $\|\bm{A}\|^2$ is that 
the gradient flow of $\|\bm{A}\|^2$ enjoys the following clean formula: 
\begin{equation}
    \frac{\mathrm {d}\|\bm{A}\|^2}{\mathrm{d}t} = - 2 \left(\frac{\partial \mathcal L}{\partial \bm{A}}\right)^\top \bm{A} = -\frac{4}{n} \bm{D}^\top\left(\frac{\partial\bm{F}}{\partial \bm{A}}\right) \bm{A} = -\frac{4}{n}\bm{D}^\top\bm{F}.
    \label{A norm dynamics}
\end{equation} 
In this work, we do experiments on two-layer linear networks, fully connected deep neural networks, and Resnet18, and all of them have such output layer structures. From Figure~\ref{sha assumption fig}, we can observe that \textbf{the output layer norm
$\|\bm{A}\|^2$ and the sharpness $\sha$ change in the same direction} most of the time along the gradient descent trajectory, i.e., they both increase or decrease at the same time. 
We note that they may change in different directions very occasionally 
around the time when $\|\bm{A}(t+1)\|^2-\|\bm A(t)\|^2$ changes its sign
(see the experiments in Figure~\ref{fig:Anorm and sharpness}).

\subsection{Detailed Analysis of Each Phase}\label{3.3}
\label{subsection 4-phase proof}

In this section, we explain the dynamics of each phase in more detail.
For clarity, we first list the assumptions we need in this section. For different phases, we may need some different assumptions to simplify the arguments. Most of the assumptions are consistent with the experiments or the findings in the literature. Some of them are somewhat stronger, and we also discuss how to relax them.

\subsubsection{Assumptions Used in Section~\ref{3.3}}

\begin{Assumption}
\label{Ass: A norm assumption}
(A-norm and sharpness) 
Along the gradient descent training trajectory, for all time 
$t$, the norm $\|\bm{A}(t)\|$ of the output layer and the sharpness 
$\sha(t)$ moves in the same directions, i.e.,
$\mathrm{sign}(\sha(t+1) - \sha(t)) = \mathrm{sign}(\|\bm{A}(t+1)\| - \|\bm{A}(t)\|)$.
\end{Assumption}

It is the key observation that we have discussed in Section~\ref{3.2}. 
The following are two assumptions about the gradient descent trajectory. 
The first one assumes that $\bm{D}(t)$ and $\|\bm{A}\|^2$ are updated
according their first order approximations. Empirical justification of this approximation 
can be found in Appendix~\ref{Appen: order one apprx}.
\begin{figure}[ht]
    \vspace{-0.5cm}
    \centering
    \subfigure[Two-layer Linear Network]{
    \label{linear fig}
    \includegraphics[width=0.31\textwidth]{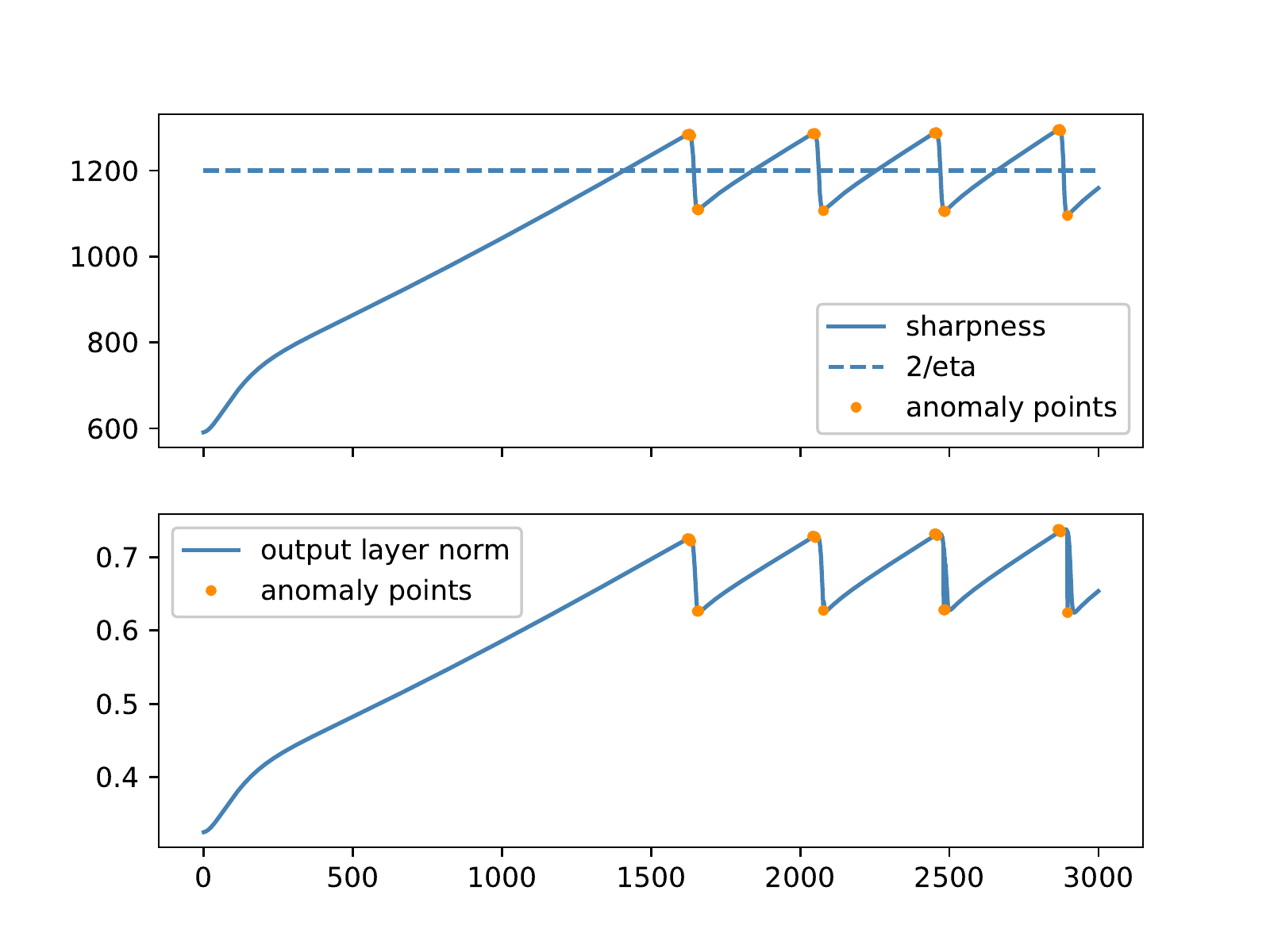}
    }
    \subfigure[Deep 5-layer tanh network]{
    \label{tanh fig}
    \includegraphics[width=0.31\textwidth]{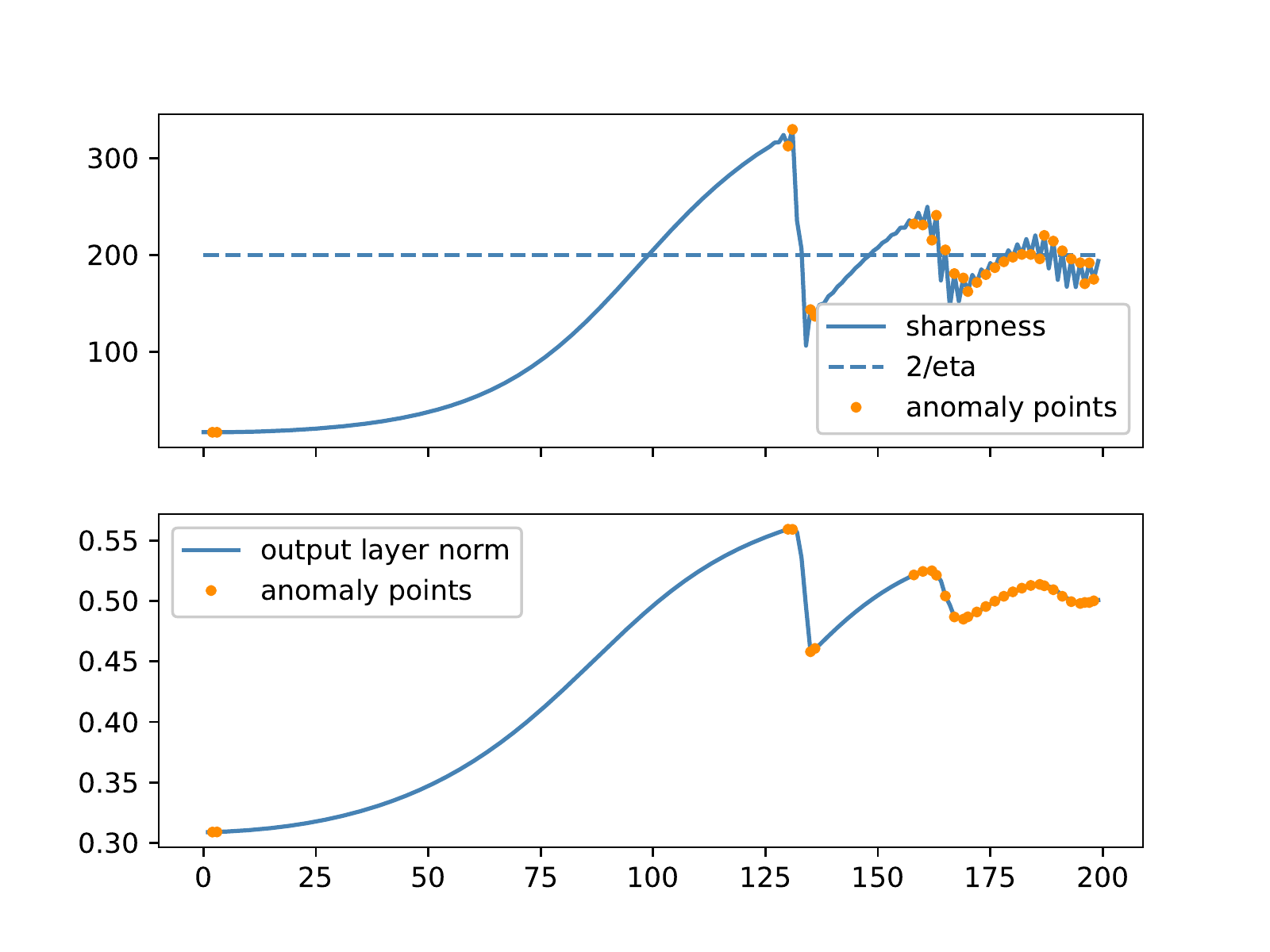}
    }
    \subfigure[Resnet18]{
    \label{Resnet fig}
    \includegraphics[width=0.31\textwidth]{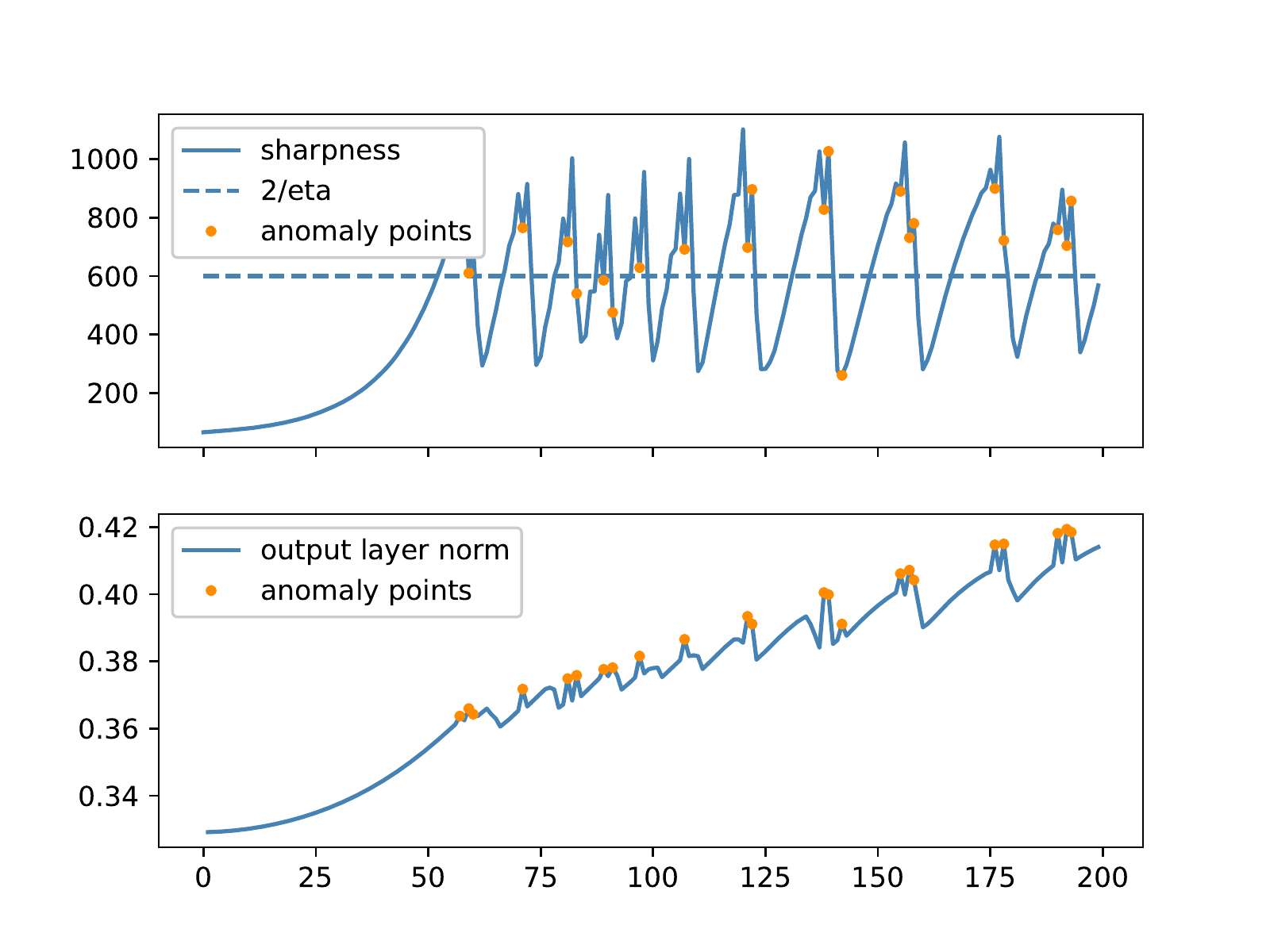}
    }
    \caption{Sharpness and A norm are highly correlated; {\em Anomaly points} (orange points) in the figure correspond to the iterations when $\mathrm{sign}(\sha(t+1) - \sha(t))\ne \mathrm{sign}(\|\bm{A}(t+1)\| - \|\bm{A}(t)\|)$. More discussions 
    can be found in Section \ref{5} and Appendix~\ref{appendixB}.}
    \label{fig:Anorm and sharpness}
    \vspace{-0.3cm}
\end{figure}

\begin{Assumption}
\label{Ass: Order one app}
(First Order Approximation of GD) Along the gradient descent trajectory, the update rule is assumed as the first order approximation
\begin{equation}
    \bm{D}(t+1) - \bm{D}(t) = - \eta \bm{M}(t)\bm{D}(t),\ 
    \|\bm{A}(t+1)\|^2 - \|\bm{A}(t)\|^2 = -\frac{4\eta}{n}\bm{D}(t)^\top\bm{F}(t)
    \label{Order one dynamics}
\end{equation}
\end{Assumption}

\begin{Assumption} (Gradient flow for the PS phase)
When $\sha(t) < 2/\eta$, $\bD(t)$ follows the gradient flow trajectory:
$\frac{\mathrm d\bD(t)}{\mathrm d t} = - \bM(t) \bD(t)$.
    \label{gf assumption}
\end{Assumption}
Assumption~\ref{gf assumption} holds empirically, especially in the progressive sharpening phase (see Figure 5 or Appendix J.1 in \citet{cohen2021gradient}) when the networks are continuously differentiable. We
include these experimental details in Appendix~\ref{appendixD}.
See also  (Theorems 4.3 and 4.5) in \citet{arora2022understanding} for further theoretical justification. We need this assumption for the proof in the progressive sharpening phase.

Then we state an assumption on the upper bound of the sharpness to restrict the regime we discuss:
\begin{Assumption}
(Sharpness upper bound) If the training does not diverge, there exists some constant $B_\sha$, such that  
$
0 < \sha(t) \leq \frac{B_\sha}{\eta}
$
for all $t$.
\label{Ass: sha upper bound section 3} 
\end{Assumption}

This assumption states that there is an upper bound of the sharpness throughout the optimization process. 
Actually, in \citet{lewkowycz2020large}, they proved that $4/\eta$ is an upper bound of the sharpness in a two-layer linear network with one datapoint, otherwise the training process (loss) would diverge.  They empirically found that similar upper bounds exist also for nonlinear activations, albeit with somewhat larger constant $B_\sha$.
In the work, We focus on the case when the loss does not diverge and hence we make Assumption~\ref{Ass: sha upper bound section 3}.

The main set of assumptions we need is about the change of $\bM$'s eigendirections.

\begin{Assumption}
\label{Ass: eigspace assumptions}
Denote $\{\bm{v}_i\}_{i = 1}^n$ to be the set of eigenvectors of $\bM (t)$. 
We have three levels of assumptions on $\bM$'s eigenspace.
\begin{enumerate}[label=(\roman*)]
    \item (fixed eigendirections) the set $\{\bm{v}_i\}_{i = 1}^n$
     is fixed throughout the phase under consideration;
    \item (eigendirections move slowly) at all time $t$ and for any $i$,  
    $\bF(t)^\top \frac{\mathrm d \bv_i(t)}{\mathrm d t} < \lambda_i(t)\bD(t)^\top \bv_i(t)$; 
    \item (principal directions moves slowly) at all time $t$, there is a small constant $\epsilon_2\geq 0$ such that
$\langle \bv_1(t), \bv_1(t+1) \rangle \geq 1-\epsilon_2$.
\end{enumerate}
\end{Assumption}

Clearly, these three assumptions are increasingly weaker from (i) to (iii).
Assumption~\ref{Ass: eigspace assumptions} (i) 
on the eigenvectors is somewhat strong, and the eigenvectors corresponding to 
small eigenvalues may change notably in our experiments. We use it to illustrate a basic proof idea of the progressive sharpening phase, but later we relax this assumption to Assumption~\ref{Ass: eigspace assumptions} (ii). 
 
Moreover, for the proof in Phase II and III, Assumption~\ref{Ass: eigspace assumptions} (iii), which only assume that the main direction changes slightly, is sufficient for our proof. Actually, we note that $\bv_1(t)$ (the eigenvector corresponding to the largest eigenvalue) changes slowly and the inner product of its initial direction and its direction at the end of the phase is also large (see Appendix~\ref{appendixD} for the empirical verification). 

For the proof in Phase II, we need another small technical assumption:
\begin{Assumption}
Assume $\bm{D}(t)^\top\bm{v}_1(t) \geq c \epsilon_2 \|\bm{D}(t)\|$ for some $c > 1$ 
for some $t=t_0$ at the beginning of this phase. Here $\epsilon_2$ is defined in Assumption~\ref{Ass: eigspace assumptions} (iii).
\label{noise assumption}
\end{Assumption}

Assumption \ref{noise assumption} says that
$\bm{D}(t)$ has a non-negligible component in the direction of $\bv_1$.
Since $\epsilon_2>0$ is a small constant, 
this is not a strong assumption as some small perturbation
(due to discrete updates) would make the assumption hold for some $c>1$.

\subsubsection{Detailed Analysis}
In each phase, we attempt to explain the main driving force of the change of the sharpness and the loss. 

\textbf{Phase I:} In this phase, we show that $\bD(t)^\top\bF(t) < 0$ under certain assumptions (detailed shortly) on the spectral properties of $\bM(t)$ (see Lemma~\ref{ps lemma} below). By Assumption~\ref{Ass: Order one app}, we have $\|\bm A(t+1)\|^2 - \|\bm A(t)\|^2 > 0$, implying that the sharpness $\sha(t)$ also increases based on Assumption~\ref{Ass: A norm assumption}. This phase stops if $\sha(t)$ grows larger than $2/\eta$.

We assume the output vector $\bm{F}(t)$ is initialized to be small (this is true if we use very small initial
weights). For simplicity, we assume $\bm{F}(0) = 0$ in the following argument.

\begin{lemma}
For all $t$ in Phase I, 
under Assumption~\ref{Ass: eigspace assumptions} (i) and \ref{gf assumption},
it holds that $\bD(t)^\top\bF(t) < 0$.\label{ps lemma}
\end{lemma}
\vspace{-0.1cm}



From this lemma, $\|\bA\|$ keeps increasing by Assumption~\ref{Ass: Order one app}; hence the sharpness keeps increasing by Assumption~\ref{Ass: A norm assumption} until it reaches $2/\eta$ or the loss converges to 0. In the former case, the training process enters Phase II, while the latter case is also possible when $\eta$ is very small (e.g., even the largest possible sharpness value is less than $2/\eta$). We admit that Assumption~\ref{Ass: eigspace assumptions} (i) is somewhat strong. In fact, the assumption can be relaxed significantly to Assumption~\ref{Ass: eigspace assumptions} (ii).
We show in Appendix~\ref{Subsection: relaxed PS} that under Assumption~\ref{Ass: eigspace assumptions} (ii) and Assumption~\ref{gf assumption}, we can still guarantee $\bD(t)^\top\bF(t)<0$. Moreover, we provide a dynamical system view of the dynamics of $\bD(t)^\top\bF(t)$ in that Appendix.

\textbf{Phase II:} When the training process just enters Phase II, the sharpness keeps increasing. We show shortly that $\bm{D}(t)^\top\bm v_1(t)$ starts to increase geometrically, and this causes the sharpness to stop increasing at some point, thus entering Phase III.

In this phase, we adopt a weaker assumption on the sharpness direction $\bv_1$:  Assumption~\ref{Ass: eigspace assumptions} (iii). This assumption holds in our experiments 
(See Figure~\ref{Largest direction}). Also, Assumption~\ref{noise assumption} is necessary.

\begin{lemma}
\label{lm:Dincrease_exponentially}
Suppose Assumption~\ref{Ass: eigspace assumptions} (iii) and \ref{noise assumption}
hold during this phase (with constants $\epsilon_2>0$ and $c>1$).
If $\sha(t) = (2 + \tau)/\eta$ and $\tau > \frac{1}{1-\epsilon_2 - 1/c}-1$, then $\bm{D}(t)^\top\bm{v}_1(t)$ increases geometrically with factor $(1+\tau)(1-\epsilon_2 - 1/c) > 1$
for $t\geq t_0$ in this phase.
\label{divergence proposition}
\end{lemma}
\vspace{-0.05cm}
Since $\bm{D}(t)^\top\bm{v}_1(t)$ increases geometrically, 
$\|\bD\|\geq \bm{D}(t)^\top\bm{v}_1(t)$ will exceed $\|\bY\|$ eventually. 
Next, the following proposition states that when this happens, 
$\bm{D}(t) ^\top\bm{F}(t)>0$. Consequently, $\|\bm{A}\|$ decreases by Assumption~\ref{Ass: Order one app},
leading to the decrement of the sharpness based on our Assumption~\ref{Ass: A norm assumption}. 

\begin{proposition}
\label{prop:DlargerthanY}
If $\|\bm{D}(t)\| > \|\bm{Y}\|$, then $\bm{D}(t) ^\top\bm{F}(t)>0$.
\end{proposition}

\textbf{Phase III:} The sharpness is still larger than $2/\eta$, but it starts decreasing. Meanwhile, the loss continues to increase rapidly due to Lemma~\ref{lm:Dincrease_exponentially}. 
Eventually, the sharpness will fall below $2/\eta$ and then the training process enters phase IV.

By Lemma~\ref{lm:Dincrease_exponentially}, if the sharpness stays above $2/\eta$, then we can have an arbitrarily large loss. According to Proposition~\ref{prop:DlargerthanY}, if the loss is large enough, the sharpness keeps decreasing.

Now we show that if the sharpness stays above $2/\eta$, 
$\|\bA(t)\|^2$ will decrease by a significant amount. 
This partially explains that the sharpness should also decrease significantly until it drops below $2/\eta$ (instead of decreasingly converging to a value above $2/\eta$ without ever entering the next phase).

\begin{proposition}
\label{prop:Adrop}
Under Assumption~\ref{Ass: Order one app},
if $\|\bD(t)\| > \|\bY\|$, then $\|\bm A(t+1)\|^2 - \|\bm A(t)\|^2 < -\frac{4\eta}{n}(\|\bD(t)\| - \|\bY\|)^2$.
\end{proposition}

From the above argument, we can see that if
$\bm D(t)^\top\bm v_1(t)$ is larger than $\|\bY\|$, then $\bm D(t)^\top\bm v_1(t)$ does not decrease in Phase III, and according to Proposition~\ref{prop:Adrop}, $\|\bA(t)\|^2$ decreases significantly, implying the sharpness drops below $2/\eta$ eventually. 

\textbf{Remark:} The fact that the sharpness can provably drop below $2/\eta$ in this phase can be proved more rigorously in Section~\ref{theoretical analysis} for the two-layer linear setting. See Theorem~\ref{EOS main thm}.

\begin{figure}[ht]
    \vspace{-0.8cm}
    \centering
    \subfigure[Sharpness, $\lambda_{\max}(\bm{M_A})$, and 
    $\lambda_2(\bM)$.]{
    \label{sha assumption fig}
    \includegraphics[width=0.43\textwidth]{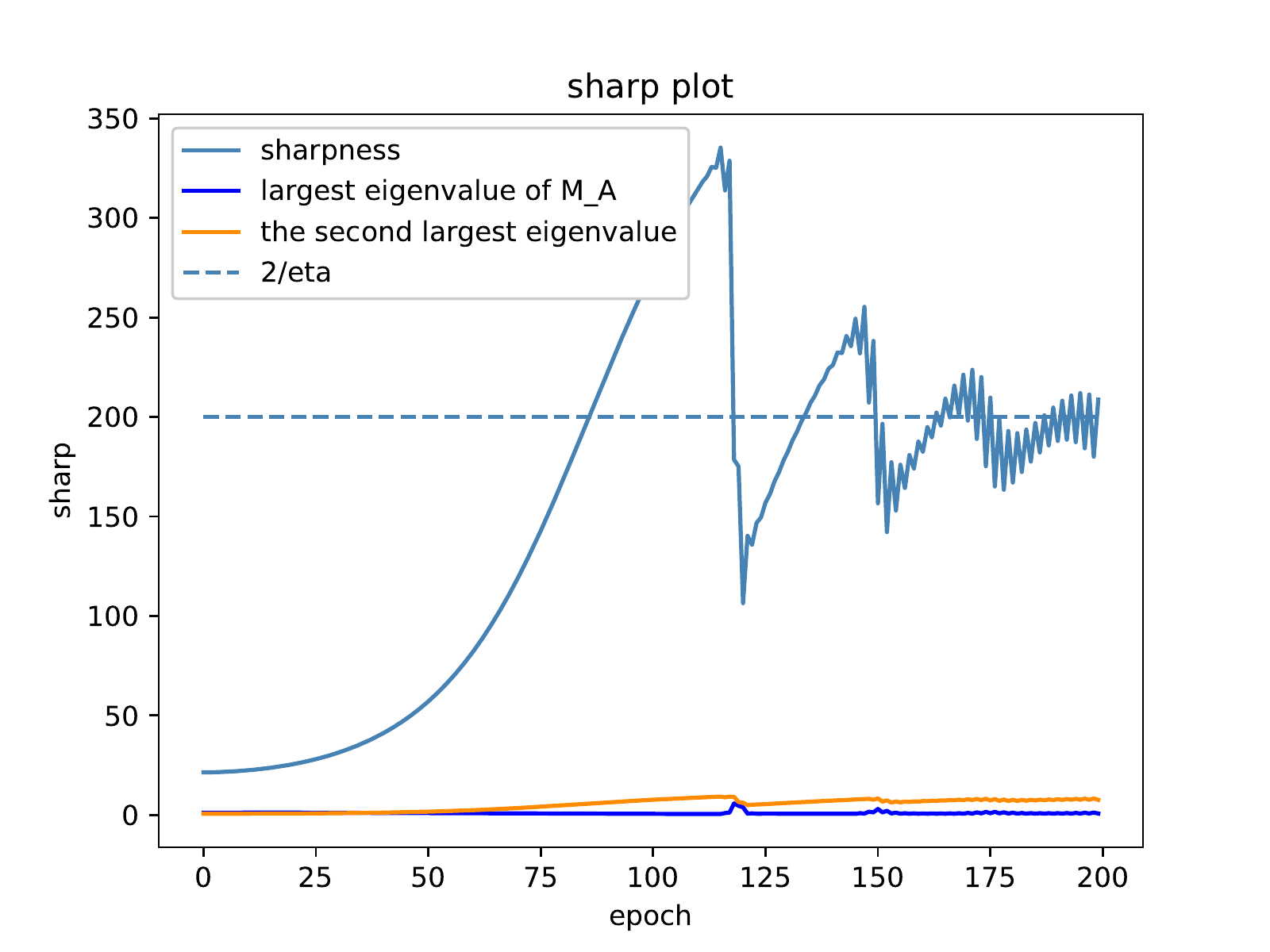}
    }
    \subfigure[The loss ($\|\bD(t)\|^2/n$) and $\|\bR(t)\|^2/n$]{
    \label{loss and R fig}
    \includegraphics[width=0.43\textwidth]{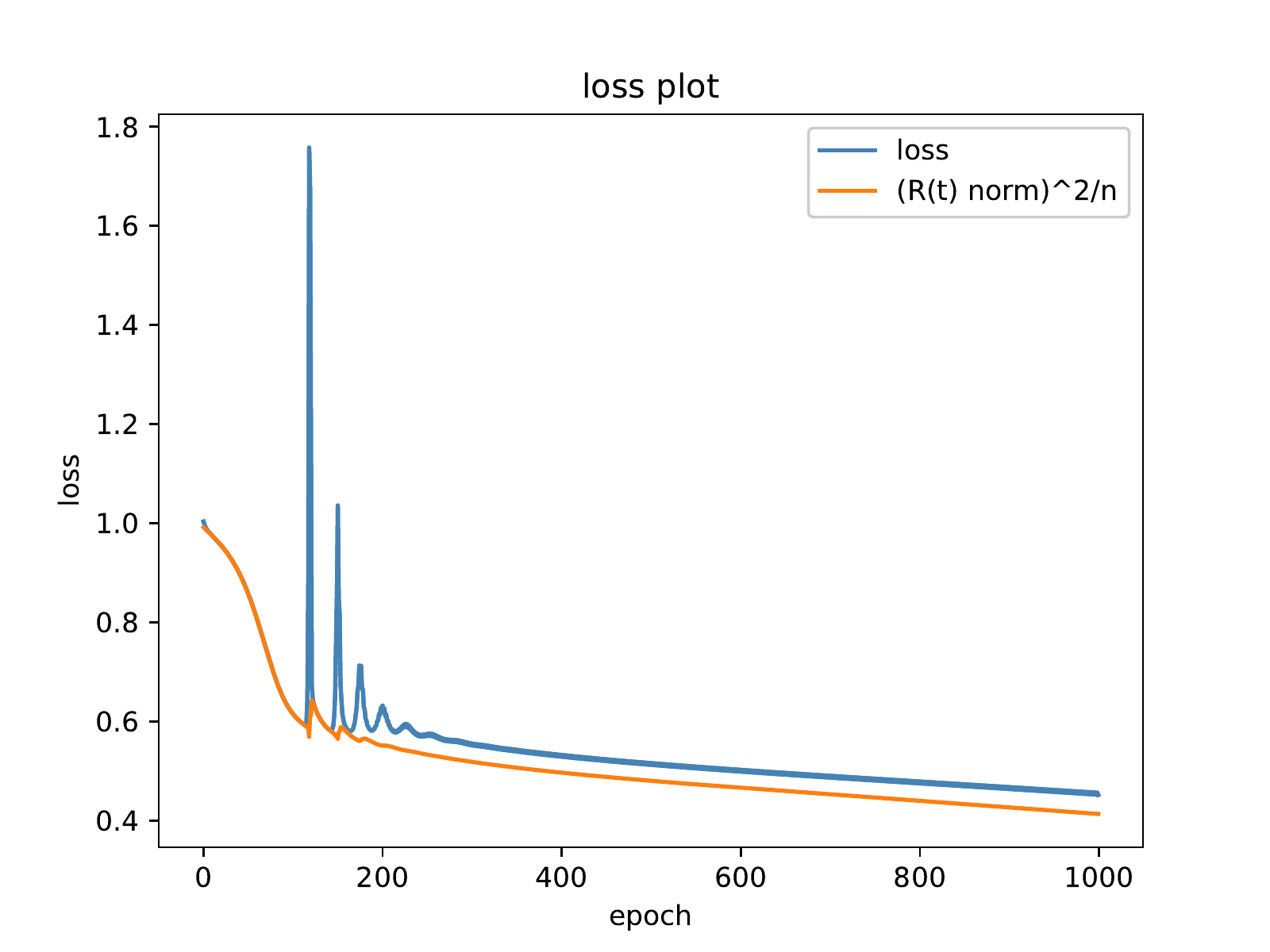}
    }
    \caption{In Figure (a), we verify the assumptions in Section~\ref{3.2} and \ref{3.3} empirically. The steel blue line represents the largest eigenvalue of $\bM$, i.e. the sharpness $\lambda_{\max}(\bM)$. The orange line: the second largest eigenvalue is far below $1/\eta$, illustrating Assumption~\ref{outlier assp}. The blue line: the largest eigenvalue of $\bm {M_A}$ is negligible compared to the sharpness. In Figure (b), we show that even though the total loss $\mathcal{L}(t) = \|\bD(t)\|^2/n$ 
    oscillates considerably, $\|\bm R(t)\|^2/n$ decreases more steadily. }
    \label{Verifying Exp1}
    \vspace{-0.3cm}
\end{figure}
\textbf{Phase IV:} First, 
since the training process has just left phase III, 
$\bm{D}(t)^\top\bm{F}(t)$ is still positive and large, hence $\|\bm A(t)\|^2$ keeps decreasing and the sharpness decreases as well.
Since the sharpness stays below $2/\eta$, the loss decreases 
due to the following descent lemma (with $\bm u$ replaced by $\bD(t)$).

\begin{lemma}
If $\sha(t)<2/\eta$, then for any vector $\bm u\in\mathbb R^n$, $\|\bm u^\top (\bm I - \eta\bM(t))\|\leq (1-\eta\alpha)\|\bm u\|$, where $\alpha = \min\{2/\eta - \sha(t), \lambda_{\min}(\bM(t))\}$.
In particular, replacing $\bm u$ with $\bD(t)$, we can see 
$\|\bD(t+1)\|\leq (1-\eta \alpha)^2\|\bD(t)\|$.
\label{Convergence lemma}
\end{lemma}

Next we argue that $\bm{D}(t)^\top\bm{F}(t)$ will become negative eventually, which indicates that $\|\bA(t)\|^2$ and hence the sharpness will grow again. 
Since the sharpness is below $2/\eta$, $\bm D(t)^\top \bm v_1(t)$ decreases geometrically due to Lemma~\ref{Convergence lemma} (replacing $\bm u$ with $\bD(t)\bv_1(t)\bv_1(t)^\top$).

In fact, $\bD$ can be decomposed into the $\bv_1$-component $\bv_1 \bv_1^\top \bD $ and the remaining part $\bR$ defined as
$\bm R(t) := (\bm I - \bm v_1(t)\bm v_1(t)^\top) \bm D(t)$.
Then we have 
\begin{equation}
    \bD(t)^\top \bF(t) = (\bm v_1(t)\bm v_1(t)^\top \bD(t))^\top(\bm v_1(t)\bm v_1(t)^\top \bD(t)+\bm Y) + \bR(t)^\top (\bR(t) + \bm Y).
    \label{eqn: DF decomposition}
\end{equation}

As shown in the next subsection, $\bR(t)$ almost follows a similar gradient descent trajectory $\bR'(t)$ (Lemma~\ref{RT pert prop}). More precisely, $\bm R'(t)$ is defined as $\bm R'(t+1) = (\bm I-\eta\bm M(t)(\bm I - \bv_1(t)\bv_1(t)^\top))\bm R'(t)$ (Lemma~\ref{lm:Rt_update}). While $\bD$'s dynamics is $\bD(t+1) = (\bm I - \eta \bm M(t))\bD(t)$, $\bR'$ follows a similar dynamics $\bR(t+1) = (\bm I - \eta \bm M'(t))\bR'(t)$, where $\bm M'(t) = \bm M(t)(\bm I - \bv_1(t)\bv_1(t)^\top)$. Note that $\bm M'(t)$ has eigenvalues smaller than $1/\eta$ for any time $t$ (by Assumption~\ref{outlier assp}), hence with an assumption similar to Assumption~\ref{Ass: eigspace assumptions} (i) (or a similar version of our relaxed assumption in Appendix~\ref{Subsection: relaxed PS} for $\bm M'$), we can prove that $\bR'(t)^\top (\bR'(t) + \bm Y)< 0$ for any time $t$ (See Appendix~\ref{appendix: proofs in sec3} for the rigorous proof).

Since $\bR(t)\approx\bR'(t)$, the second term in the decomposition~\eqref{eqn: DF decomposition} is always negative and the first term ($\bv_1$ direction term) is decreasing geometrically. Therefore, there are only two possible cases. The first possibility is that the first term decreases to a small value near 0 and the second term remains largely negative. Then their sum will be negative, which is $\bD(t)^\top\bF(t)< 0$, thus implying the training enters Phase I. The second possibility is that when the first term decreases to a small value near 0, the second term is also a small negative value. In this case both $\bR(t)$ and $\bD(t)^\top \bv_1(t)$ are small, implying the loss is almost 0, which is indeed the end of the training.

\subsection{Explaining Non-monotonic Loss Decrement}\label{3.4}
\label{subsection loss decrement}
In this subsection, we attempt to explain the non-monotonic decrement of the loss during the entire GD trajectory.
See Figure~\ref{loss and R fig}.
As defined in the last section, 
we decompose $\bD$ into the $\bv_1$-component $\bv_1 \bv_1^\top \bD $ and the remaining part $\bR$.
Below, we prove that $\bm R(t)$ is not affected much by the exponential growth of the loss (Proposition~\ref{RT pert prop}) in Phase II and III, and almost follows a converging trajectory (which is defined as $\bR'(t)$ later in this section). 

The arguments in this subsection need Assumption~\ref{Ass: sha upper bound section 3} and Assumption~\ref{Ass: eigspace assumptions} (iii), both very consistent with the experiments. 
We need an additional assumption on the spectrum of $\bM$.
\begin{Assumption}
All $\bM(t)$'s eigenvalues except $\sha(t)=\lambda_{\max}(\bM(t))$ are smaller than $1/\eta$ for all $t$.
\label{outlier assp}
\end{Assumption}
\vspace{-0.5cm}
Recall that the largest eigenvalue is at most $\frac{B_\sha}{\eta}$ by Assumption~\ref{Ass: sha upper bound section 3}.
Empirically, the largest eigenvalue is an outlier in the spectrum, i.e., it is much larger than the other eigenvalues.
Hence, we make Assumption~\ref{outlier assp} which states that all other eigenvalues
are at most $1/\eta$, which is consistent with our experiments.
See Figure \ref{Verifying Exp1}(b).
Similar fact is also mentioned in \cite{sagun2016eigenvalues,sagun2017empirical}.

First, we let $B_D$ be an upper bound of $\bm D(t)$, i.e.,
for all $t$, $\|\bm D(t)\|\leq B_D$. 
In the two-layer linear network case, we can have an explicit form of $B_D$. 
(see Lemma~\ref{D upper bound} in Appendix~\ref{appen c}.)
Recall that in Assumption \ref{Ass: sha upper bound section 3}, $B_\sha$ is the upper bound of $\eta \sha$. 

\begin{lemma}
\label{lm:Rt_update}
Suppose Assumption~\ref{Ass: eigspace assumptions} (iii) holds.
$\bm R(t)$ satisfies the following:
$$
 \bm{R}(t+1) = (I - \eta \bm{M}(t))\bm{R}(t) + \bm e_1(t), \quad \text{where}\quad
 \|\bm e_1(t)\|\leq 6\sqrt{\epsilon_2}\|\bm D(t)\|(B_\sha -1)
$$
\end{lemma}

\begin{lemma}
Define an auxiliary sequence $\bm R'(t)$ by $\bm R'(0) = \bm R(0)$, and $\bm R'(t+1) = (\bm I-\eta\bm M(t)(\bm I - \bv_1(t)\bv_1(t)^\top))\bm R'(t)$.
If Assumption~\ref{Ass: sha upper bound section 3},
Assumption~\ref{Ass: eigspace assumptions} (iii),
Assumption~\ref{outlier assp} hold, and for any time $t$ there exists a quantity $\lambda_r>0$, such that 
the smallest eigenvalue of $\bM(t)$, i.e. $\lambda_{\min}(\bM(t)) > \lambda_r$, 
then there exists a constant $c_r>0$ such that $\|\bm R(t) - \bm R'(t)\| \leq c_r \frac{B_D(B_\sha -1)\sqrt{\epsilon_2}}{\eta \lambda_r}$.
\label{RT pert prop}
\end{lemma}

Now, in light of Lemma~\ref{lm:Dincrease_exponentially} and Lemma~\ref{RT pert prop}, 
we arrive at an interesting explanation of the phenomena of non-monotonic decrease of the loss. Basically, $\bD$ can be decomposed into the $\bv_1$-component $\bv_1 \bv_1^\top \bD $ and the remaining part $\bR=(\bm I - \bm v_1\bm v_1^\top) \bm D$. The $\bv_1$-component may increase geometrically during the EOS (Lemma~\ref{lm:Dincrease_exponentially}), but the behavior of the 
remaining part $\bR(t)$ is close to $\bm R'(t)$, which 
follows the simple updating rule $\bm R'(t+1) = (\bm I-\eta\bm M(t))\bm R'(t)$,
so Lemma~\ref{Convergence lemma} implies that the $\bR$ part almost keeps
decreasing during the entire trajectory
(here Lemma~\ref{Convergence lemma} applies with $\bm u$ replaced by $\bm R'(t)$, 
noticing that the eigenvalues except the first are well below $2/\eta$). 
Hence, the non-monotonicity  of the loss is mainly due to 
the $\bv_1$-component of $\bD$, and the rest part $\bR$ is optimized 
in the classical regime
(step size well below 2/(the operator norm)) and hence steadily decreases.
See Figure~\ref{loss and R fig}.

\section{A Theoretical Analysis for 2-Layer Linear NN}
\label{theoretical analysis}
In this section, we aim to provide a more rigorous explanation of the EOS phenomenon in two-layer linear networks. The proof ideas follow similar high-level intuition as the proofs in Section~\ref{subsection 4-phase proof}. In particular, we can remove or replace the assumptions in Section~\ref{subsection 4-phase proof} with arguably weaker assumptions.
Due to space limit, we state our main theoretical results and elaborate their relation with the proofs in Section~\ref{subsection 4-phase proof}. 
The detailed settings and proof are more tedious and can be found in Appendix~\ref{appen c}.

\subsection{Setting and basic notations}
\textbf{Model:} In this section, we study a two-layer neural network with linear activation, i.e. 
    $
    f(x) = \sum_{q=1}^m\frac{1}{\sqrt m}a_q\bm w_qx = \frac{1}{\sqrt m} \bm A^\top \bm Wx
    $
    where $\bm W = [\bm w_1, ...,\bm w_m]^\top\in \mathbb R^{m\times d}$, $\bm A = [a_1, ..., a_m]\in \mathbb{R}^{m}$.
    
\noindent  \textbf{Dataset:} For simplicity, we assume $y_i = \pm 1$ for all $i\in [n]$, and $\|\mathbf X^\top \mathbf X\|_2 = \Theta(n)$. 
    We assume $\mathbf{\bX^\top \bX}$ has rank $r$, and we decompose $\mathbf{\bX^\top \bX}$ and $\bY$ according to the
    orthonormal basis $\{\bv_i\}$, the eigenvectors of $\mathbf{\bX^\top \bX}$:
    $
    \mathbf{\bX^\top \bX} = \sum_{i=1}^r \lambda_i\bm v_i\bm v_i^\top, \ \bY = \sum_{i=1}^r (\bY^\top \bm v_i) \bm v_i :=\sum_{i=1}^r z_i \bm v_i
    $
    where $\bv_i$ is the eigenvector corresponding to the $i$-th largest eigenvalue $\lambda_i$ of $\mathbf{\bX^\top \bX}$. $z_i = \bY^\top \bm v_i$ is the projection of $\bY$
    onto the direction $\bv_i$. Here we suppose $n\gg r$ and the global minimum $(\bm A^*, \bm W^*)$ exists.
        
    
    
\noindent    \textbf{Update rule:} We write explicitly the GD dynamics of $\bm D(t)$: 
        $\bm D(t+1)= (\bm I-\eta \bm M^*(t))\bm D(t)$, 
    where $\bm M^*(t)=\frac{2}{mn}(\|\bm A(t)\|^2\mathbf X^\top\mathbf X +\mathbf X^\top \bm W^\top (t) \bm W(t)\mathbf X)- \frac{4\eta}{n^2m}(\bm D(t)^\top\bm F (t))\mathbf X^\top \mathbf X$ is the Gram matrix combined with second order terms.
  
\subsection{Main Theorem and The Proof Sketch}

\noindent\textbf{Phase I and Progressive Sharpening:}

\begin{Assumption}
There exists some constant $\chi>1$, s.t. for all $i\in[r-1]$,
$ \lambda_i(\bX^\top \bX) \leq \chi\lambda_{i+1}(\bX^\top \bX)$. Moreover, $\lambda_1(\bX^\top \bX)\geq 2\lambda_2(\bX^\top \bX).$
\label{eigenspectrum assumption}
\end{Assumption}

\begin{Assumption}
There exists $\smallconst = \Omega(r^{-1})$ 
such that $\min_{i\in [r]}\{z_i/\sqrt n\}\geq \smallconst$.
\label{y assumption}
\end{Assumption}\vspace{-0.25cm}

The first assumption is about the eigenvalue spectrum of $\bX^\top \bX$. \footnote{It guarantees the gap between two adjacent eigenvalues is not very large, and there is a gap between the largest and the second largest eigenvalue. 
Note the second part of the assumption is a relaxed version of Assumption~\ref{outlier assp}. In our CIFAR-10 1k-subset with samples' mean subtracted, $\lambda_1/\lambda_2 = \chi \approx 3$ (See Figure~\ref{eigenspec assumption fig}). }
The second assumes that all component $z_i=\bY^\top\bv_i$ are not too small. 



\begin{theorem}[Informal]
Suppose Assumption~\ref{eigenspectrum assumption}, Assumption~\ref{y assumption} hold, 
the smallest nonzero eigenvalue $\lambda_r=\lambda_{r}(\bX^\top \bX)>0$
and $\lambda_1 = \lambda_{\max}(\bX^\top \bX) = c_1n$. Then for any $\epsilon>0$, if $m= \Omega(\frac{c_1n^2}{\lambda_r^2})$, and $n =\Omega(\frac{\lambda_r^2}{\smallconst^4\epsilon^2})$,
we have the progressive sharpening property: $\sha(t+1)-\sha(t)>0$ for $t= 1,2, ...,  t_0-1$ where $t_0$ is the time when $\|\bD(t)\|^2\leq O(\epsilon^2)$ or $\lambda_{\max}(\bM^*(t))>1/\eta$ for the first time.
\label{PS theorem}
\end{theorem}

In the proof of this theorem, we show that
the Gram matrix $\bM(t) \approx \frac{2}{mn}(\|\bm A(t)\|^2+\frac{m}{d})\bX^\top \bX$,
which serves as a justification of 
Assumption~\ref{Ass: eigspace assumptions} we made in Section~\ref{subsection 4-phase proof}.
That shows all $\bM(t)$ approximately share the same set of eigenvectors as $\bX^\top \bX$. In our proof, we also prove more rigorously that $\|\bm A(t)\|^2$ is an indicator of the sharpness in this simpler setting.


\noindent\textbf{Edge of Stability (Phase II - IV):}

\begin{Assumption}
There exists some constant $c_2>0$, such that
$
\|\bm \Gamma(t)\|\leq \frac{c_2}{m}.
$
\label{Lambda assumption}
\end{Assumption}\vspace{-0.25cm}

This assumption is based on Theorem~\ref{PS theorem}. In Theorem~\ref{PS theorem}, we state that in the progressive sharpening phase, $\|\bm\Gamma(t)\|$ has an upper bound of $O(1/m)$. Now in the EOS phase, we assume that $\|\bm\Gamma(t)\|$ grows larger by at most a constant factor. Further discussions refer to Appendix \ref{D.2.2 gamma assumption verification}.

\begin{Assumption}
There exists some constant $\beta>0$, such that 
$
\sha \leq \frac{4}{\eta}(1-\beta).
$
\label{section 4 sharp upper bound assumption}
\end{Assumption}
\vspace{-0.25cm}

This assumption is consistent with Assumption~\ref{Ass: sha upper bound section 3}, which assumes an upper bound of the sharpness. 

\begin{Assumption}
There exist some constant $c_3$ such that $|\bD(t)^\top \bv_1| > c_3 \sqrt{n}/m$ for some $t=t_0$ at the beginning of phase II.
\label{divergence assumption section 4}
\end{Assumption}
\vspace{-0.25cm}
This assumption is in the same spirit of Assumption~\ref{noise assumption} with the only change of the bound in terms of $m$ and $n$.
Now, we are ready to state our theorem in this stage.

\begin{theorem}
Denote the smallest nonzero eigenvalue as $\lambda_r\triangleq\lambda_{r}(\bX^\top \bX)>0$
and the largest eigenvalue as $\lambda_1 \triangleq \lambda_{1}(\bX^\top \bX)$. Under Assumption~\ref{Lambda assumption}, \ref{section 4 sharp upper bound assumption}, \ref{divergence assumption section 4}, and $\lambda_1(\bX ^\top\bX)\geq 2\lambda_2(\bX^\top \bX)$ in Assumption~\ref{eigenspectrum assumption},
there exist constants $c_4, c_5, c_6$ such that if $n>c_6\lambda_r\eta, m > \max\{\frac{c_4d^2n^2}{\lambda_r^2}, c_5\eta\}$, then 
\vspace{-0.1cm}
\begin{itemize}
    \item There exists $\rho = O(1)$ which depends on $c_3$ such that if $\sha(t_0)>\frac{2}{\eta}(1+\rho)$ for some $t_0$, there must exist some $t_1>t_0$ such that $\sha(t_1) < \frac{2}{\eta}(1+\rho)$.
    \item If $\sha(t), \sha(t+1) > \frac{2}{\eta}(1+\rho)$, then there is a constant $c_7>0$ (depending on $c_3$) such that $|\bD(t+1)^\top \bv_1|>|\bD(t)^\top \bv_1|(1+c_7)$.
    \item Define $\bR(t):= (\bm I - \bv_1 \bv_1^\top)\bD(t)$, and $\bR'(t) := (\bm I - \eta\bM^*(t)(\bm I - \bv_1 \bv_1^\top))\bR'(t-1)$.
    It holds that $\|\bR(t) - \bR'(t)\| = O(\frac{\sqrt{n^3}d}{\lambda_r \sqrt{m}})$.
\end{itemize}
\label{EOS main thm}
\end{theorem}
\vspace{-0.25cm}
We can conclude the following from Theorem~\ref{EOS main thm}: 
(1) The first statement of the theorem states that if the progressive sharpening phase
causes the sharpness to grow over $2/\eta$, then the sharpness eventually goes below $2/\eta$.
This illustrates the regularization effect of gradient descent on the sharpness
(this is consistent with the analysis of Phase III in Section \ref{subsection 4-phase proof}).
(2) The second states that $|\bD(t)^\top \bv_1|$ geometrically increases in Phase II and III.
Note that we proved a similar Lemma \ref{divergence proposition} for Phase II in the more general setting in Section \ref{subsection 4-phase proof}.
(3) The third conclusion gives an upper bound for the distance between $\bR(t)$'s trajectory and $\bR'(t)$'s. 
This bound helps illustrate why $\bR(t)$'s trajectory is similar with $\bR'(t)$ in Phase IV of Section \ref{subsection 4-phase proof}.






\section{Discussions and Open Problems}\label{5}

In this section, we discuss the limitation of our theory and some related findings.
First, our argument crucially relies on the assumption that $\|\bA\|$ changes in the same direction 
as $\sha$ does most of the time. Here, we elaborate more on this point.
Seeing from a longer time scale, $\|\bA\|^2$ and the sharpness may have very different overall trends (See Figure (c) in \ref{fig:Anorm and sharpness}), i.e., the sharpness oscillates around $2/\eta$ but $\|\bA\|^2$ increases. 
Moreover, the sharpness may oscillate more frequently than $\|\bA\|^2$,
while the low-frequency trends seem to match well (See the late training phases in Figure (b) in \ref{fig:Anorm and sharpness}).
Currently, our theory cannot explain the high-frequency oscillation of 
the sharpness in Figure (b).

While we still believe the change of $\|\bA\|$ is a major driving force of the change of the sharpness, other factors (such as other layers) must be taken into consideration for a complete understanding and explanation of the sharpness dynamics.
We also carry out some experiments that reveal some interesting relation between the inner layers and the sharpness, which is not yet reflected in our theory. Due to space limit, we defer it to Appendix \ref{appendix section: discussion experiment}.
We conclude with some open problems.
It would be very interesting to remove some of our assumptions or replace them 
(especially those related to the spectrum of $\bM$) by weaker or more natural assumptions on the data or architectures, or make some of the heuristic argument more rigorous (e.g., first order approximation of the dynamics~\eqref{Order one dynamics}). 
Extending our results in Section~\ref{theoretical analysis} to deeper neural networks with nonlinear activation function is an intriguing and challenging open problem.

\bibliography{ref}

\newpage
\appendix



\section{Experimental Setup}\label{appendixA}
We provide detailed experimental setup in this section. 

\subsection{Dataset}
Under GPU memory constraints when computing Gram matrix of our models, we limit the size of dataset we use. The dataset is a 1,000-sample subset from CIFAR-10 (\citet{krizhevsky2009learning}) (\url{ https://www.cs.toronto.edu/~kriz/cifar.html}).
To make it a binary dataset, we constructed it by selecting the first 500 samples of class 0 and 1, respectively.\footnote{\citet{cohen2021gradient} selects the first 5,000 examples from CIFAR-10. Our subset is smaller because of the computation limitation when calculating the Gram matrix. Experiments show that the properties along GD trajectory (e.g. the loss, the sharpness) is similar on both datasets.} Then we label the samples by $\pm1$. In the experiments in Appendix~\ref{appendixB} and ~\ref{appendixD}, we fix the objective as training on the 1000-example subset of CIFAR-10.

\subsection{Network Architecture}
In general settings, we experiment with three architectures from simple models to more complicated models: one-hidden-layer linear neural network, four-hidden-layer fully-connected network, convolutional network and a ResNet18 (\citet{he2016deep}) model. The initialization of each layer follows the default initialization of PyTorch (\citet{paszke2019pytorch}).

\textbf{Linear Network } We first use a simple two-layer linear neural network. The hidden layer has 200 neurons. We empirically show that even a simple linear network can enter the EOS regime. 
\begin{table}[H]
  \caption{Linear network}
  \label{linear network}
  \centering
  \begin{tabular}{lllll}
    \toprule
    Layer \# & Name     & Layer     & In shape & Out shape\\
    \midrule
    1 &   & \texttt{Flatten()} & $(3,32,32)$  & 3072     \\
    2     & \texttt{fc1} & \texttt{nn.Linear(3072, 200)}  & 3072 & 200     \\
    3     & \texttt{fc2}       & \texttt{nn.Linear(200, 1)} &  200  & 1  \\
    \bottomrule
  \end{tabular}
\end{table}

\textbf{Fully-connected Network } We conduct further experiments on several different fully-connected networks with 4 hidden layers with various activation functions. We consider tanh, ReLU, ELU activations. For example, the structure of a fully-connected tanh network is shown in Table~\ref{Fully-connected NN}.
\begin{table}[H]
  \caption{Fully-connected network}
  \label{Fully-connected NN}
  \centering
  \begin{tabular}{lllll}
    \toprule
    Layer \# & Name     & Layer     & In shape & Out shape\\
    \midrule
    1 &   & \texttt{Flatten()} & $(3,32,32)$  & 3072     \\
    2     & \texttt{fc1} & \texttt{nn.Linear(3072,200,bias=False)}  & 3072 & 200     \\
    3     &        & \texttt{nn.tanh()} &  200  & 200  \\
    4     & \texttt{fc2}       & \texttt{nn.Linear(200,200,bias=False)} &  200  & 200  \\
    5     &        & \texttt{nn.tanh()} &  200  & 200  \\
    6     & \texttt{fc3}       & \texttt{nn.Linear(200,200,bias=False)} &  200  & 200  \\
    7     &        & \texttt{nn.tanh()} &  200  & 200  \\
    8     & \texttt{fc4}       & \texttt{nn.Linear(200,200,bias=False)} &  200  & 200  \\
    9     &        & \texttt{nn.tanh()} &  200  & 200  \\
    10     & \texttt{fc5}       & \texttt{nn.Linear(200,1,bias=False)} &  200  & 1  \\
    \bottomrule
  \end{tabular}
\end{table}

\textbf{Convolutional Network } We also conduct  experiments on several different convolutional networks with two convolutional layers and two max-pooling layers. Like the fully-connected network experiments, we consider tanh, ReLU and ELU activations. For example, the structure of a convolutional tanh network is shown in Table~\ref{Conv NN}.
\begin{table}[H]
  \caption{Convolutional network}
  \label{Conv NN}
  \centering
  \begin{tabular}{lllll}
    \toprule
    \# & Name     & Layer     & In shape & Out shape\\
    \midrule
    1 &   \texttt{conv1}& \texttt{nn.Conv2d(3,32,kernel\_size=3,padding=1)} & $(3,32,32)$  & (32,32,32)     \\
    2     &  & \texttt{nn.tanh()}  & (32,32,32) & (32,32,32)     \\
    3     &        & \texttt{nn.MaxPool2d(2)} &  (32,32,32)  & (32,16,16)  \\
    4     & \texttt{conv2}       & \texttt{nn.Conv2d(32,32,kernel\_size=3,padding=1)} &  (32,16,16)  & (32,16,16)  \\
    5     &        & \texttt{nn.tanh()} &  (32,16,16)  & (32,16,16)  \\
    6     &        & \texttt{nn.MaxPool2d(2)}  &  (32,16,16)  & (32,8,8)  \\
    7     &        & \texttt{Flatten()} &  (32,8,8)  & 2048  \\
    8     & \texttt{fc1}       & \texttt{nn.Linear(2048,1,bias=False)} &  2048  & 1  \\
    \bottomrule
  \end{tabular}
\end{table}
\textbf{ResNet18 }  We also conduct experiment on the ResNet18 architecture proposed by \citet{he2016deep}. We use the default architecture implemented in PyTorch (\citet{paszke2019pytorch}). When calculating the sharpness, we use the numerical methods in the package (\citet{golmantpytorchhessian}) to calculate the top eigenvalue of the Hessian matrix.

\section{Further Experiments and Discussions}\label{appendixB}


In this appendix, we use the 1000-sized subset of CIFAR-10 introduced in Appendix~\ref{appendixA} to conduct further experiments on various architectures. We verify our main observation about the correlation between the sharpness and the weight norm of the output layer (A-norm) of the neural network through the following experiments. 

We consider simple linear networks, fully-connected networks, convolutional networks in this appendix. For nonlinear networks, we choose tanh, ReLU, ELU as activation functions. We train the networks with MSE loss. Here we exclude ResNet18 experiment since it is not feasible to compute the Gram matrix or the leading Gram matrix eigenvector due to GPU memory limitation. The sharpness and A-norm correlation of ResNet18 are included in Figure~\ref{fig:Anorm and sharpness} in Section~\ref{3.2}.

Here we run full-batch gradient descent with a selected step size $\eta$ such that the sharpness at initialization is smaller than $2/\eta$.
\subsection{Further Experiments}
\subsubsection{Linear Networks}
We first verify our four-phase division of EOS phenomena in a simple linear network. This experiment shows that even linear networks can also enter EOS regime and the four-phase division of the gradient descent trajectory is quite apparent in this setting. The following Figure~\ref{fully-connected linear network} illustrates the positive correlation between the sharpness and the A-norm, and the relationship between the loss $\|\bD(t)\|^2$ and $\|\bR(t)\|^2$ along the trajectory.
\begin{figure}[H]
    \vspace{-0.5cm}
    \centering
    \subfigure[Sharpness and A-norm]{
    \label{linear sha assumption fig}
    \includegraphics[width=0.4\textwidth]{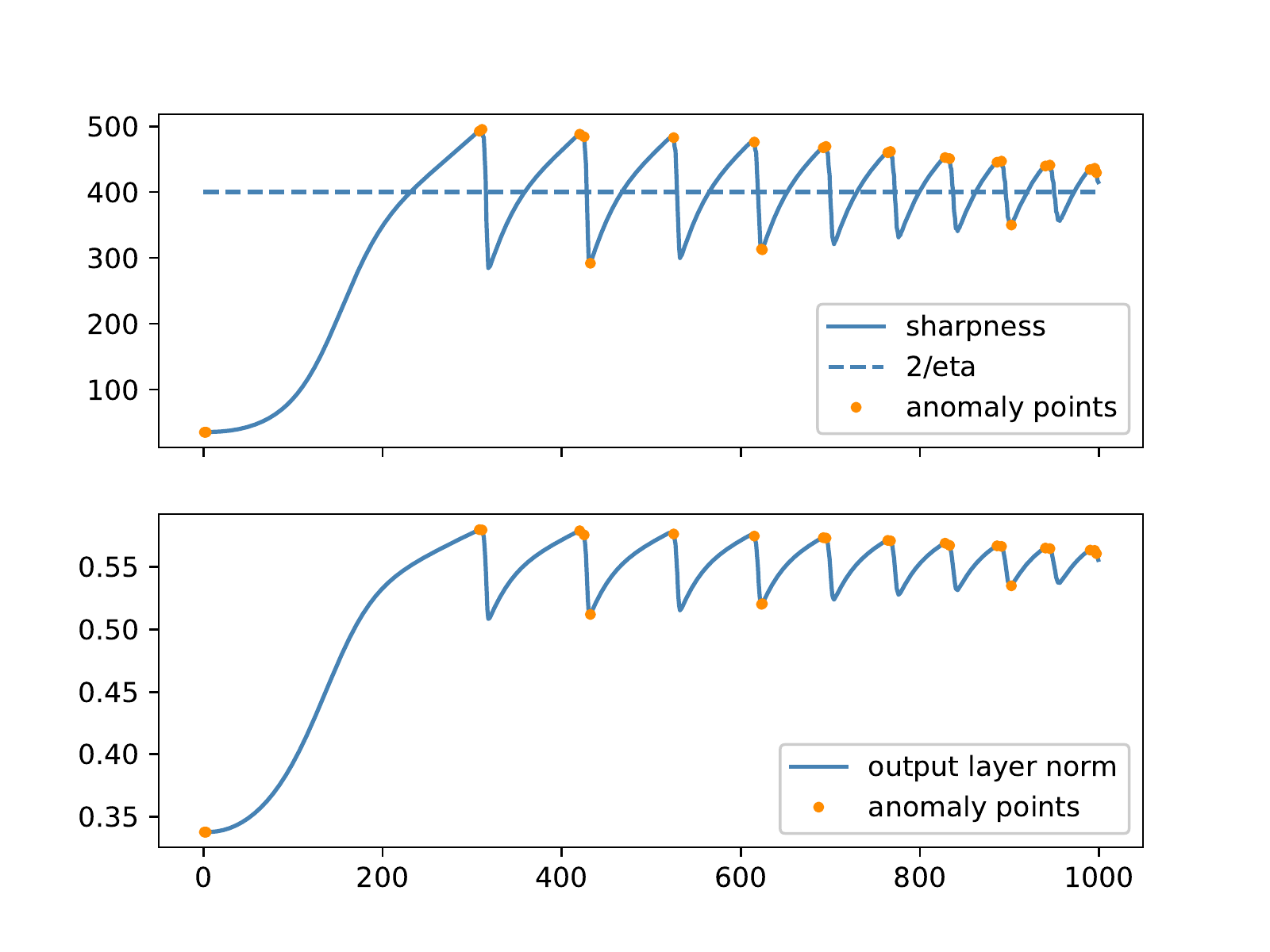}
    }
    \subfigure[The loss ($\|\bD(t)\|^2/n$) and $\|\bR(t)\|^2/n$]{
    \label{linear loss and R fig}
    \includegraphics[width=0.4\textwidth]{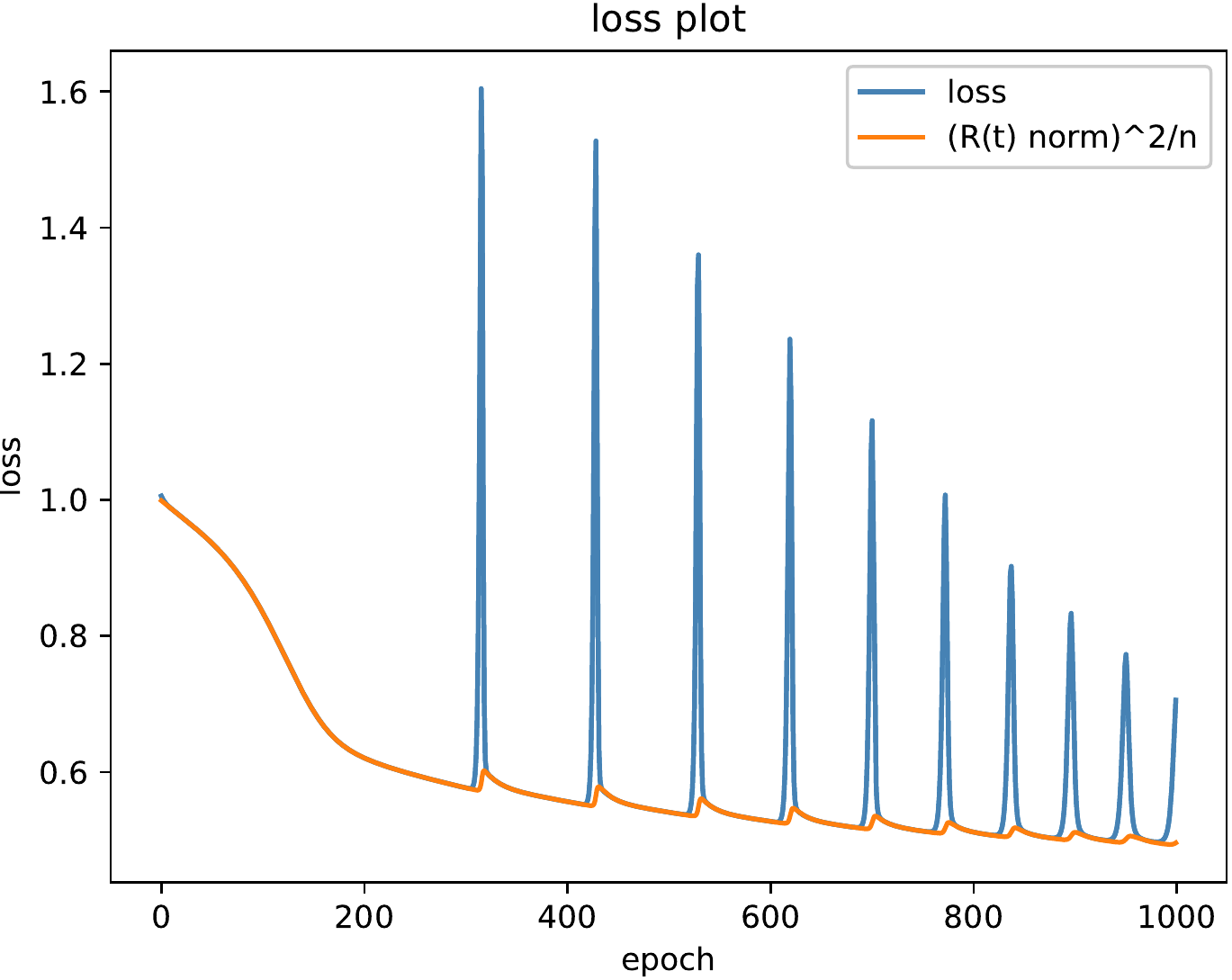}
    }
    \caption{Fully-connected linear network. In Figure (a), we verify the correlation between sharpness and A-norm in convolutional network setting. In Figure (b), we show that in this case, even though the total loss $\mathcal{L}(t) = \|\bD(t)\|^2/n$ 
    oscillates considerably, $\|\bm R(t)\|^2/n$ decreases more steadily. }
    \label{fully-connected linear network}
\end{figure}
\subsubsection{Fully-connected networks}
We train three 5-layer fully-connected networks with different activation functions: tanh, ReLU, and ELU activations. In Figure~\ref{fully-connected tanh network}, ~\ref{fully-connected relu network}, and ~\ref{fully-connected elu network}, we verify that in these fully-connected networks, the sharpness is positively correlated to the dynamics output layer norm (A-norm) most of the time in the progressive sharpening stage and the first few oscillations. 

Note that anomaly points appear much more frequently after a few oscillations. Meanwhile, the sharpness oscillates more frequently around $2/\eta$ in every few iterations.
We further discuss the phenomenon in Appendix~\ref{discussion on anomaly}, in which we elaborate the complicated relationship of the sharpness, A-norm and parameters of other layers.

\begin{figure}[H]
    \vspace{-0.5cm}
    \centering
    \subfigure[Sharpness and A-norm]{
    \label{fc tanh sha assumption fig}
    \includegraphics[width=0.38\textwidth]{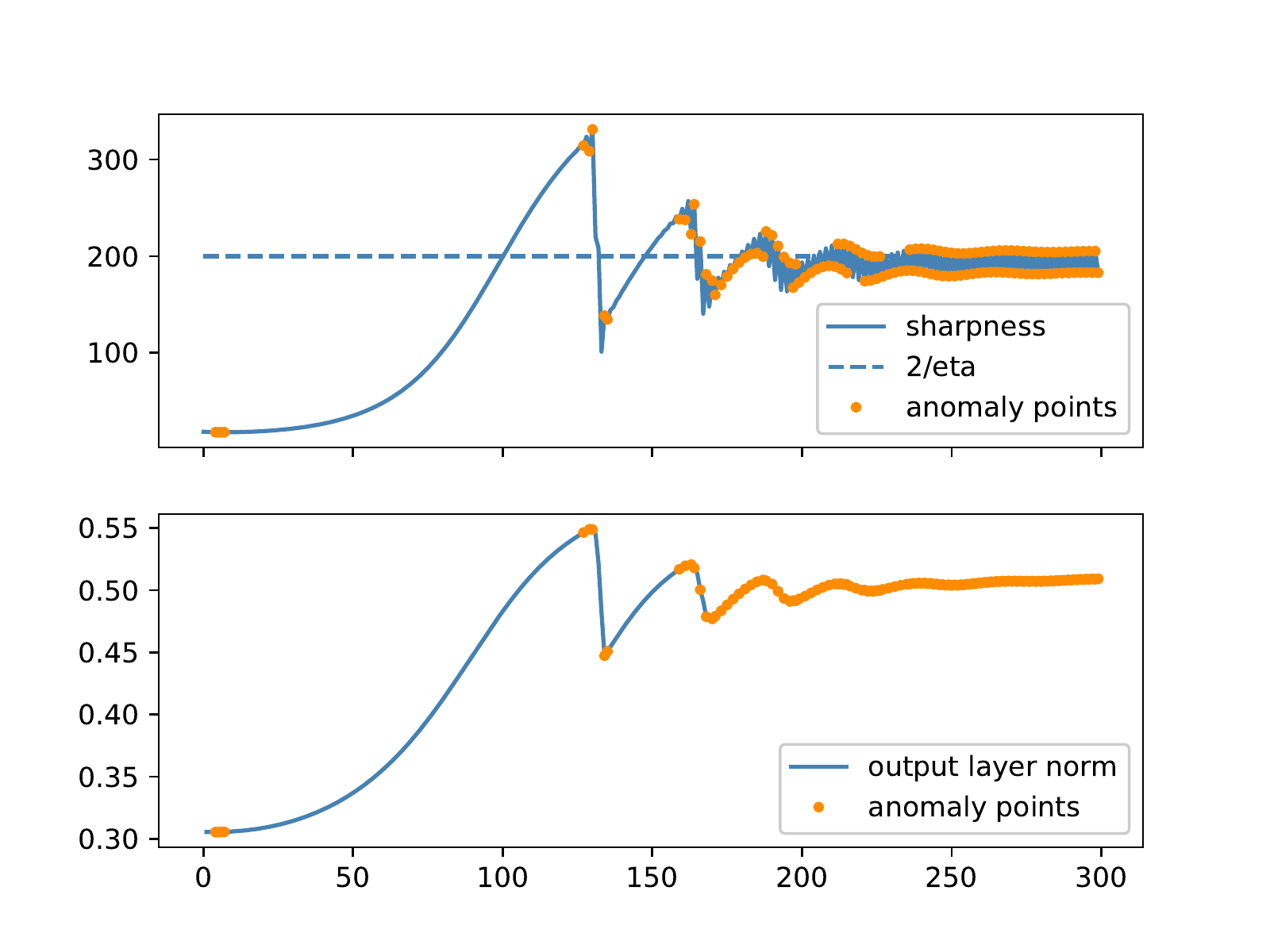}
    }
    \subfigure[The loss ($\|\bD(t)\|^2/n$) and $\|\bR(t)\|^2/n$]{
    \label{fc tanh loss and R fig}
    \includegraphics[width=0.38\textwidth]{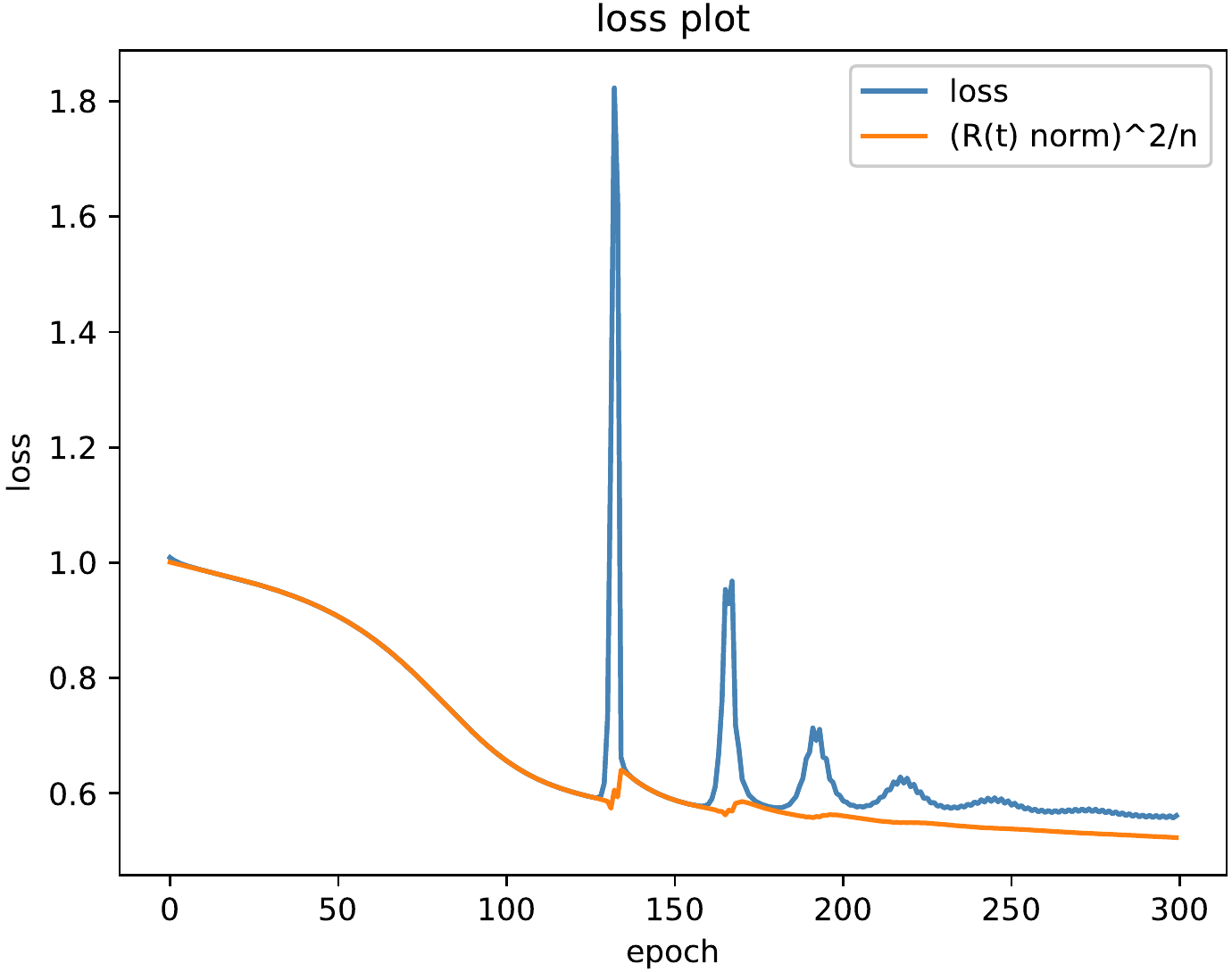}
    }
    \caption{Fully-connected tanh Network. Refer to Figure \ref{fully-connected linear network} for more information. }
    \label{fully-connected tanh network}
\end{figure}
\begin{figure}[H]
    \vspace{-0.5cm}
    \centering
    \subfigure[Sharpness and A-norm]{
    \label{fc relu sha assumption fig}
    \includegraphics[width=0.38\textwidth]{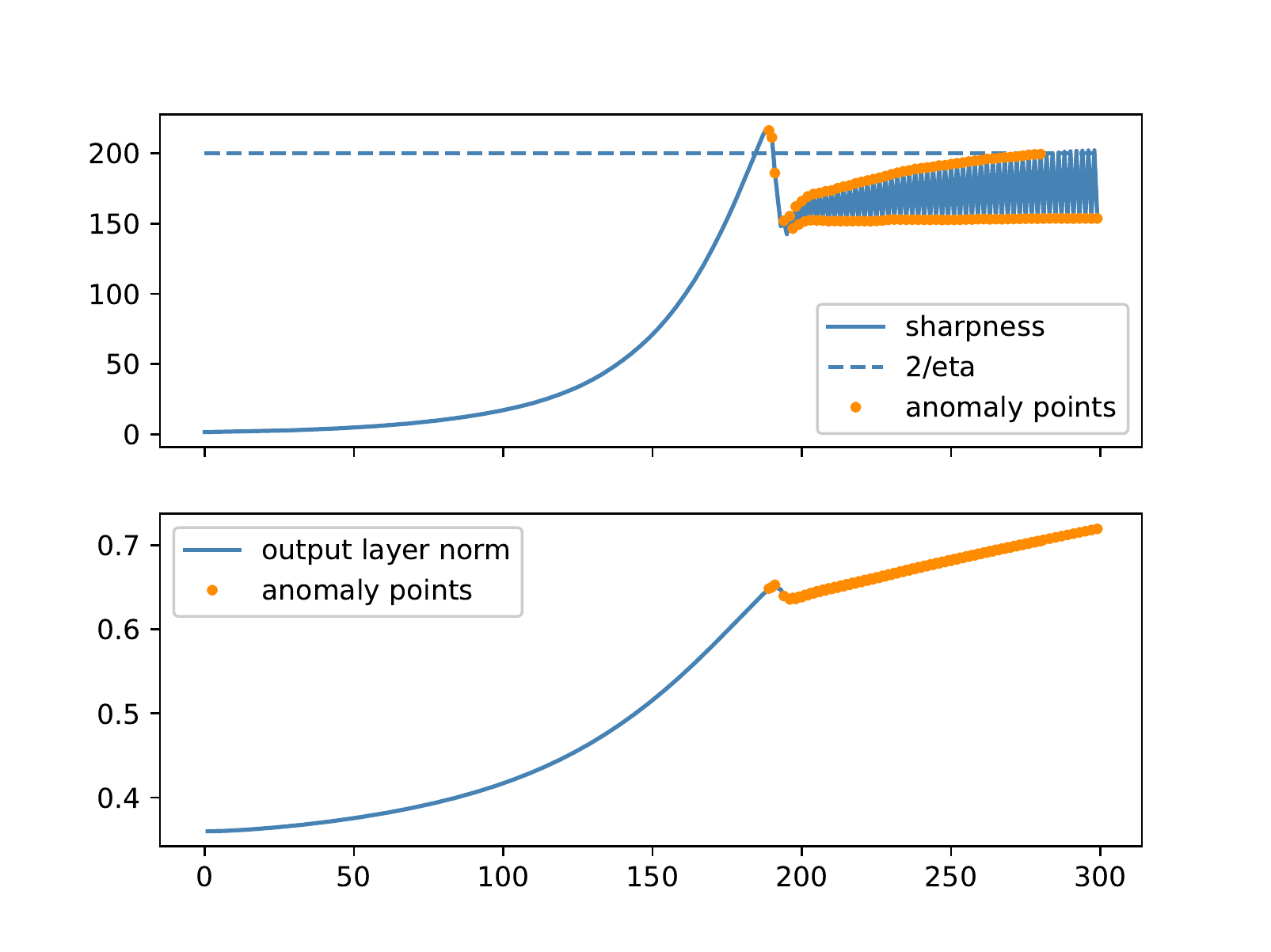}
    }
    \subfigure[The loss ($\|\bD(t)\|^2/n$) and $\|\bR(t)\|^2/n$]{
    \label{fc relu loss and R fig}
    \includegraphics[width=0.38\textwidth]{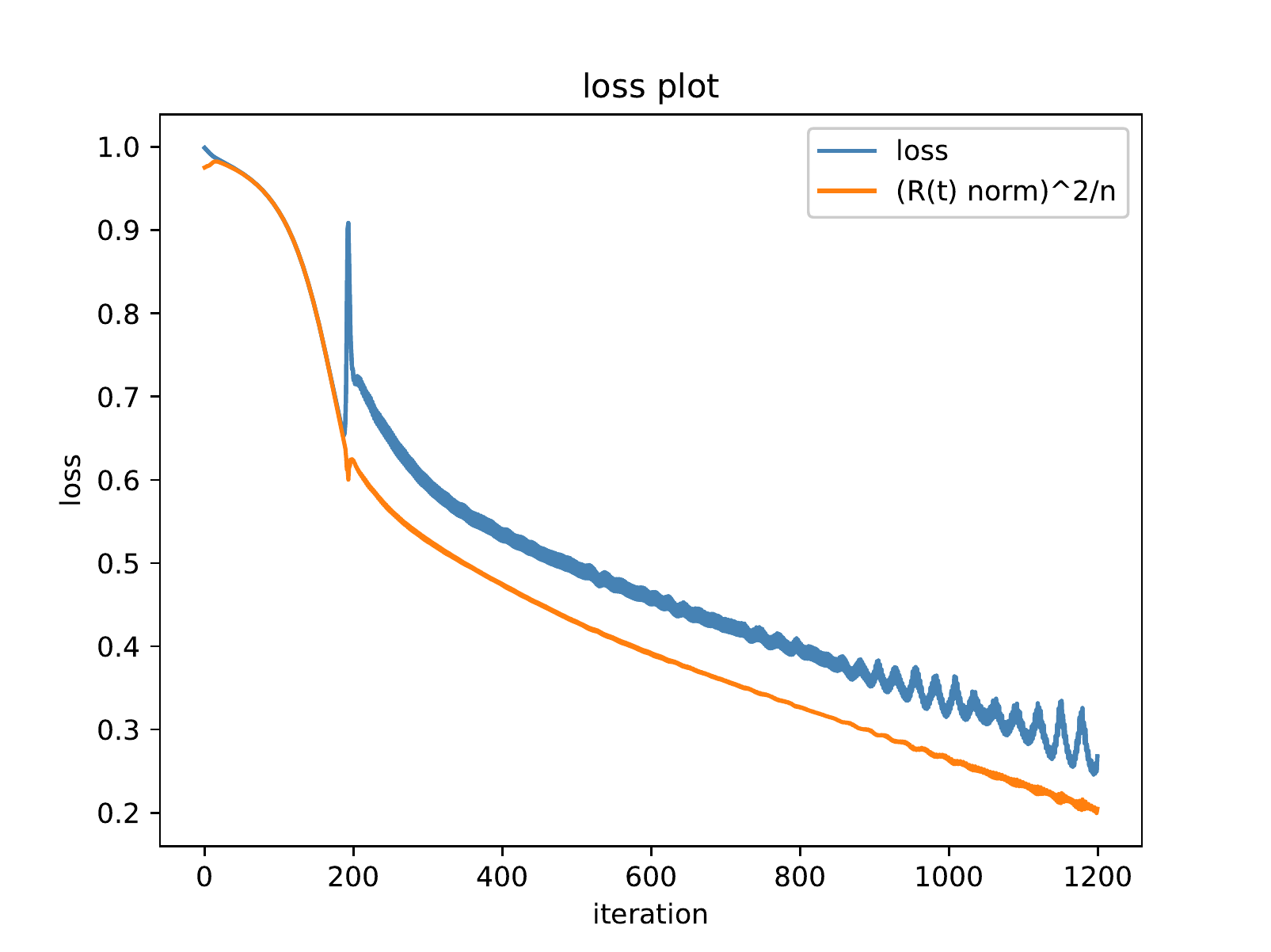}
    }
    \caption{Fully-connected ReLU Network. Refer to Figure \ref{fully-connected linear network} for more information. }
    \label{fully-connected relu network}
\end{figure}
\begin{figure}[H]
    \vspace{-0.5cm}
    \centering
    \subfigure[Sharpness and A-norm]{
    \label{fc elu sha assumption fig}
    \includegraphics[width=0.38\textwidth]{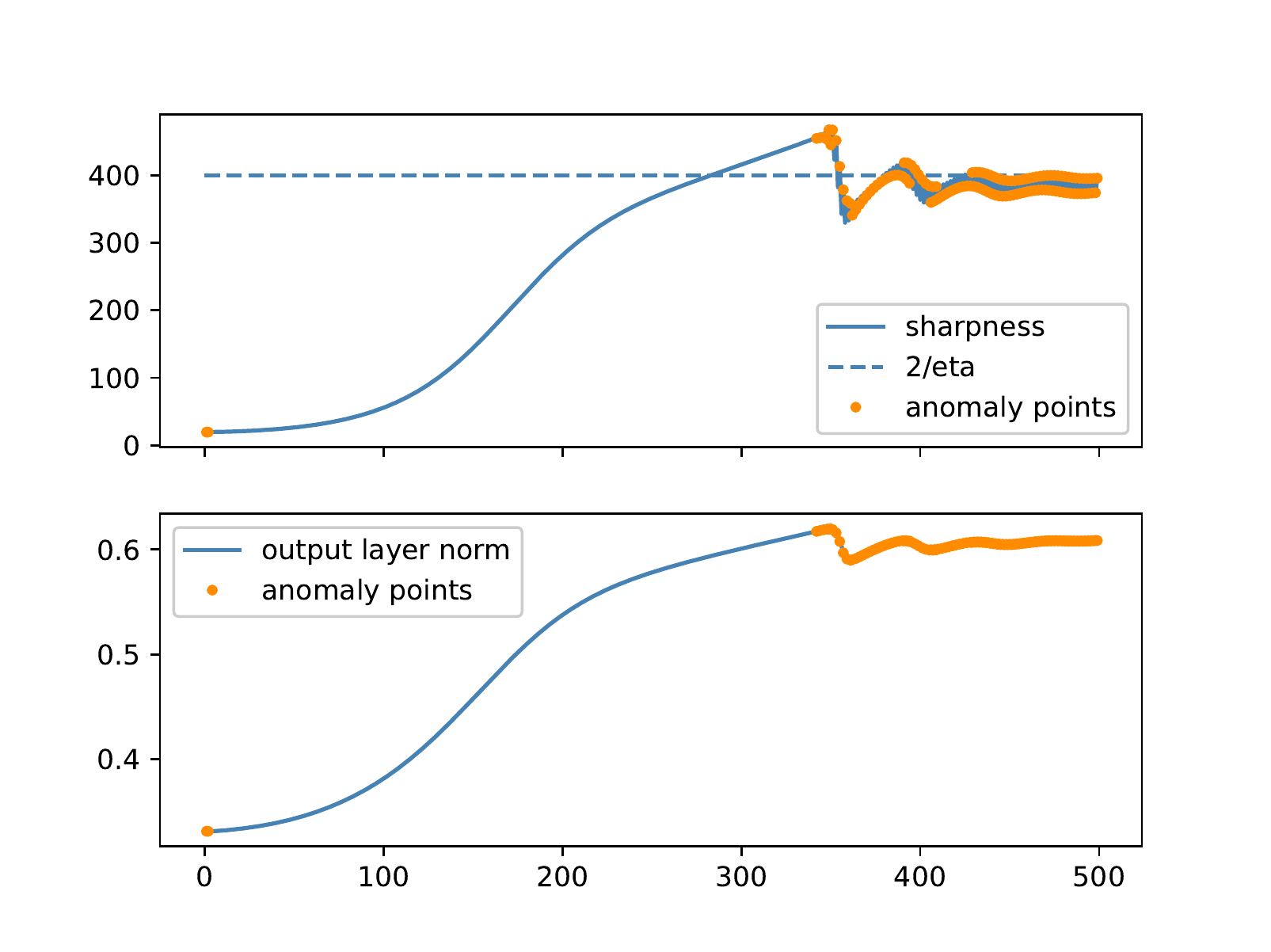}
    }
    \subfigure[The loss ($\|\bD(t)\|^2/n$) and $\|\bR(t)\|^2/n$]{
    \label{fc elu loss and R fig}
    \includegraphics[width=0.38\textwidth]{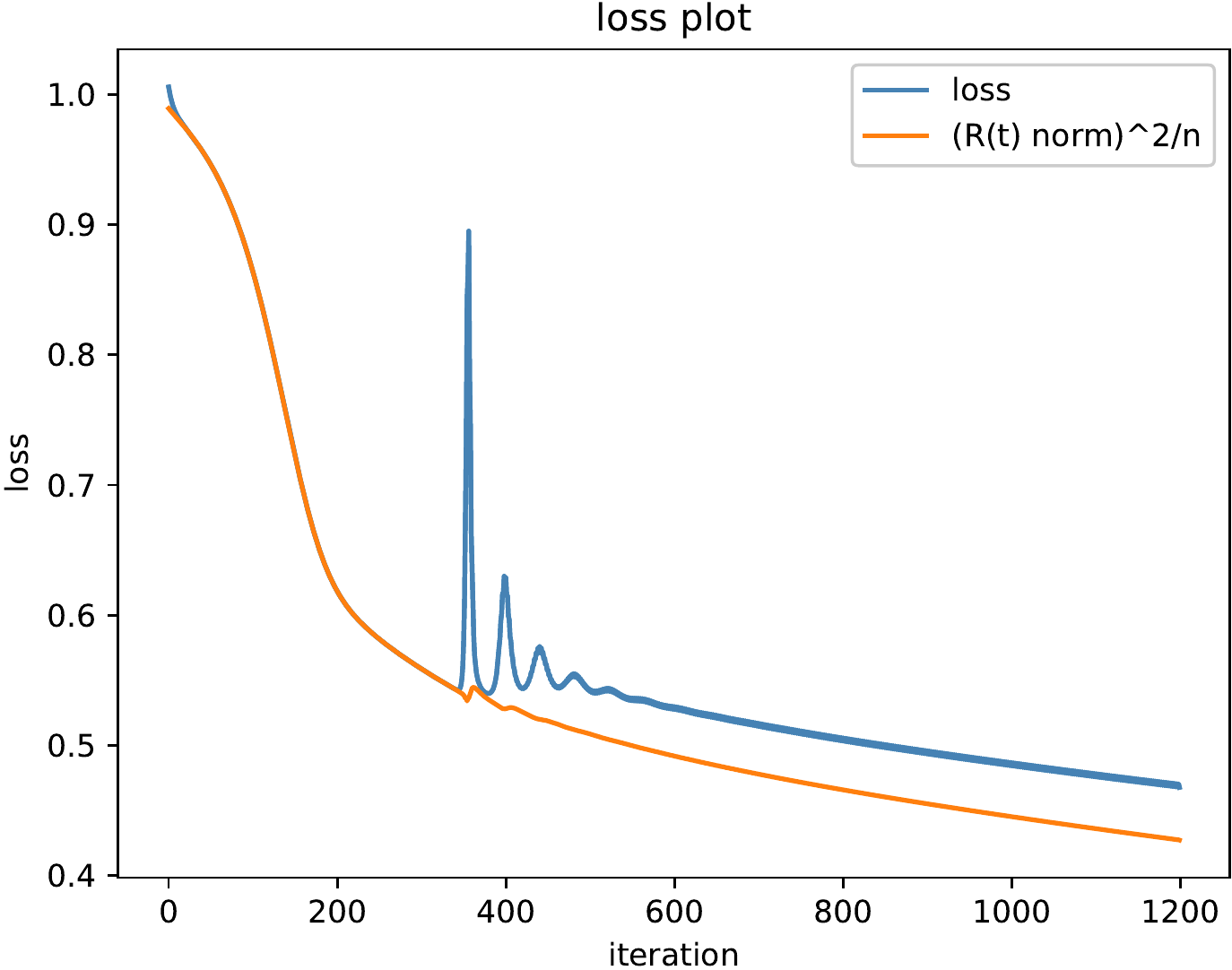}
    }
    \caption{Fully-connected ELU Network. Refer to Figure \ref{fully-connected linear network} for more information.  }
    \label{fully-connected elu network}
\end{figure}
\subsubsection{Convolutional Networks}
We train three convolutional networks with different activation functions: ReLU, tanh, and ELU activations. In Figure~\ref{Convolutional tanh network}, ~\ref{Convolutional ReLU network}, and ~\ref{Convolutional ELU network}, we verify that in convolutional networks, the positive correlation between the sharpness and the output layer norm (A norm) is still correct in the training process. In the first few oscillations of the sharpness, the four-phase division is also valid. On the other hand, we notice that the same anomaly appears in this convolutional setting as in the fully-connected examples. We defer the discussion to Section~\ref{discussion on anomaly}.

\begin{figure}[H]
    \vspace{-0.5cm}
    \centering
    \subfigure[Sharpness and A-norm]{
    \label{conv tanh sha assumption fig}
    \includegraphics[width=0.4\textwidth]{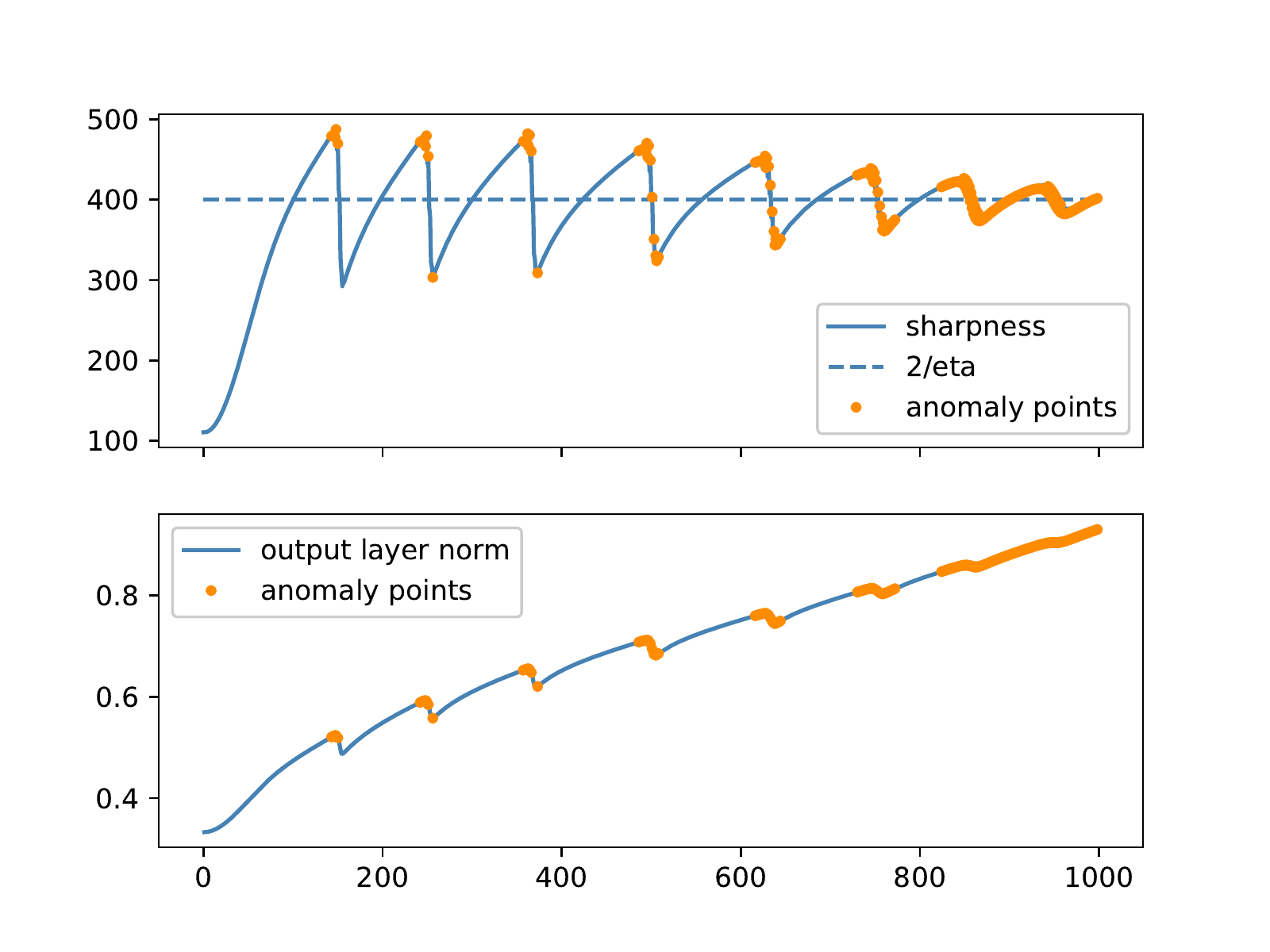}
    }
    \subfigure[The loss ($\|\bD(t)\|^2/n$) and $\|\bR(t)\|^2/n$]{
    \label{conv tanh loss and R fig}
    \includegraphics[width=0.4\textwidth]{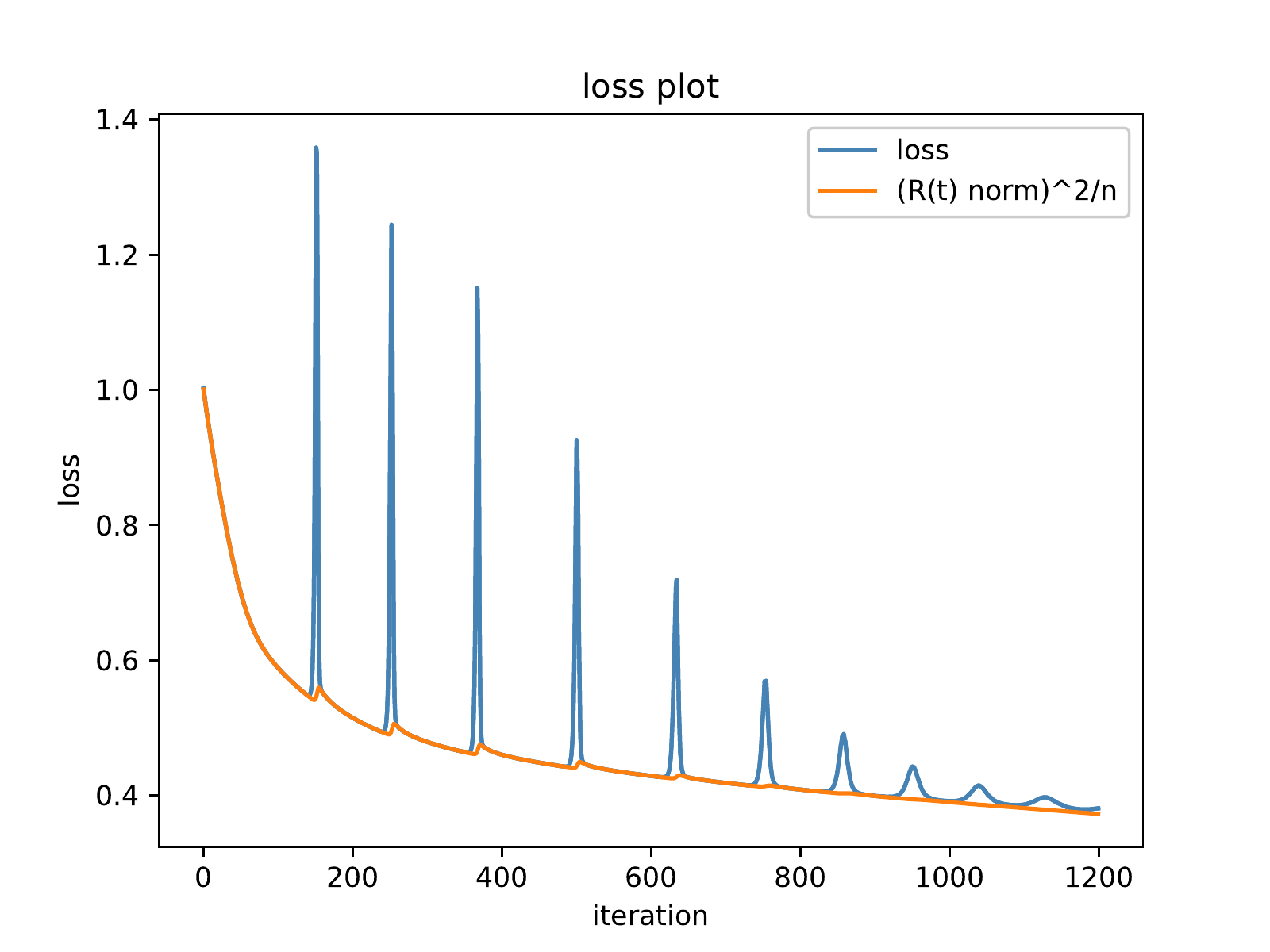}
    }
    \caption{Conv. tanh Network. In Figure (a), we verify the correlation between sharpness and A-norm in convolutional network setting. In Figure (b), we show that in this case, even though the total loss $\mathcal{L}(t) = \|\bD(t)\|^2/n$ 
    oscillates considerably, $\|\bm R(t)\|^2/n$ decreases more steadily. }
    \label{Convolutional tanh network}
\end{figure}

\begin{figure}[H]
    \vspace{-0.5cm}
    \centering
    \subfigure[Sharpness and A-norm]{
    \label{conv relu sha assumption fig}
    \includegraphics[width=0.4\textwidth]{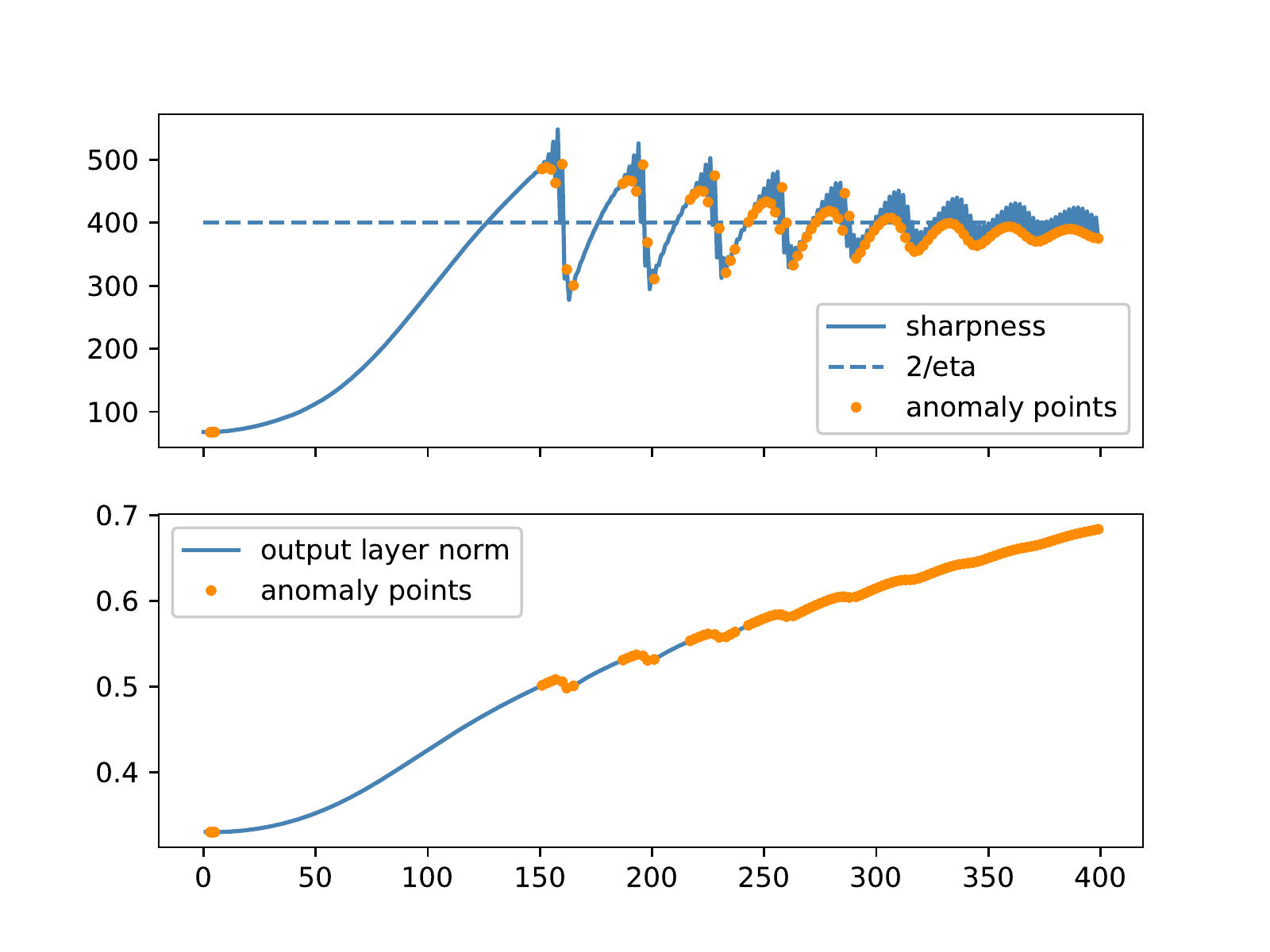}
    }
    \subfigure[The loss ($\|\bD(t)\|^2/n$) and $\|\bR(t)\|^2/n$]{
    \label{conv relu loss and R fig}
    \includegraphics[width=0.4\textwidth]{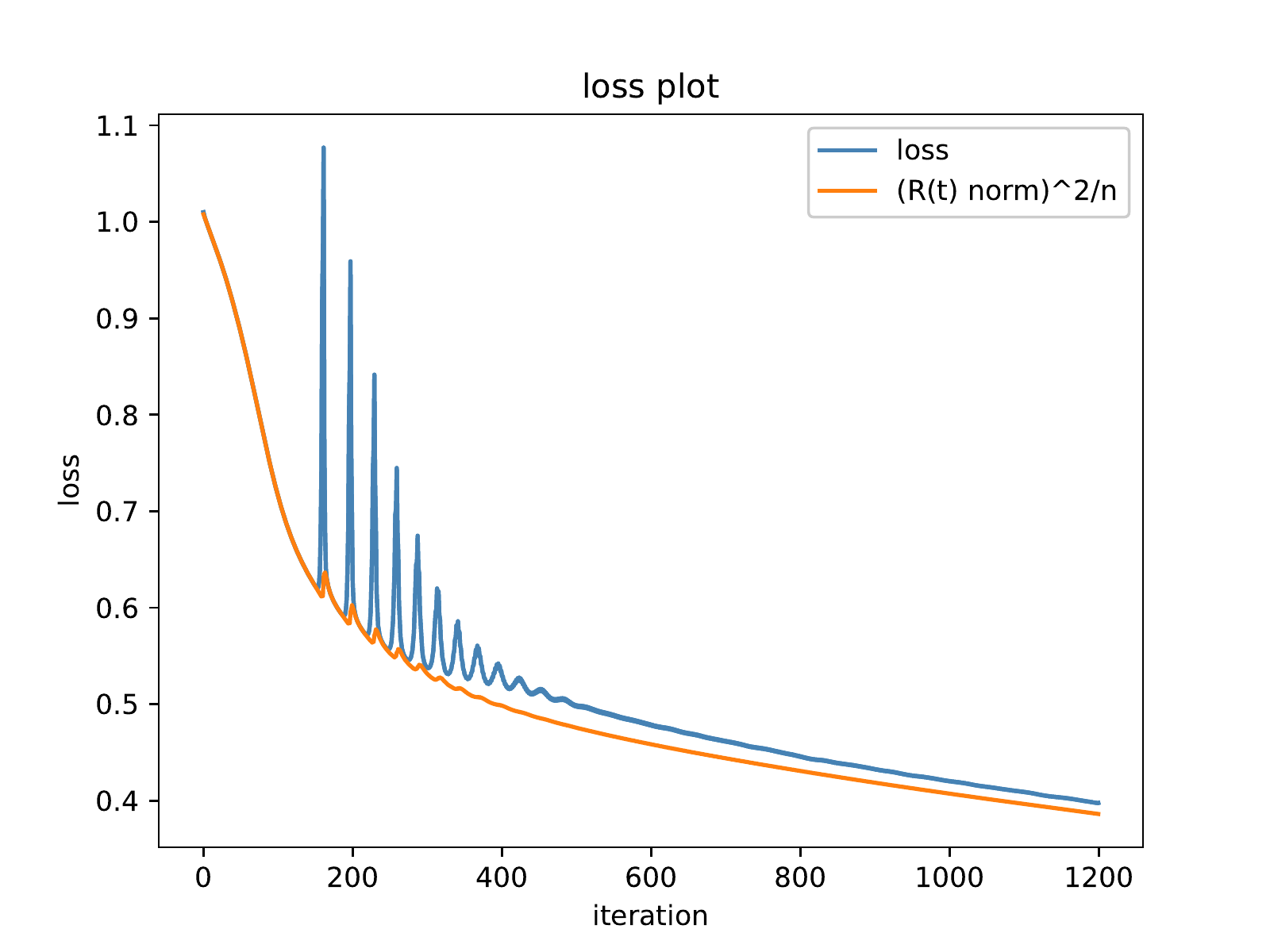}
    }
    \caption{Conv. ReLU Network. Refer to Figure \ref{Convolutional tanh network} for more information. }
    \label{Convolutional ReLU network}
\end{figure}

\begin{figure}[H]
    \vspace{-0.5cm}
    \centering
    \subfigure[Sharpness and A-norm]{
    \label{conv elu sha assumption fig}
    \includegraphics[width=0.4\textwidth]{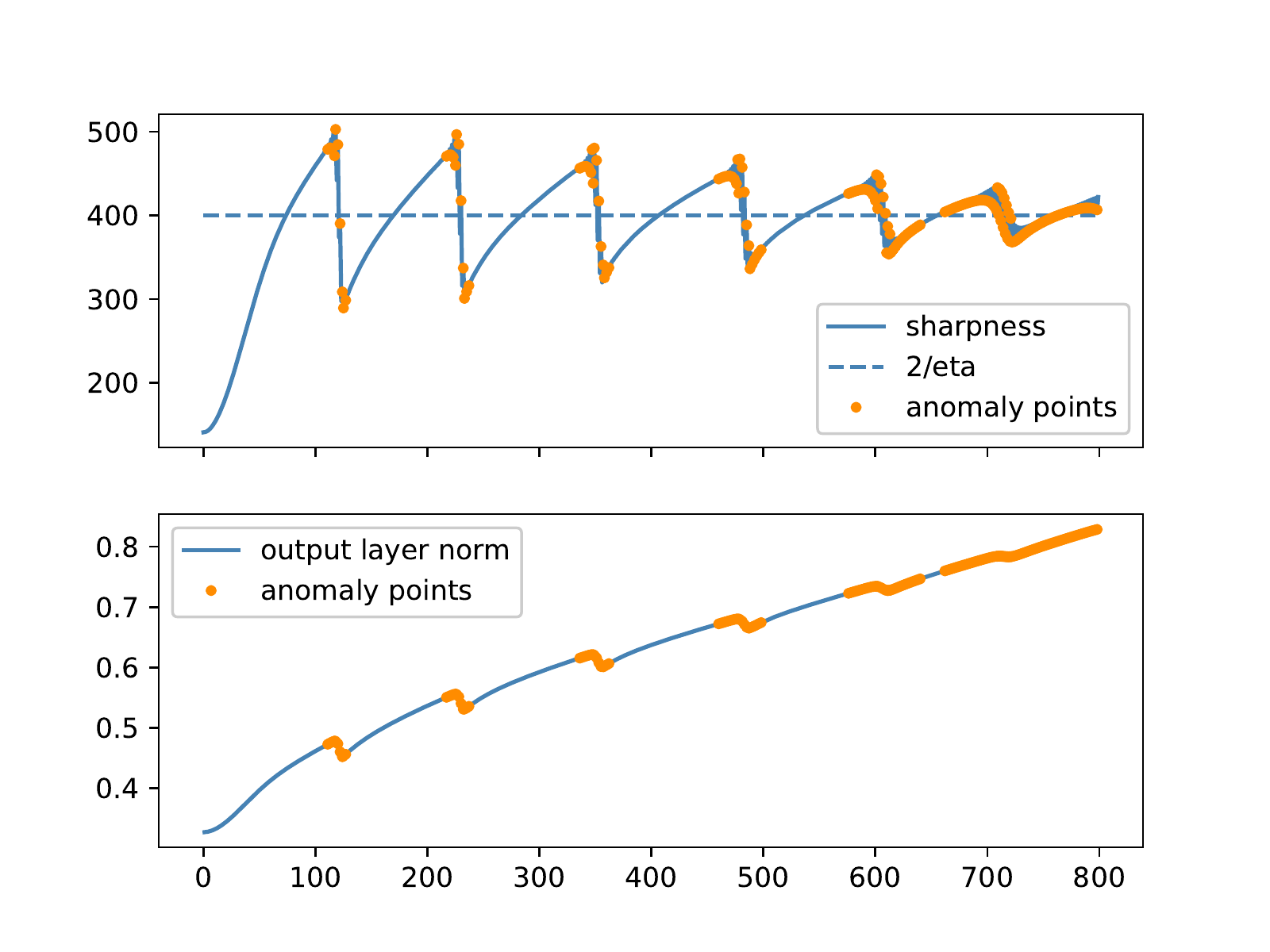}
    }
    \subfigure[The loss ($\|\bD(t)\|^2/n$) and $\|\bR(t)\|^2/n$]{
    \label{conv elu loss and R fig}
    \includegraphics[width=0.4\textwidth]{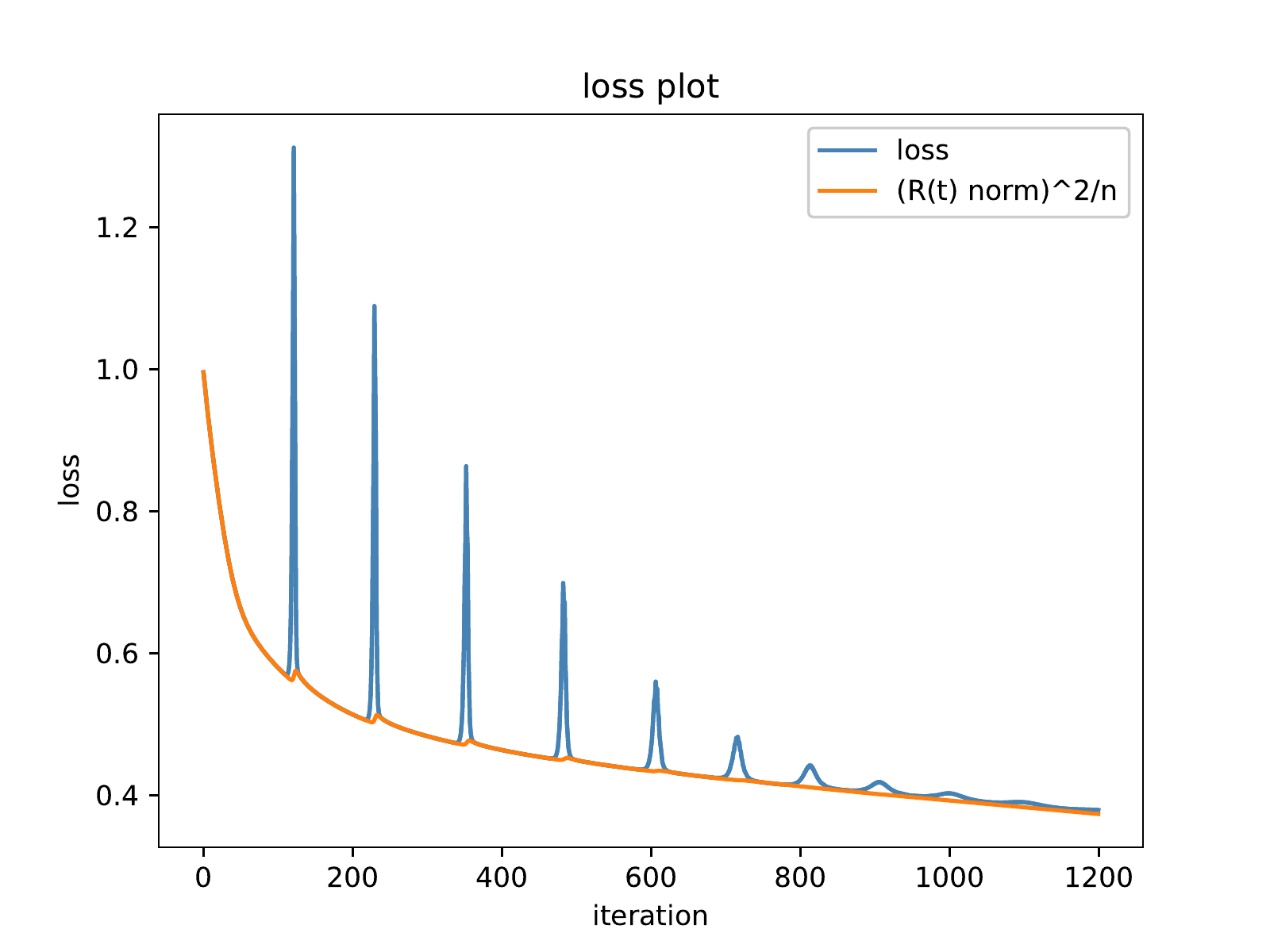}
    }
    \caption{Conv. ELU Network. Refer to Figure \ref{Convolutional tanh network} for more information. }
    \label{Convolutional ELU network}
\end{figure}

\subsubsection{Gaussian Data on Linear Networks}
\label{exp: Guassian}
In this section, we train two-layer fully-connected linear networks with datapoints sampled from Gaussian distribution and different label vectors. 

In Figure~\ref{fig:Gaussian-data}, we train two-layer fully-connected linear networks with datapoints sampled from Gaussian distribution. In particular, the width of the network is 200. The data $\bX$ is 1000 datapoints sampled from $\mathcal{N}(0,\mathbb{I}_{3072})$, where $\mathbb{I}_{3072}\in \mathbb{R}^{3072\times 3072}$ is the identity matrix. The label $\bY$ is uniformly sampled from $\text{Unif}\{-1,1\}^n$. 

In Figure~\ref{fig:Gaussian-data}, we verify that even with simple two-layer linear network and Gaussian data, progressive sharpening and EOS can still be observed. However, because Gaussian data is easy for the network to learn, the convergence is so fast that within tens of epochs the training loss converges to zero. In this case, the EOS phenomenon is not quite typical.


\begin{wrapfigure}{r}{0.5\textwidth}
  \centering
    \includegraphics[width=0.48\textwidth]{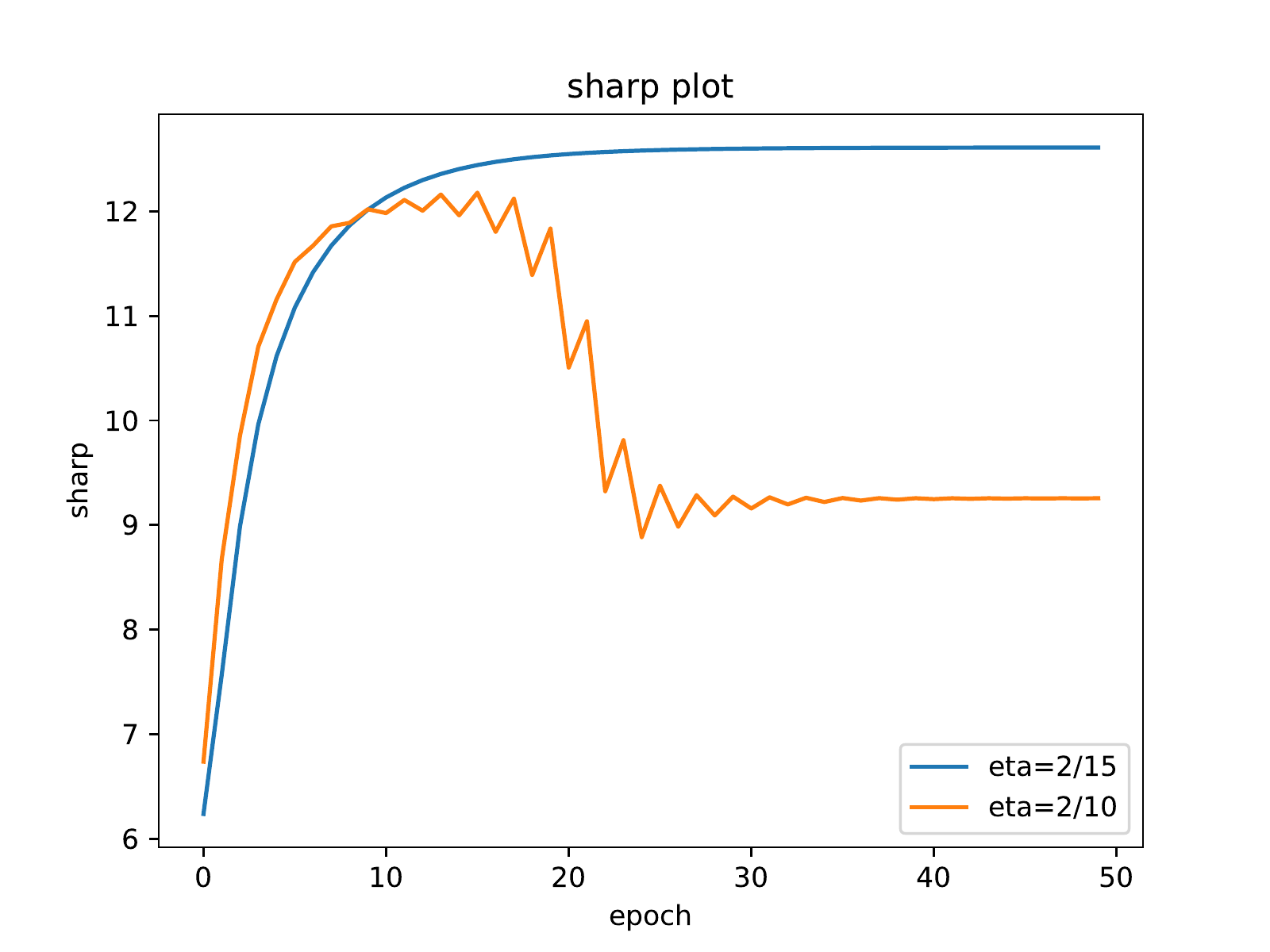}
    \caption{Two-layer Linear Network. Trained with data sampled from the standard Gaussian distribution. Both progressive sharpening and EOS are observed.}
    \label{fig:Gaussian-data}
\end{wrapfigure}

To further explore the effect of different factors on the degree of progressive sharpening, we train the network with data points sampled from Gaussian distribution and different label vectors. In particular, the width of the network is 400. The  input $\bX$ consists of 500 data points sampled from $\mathcal{N}(\mu, \Sigma)$, where the mean $\mu\in \mathbb{R}^{3072}$ and the covariance $\Sigma\in \mathbb{R}^{3072\times 3072}$ are the mean vector and the covariance matrix of 5000 CIFAR-10 data points, respectively. We note that to illustrate the degree of progressive sharpening, we choose a very small learning rate $\eta= 2/20000$
(so that the training converges before EOS can happen).

In Figure ~\ref{fig: Gaussian and CIFAR}, the label $\bY$ is uniformly sampled from $\text{Unif}\{-1,1\}^n$. 
Comparing the experimental results with those obtained using the same network trained 
with 500 CIFAR-10 data points, we find out that both have a similar degree of progressive sharpening.  
In Figure ~\ref{fig: Gaussian y along v1} and Figure ~\ref{fig: Gaussian y along v400}, we let the label $\bY$ be $\sqrt{n}v_1$ and $\sqrt{n}v_{400}$ respectively. The result shows that when the label vector is aligned with the top eigenvector $v_1$, the degree of progressive sharpening is relatively small,
and when the label vector is aligned with a bottom eigenvector, the degree of progressive sharpening is relatively large. 

This phenomenon is consistent with our intuition. 
Empirically we found that a dataset that is easier to learn leads to faster convergence rate, which then leads to a smaller degree of progressive sharpening. For example, training standard Gaussian distributed dataset converges faster than that with CIFAR-10 dataset, hence the degree of progressive sharpening of the former is much smaller; as another example
(Figure ~\ref{fig: Gaussian y along v1} and Figure ~\ref{fig: Gaussian y along v400}), 
if the label vector is aligned with the top eigenvector, the convergence in the first phase (the PS phase) is faster, which leads to a shorter first phase and thus a smaller degree of progressive sharpening compared to the other case.

\begin{figure}[ht]
    \vspace{-0.5cm}
    \centering
    \subfigure[Gaussian data with random label vector, compared with CIFAR]{
    \label{fig: Gaussian and CIFAR}
    \includegraphics[width=0.42\textwidth]{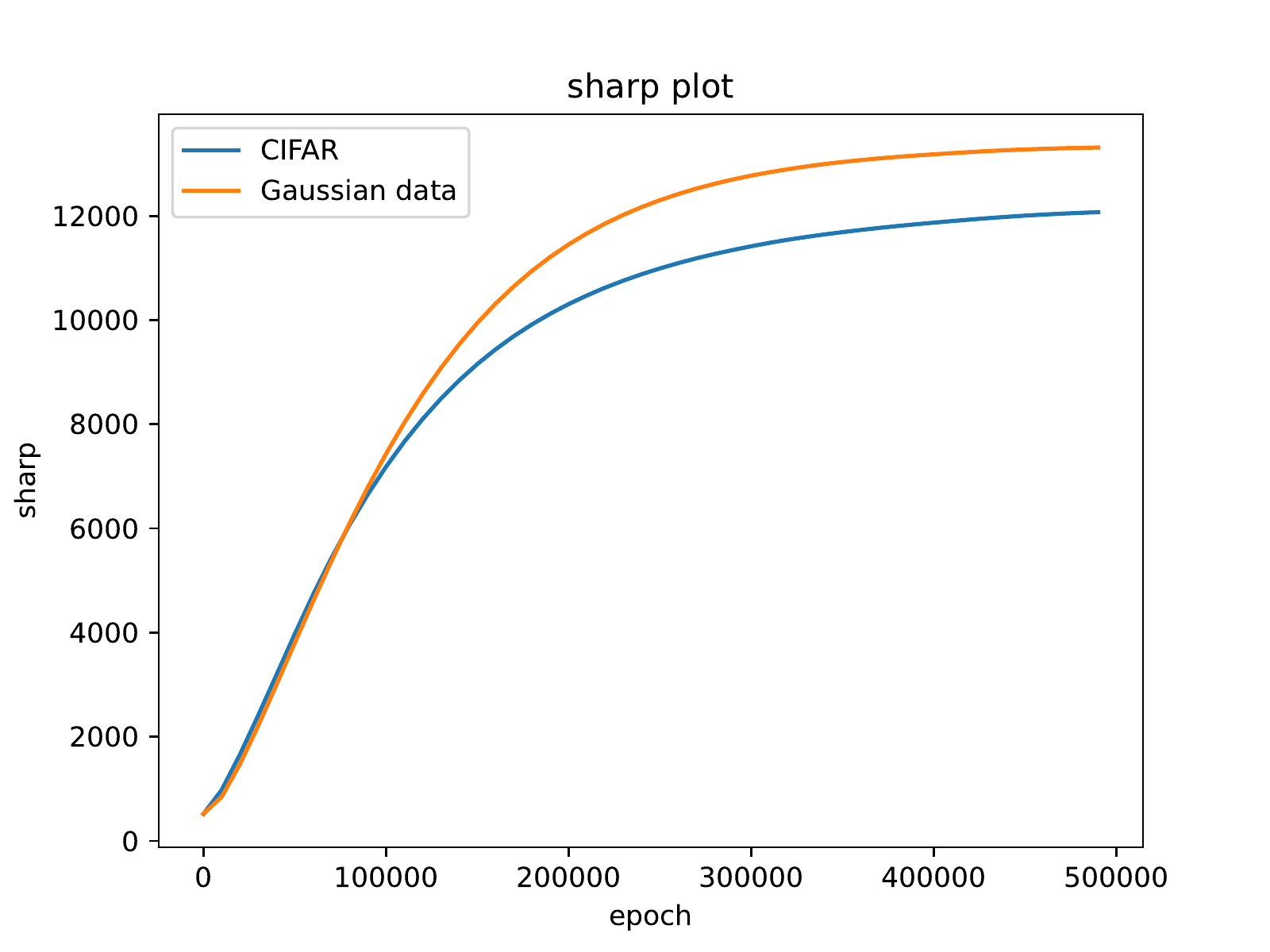}
    }
    \subfigure[Gaussian data with label vector align with $v_1$]{
    \label{fig: Gaussian y along v1}
    \includegraphics[width=0.42\textwidth]{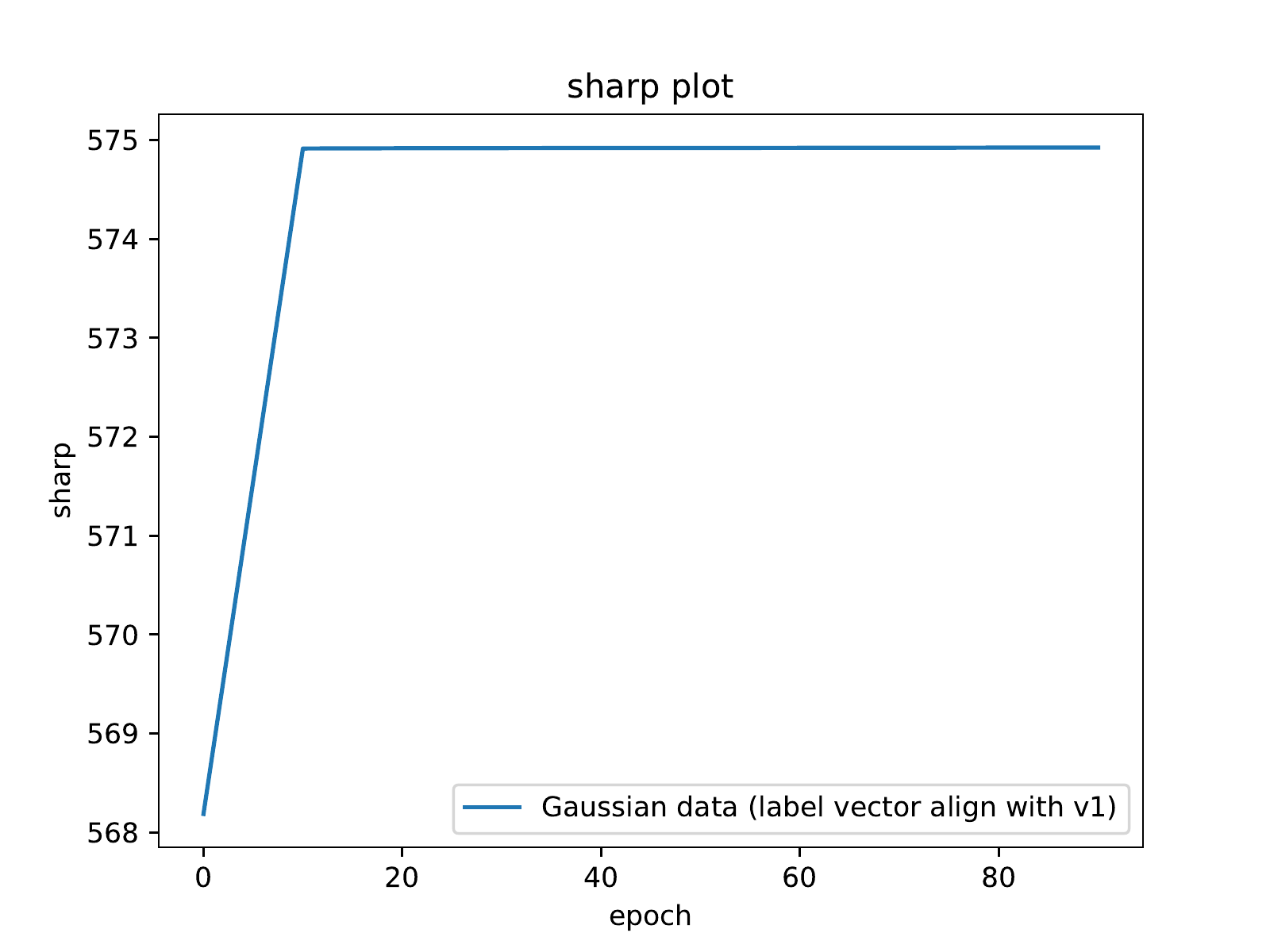}
    }
    \subfigure[Gaussian data with label vector align with $v_{400}$]{
    \label{fig: Gaussian y along v400}
    \includegraphics[width=0.42\textwidth]{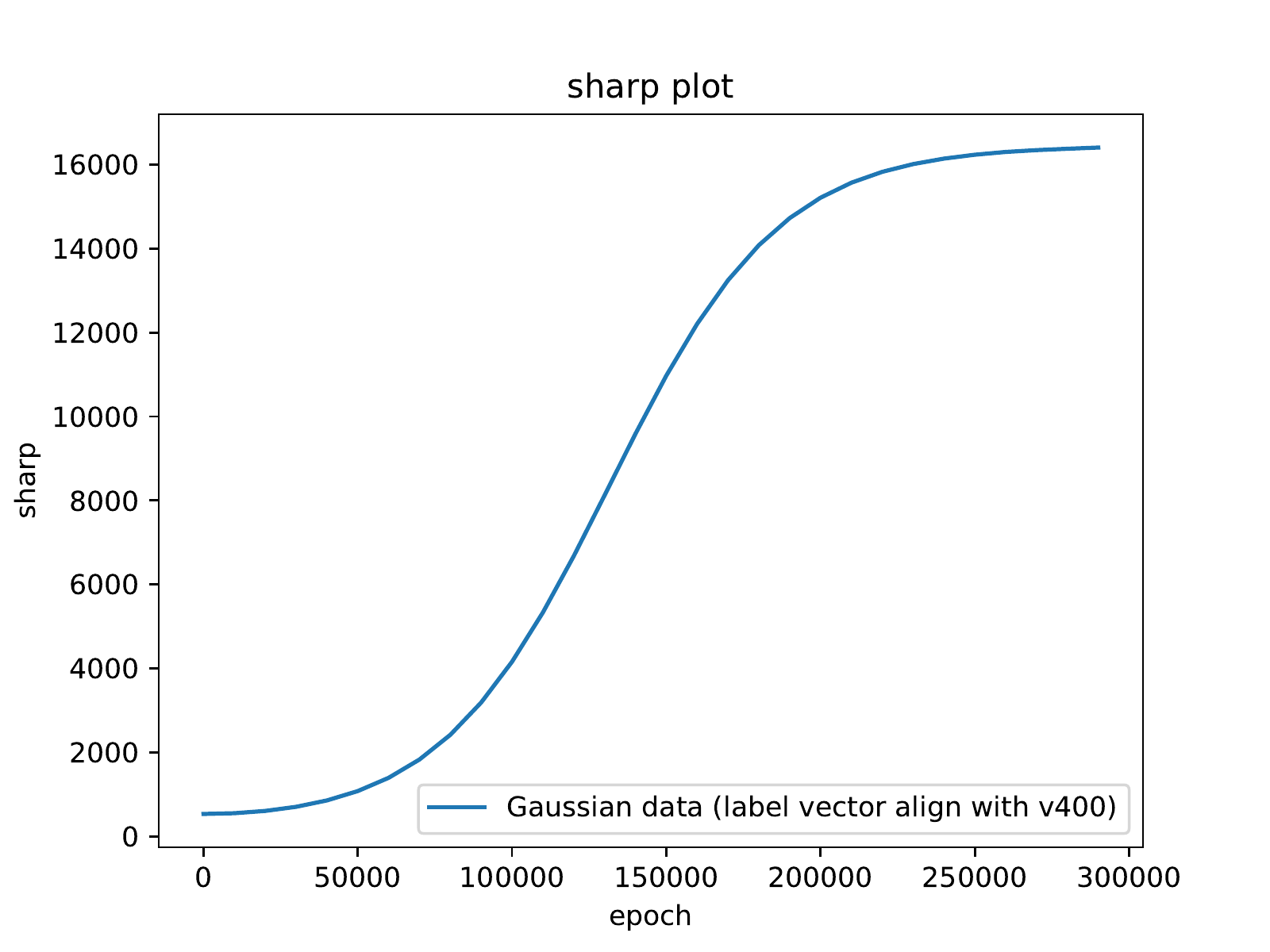}
    }
    \caption{progressive sharpening of two-layer fully-connected linear networks trained with data points sampled from anisotropic Gaussian distribution. In Figure (a), the label $\bY$ is uniformly sampled from $\text{Unif}\{-1,1\}^n$ and the result is compared with that trained with CIFAR data points. In Figure (b) and (c), the label $\bY$ takes $\sqrt{n}v_1$ and $\sqrt{n}v_{400}$ respectively.}
    \label{Fig: Anisotropic Gaussian}
\end{figure}

\subsection{Further Discussions on the Relation between A-norm and Sharpness}\label{discussion on anomaly}

In our paper, we use the observation that $\|\bm A\|^2$ shares the same trend as the sharpness to explain the dynamics of the sharpness. 
However, as shown in the experiments in this section (see Figure~\ref{fully-connected tanh network}, \ref{fully-connected relu network}, \ref{fully-connected elu network}, \ref{Convolutional tanh network}, \ref{Convolutional ReLU network}, \ref{Convolutional ELU network}),
anomaly points exist during the training process. Here we briefly discussion these anomaly points further.
We divide them into three kinds.

\textbf{At the time that $\|\bm A\|^2$ changes its trend:}
This kind of anomaly points appear when $\|\bm A\|^2$ changes its trend. 
When $\|\bm A\|^2$ changes its trend, its changing rate (or differential, the $\bm{D}^\top\bm{F}$ term) changes its sign, hence the changing rate's absolute value is small.
Now because the dynamics of the sharpness is effected by both $\|\bm A\|^2$ and the inner layers, when changing rate of $\|\bm A\|^2$ is small, the inner layers may play a larger role in the direction of sharpness, which may cause the anomaly.

\textbf{When the sharpness oscillates more frequently:} We notice that in some cases, $\|\bm A(t)\|^2$ and the sharpness $\sha(t)$ have
very similar overall trend, 
but the sharpness oscillates more frequently but the magnitude is  small (i.e., the sharpness curve has higher frequency oscillations. See Figure~\ref{fully-connected relu network} or Figure~\ref{Convolutional ELU network}). In this case, while we believe the change of $\|\bm A\|^2$ is a major driving force of the change of the sharpness, other layers must be taken into consideration for understanding such small and frequent oscillations of sharpness.

\textbf{In late training phases: } We notice that in the late training phases in most settings, sharpness oscillates and crosses $2/\eta$ more frequently
(it changes directions in a few iterations). 
In this case, our four-phase division does not strictly apply.
At the same time, $\|\bm A\|^2$ may also change direction more frequently, resulting more anomaly points during this period of time.
Hence, understanding the behavior of the late stages is beyond our current analysis and requires new insights.

\section{Proof for Section 4}\label{appen c}
\subsection{Detailed Settings}

\textbf{Model:} In this section, we study a two-layer neural network with linear activation, i.e. 
    $$
    f(\bx) = \sum_{q=1}^m\frac{1}{\sqrt m}a_q\bm w_q\bx = \frac{1}{\sqrt m} \bm A^\top \bm W\bx
    $$
    where $\bm W = [\bm w_1,\bm w_2, ...,\bm w_m]^\top\in \mathbb R^{m\times d}$ is the hidden layer's weight matrix, and $\bm A = [a_1, a_2, ..., a_m]\in \mathbb{R}^{m}$ is the weight vector of the output layer.
    $\bx\in \mathbb{R}^d$ is the input vector. 
    
    \textbf{Input distribution: } Denote by $\bX = [\bx_1, \bx_2, ..., \bx_n] \in \mathbb{R}^{d\times n}$ the training data matrix and by $\bY = (y_1, y_2, ..., y_n)\in \mathbb{R}^n$ the label vector. We assume $y_i = \pm 1$ for all $i\in [n]$, and $\|\bX^\top \bX\| = \Theta(n)$.\footnote{This property is  mentioned in \citet{hu2020surprising}. Empirically, for a randomly selected $k$-sample subset from CIFAR-10, $\|\bX^\top\bX\|/k$ is nearly constant.} We assume $\mathbf{\bX^\top \bX}$ has rank $r$, and we decompose $\mathbf{\bX^\top \bX}$ and $\bY$ according to the
    orthonormal basis $\{\bv_i\}$, the eigenvectors of $\mathbf{\bX^\top \bX}$:
    $
    \mathbf{\bX^\top \bX} = \sum_{i=1}^r \lambda_i\bm v_i\bm v_i^\top, \ \bY = \sum_{i=1}^r (\bY^\top \bm v_i) \bm v_i :=\sum_{i=1}^r z_i \bm v_i
    $
    where $\bv_i$ is the eigenvector corresponding to the $i$-th largest eigenvalue $\lambda_i$ of $\mathbf{\bX^\top \bX}$. $z_i = \bY^\top \bm v_i$ is the projection of $\bY$
    onto the direction $\bv_i$. Here we assume that $n\gg r$.
    
    Also, we suppose there exists some parameter $\bA^*$, $\bm W^*$, s.t. $$ \bY^\top = \bA^{*\top}\bm W^* \bX$$
    Note when $n\gg d$, the matrix $\bX^\top \bX$ cannot be full rank. This condition guarantees that Gradient Descent (GD) only travels in the column space of $\bX^\top \bX$ according to Lemma~\ref{colspace} at Appendix~\ref{appendixE}.

    Define the output vector as $\bF$ and the residual vector as $\bD$ in the training process. 
    
    $$\bF = (f(\bx_1), f(\bx_2), ..., f(\bx_n))\in \mathbb R^n, \bD = \bF - \bY$$ 
    (Note that they are functions of time $t$. So we use $\bD(t),\bF(t),\bA(t),\bW(t)$ to denote $\bD,\bF,\bA,\bW$ at time $t$.)
    
    Consider the mean square error MSE loss during the training process:
    $$
    \mathcal{L}(\bA,\bm W) = \frac{1}{n}\sum_{i=1}^n(f(\bx_i)-y_i)^2 = \frac{1}{n}\|\bD\|^2
    $$
    
    \textbf{Initialization: }We run GD on the loss and start from \textit{symmetric initialization} for the weights, which guarantees $\bF(0) = 0.$
    
    $$
    a_q\sim\text{Unif}(\{-1, 1\}), \ a_{q+m/2} = - a_{q}, \ q = 1,2,...,m/2
    $$
    $$
    \sum_{q=1}^{m/2} \bm w_q\bm w_q^\top = \frac{m}{2d} \bm I_d, \bm w_{q+m/2} =\bm w_q, \ q = 1,2,...,m/2
    $$
    where $\bm I_d \in \mathbb R^{d\times d}$ is the identity matrix of $d$-dimension.
    
    \textbf{Sharpness:} Recall the definition of the sharpness: $\sha(t) = \lambda_{\max}(\bM(t))$, where $\bM(t)$ is the Gram matrix \eqref{C3}.
    
    \textbf{Learning rate selection: } We select a learning rate $\eta$ such that $\sha(0)<2/\eta$.
    Specially, based on our method of initialization, $\sha(0) =\frac{m\lambda_1 (d+1)}{d}\cdot \frac{2}{mn}$, hence we have
    \begin{equation}
    \label{initialization inequality}
        \eta < \frac{nd}{(d+1)\lambda_1}
    \end{equation}
    
    \textbf{Gradient Descent Update Rule: } Following GD, the training dynamics is as follows:
    $$
    a_q(t+1)-a_q(t)=-\eta \frac{\partial \mathcal L}{\partial a_q} = -\frac{2\eta}{n\sqrt m}\sum_{i=1}^n(f(\bx_i)-y_i)\bm w_q\bx_i, \ 
    \forall q\in [m]$$
    Similarly, we have
    $$
    \bm w_q(t+1)-\bm w_q(t)=-\frac{2\eta}{n\sqrt m}\sum_{i=1}^n(f(\bx_i)-y_i)a_q\bx_i , \ 
    \forall q\in [m]
    $$
    Write them into matrix forms and we have:
    \begin{equation}
    \begin{aligned}
    \bA(t+1)-\bA(t)&=-\frac{2\eta}{n\sqrt m}\bm W (t)\bX \bD(t)\\
        \bW(t+1)-\bW(t)&=-\frac{2\eta}{n\sqrt m}\bA(t) \bD(t)^\top\bX^\top
    \end{aligned}
    \label{C1}
    \end{equation}
    
    Then we have the $\|\bA(t)\|$'s dynamics according to $\eqref{C1}$:
    \begin{equation}
    \|\bA(t+1)\|^2-\|\bA(t)\|^2 = -2\eta\bA(t)^\top\frac{\partial \mathcal{L}}{\partial \bA} + \eta^2\|\frac{\partial \mathcal{L}}{\partial \bA}\|^2 = -\frac{4\eta}{n}\bF(t)^\top\bD(t) +\eta^2 \|\frac{\partial \mathcal{L}}{\partial \bA}\|^2
    \label{C2}
    \end{equation}
    where $\bF(t)^\top \bD(t) = \bD(t)^\top\bF(t)\in \mathbb{R}$ is a real number.
    
    Define the Gram matrix as:
    \begin{equation}
        \bM(t) = \frac{2}{mn}(\|\bA(t)\|^2\bX^\top \bX +\bX^\top \bW (t)^\top \bW(t)\bX)
        \label{C3}
    \end{equation}

    We can also show the dynamics of $\bD(t)$ according to \eqref{C1} and \eqref{C3}. See Appendix~\ref{appendixE} for a proof.
    \begin{lemma}\label{update rule of d appenc} (See Lemma~\ref{update rule of d appenE})
    The update rule of the residual vector of $\bD(t)$:
    \begin{equation}
        \begin{aligned}
        \bD(t+1)^\top-\bD(t)^\top=-\eta\bD(t)^\top \bM(t) + \frac{4\eta^2}{n^2m}\bD(t)^\top(\bF(t)^\top\bD(t))\bX^\top \bX
        \end{aligned}
    \end{equation}
    \end{lemma}
    
    We define some extra notations for preparation. 
    \begin{equation}
    \begin{aligned}
       & \bM^*(t) = \bM(t) - \frac{4\eta}{n^2m}(\bD(t)^\top\bF (t))\bX^\top \bX\\
    &\bm\Gamma(t) = \frac{2}{mn}(\bX^\top \bW (t)^\top \bW(t)\bX - \frac{m}{d} \bX^\top \bX) \\
    &\widetilde \bM(t) = \bM^*(t) - \bm\Gamma(t)
    \end{aligned}
    \label{C4}
    \end{equation}
    \textbf{Remark.} Here we explain the notations above and what they are for.
    \begin{itemize}
        \item $\bM^*(t)$ is the Gram matrix $\bM(t)$ plus a second order term ($- \frac{4\eta}{n^2m}(\bD(t)^\top\bF (t))\bX^\top \bX$). Since the second order term is small, we have $\bM(t)\approx\bM^*(t)$. Also, in Section~\ref{edge of stab(phase ii-iv) appen} we prove $\bM^*(t)$ is almost a linear interpolation of $\bM(t)$ and $\bM(t+1)$, which corroborates this argument. 
        \item $\bm \Gamma(t)$ is the difference between $\frac{2}{mn}\bX^\top \bW (t)^\top \bW(t)\bX$ and $\frac{2}{mn}\bX^\top \bW (0)^\top \bW(0)\bX$. Note that $\bm\Gamma(0)=0$ because of the initialization condition. 
        In the following proof, we will show a small upper bound of the norm of $\bm\Gamma(t)$. Hence, we can show that the eigenvectors of $\bM(t)$ are approximately aligned with the eigenvectors of $\bX^\top\bX$.
        \item $\widetilde \bM(t)$ is the Gram matrix with the second order term, yet with $\bm\Gamma(t)$ excluded. $\widetilde \bM(t)$ is used for bounding the main part of $\bm \Gamma(t)$ (i.e., the terms without the noise $\be(t)$, e.g. \eqref{term A1},\eqref{term A5}) which is defined in Theorem~\ref{Theorem C.1, ps}) in the proof of Theorem~\ref{Theorem C.1, ps}.  Note that there exists some number $\gamma$ s.t. $\widetilde \bM(t) = \gamma \bX^\top\bX$. This is because
    \begin{equation}
        \widetilde \bM(t) = \bM^*(t) - \bm\Gamma(t) = \frac{2}{mn}(\|\bA(t)\|^2+\frac{m}{d}-\frac{2\eta}{n}(\bD(t)^\top\bF(t)))\bX^\top \bX
        \label{widetilde m}
    \end{equation}
    Hence for any eigenvector $\bv_i$ of $\bX^\top \bX$, $\lambda_i(\widetilde \bM(t)) = \bv_i^\top\widetilde \bM(t)\bv_i.$
    \end{itemize}

    Then the dynamics of $\bD(t)$ can be written as:
    \begin{equation}
        \bD(t+1)=(\bm I_n-\eta \bM^*(t))\bD(t)
        \label{C5}
    \end{equation}
    
    \textbf{Update Rule of $\bM(t)$: } Using the update rule of $\bA(t)$ and $\bW(t)$, we can also derive the update rule of the Gram matrix $\bM(t)$ (see Appendix~\ref{appendixE} for the proof detail):
    \begin{lemma} (See Lemma~\ref{update rule of M appenE}) The update rule of the Gram matrix $\bM(t)$ is,
    $$\begin{aligned}
  \bM(t+1) - \bM(t)  
=&-\frac{4\eta}{n^{2}m} \left( 2(\bF(t)^{\top} \bD(t)) \bX^\top \bX+\bF(t) \bD(t)^{\top}\bX^\top \bX +\bX^\top \bX\bD(t)\bF(t)^\top \right) \\
&+\frac{8\eta^2}{n^3m^2} (\bD(t)^{\top} \bX^{\top} \bm W(t)^{\top}\bm W(t) \bX \bD(t)) \bX^\top \bX\\ &+\frac{8\eta^2}{n^3m^2}\|\bA(t)\|^{2}\ \bX^\top\bX \bD(t)\bD(t)^\top\bX^\top\bX
\end{aligned}$$
\label{update rule of M}
    \end{lemma}


\subsection{Phase I and Progressive Sharpening}\label{phase i and ps appen}
    We suppose $\bX^\top\bX$ has rank $r$, and we decompose $\bX^\top\bX$ and $\bY$ into a basis $\{\bv_i\}$ composed of orthonormal eigenvectors of $\bX^\top\bX$.
    $$
    \bX^\top \bX = \sum_{i=1}^r \lambda_i\bv_i\bv_i^\top, \bY = \sum_{i=1}^r (\bY^\top \bv_i) \bv_i :=\sum_{i=1}^r z_i \bv_i
    $$
    
where $\bv_i$ is the eigenvector corresponding to the $i$-th largest eigenvalue $\lambda_i$ of $\bX^\top \bX$. $z_i =\bY^\top \bv_i$ is the $\bY$'s projection onto the direction $\bv_i$. In particular, $\bv_1$ is the eigenvector corresponding to the largest eigenvalue of $\bX^\top \bX$.

Recall that the sharpness $\sha(t) = \lambda_{\max}(\bM(t))$. Here we propose an approximation of sharpness $\sha$ by the eigenvector $\bv_1$: 

\begin{equation}
    \sha^*(t) = \bv_1^\top \bM(t)\bv_1
    \label{B.0}
\end{equation}


The following corollary rationalizes this approximation. The proof is deferred to Corollary~\ref{v_1 similarity coro} in Section~\ref{edge of stab(phase ii-iv) appen}. 

\begin{Corollary} 
There exists some constant $c_2, c_4$, s.t. for all time t $\langle \bv_1, \bm{u} \rangle^2\geq 1 - c_4\sqrt{\frac{1}{nm}}$ and $$ \sha(t)\geq\sha^*(t) \geq \sha(t)-\frac{2c_2}{m} $$
where $\bm{u}$ is the top eigenvector of $\bM(t)$.
\end{Corollary}
By this corollary, we can see that the top eigenvector of $\bM(t)$ is approximately $\bv_1$, and the approximation $\sha^*(t)$ is very close to the real sharpness $\sha(t)$. Thus, we consider $\sha^*(t)$ as an approximation of the sharpness and analyze its dynamics. 

We first show the main theorem of progressive sharpening under Assumption~\ref{eigenspectrum assumption appen} and \ref{y assumption appen} below. 
\begin{Assumption}
There exists some constant $\chi>1$, s.t. for all $i\in[r-1]$,
$ \lambda_i(\bX^\top \bX) \leq \chi\lambda_{i+1}(\bX^\top \bX)$. Moreover, $\lambda_1(\bX^\top \bX)\geq 2\lambda_2(\bX^\top \bX).$
\label{eigenspectrum assumption appen}
\end{Assumption}

\begin{Assumption}
There exists $\smallconst = \Omega(r^{-1})$ 
such that $\min_{i\in [r]}\{z_i/\sqrt n\}\geq \smallconst$.
\label{y assumption appen}
\end{Assumption}

The first assumption is about the eigenvalue spectrum of $\bX^\top \bX$. It guarantees the gap between two adjacent eigenvalues is not very large, and there is a gap between the largest and the second largest eigenvalue. Note the second part of the assumption is a relaxed version of Assumption~\ref{outlier assp}. In our CIFAR-10 1k-subset with samples' mean subtracted, $\lambda_1/\lambda_2 = \chi \approx 3$ (See Appendix~\ref{Assumption 4.1}, Figure~\ref{eigenspec assumption fig}). 
The second assumes that all components $z_i=\bY^\top\bv_i$ are not too small.

Under the two assumptions, we have the following theorem:
\begin{theorem}[\textbf{Progressive Sharpening}]\label{Theorem C.1, ps}

    Suppose Assumption~\ref{eigenspectrum assumption appen}, \ref{y assumption appen} hold. Suppose $\lambda_r=\lambda_{\min}(\bX^\top \bX)>0, \lambda_1 = \lambda_{\max}(\bX^\top \bX) = \Theta(n)$. For any $\epsilon > 0$, if $m=\Omega(\frac{ n^2}{\lambda_r^2})$, and $n = \max\{\Omega(\frac{\lambda_r}{\kappa^2}), \Omega(\frac{\lambda_r^2}{\kappa^2}), \Omega(\frac{\lambda_r^2}{\kappa^4\epsilon^2}),\Omega(\frac{\lambda_r^4}{\kappa^4\epsilon^2})\}$, we have the following properties for $t= 1,2, ...,  t_0-1$  where $t_0$ is the time when $\|\bD(t)\|^2= O(\epsilon^2)$ or $\lambda_{\max}(\bM^*(t))>1/\eta$ for the first time.

\begin{itemize}
    \item (Progressive Sharpening) (Lemma~\ref{progressive shap appen}) $\sha^*(t+1)-\sha^*(t)>0.$
    \item (Lemma~\ref{direction guarantee}) $\|\bm\Gamma(t)\|=O(1/m)$.
\end{itemize}
\end{theorem}

    We first prove Lemma~\ref{direction guarantee}, which implies that $\bD(t)$  approximately converges independently in each direction of $\bX^\top \bX$'s non-zero eigenspace. Lemma~\ref{direction guarantee} also proves the second point $\|\bm\Gamma(t)\|\leq O(1/m)$.
    (Recall that $\|\bm\Gamma(t)\|$ is defined in \eqref{C4}.) 

    Then, based on the conclusion in Lemma~\ref{direction guarantee}, Assumption~\ref{eigenspectrum assumption appen} and ~\ref{y assumption appen}, we prove that $\|\bA(t)\|^2$ grows in Lemma~\ref{Norm growth} and Lemma~\ref{lemma c5}. Specifically, here we prove a sufficient condition of $-\bD(t)^\top\bF(t)>0$.

    Finally, we use the dynamics of $\sha^*(t)$ and its dependence on $\|\bA(t)\|^2$ dynamics to prove the sharpness grows (Lemma~\ref{progressive shap appen}), thus finishing the proof.

\textbf{Remark 1.} In the theorem, $\lambda_{\max}(\bM^*(t))>1/\eta$ is the termination condition for this theorem. Here, $\bM^*(t)$ is the Gram matrix $\bM(t)$ plus the second order term ($- \frac{4\eta}{n^2m}(\bD(t)^\top\bF (t))\bX^\top \bX$). 
Since the second order term is small, we have $\bM(t)\approx\bM^*(t)$. Also, in Section~\ref{edge of stab(phase ii-iv) appen}, we prove $\bM^*(t)$ is almost a linear interpolation of $\bM(t)$ and $\bM(t+1)$. Thus, $\lambda_{\max}(\bM^*(t))$ can be seen as an approximation of the sharpness.

\textbf{Remark 2.} 
From the experiments, one can see that gradient descent is in progressive sharpening phase until the sharpness crosses the threshold $2/\eta$. Right now, our proof only works till $\lambda_{\max}(\bM^*(t))$ reaches $1/\eta$. It would be interesting to extend our result
to $2/\eta$.

\textbf{Remark 3.} In this theorem, we require mild over-parameterization ($m=\Omega(n^2)$, assuming $\lambda_r = \Theta(1)$, to prove the direction guarantee of the Gram matrix $\bM(t)$ and the monotone increment of $\|\bA(t)\|^2$. 
Astute readers may find it similar to the NTK regime, where the parameters do not move far from the initialization. 
However, we stress that our analysis does not necessarily require NTK regime, and can go beyond NTK. We defer the discussion on the difference between our results and the NTK regime to Appendix~\ref{our result and ntk}. 

Then we break the proof of theorem into lemmas.
The first lemma (Lemma~\ref{direction guarantee}) proves that the Gram matrix $\bM(t) \approx \frac{2}{mn}(\|\bm A(t)\|^2+\frac{m}{d})\bX^\top \bX$ by bounding the movement of $\|\bm\Gamma(t)\|$. In this way, $\bD(t)\approx -\prod_{j=0}^{t-1}(\bm I_n - \eta \widetilde \bM(j))\bY$ can approximately descend in each eigenvector $\bv_i$ of $\bX^\top\bX$ independently. Also, by proving $\bM(t) \approx \frac{2}{mn}(\|\bm A(t)\|^2+\frac{m}{d})\bX^\top \bX$, we justify $\|\bA(t)\|^2$ as an indicator of the sharpness.

\begin{lemma}[\textbf{Direction Guarantee}]\label{direction guarantee}
Along GD training trajectory, if $m\geq \frac{112 c_1n^2}{\lambda_r^2}$ , we have the following properties until $\lambda_{\max}(\bM^*(t)) > 1/\eta $:
\begin{enumerate}
    \item $\|\bA(t)\|^2 \geq m / 2$ \textit{and} $ \|\bm\Gamma(t)\| \leq R_w := 40/m$. 
    \item $\bD(t) = -\prod_{j=0}^{t-1}(\bm I_n - \eta \widetilde \bM(j))\bY  + {\bm e(t)}$, \textit{where} $\|\bm e(t)\|_2\leq \frac{40 n^{3/2}}{\lambda_r m} $
    \item $\lambda_{r}(\widetilde \bM(t))\geq \lambda_r/n, \ \lambda_{r}(\bM^*(t))\geq \lambda_r/n$.
\end{enumerate}
\end{lemma}
\begin{proof}
We prove the theorem by induction. 

We first consider the base case. For property 1, $\|\bA(0)\|^2=m$ and $\|\bm\Gamma(0)\|=0$.  For property 2, with \textit{symmetric initialization}, $\bF(0) = 0$, hence $\bD(0) = -\bY$, $\|\bm e(0)\|=0$. For property 3, $$
\bM^*(0)=\widetilde \bM(0) = \bM(0)\succeq \frac{2}{nm}\|\bA(0)\|^2\bX^\top \bX
$$
The minimal eigenvalue of $\frac{2}{nm}\|\bA(0)\|^2\bX^\top \bX$ is $2\lambda_r/n>\lambda_r/n$.
Thus all the properties hold at iteration $t=0$.

Suppose for all $k\leq t$, these properties hold. Then we consider the case in iteration $t+1$. 

We first show a worst-case upper bound for $\|\bD(t)\|$.
\begin{align*}
    \|\bD(t)\| &=\left\|\prod_{k=0}^{t-1}(\bm I - \eta \bM^*(k))\bD(0)\right\|\\
    &\leq \left\|\prod_{k=0}^{t-1}(\bm I - \eta \bM^*(k))\right\|\left\|\bD(0)\right\|\\
    &\leq (1-\frac{\eta\lambda_r}{n})^t \|\bY\|= \sqrt n(1-\frac{\eta\lambda_r}{n})^t
    \tag{D}\label{worst case D}
\end{align*}
The second inequality uses property 3 in $k\leq t$. Note $\|\bY\|=\sqrt n$ in our setting, so the last equality holds. That means $\|\bD(t)\|\leq \sqrt n$ and $\|\bF(t)\|=\|\bD(t)+\bY\|\leq 2\sqrt n.$

Now, we show the error $\bm e(t)$ is bounded. We have 

$\begin{aligned} \bD(t+1) &=(\bm I-\eta \widetilde{\bM}(t))\bD(t)-\eta   \bm\Gamma(t)\bD(t) \\ &=- \prod_{j=0}^{t}(\bm I-\eta \widetilde{\bM}(j))\bY+ (\bm I-\eta \widetilde{\bM}(t))\bm e(t)-\eta   \bm\Gamma(t)\bD(t) \\ &=- \prod_{j=0}^{t}(\bm I-\eta \widetilde{\bM}(j))\bY+ (\bm I-\eta \bM^*(t))\bm e(t)-\eta  \bm \Gamma(t)\prod_{j=0}^{t-1}(\bm I-\eta \widetilde{\bM}({j}))\bD(0) . \end{aligned}$

Hence we have $\bm e(t+1)=\bm e(t) (\bm I-\eta \bM^*(t))-\eta  \bm\Gamma(t)\prod_{j=0}^{t-1}(\bm I-\eta \widetilde{\bM}(j))\bD(0)$. 

Then we can have the following bound by this recursion:
\begin{align*}\|\bm e(t+1)\|_{2} & \leq \|\bm e(t)\|_{2}\|\bm I-\eta \bM^*(t)\|_{2}+\eta\|\bD(0)\|_{2}\left\|\prod_{j=0}^{t-1}(\bm I-\eta \widetilde{\bM}(j))\right\|_{2}\|\bm\Gamma(t)\|_{2} \\ & \leq \|\bm e(t)\|_{2}\left(1-\frac{\eta\lambda_r}{n}\right)+\eta \sqrt{n} \left(1-\frac{\eta\lambda_r}{n}\right)^{t} R_{w} \\ 
&\leq \sum_{k=0}^{t} \eta \sqrt{n}\left(1-\frac{\eta\lambda_r}{n}\right)^{t} R_{w} \tag{E1}\label{e bound}\\&\leq \frac{\sqrt{n}}{\frac{1}{n}\lambda_{r}} \cdot R_{w}\\&\leq \frac{40n^{3 / 2}}{m\lambda_{r}} \end{align*}

Here the third inequality holds because $\|\bm e(0)\| = 0$. Thus the second property holds for $t+1$.

Then we consider the lower bound of $\|\bA(t)\|^2$.

Consider the dynamics of $\|\bA(t)\|^2$ in \eqref{C2} and sum the difference up from 0 to $t$. Recall that we proved $\|\bD(t)\|\leq \sqrt n$ by \eqref{worst case D} above.

\begin{align*}
    \|\bA(t+1)\|^2-\|\bA(0)\|^2 &= \frac{4\eta} {n} \sum_{k=0}^t(-\bD(k)^\top\bF(k)) + \eta^2 \sum_{k=0}^t\|\frac{\partial \mathcal{L}(k)}{\partial \bA(k)}\|^2\\
    &\geq -\frac{4\eta}{n}\sum_{k=0}^{t} \bD(k)^\top(\bD(k)+\bY)&&\text{ [$\|\frac{\partial \mathcal{L}(k)}{\partial \bA(k)}\|^2>0$]}\\
    &\geq -\frac{4\eta}{n}\sum_{k=0}^{t} (\|\bD(k)\|^2 + \|\bD(k)\| \|\bY\|)&&\text{ [$\ell_2$-norm inequality]} \\
    &\geq -\frac{4\eta}{n}\sum_{k=0}^{t} \|\bD(k)\|(\sqrt n + \sqrt n)&&\text{ [$\|\bD(t)\|\leq \sqrt n$]}\\
    &\geq -\frac{8\eta}{n} \sum_{k=0}^{t} n (1- \frac{\eta\lambda_r}{n})^k &&\text{[ by \eqref{worst case D}]}\\
    &\geq -\frac{8n}{\lambda_r}\geq -\frac{m}{2}\tag{\#}\label{N1}
\end{align*}

Since $\|\bA(0)\|^2=m$, we have the lower bound $\|\bA(t)\|^2\geq m/2$.

Next, we lower bound the minimal non-zero eigenvalue (i.e. the $r$-th largest eigenvalue since $\bX^\top\bX$ is rank $r$) of $\bM^*(t+1)$ and $\widetilde \bM(t+1)$.

For $\bM^*(t+1)$ and $\widetilde \bM(t+1)$, since the $\bX^\top\bW(t)^\top\bW(t)\bX$ part is PSD, we have

$$
\bM^*(t+1) \succeq \frac{2}{mn}\|\bA(t+1)\|^2\bX^\top \bX  - \frac{4\eta}{n^2m}(\bD(t+1)^\top\bF(t+1))\bX^\top \bX 
$$
$$
\widetilde\bM(t+1) \succeq \frac{2}{mn}\|\bA(t+1)\|^2\bX^\top \bX  - \frac{4\eta}{n^2m}(\bD(t+1)^\top\bF(t+1))\bX^\top \bX 
$$
From the inequality in \eqref{N1}, and from \eqref{worst case D} $\|\bD(t)\|\leq \sqrt n,\|\bF(t)\|\leq 2\sqrt n$, we know 

\begin{align*}
    \|\bA(t+1)\|^2-\|\bA(0)\|^2 - \frac{2\eta}{n}(\bD(t+1)^\top\bF(t+1))&\geq -\frac{8n}{\lambda_r}- \frac{2\eta}{n} \sqrt n\cdot 2\sqrt n\geq -m/2 \\\
\end{align*}
$$\bM^*(t+1)\succeq \frac{2}{mn}(\|\bA(t+1)\|^2 - \frac{2\eta}{n}(\bD(t+1)^\top\bF(t+1)))\bX^\top\bX\succeq\frac{1}{n}\bX^\top\bX$$
Thus $\bM^*(t+1)$ (and also $\widetilde\bM(t+1)$) has its smallest eigenvalue larger than $\lambda_r/n$. Property 3 holds. To extend this conclusion, we actually have $\lambda_i(\widetilde\bM(t+1))\geq \lambda_i /n$ for all $i\in[r]$ from the argument above. 

Finally we bound $\|\bm\Gamma(t+1)\|$. We first write down the dynamics of $\bm\Gamma(t)$ according to Lemma~\ref{update rule of M}.

\begin{align*}
    \bm\Gamma(t+1)-\bm\Gamma(t)={}&\frac{2}{nm}\left(\bX^\top (\bm W(t+1)^{\top}\bm W(t+1)-\bm W(t)^{\top}\bm W(t))
        \bX\right)\\
    ={}& \frac{2}{nm}\left(\bX^\top ((\bm W(t+1)^{\top}-\bm W(t)^{\top} )\bm W(t) +\bW(t)^\top (\bm W(t+1)-\bm W(t)))
        \bX\right.\\
        &\left.+\ \bX^\top (\bm W(t+1)^{\top}-\bm W(t)^{\top} )(\bm W(t+1)-\bm W(t))\bX\right)\tag{Use Equation~\eqref{C1}}\\
    ={}&-\frac{4}{n^{2} m} \eta\left(\bF (t) \bD(t)^\top \bX^{\top} \bX 
    +\bX^{\top} \bX \bD{(t)} \bF(t)^{\top}\right)\\&{}+\frac{8 \eta^{2}}{n^{3} m^{2}}\|\bA(t)\|^{2} \cdot \bX^{\top} \bX \bD(t) \bD(t)^\top \bX^{\top} \bX
\end{align*}

Hence we sum it up and get:
\begin{align*}
\left\|\bm\Gamma(t+1)-\bm\Gamma(0)\right\|_{2}
=&\left\| \frac{4}{n^{2} m} \eta \sum_{k=0}^{t}\left(\bF(t) \bD(t)^{\top} \bX^{\top} \bX+\bX^{\top} \bX \bD(t) \bF(t)^{\top}\right)\right.\\
&\left.+\frac{8 \eta^{2}}{n^{3} m^{2}} \sum_{k=0}^{t}\|\bA(t)\|^{2} \cdot \bX^{\top} \bX \bD(t) \bD(t)^{\top} \bX^{\top} \bX \right\|_{2}\tag{Triangle Inequality}\\
\leq{}& \frac{4}{n^{2} m} \eta \left\| \sum_{k=0}^{t} \left( \bD (t) \bD(t)^\top \bX^{\top} \bX+\bX^{\top} \bX \bD(t) \bD(t)^{\top}\right) \right\|_{2} \quad \tag{G1}\label{Gamma (1)}\\
&+\frac{4}{n^{2} m} \eta\left\|\sum_{k=0}^{t}\left(\bY \bD(t)^{\top} \bX^{\top} \bX+\bX^{\top} \bX \bD(t)\bY^\top\right)\right\|_{2} \quad \tag{G2}\label{Gamma (2)}\\
&+\frac{8 \eta^{2}}{n^{3} m^{2}}\left\|\sum_{k=0}^{t}\left\| \bA(t)\right\|^{2} \cdot \bX^{\top} \bX \bD(t) \bD(t)^\top \bX^{\top} \bX\right\|_{2}\quad  \tag{G3}
\label{Gamma (3)}
\end{align*}

Then we bound these three terms one by one. 

Term \eqref{Gamma (1)}: $\frac{4}{n^{2} m} \eta \left\| \sum_{k=0}^{t} \left( \bD (t)\bD(t)^\top \bX^{\top} \bX+\bX^{\top} \bX\bD(t) \bD(t)^{\top}\right) \right\|_{2}$

By symmetry, we just consider the first term $\bD^\top (t) \bD(t) \bX^{\top} \bX.$
\begin{align*}
    \bD (t) \bD(t)^\top \bX^{\top}\bX =&\prod_{j=0}^{t-1}(\bm I_n - \eta \widetilde \bM(j))\bY\bY^\top\prod_{j=0}^{t-1}(\bm I_n - \eta \widetilde \bM(j)) \bX^\top \bX \tag{A1}\label{term A1} \\
    &+\bm e(t)\bY^\top\prod_{j=0}^{t-1}(\bm I_n - \eta \widetilde \bM(j))\bX^\top \bX\tag{A2}\label{term A2}\\
    &+\prod_{j=0}^{t-1}(\bm I_n - \eta \widetilde \bM(j))\bY \be(t)^\top\bX^\top \bX\tag{A3}\label{term A31}\\
    &+\be(t)\be(t)^\top\bX^\top \bX\tag{A4}\label{term A4} 
\end{align*}

We bound each term in $\ell_2$-norm and then add them up.

Term~\eqref{term A1} 
\begin{align*}
&\frac{4\eta}{n^{2} m} \left\|\sum_{k=0}^t\prod_{j=0}^{k-1}(\bm I_n - \eta \widetilde \bM(j))\bY\bY^\top\prod_{s=0}^{k-1}(\bm I_n - \eta \widetilde \bM(s)) \bX^\top \bX\right\|_2\\
=&\frac{4\eta}{n^{2} m} \left\|\sum_{k=0}^t\prod_{j=0}^{k-1}(\bm I_n - \eta \widetilde \bM(j))\sum_{i=1}^r z_i\bv_i\sum_{j'=1}^rz_{j'}\bv_{j'}^\top\prod_{s=0}^{k-1}(\bm I_n - \eta \widetilde \bM(s)) \bX^\top \bX\right\|_2\\
=&\frac{4\eta}{n^{2} m} \left\|\sum_{k=0}^t\left(\sum_{i=1}^r\prod_{j=0}^{k-1}(1 - \eta  \lambda_i(\widetilde \bM(j)))z_i\bv_i\right)\left(\sum_{j'=1}^r\bv_{j'}^\top z_{j'}\prod_{s=0}^{k-1}(1 - \eta \lambda_{j'}(\widetilde \bM(s))) \lambda_{j'}\right)\right\|_2\\
&\text{where $\lambda_j(\widetilde\bM) =  \bv_j^\top  \widetilde\bM\bv_j$ (See \eqref{widetilde m}). We know since $\|\bA(t)\|^2\geq m/2$, $\lambda_j(\widetilde \bM(t))\geq \frac{1}{n}\lambda_j$, $\forall t, j$}\\
= & \frac{4\eta}{n^{2} m}\left\|\sum_{i=1}^r\sum_{j'=1}^r\sum_{k=0}^t\left(\prod_{j=0}^{k-1}(1 - \eta  \lambda_i(\widetilde \bM(j)))z_iz_{j'} (1 - \eta \lambda_{j'}(\widetilde \bM(j))) \lambda_{j'}\right)\bv_i\bv_{j'}^\top\right\|_2\\
\leq & \frac{4\eta}{n^{2} m}\left\|\sum_{i=1}^r\sum_{j'=1}^r\sum_{k=0}^t\left(\prod_{j=0}^{k-1}(1 - \eta  \lambda_i(\widetilde \bM(j)))z_iz_{j'} (1 - \eta \lambda_{j'}(\widetilde \bM(j))) \lambda_{j'}\right)\bv_i\bv_{j'}^\top\right\|_F \\ 
= & \frac{4\eta}{n^{2}m}\sqrt{\sum_{i=1}^r\sum_{j'=1}^r\left(\sum_{k=0}^t\prod_{j=0}^{k-1}(1 - \eta  \lambda_i(\widetilde \bM(j)))z_iz_{j'} (1 - \eta \lambda_{j'}(\widetilde \bM(j))) \lambda_{j'}\right)^2}\\
\leq&\frac{4\eta}{n^{2}m}\sqrt{\sum_{i=1}^r\sum_{j'=1}^r\left(\sum_{k=0}^t(1 - \eta \frac{1}{n}\lambda_{j'} )^kz_iz_{j'} (1 - \eta\frac{1}{n} \lambda_{i})^k \lambda_{j'}\right)^2} \\
\leq & \frac{4}{n^{2}m}\sqrt{\sum_{i=1}^r\sum_{j'=1}^rz_i^2z_{j'}^2 \lambda_{j'}^2\frac{\eta^2}{(\frac{1}{n}\eta (\lambda_{i}+\lambda_{j'})-\frac{\eta^2}{n^2}\lambda_i\lambda_{j'})^2}}  \\ 
\leq & \frac{4}{n^2 m}\sqrt{\|\bY\|_2^4\cdot n^2}\\
= & \ \frac{4}{m}
\end{align*}

The first inequality comes from $\|\cdot\|_2\leq \|\cdot\|_F$. The last three inequalities require that $\eta\lambda_{\max{}}(\bM^*)<1.$ 

Combined with its symmetric counterpart in $\bX^\top \bX \bD(t) \bD(t)^\top$, the sum of \eqref{term A1} is smaller than $\frac{8}{m}$. 

Term~(\ref{term A2}) and (\ref{term A31}):
\begin{align*}
&\frac{4\eta}{n^{2} m} \left\|\sum_{k=0}^t \be(k)\bY^\top\prod_{j=0}^{k-1}(\bm I_n - \eta \widetilde \bM(j))\bX^\top \bX
    +\prod_{j=0}^{k-1}(\bm I_n - \eta \widetilde \bM(j))\bY \be(k)^\top\bX^\top \bX\right\|_2\\
 \leq & \frac{8\eta}{n^{2} m}\left\|\sum_{k=0}^t \be(k)\bY^\top\prod_{j=0}^{k-1}(\bm I_n - \eta \widetilde \bM(j))\bX^\top \bX\right\|_2 \tag{Triangle Inequality}\\
 \leq & \frac{8\eta}{n^{2} m}\sum_{k=0}^t\left\|\be(k)\right\|_2\left\|\bY^\top\prod_{j=0}^{k-1}(\bm I_n - \eta \widetilde \bM(j))\right\|_2\left\|\bX^\top \bX\right\|_2\tag{Cauchy-Schwarz}\\
 \leq & \frac{8}{n^2m}\left(\sum_{k=0}^t\eta(1- \frac{\eta}{n}\lambda_r)^k\sqrt n\right) \cdot \frac{n^{3/2}}{\lambda_r}R_w \cdot c_1n \tag{\eqref{e bound} and Algebra}\\
 \leq & \frac{8c_1n^2}{m\lambda_r^2}R_w
\end{align*}

Similarly, we have their symmetric counterpart added and get a bound of $\frac{16c_1n^2}{m\lambda_r^2}R_w$.

Term~(\ref{term A4}):
\begin{align*}
    &\frac{4\eta}{n^{2} m} \left\|\sum_{k=0}^t \be(k) \be(k)^\top\bX^\top \bX\right\|_2\\
    \leq& \frac{4}{n^2m}\eta\sum_{k=0}^t\|\be(k)\|^2\|\bX^\top \bX\|_2\tag{Cauchy-Schwarz}\\
    (&\|\be(k)\|_2 \leq \|\bD(k)\|+\|\prod_{j=0}^{k-1}(I-\eta \widetilde{\bM}(j))\bY\|\leq 2\sqrt n (1-\frac{\eta}{n}\lambda_r)^k)\\
     \leq& \frac{4\eta}{N^2m}\sum_{k = 0}^t 2(1-\frac{\eta }{n}\lambda_r)^k\sqrt n\cdot\frac{n^{3/2}}{\lambda_r}R_w\cdot c_1n\tag{\eqref{e bound} and Algebra}\\
     \leq& \frac{8n^2c_1}{m\lambda_r^2}R_w
\end{align*}

Similarly, we can combine its symmetric counterpart and get a bound of $\frac{16c_1n^2}{m\lambda_r^2}R_w$.

Add them up, and we can get the bound of the first part. $$\|(1)\|_2\leq \frac{8}{m} + \frac{32c_1n^2}{m\lambda_r^2}R_w$$

Term~(\ref{Gamma (2)}): 
\begin{align*}
&\frac{4}{n^{2} m} \eta\left\|\sum_{k=0}^{t}\left(\bY \bD(t)^{\top} \bX^{\top}\bX+\bX^{\top}\bX \bD(t)\bY^\top\right)\right\|_{2}\tag{Triangle Inequality}\\
\leq &\frac{8}{N^{2} m} \eta\left\|\sum_{k=0}^{t}\bY\bY^\top\prod_{j=0}^{k-1}(\bm I_n - \eta \widetilde \bM(j)) \bX^{\top}\bX\right\|_{2}\tag{A5}\label{term A5}\\
 &+\frac{8}{N^{2} m}\left\|\sum_{k=0}^t\bY \be(k)^\top\bX^\top \bX\right\|_2\tag{A6}\label{term A6}
\end{align*}

We bound the two parts separately.

Term~(\ref{term A5}):
\begin{align*}
   & \frac{8}{N^{2} m} \eta\left\|\sum_{k=0}^{t}\bY\bY^\top\prod_{j=0}^{k-1}(\bm I_n - \eta \widetilde \bM(j)) \bX^{\top}\bX\right\|_{2}\tag{Cauchy-Schwarz}\\
   \leq& \frac{8\eta}{n^2m}\|\bY\bY^\top\|_2 \left\|\sum_{k=0}^t\sum_{i=1}^r(1-\eta \frac1n \lambda_i)^k\lambda_i\bv_i\bv_i^\top\right\|_2
   \leq  \frac{8}{m} \tag{Algebra}
\end{align*}

Term~(\ref{term A6}):
\begin{align*}
    \frac{8}{n^{2} m}\left\|\sum_{k=0}^t\bY \be(k)^\top\bX^\top \bX\right\|_2
    \leq &\frac{8}{n^{2} m}\|\bY\|_2\left\|\sum_{k=0}^t \be(k)\right\|_2\cdot c_1n\tag{Cauchy-Schwarz}\\
    \leq & \frac{8}{n^2m}\sqrt n \eta \sum_{k=0}^t\eta \sqrt n k(1-\eta\frac{1}{n}\lambda_r)^{k-1}R_w c_1n\tag{\eqref{e bound} and $\|\bY\|=\sqrt n$}\\
    \leq &\frac{8c_1}{m}\eta^2\sum_{k=0}^t\sum_{j=k}^{t-1}(1-\eta \frac{1}{n}\lambda_r)^jR_w\tag{Abel's lemma}\\
    \leq &\frac{8c_1}{m}\eta^2\sum_{k=0}^t \frac{n}{\eta \lambda_r} (1-\eta\frac{1}{n}\lambda_r)^k  R_w\tag{Algebra}\\
    \leq&\frac{8c_1}{m}\frac{n^2}{\lambda_r^2}R_w = \frac{8c_1n^2}{m\lambda_r^2}R_w
\end{align*}

Sum up and we get 
$$
\|(2)\|_2 \leq \frac{8}{m} + \frac{8c_1n^2}{m\lambda_r^2}R_w
$$
Term~(\ref{Gamma (3)}): 
\begin{align*}
    &\frac{8 \eta^{2}}{n^{3} m^{2}}\left\|\sum_{k=0}^{t}\left\| \bA(t)\right\|^{2}  \bX^{\top}\bX \bD(t) \bD(t)^{\top} \bX^{\top}\bX\right\|_{2}\\&\text{(Since $\bX^{\top}\bX \bD(t) \bD(t)^{\top} \bX^{\top}\bX$ is PSD)}\\
    \leq&\frac{4\eta}{n^2m}\eta \frac{2}{mn}\max_{t}\left\|\|\bA(t)\|^2\bX^\top \bX\right\|_2\cdot\left\|\sum_{k=0}^t\bD(t) \bD(t)^\top\bX^\top \bX\right\|_2\\&(\text{Since } \lambda_{\max}(\bM^*)<1/\eta)\\
    \leq&\frac{4\eta}{n^2m}\left\|\sum_{k=0}^t\bD(t) \bD(t)^\top\bX^\top \bX\right\|_2\leq \frac{4}{m}+\frac{16c_1n^2}{m\lambda_r^2}R_w\\
\end{align*}
The last inequality is the same one as in the term~(\ref{Gamma (1)}) bound.

Adding the three terms together and using the induction hypothesis, we have 
$$
\|\eqref{Gamma (1)}\|_2+\|\eqref{Gamma (2)}\|_2 +\|\eqref{Gamma (3)}\|_2\leq\frac{40}{m}.
$$

Property 1 holds. Therefore the proof is completed. 
\end{proof}

Lemma~\ref{direction guarantee} tells us that $\bD(t)$ decreases along GD trajectory in some fixed directions independently depending on $\bX^\top\bX$. After we have this GD trajectory, we can have the following Lemma~\ref{Norm growth} about the dynamics of $\|\bA(t)\|^2$ under the condition in Lemma~\ref{direction guarantee}. It shows a sufficient condition for $\|\bA(t)\|^2$ to grow.

\begin{lemma}\label{Norm growth}
Under the Assumption~\ref{eigenspectrum assumption appen} and Assumption~\ref{y assumption appen}, if
$$n>(70\lambda_r+25\lambda_r^2)/(196\kappa^2\min\{a_1-a_1^2, a_2-a_2^2\})$$ and as long as there exist two number $a_1, a_2$ (to be determined) and some $i\in[r]$ at time t s.t.
\begin{equation}
    0 < a_1 \leq \prod_{j=0}^{t-1}(1-\eta \lambda_{i}(\widetilde{\bM}(j)))\leq a_2<1
    \label{B.1}
\end{equation}
we have $-\bD(t)^\top\bF(t)>0$.

\end{lemma}
\begin{proof} We use property 2 of Lemma~\ref{direction guarantee} to obtain the following expression:
\begin{align*}
    -\bD(t)^\top\bF(t) ={}&\bY^\top \prod_{j=0}^{t}(\bm I_n-\eta \widetilde{\bM}(j))(\bm I_n-\prod_{j=0}^{t}(\bm I_n-\eta \widetilde{\bM}(j)))\bY\\
    & + \bY^\top (2\prod_{j=0}^{t}(\bm I_n-\eta \widetilde{\bM}(j))-\bm I_n)\be(t) - \|\be(t)\|^2\tag{Use \eqref{e bound} and $m\geq \frac{112c_1n^2}{\lambda_r^2}$}\\
    \geq{}&\sum_{i=1}^r z_i^2\prod_{j=0}^{t-1}(1-\eta \lambda_{i}(\widetilde{\bM}(j))(1-\prod_{j=0}^{t-1}(1-\eta \lambda_{i}(\widetilde{\bM}(j))) - \frac{40n^2}{\lambda_r m} -\frac{1600n^3}{\lambda_r^2m^2}\\
    \geq{}& \kappa^2n\max_i\left\{\prod_{j=0}^{t-1}(1-\eta \lambda_{i}(\widetilde{\bM}(j))(1-\prod_{j=0}^{t-1}(1-\eta \lambda_{i}(\widetilde{\bM}(j)))\right\}- \frac{5\lambda_r}{14}-\frac{25\lambda_r^2}{196} \tag{C.4}\label{lemma C.4 expression}
\end{align*}

Notice that all inequalities hold since $\lambda_i(\bM^*(t))<1/\eta$ for all $i$. In the first inequality, we use the $\|\be(t)\|$ bound in \eqref{e bound}, and in the second we just replace $m$ with its lower bound $\frac{112c_1n^2}{\lambda_r^2}.$ Then we use $n>(70\lambda_r+25\lambda_r^2)/(196\kappa^2\min\{a_1-a_1^2, a_2-a_2^2\})$
and complete the proof.
\begin{align*}
    -\bD(t)^\top\bF(t)\geq \kappa^2n\min\{a_1-a_1^2,a_2-a_2^2\}-\frac{5\lambda_r}{14}-\frac{25\lambda_r^2}{196}>0
\end{align*}
\end{proof}

Then, we use Assumption~\ref{eigenspectrum assumption appen} to prove the next lemma. It tells that condition~\eqref{B.1} is satisfied under this assumption in a time interval. That means $-\bD(t)^\top\bF(t)>0$ during this period.

\begin{lemma}\label{lemma c5}
Under Assumption~\ref{eigenspectrum assumption appen}, we have condition \eqref{B.1} satisfied with $a_2 = e^{(a_1-1)/\chi}$ in the time interval $[t_1, t_2)$. Here, $t_1$ is the iteration that $\prod_{j=0}^{t-1}(1-\eta \lambda_{1}(\widetilde{\bM}(j)))< a_2$ for the first time, and $t_2$ is the iteration when $\prod_{j=0}^{t-1}(1-\eta \lambda_{r}(\widetilde{\bM}(j)))< a_1$ for the first time.
\end{lemma}
\begin{proof}
We prove for all $i\in [r-1]$, 
\begin{equation}
    \text{If }\prod_{j=0}^{t-1}(1-\eta \lambda_{i}(\widetilde{\bM}(j)))< a_1, \text{then }\prod_{j=0}^{t-1}(1-\eta \lambda_{i+1}(\widetilde{\bM}(j)))< e^{(a_1-1)/{\chi}}.\label{condition}
\end{equation}

In this way, the only two possibility that all $i\in [r]$ doesn't satisfy the condition~\eqref{B.1} is: (1) the $\prod_{j=0}^{t-1}(1-\eta \lambda_{1}(\widetilde{\bM}(j)))>a_2$; (2) $\prod_{j=0}^{t-1}(1-\eta \lambda_{r}(\widetilde{\bM}(j)))< a_1$. Otherwise, there must be some $i$ s.t. $\prod_{j=0}^{t-1}(1-\eta \lambda_{i}(\widetilde{\bM}(j)))\in [a_1, a_2].$ Thus, if condition~\eqref{condition} is satisfied, we have this lemma proved.

Now suppose $\prod_{j=0}^{t-1}(1-\eta \lambda_{i}(\widetilde{\bM}(j)))< a_1$. By Bernoulli's inequality, 
$$(1+x_1)(1+x_2)...(1+x_n)\geq 1+x_1+x_2+...+x_n, \text{if } x_i\geq -1, \ \forall i\in[n]$$
since $-\eta \lambda_{i}(\widetilde{\bM}(j)) > -1$ for all $i$, we have 
$$
\sum_{j=0}^{t-1}\eta \lambda_{i}(\widetilde{\bM}(j))> 1-a_1
$$

By Jensen inequality, we have
$$
\sum_{j=0}^{t-1}\log(1-\eta \lambda_{i}(\widetilde{\bM}(j))/\chi) \leq t\log\left(1-\frac{\sum_{j=0}^{t-1}\eta \lambda_{i}(\widetilde{\bM}(j))}{t\chi}\right)
$$

Hence we have the following inequalities:
\begin{align*}
    \prod_{j=0}^{t-1}(1-\eta \lambda_{i+1}(\widetilde{\bM}(j))) &\leq \prod_{j=0}^{t-1}(1-\eta \lambda_{i}(\widetilde{\bM}(j))/\chi)\\
    &\leq \exp\left\{t\log\left(1-\frac{\sum_{j=0}^{t-1}\eta \lambda_{i}(\widetilde{M}(j))}{t\chi}\right)\right\}\\
    &\leq (1+\frac{a_1-1}{t\chi})^t\\
    &\leq e^{(a_1-1)/\chi}
\end{align*}

So we complete the proof.
\end{proof}

Then we pay attention back to sharpness. We have the sharpness's dynamics by Lemma~\ref{update rule of M}.
\begin{align*}
\sha^*(t+1)-\sha^*(t) ={}& \bv_1^\top(\bM(t+1)-\bM(t))\bv_1\tag{Definition \eqref{B.0}}\\
={}&-\frac{8\eta \lambda_1}{n^2m}(\bD(t)^\top\bF(t)+ (\bD(t)^\top\bv_1)(\bF(t)^\top\bv_1))\\&+\frac{8\eta^2\lambda_1}{m^2n^3}(\bD(t)^\top\bX^\top \bW^\top(t) \bW(t)\bX\bD (t)+\|\bA(t)\|^2\lambda_1(\bD(t)^\top\bv_1)^2)
\end{align*}

This equation shows that the dynamics of sharpness is closely related to the dynamics of $\|\bA(t)\|^2$, i.e. highly dependent on $-\bD(t)^\top\bF(t)$. Based on the lemmas above, we can prove progressive sharpening happens almost along the whole training trajectory (Lemma~\ref{progressive shap appen}) until the loss $\mathcal{L} = \frac{1}{n}\|\bD(t)\|^2$ converges to $O(n^{-1})$. 

\begin{lemma}\label{progressive shap appen}
Under Assumption~\ref{eigenspectrum assumption appen} and \ref{y assumption appen}, if $m > \frac{112 c_1n^2}{\lambda_r^2}$ and
$$
n > \max\{(70\lambda_r+25\lambda_r^2)/(98\kappa^2\min\{ e^{-1/\chi}-e^{-2/\chi}, \eta c_1(1-\eta c_1)\}), (\frac{70\lambda_r+25\lambda_r^2}{98\kappa^2\epsilon})^2\}
$$

\textit{we have $\sha^*(t+1)-\sha^*(t)>0$ until the time t when $\|\bD(t)\|^2\leq \epsilon^2 + \frac{5\epsilon\lambda_r}{7\sqrt n}  +\frac{25\lambda_r^2}{196n}$.}
\end{lemma}
\textbf{Remark. } When $a_1 = \epsilon/\sqrt n, a_2=e^{(a_1-1)/\chi} < e^{-1/\chi}$, the lower bound of $n$ guarantee that 
\begin{equation}\label{n lower bound lemma c6}
    \kappa^2n\min\{a_1-a_1^2, a_2-a_2^2, (1-\eta c_1)\eta c_1\} > \frac{5\lambda_r}{7}-\frac{25\lambda_r^2}{98}
\end{equation}
\begin{proof}
Note that the second order term
$$
\frac{8\eta^2\lambda_1}{m^2n^3}(\bD(t)^\top\bX^\top \bW(t) ^\top\bW(t)\bX\bD (t)+\|\bA(t)\|^2\lambda_1(\bD(t)^\top\bv_1)^2)
$$
is larger than 0. So as long as the first order term 
$$
-\frac{8\eta \lambda_1}{n^2m}(\bD(t)^\top\bF(t)+ (\bD(t)^\top\bv_1)(\bF(t)^\top\bv_1)) >0
$$
the approximate sharpness will grow. 

First, we give a lower bound for the number $- (\bD(t)^\top\bv_1)(\bF(t)^\top\bv_1)$:
\begin{align*}
    - (\bD(t)^\top\bv_1)(\bF(t)^\top\bv_1) =& z_1^2\prod_{j=0}^{t-1}(1-\eta \lambda_{1}(\widetilde{\bM}(j))(1-\prod_{j=0}^{t-1}(1-\eta \lambda_{1}(\widetilde{\bM}(j)))\\
    &+z_1 (2\prod_{j=0}^{t-1}(1-\eta \lambda_1(\widetilde{\bM}(j)))-1)(\be(t)^\top\bv_1) - (\be(t)^\top \bv_1)^2
   \\
    \geq& -|z_1|\|\be(t)\| - \|\be(t)\|^2\tag{$\lambda_{1}(\bM^*(t))<1/\eta$}\\
    \geq& - \frac{40n^2}{\lambda_r m} -\frac{1600n^3}{\lambda_r^2m^2} \tag{Use \eqref{e bound}}\\
    \geq& -\frac{5\lambda_r}{14}-\frac{25\lambda_r^2}{196}\tag{$m\geq \frac{112c_1n^2}{\lambda_r^2}$}
\end{align*}
where the first equation holds due to Property 2 of Lemma~\ref{direction guarantee}.

With the lower bound of $ -(\bD(t)^\top\bv_1)(\bF(t)^\top\bv_1)$, we show the dynamics of the first order term by similar technique in the expression~\eqref{lemma C.4 expression} :
\begin{align*}
    &- \bD(t)^\top\bF(t) - (\bD(t)^\top\bv_1)(\bF(t)^\top\bv_1)\tag{Property 2 of Lemma~\ref{direction guarantee}}\\
    ={}&\bY^\top \prod_{j=0}^{t}(\bm I_n-\eta \widetilde{\bM}(j))(\bm I_n-\prod_{j=0}^{t}(\bm I_n-\eta \widetilde{\bM}(j)))\bY + \bY^\top (2\prod_{j=0}^{t}(\eta \widetilde{M}(j)-\bm I_n)-\bm I_n)\be(t) \\
    & - \|\be(t)\|^2 - (\bD(t)^\top\bv_1)(\bF(t)^\top\bv_1) \tag{Use \eqref{e bound} and $m\geq \frac{112c_1n^2}{\lambda_r^2}$}\\
    \geq{} & \kappa^2n\max_i\left\{\prod_{j=0}^{t-1}(1-\eta \lambda_{i}(\widetilde{\bM}(j))(1-\prod_{j=0}^{t-1}(1-\eta \lambda_{i}(\widetilde{\bM}(j)))\right\}- \frac{5\lambda_r}{14}-\frac{25\lambda_r^2}{196} \\
    &{}+(-\frac{5\lambda_r}{14}-\frac{25\lambda_r^2}{196})\tag{Property 2 of Lemma~\ref{direction guarantee}}\\
    \geq {}&{} \kappa^2n\max_i\left\{\prod_{j=0}^{t-1}(1-\eta \lambda_{i}(\widetilde{\bM}(j))(1-\prod_{j=0}^{t-1}(1-\eta \lambda_{i}(\widetilde{\bM}(j)))\right\} - \frac{5\lambda_r}{7}-\frac{25\lambda_r^2}{98} \\
    \geq {}&{} \kappa^2n(1-\eta c_1)\eta c_1 - \frac{5\lambda_r}{7}-\frac{25\lambda_r^2}{98}> 0
\end{align*}
The last inequality holds due to the lower bound of $n$ \eqref{n lower bound lemma c6} assumed in Lemma~\ref{progressive shap appen}.

Now, $\prod_{j=0}^{t-1}(1-\eta \lambda_{1}(\widetilde{\bM}(j)))$ begins to decrease each iteration. Before the time when $\prod_{j=0}^{t-1}(1-\eta \lambda_{1}(\widetilde{\bM}(j)))$ becomes smaller than $\frac{\epsilon}{\sqrt n}$, $-\bD(t)^\top\bF(t)>0$ always holds because of the lower bound of $n$ \eqref{n lower bound lemma c6} assumed in Lemma~\ref{progressive shap appen}.

Then, after the time $t_1$ when $\prod_{j=0}^{t-1}(1-\eta \lambda_{1}(\widetilde{\bM}(j))<\frac{\epsilon}{\sqrt n}$ and before the time $t_2$ when $\prod_{j=0}^{t-1}(1-\eta \lambda_{i}(\widetilde{\bM}(j)) < \frac\epsilon{\sqrt n}$ for all $i$, we enter the time interval $[t_1, t_2)$ where Lemma~\ref{lemma c5} begins to hold. We use Lemma~\ref{lemma c5} to show that there exists some $i$ to make
$$
\prod_{j=0}^{t-1}(1-\eta \lambda_{i}(\widetilde{\bM}(j)) \in [\frac\epsilon{\sqrt n}, e^{-1/\chi}]
$$
Before the time $t_2$, we have this inequality always hold. Thus $$- \bD(t)^\top\bF(t) - (\bD(t)^\top\bv_1)(\bF(t)^\top\bv_1)>0$$ holds until the iteration $t_2$. Thus during this period, $\sha^*(t)$ keeps increasing.

At this iteration, $\left\|\prod_{j=0}^{t-1}(\bm I_n - \eta \widetilde \bM(j))\right\|\leq \frac{\epsilon}{\sqrt n}$, and we can bound the norm of the residual $\bD(t)$ with the inequality below. We have 
\begin{align*}
    \|\bD(t)\|^2&=\|-\prod_{j=0}^{t-1}(\bm I_n - \eta \widetilde \bM(j))\bY + {\be(t)}\|^2\tag{Triangle Inequality}\\
    &\leq \left\|\prod_{j=0}^{t-1}(\bm I_n - \eta \widetilde \bM(j))\bY\right\|^2 +2\left\|\prod_{j=0}^{t-1}(\bm I_n - \eta \widetilde \bM(j))\bY\right\|\|{\be(t)}\|+ \|{\be(t)}\|^2\tag{Cauchy-Schwarz}\\
    &\leq \left\|\prod_{j=0}^{t-1}(\bm I_n - \eta \widetilde \bM(j))\right\|^2\left\|\bY\right\|^2 +2\left\|\prod_{j=0}^{t-1}(\bm I_n - \eta \widetilde \bM(j))\right\|\left\|\bY\right\|\|{\be(t)}\|+ \|{\be(t)}\|^2\tag{$\|\bY\|=\sqrt n$}\\
    &\leq n\max_i\left\{\prod_{j=0}^{t-1}(1-\eta \lambda_{i}(\widetilde{\bM}(j))\right\}^2+ \frac{80n^2}{\lambda_r m}\max_i\left\{\prod_{j=0}^{t-1}(1-\eta \lambda_{i}(\widetilde{\bM}(j))\right\} +\frac{1600n^3}{\lambda_r^2m^2}\\
    &< \epsilon^2 + \frac{5\epsilon\lambda_r}{7\sqrt n}  +\frac{25\lambda_r^2}{196n}\tag{$\left\|\prod_{j=0}^{t-1}(\bm I_n - \eta \widetilde \bM(j))\right\|\leq \frac{\epsilon}{\sqrt n}$}
\end{align*}

    Thus, before the norm of the residual $\bD(t)$ decreases to this value, $\sha^*(t)$ keeps increasing.
\end{proof}

\subsection{Edge of Stability (Phase II-IV)}\label{edge of stab(phase ii-iv) appen}

In the edge of stability regime, we focus on the largest eigenvalue $\sha$  and its corresponding eigenvector $\bm{u}$. Since $\bM(t)$ has a large similarity with $\bX^\top \bX$ in progressive sharpening phase, we consider the eigenvector $\bv_1$ corresponding to the largest eigenvalue $\lambda_1$ of $\bX^\top \bX$.

After $\sha \geq 2/\eta$, the proof in Section~\ref{phase i and ps appen} does not extend to this phase. However, the bound of $\|\bm\Gamma(t)\|$ (Lemma~\ref{direction guarantee}) still holds up to a constant factor empirically (See Figure~\ref{gamma assumption}). Hence, we make this bound an assumption as follows.


\begin{Assumption}
\label{Gamma assumption appendix}
There exists some constant $c_2>0$, such that for any time $t$,
$$
\|\bm\Gamma(t)\|\leq \frac{c_2}{m}
$$
\end{Assumption}

Note that in above progressive sharpening stage, the assumption holds by Theorem~\ref{Theorem C.1, ps}. We propose this assumption to keep the gram matrix from deviating too far from the original trajectory even in other phases. The verification of this assumption can be found in Appendix~\ref{D.2.2 gamma assumption verification}.



\begin{Corollary}
\label{v_1 similarity coro}
Recall that $\bv_1$ is the largest eigenvector of $\bX^\top \bX$ and $\bm u$ is the largest eigenvector of $\bM$. 
There exists a constant $c_4$, such that $\langle \bv_1, \bm u \rangle^2\geq 1 - c_4\sqrt{\frac{1}{nm}}$, and $\sha(t) \geq \sha^*(t) \geq \sha(t) -\frac{2c_2}{m} $.
\end{Corollary}

\begin{proof}
By difinition of $\bm \Gamma(t)$ in \ref{C4}, $\bM = \gamma(t) \bX^\top \bX + \bm\Gamma(t)$, here $\gamma(t) = \frac{2}{mn}(\|\bA(t)\|^2 + \frac{m}{d}) \geq \frac{2}{nd}$.

Let $\langle \bv_1, \bm u \rangle = \cos\theta$ and decompose $\bm u$ by $\bm u = \bv \cos\theta + \bv_{\perp}\sin\theta$.
Then we have $\sha = \bm u^\top \bM \bm u \leq \lambda_1\gamma(t)\cos^2\theta + \frac{\lambda_1}{2}\sin^2\theta\gamma(t) + \frac{c_2}{m}$.

Also we have $\sha \geq \bv_1^\top \bM \bv_1 \geq \gamma(t)\lambda_1 - \frac{c_2}{m}$.
So $\lambda_1\gamma(t)\cos^2\theta + \frac{\lambda_1}{2}\sin^2\theta\gamma(t) + \frac{c_2}{m} \geq \gamma(t)\lambda_1 - \frac{c_2}{m}$, which induces
$\sin^2\theta\leq \frac{2c_2}{m\gamma(t)}\cdot\frac{2}{\lambda_1}\leq \frac{2c_2^2d}{m\lambda_1}$.
Because in our setting $\lambda_1 = \Theta(n)$, there exists a constant $c_4$ such that 
$\cos\theta\geq \sqrt{1 - \frac{c_4^2 }{nm}} \geq 1 - c_4\sqrt{\frac{1}{mn}}$.

The inequality $\sha(t)\geq\sha^*(t)$ is because $\bm u^\top \bM \bm u \geq \bv_1^\top \bM \bv_1$ by definition of $\bm u$. The other side can be proved as the following:

$$
\begin{aligned}
\bv_1^\top \bM \bv_1 =& \bv_1^\top (\bM-\bm\Gamma) \bv_1 + \bv_1^\top \bm\Gamma \bv_1\\ \geq & \bm u^\top (\bM-\bm\Gamma) \bm u - c_2/m \\
\geq& \bm u^\top \bM \bm u - 2c_2/m.
\end{aligned}
$$

\end{proof}

To prove the main theorem (Theorem~\ref{EOS main thm appendix}), we need two more assumptions.

\begin{Assumption}
There exists some constant $c_3$ such that $|\bD(t)^\top \bv_1| > c_3 \sqrt{n}/m$ for some $t=t_0$ at the beginning of phase II.
\label{divergence assumption appendix}
\end{Assumption}

\begin{Assumption}
There exists some constant $\beta>0$, such that 
$$
\frac{4}{\eta}(1-\beta)\geq \sha
$$
\label{sharpness upper bound appendix}
\end{Assumption}

The above assumption is consistent with Assumption~\ref{Ass: sha upper bound section 3}, in which we assume an upper bound of the sharpness. \citet{lewkowycz2020large} showed that $4/\eta$ is a upper bound of the sharpness in two-layer linear network with one datapoint, otherwise the training process would diverge. Here we make it an assumption.

\begin{theorem}
 Suppose the smallest nonzero eigenvalue $\lambda_r=\lambda_{r}(\bX^\top \bX)>0, \lambda_1 = \lambda_{\max}(\bX^\top \bX)$. Under Assumption~\ref{Gamma assumption appendix}, \ref{sharpness upper bound appendix}, \ref{divergence assumption appendix}, and $\lambda_1(\bX ^\top\bX)\geq 2\lambda_2(\bX^\top \bX)$ in Assumption~\ref{eigenspectrum assumption appen},
there exists constants $C_1, C_2, C_3$, such that if $n>c_1\lambda_r\eta, m > \max\{\frac{C_2d^2n^2}{\lambda_r^2}, C_3\eta\}$, we have

\begin{itemize}
    \item (Lemma~\ref{first conclusion of EOS thm}) There exists $\rho = O(1)$ which is related to $c_3$ such that if $\sha(t_0)>\frac{2}{\eta}(1+\rho)$ for some $t_0$, there must exist some $t_1>t_0$ such that $\sha(t_1) < \frac{2}{\eta}(1+\rho)$.
    \item (Lemma~\ref{EOS divergence thm}) If $\sha(t), \sha(t+1) > \frac{2}{\eta}(1+\rho)$, then there is a constant $c_7>0$ (related to $c_3$) such that $|\bD(t+1)^\top \bv_1|>|\bD(t)^\top \bv_1|(1+c_7)$.
    \item (Lemma~\ref{R dynamics in appendix}) Define $\bR(t):= (\bm I - \bv_1 \bv_1^\top)\bD(t)$, and $\bR'(t) := (\bm I - \eta\bM^*(t)(\bm I - \bv_1 \bv_1^\top))\bR'(t-1)$. We have $\|\bR(t) - \bR'(t)\| = O(\frac{\sqrt{n^3}d}{\lambda_r \sqrt{m}})$.
\end{itemize}
\label{EOS main thm appendix}
\end{theorem}

Next, we prove this theorem in three parts: first we prove the third statement, which gives $\bR(t)$ an upper bound (Lemma~\ref{R dynamics in appendix}), then we prove the second statement, in which we use Assumption~\ref{divergence assumption appendix} to prove that when sharpness is above $2/\eta$, $\bD^\top \bv_1$ increases geometrically (Lemma~\ref{EOS divergence thm}), and lastly we prove that the sharpness eventually drops below $2/\eta$ (Lemma~\ref{first conclusion of EOS thm}), which is the first statement in our theorem.

Now we first give a key equation:
\begin{lemma}
\label{Key equation lemma}
The dynamics of the approximation on sharpness is:
\begin{equation}
\label{Key equation in EOS}
\begin{aligned}
&\sha^*(t+1)-\sha^*(t) \\
=&- \frac{8\eta \lambda_{1}}{mn^2} \left(\bF(t)^{\top} \bD(t)+(\bF(t)^\top \bv_1)(\bD(t)^\top \bv_1)-\frac{\eta}{2}(\bD(t)^\top \bv_1)^{2}  \sha^*(t)  \right.\\
&\left.-{}\frac{\eta}{2}\bR(t)^\top \bm\Gamma(t) \bR(t)  -\ \eta \bR(t)^\top \bm\Gamma(t)\left(\bv_1 \bv_1^\top\bD(t)\right)-\frac{\eta}{m n}\bR(t)^{\top}\left(\frac{m}{d} \bX^{\top} \bX\right) \bR(t)  \right).
\end{aligned}
\end{equation}
\end{lemma}

The proof of the equation is long and tedious, so we leave that in Lemma~\ref{calculations for equations}.

Next we deal with the gap between $\bM(t)$ and $\bM^*(t)$. Note that $\bM$ is the Gram Matrix, but in gradient descent trajectory, $\bD(t+1) - \bD(t) = -\bM^*(t)\bD(t)$, so $\bM^*$ is the one that truly controls $\bD$'s dynamics. From the following lemma, we can see that $\bM^*(t)$ is a smoothed version of $\bM(t)$ and $\bM(t+1)$ plus a small perturbation.

\begin{lemma}
\label{M* lemma}
If $m \geq \eta c_2(d+1)$ (recall that $c_2$ is defined in Assumption~\ref{Gamma assumption appendix}), then there exists $k_s \in [0,1)$ and a constant $c_6$ such that $\|\bM^*(t) - (1-k_s)\bM(t) - k_s \bM(t+1)\|\leq \frac{c_6}{m}$.
\end{lemma}

\begin{proof}
Recall that $\bM^*(t) = \bM(t) - \frac{4\eta}{N^2m} (\bD(t)^\top\bF(t)) \bX^\top \bX$. Then we consider two different cases.

\textbf{Case 1}: $|\bD(t)^\top\bF(t)| \leq \frac{d+1}{d-2} (1 + \frac{d+1}{d-2}) n$.

In this case $\|\bM^*-\bM\|_2 \leq \frac{4\eta \lambda_1}{mn^2} \frac{d+1}{d-2} (1 + \frac{d+1}{d-2}) n = O(\frac{\eta \lambda_1}{mn}) = O(\frac{1}{m})$. Here the last inequality follows from the inequality~\eqref{initialization inequality}.

\textbf{Case 2}: $|\bD(t)^\top\bF(t)| > \frac{d+1}{d-2} (1 + \frac{d+1}{d-2}) n$. 

We claim that in this case we have $\bF(t)^\top\bD(t) > (\frac{2\eta}{mn}\frac{m}{d}\lambda_1 + \frac{c_2\eta}{m})\|\bD(t)\|^2$.
If it does not hold, then 
$R.H.S. \geq \bF(t)^\top\bD(t) \geq \|\bD(t)\|^2 - \|\bD(t)\|\|\bm Y\|$. Hence we can have 
$$
\begin{aligned}
\|\bm Y\| &\geq \|\bD(t)\| (1 - \frac{2\eta}{mn}\frac{m}{d}\lambda_1 - \frac{c_2\eta}{m}) \\
&= \|\bD(t)\|(1 - \frac{2\eta\lambda_1}{nd} - \frac{\eta c_2}{m})\\
&\geq \|\bD(t)\|(1 - \frac{2}{d+1} - \frac{1}{d+1})
\end{aligned}
$$

The last inequality uses the restriction of $m$ in this lemma and the inequality~\eqref{initialization inequality}. Hence we have $|\bF(t)^\top\bD(t)|\leq \|\bD(t)\|(\|\bD(t)\|+\|Y\|)\leq \frac{d+1}{d-2} (1 + \frac{d+1}{d-2}) n$, which leads to a contradiction. 

Hence now $\bF(t)^\top\bD(t) > (\frac{2\eta}{mn}\frac{m}{d}\lambda_1 + \frac{c_2\eta}{m})\|\bD(t)\|^2$ holds. 
Since $\bX^\top \bm W^\top \bm W\bX = \frac{m}{d} \bX^\top \bX + \bm\Gamma(t)$ and $\|\bX^\top \bX\| = \lambda_1, \|\bm\Gamma(t)\|\leq \frac{c_2}{m}$, 
we can have:
$$\bF(t)^\top\bD(t) > \frac{2\eta}{mn}\bD(t)^\top\bX^\top \bm W^\top \bm W\bX\bD(t).$$

Now let $k_s := \frac{\bF(t)^\top\bD(t)}{2\bF(t)^\top\bD(t) - \frac{2\eta}{mn}\bD(t)^\top\bX^\top \bm W^\top \bm W\bX\bD(t)}$, by the inequality above we have $k_s<1$.

By the equation 
$$
\begin{aligned}
&\bM(t+1)-\bM(t)=\frac{2}{m n}\left(\|\bA(t+1)\|^{2}-\|\bA(t)\|^{2}\right) \bX^{\top} \bX+(\bm\Gamma(t+1)-\bm\Gamma(t))\\
&\|\bA(t+1)\|^{2}-\|\bA(t)\|^{2}=-\frac{4 \eta}{n} \bF(t)^\top \bD(t)+\frac{4 \eta^{2}}{n^{2}} \cdot \frac{1}{m} \bD(t)^{\top} \bX^{\top} \bm W^{\top}\bm W\bX \bD(t)
\end{aligned}
$$

we have 
$\bM^* = \bM (t) + k_s(\bM(t+1) - \bM(t) - (\bm\Gamma(t+1) - \bm\Gamma(t))$.
Hence 
$\|\bM^* - (1-k_s)\bM(t) - k_s \bM(t+1)\|\leq \frac{2c_2}{m}$.
Now combining the conclusions in both cases together, we finish the proof. 
\end{proof}

Then we consider a corollary of this lemma. Basically, since $\bM^*(t)$ is a weighted sum of $\bM(t)$ and $\bM(t+1)$ adding a small perturbation and $\bM(t)$ has a decomposition to the $\bX^\top \bX$ component and a small noise $\bm \Gamma(t)$, $\bM^*(t)$ can also be decomposed into the $\bX^\top \bX$ component and a small noise.
\begin{Corollary}
\label{M^* corollary}
$\bM^*$ can be decomposed to $\bM^*(t) = \gamma^*(t) \bX^\top \bX + \bm\Gamma^*(t)$, where $\|\bm\Gamma^*(t)\|<\frac{c_2+c_6}{m}$ and for any eigenvector $\bm u$ of $\bX^\top \bX$ except $\bv_1$, and if let $\tau = \frac{2\lambda_r}{nd}$,  $$\tau \leq \bm u^\top (\bM^*(t) - \bm\Gamma^*(t)) \bm u \leq \frac{2}{\eta} - \tau$$
\end{Corollary}
\begin{proof}
By Lemma~\ref{M* lemma},
if we denote $\bm\Gamma'(t) = \bM^*(t) - k_s \bM(t+1) - (1-k_s)\bM(t)$, then $\|\bm\Gamma'(t)\|\leq \frac{c_6}{m}$.

Hence, we can see
$$
\begin{aligned}
 \bM^*(t) = & k_s \bM(t+1) + (1-k_s)\bM(t) + \bm\Gamma' \\
 =& \frac{2}{mn}(k_s\|\bA(t+1)\|^2 + (1-k_s)\|\bA(t)\|^2 + \frac{2m}{d})\bX^\top \bX \\
 &+k_s\bm\Gamma(t+1) + (1-k_s)\bm\Gamma(t) + \bm\Gamma'(t)
\end{aligned}
$$

Denote $k_s\bm\Gamma(t+1) + (1-k_s)\bm\Gamma(t) + \bm\Gamma'(t)$ by $\bm\Gamma^*(t)$. By Assumption~\ref{Gamma assumption appendix} and Lemma~\ref{M* lemma}, $\|\bm\Gamma^*(t)\| < \frac{c_2 + c_6}{m}$.

Note that by Assumption~\ref{eigenspectrum assumption appen}, $\lambda_i(\bX^\top \bX) \leq \frac{1}{2}\lambda_1$ for any $i>1$.
Then, we can see that
$$
\begin{aligned}
\bv_1^\top (\bM(t) - \bm\Gamma(t))\bv_1 =& \frac{2\lambda_1}{mn}(\|\bA(t)\|^2 + \frac{m}{d})  \geq \frac{4\lambda_i}{mn}(\|\bA(t)\|^2 + \frac{m}{d})
\end{aligned}
$$
Hence, we can see 
$$
\begin{aligned}
\frac{2\lambda_i}{mn}(\|\bA(t)\|^2 + \frac{m}{d}) 
\leq & \frac{1}{2} \bv_1^\top (\bM(t) - \bm\Gamma(t))\bv_1 \\
\leq & \frac{1}{2} (\sha + \frac{c_2}{m})\\
\leq & \frac{2}{\eta}(1-\beta) + \frac{c_2}{2m}.
\end{aligned}
$$
Here the last inequality holds due to Assumption~\ref{sharpness upper bound appendix}.

Now the inequality above shows that for any eigenvector $\bm u$ of $\bX^\top \bX$ except $\bv_1$,
$$
\bm u^\top (\bM^*(t) - \bm\Gamma^*(t)) \bm u \leq \frac{2}{\eta} - (\frac{2\beta}{\eta} - \frac{c_2}{2m}).
$$
On the other side, 
$
\frac{2\lambda_i}{mn}(\|\bA(t)\|^2 + \frac{m}{d})
\geq \frac{2\lambda_r}{nd}. 
$

Take $\tau = \min\{\frac{2\lambda_r}{nd}, \frac{2\beta}{\eta} - \frac{c_2}{2m}\}$. Now because $m = \Omega(\eta)$ and $\beta$ is some constant, we can see $\tau = \min\{\frac{2\lambda_r}{nd}, \frac{2\beta}{\eta} - \frac{c_2}{2m}\} = \frac{2\lambda_r}{nd}$.
\end{proof}

Before using this corollary to derive the dynamics of $\bR(t)$ (and thus gives an upper bound of $\|\bR(t)\|$), we need an upper bound for $\|\bD(t)\|$.

\begin{lemma}
\label{D upper bound}
For any constant $c_5 < \min\{2\beta, \frac{d}{d+1}\}$, if $m > \frac{3\eta c_2}{\frac{2d}{d+1} - 2c_5}$ and $m > \frac{\eta c_2}{2\beta - c_5}$, there exists a constant $c_3$ such that $\|\bD(t)\| \leq c_3 \sqrt{nm}$.
\end{lemma}
\begin{proof}

First we analyze the right hand side of equation~\eqref{Key equation in EOS}. We can get 
$$
\begin{aligned}
&\text { R.H.S. }\\
&=- \frac{8\eta \lambda_{1}}{mn^2} \left(\bF(t)^{\top}\bD(t)+(\bF(t)^{\top} \bv_1)(\bD(t)^{\top}\bv_1)-\frac{\eta}{2}(\bD(t)^{\top} \bv_1)^{2}  \lambda_{1}  \right.\\
&\left.-\frac{\eta}{2}\bR(t)^{\top} \bm\Gamma(t) \bR(t)  -\eta  \bR(t)^{\top}\bm\Gamma(t)\left(\bv_1 \bv_1^{\top}\bD(t)\right)-\frac{\eta}{m n}\bR(t)^{\top}\left(\frac{m}{d} \bX^{\top} \bX\right) \bR(t)  \right)\\
&\leq-\frac{8\eta \lambda_{1}}{mn^2} \left(\|\bD(t)\|^{2}+(\bD(t)^{\top} \bv_1)^{2}-\frac{\eta}{2}(\bD(t)^{\top} \bv_1)^{2}  \lambda_{_1}  -\frac{\eta}{2}\|\bR(t)\|^{2}\|\bm\Gamma(t)\|_{2}  \right.\\
&\left.-\eta\|\bR(t)\| \cdot\|\bD(t)\| \cdot\|\bm\Gamma(t)\|_{2}-\frac{\eta \lambda_{1}}{n d}\|\bR(t)\|^{2}  +\bY^\top \bD(t)+\left(\bm Y^{\top}\bv_1\right)(\bD(t) ^{\top}\bv_1)\right)\\
&\leq-\frac{8\eta \lambda_{1}}{mn^2} \left(2 \beta(\bD(t)^{\top} \bv_1)^{2}+\|\bR(t)\|^{2}-\frac{c_{2}\eta}{2m}\|\bR(t)\|^{2}  \right.\\
&\left.-\frac{c_{2}\eta}{m} \|\bR(t)\|\|\bD(t)\| -\frac{1}{d+1}\|\bR(t)\|^{2} -\|\bm Y\| \cdot\|\bD(t)\|-\|\bm Y\| (\bD(t)^{\top} \bv_1)\right)\\
&\leq-\frac{8\eta \lambda_{1}}{mn^2} \left(2 \beta(\bD(t)^{\top} \bv_1)^{2}+\left(\frac{d}{d+1}-\frac{\eta c_{2}n}{2 m}\right)\|\bR(t)\|^{2}-  \frac{c_{2}\eta}{m}\|\bD(t)\|^{2}-2\|\bm Y\| \cdot\|\bD(t)\|  \right)\\
&\leq-\frac{8\eta \lambda_{1}}{mn^2}\left(c_{5}\|\bD(t)\|^{2}-2\|\bY\| \cdot\|\bD(t)\|\right)
\end{aligned}
$$

Here the third inequality follows from inequality~\eqref{initialization inequality} which is $\frac{\eta \lambda_1}{nd} < \frac{1}{d+1}$, and 
the last inequality holds because $2\beta - \frac{c_2\eta }{m} > c_5$ and $\frac{d}{d+1} - \frac{3\eta c_2 }{2m} > c_5$ by the restriction of $m$ in this lemma.

On the other hand, the left hand side of equation~\eqref{Key equation in EOS} is 
$\sha^*(t+1) - \sha^*(t)  \geq  0 - \sha(t) >  -\frac{4}{\eta} $. Here the first inequality follows because of $\sha^*(t)\geq0$ and Corollary~\ref{v_1 similarity coro}, and the second inequality holds because of Assumption~\ref{sharpness upper bound appendix}. 

Hence we have 
$\frac{4}{\eta} > \frac{8}{n^2}\frac{\eta \lambda_1}{m} (c_5 \|\bD(t)\|^2 - 2\|\bm Y\|\|\bD(t)\|)$.
Now if $\|\bD(t)\| > \frac{2\|\bm Y\|}{c_5} + n\sqrt{m}\frac{1}{\eta\sqrt{2\lambda_1}}$, the inequality cannot hold. Note that $\|Y\| = \sqrt{n}$ and $\lambda_1 \in \Theta(n)$, and also from inequality~\eqref{initialization inequality}, $\eta < \frac{\lambda_1}{n} = O(1)$, we finish the proof.
\end{proof}

Now we can give a lemma on $\bR(t):= (\bm I - \bm{v}_1\bm{v}_1^\top)\bD(t)$. The proof idea is similar to Lemma~\ref{RT pert prop} in Section~\ref{Section 3}.

\begin{lemma}
\label{R dynamics in appendix}
 Define $\bR'(t) := (\bm I - \eta\bM^*(t)(\bm I - \bv_1 \bv_1^\top))\bR'(t-1)$. We have $\|\bR(t) - \bR'(t)\| = O(\frac{\sqrt{n^3}d}{\lambda_r \sqrt{m}})$.
\end{lemma}

\begin{proof}
We consider the update rule for $\bm R(t)$, whose proof is in Lemma~\ref{RT update rule appendix E}:

\begin{equation}
\label{RT update rile}
\bR(t+1)
= \left(I-\eta \bM^{*}(t)\right)\bR(t)+\eta\left(\bv_1 \bv_1^{\top}\bm\Gamma(t) - \bm\Gamma(t)\bv_1 \bv_1^{\top}\right) \bD(t)
\end{equation}

Hence if we denote $\bm e_1(t) = \bR(t+1) - (\bm I - \eta \bM^*(t)(\bm I-\bv_1\bv_1^\top))\bR(t)$, then we have 

$$\bm e_1(t) = - \eta \bM^*\bv_1\bv_1^\top \bR(t) + \eta\left(\bv_1 \bv_1^{\top}\bm\Gamma(t) - \bm\Gamma(t)\bv_1 \bv_1^{\top}\right) \bD(t) = \eta\left(\bv_1 \bv_1^{\top}\bm\Gamma(t) - \bm\Gamma(t)\bv_1 \bv_1^{\top}\right) \bD(t).$$

Using the upper bound of $\|\bD(t)\|$ in Lemma~\ref{D upper bound}, and Assumption~\ref{Gamma assumption appendix} we have 
$$\|\bm e_1(t)\| \leq 2\eta \|\Gamma(t)\|\|\bD(t)\| \leq 
2\eta \frac{c_2}{m} c_3\sqrt{nm} = 2c_2c_3\eta\sqrt{\frac{n}{m}}$$.

Now we consider the sequence $\bR'(t)$.

First by Corollary~\ref{M^* corollary}, we can see that the eigenvalues of $(\bM^*(t) - \bm\Gamma^*(t)) (\bm I - \bv_1\bv_1^\top)$ are
all the eigenvalues of $\bM^*(t) - \bm\Gamma^*(t)$ except the largest one, hence are in $(\tau, 2/\eta - \tau)$. Hence because $m = \Omega(n^2)$, if we let $\frac{\lambda_r}{nd} < \frac{c_2+c_6}{m}$, then by Corollary~\ref{M^* corollary} we have all eigenvalues of $\bM^*(t) (\bm I - \bv_1\bv_1^\top)$ are in $(\tau', 2/\eta - \tau')$ where $\tau' = \frac{\lambda_r}{nd}$ .
Hence 
$\|\bm I - \eta \bM(t) (\bm I - \bv_1(t)\bv_1(t)^\top)\|\leq 1-\eta \tau'$.

By the calculations above we can get
\begin{align*}
    \|\bm R(t+1) - \bm R'(t+1)\| 
    & = \|(\bm I - \eta \bM^*(t) (\bm I - \bv_1\bv_1^\top))(\bm R(t) - \bm R'(t)) + \bm e_1(t)\| \\
    & \leq \|(\bm I - \eta \bM(t) (\bm I - \bv_1\bv_1^\top))(\bm R(t) - \bm R'(t))\| + \|\bm e_1(t)\| \\
    &\leq \|\bm I - \eta \bM(t)(\bm I - \bv_1\bv_1^\top)\| \|\bm R(t) - \bm R'(t)\| + \|\bm e_1(t)\| \\
    &\leq \|\bm R(t) - \bm R'(t)\| (1 - \eta \tau') + 2c_2c_3\eta\sqrt{\frac{n}{m}}.
\end{align*}

Thus if we denote $\|\bm R(t) - \bm R'(t)\| - \frac{2c_2c_3\eta\sqrt{\frac{n}{m}}}{\eta \tau'}$ by $p(t)$, and replace $\|\bm R(t) - \bm R'(t)\|$ in the inequality above by $p(t) +  \frac{2c_2c_3\eta\sqrt{\frac{n}{m}}}{\eta \tau'}$, we can get $|p(t+1)| \leq |p(t)|(1-\eta \tau')$.
Hence, we can have $|p(t)| < |p(0)| = \frac{2c_2c_3\eta\sqrt{\frac{n}{m}}}{\eta \tau'}$ for any time $t$. 
Therefore, we can obtain
$\|\bm R(t) - \bm R'(t)\| < \frac{2c_2c_3\eta\sqrt{\frac{n}{m}}}{{\eta \tau'}} + |p(0)| = \frac{4c_2c_3\sqrt{\frac{n}{m}}}{\tau'} = O(d\sqrt{\frac{n^3}{m}}/\lambda_r)$.

\end{proof}

Now based on Theorem~\ref{R dynamics in appendix}, we can give an upper bound of $\|\bR(t)\|$. Because $\|\bR'(t)\|$ is always decreasing and $\|\bR'(0)\| = \sqrt{n}$, hence if $m = \Omega(n^2d^2/\lambda_r^2)$, we can have $O(d\sqrt{\frac{n^3}{m}}/\lambda_r) = O(\sqrt{n})$. Hence
\begin{equation}
\label{R upper bound appendix}
   \|\bR(t)\| \leq \|\bR'(t)\| + O(\sqrt{n}) = O(\sqrt{n}).
\end{equation}
Let $\|\bR(t)\|\leq c_7\sqrt{n}$.

Next we can use Assumption~\ref{divergence assumption appendix} to prove that $\bD(t)^\top\bv_1$ increases geometrically when $\sha > 2/\eta$, which then causes the drop of the sharpness.

\begin{lemma}
\label{EOS divergence thm}
Let $\rho^*=\frac{(c_2+c_6)c_7}{2c_3}$, $\rho = \rho^* + \frac{\eta}{2}\frac{5c_2+c_6}{m}$, $\epsilon_1 = 2\eta(\rho^*-\frac{(c_2+c_6)c_7}{2c_3})$. If $\sha(t), \sha(t+1) > \frac{2}{\eta}(1+\rho)$, we have
$|\bD(t+1)^{\top}\bv_1| > (1+\epsilon_1)\bD(t)^{\top}\bv_1$.
\end{lemma}

\begin{proof}
$$
\begin{aligned}
\bD(t+1)^\top \bv_1 &=\bD(t)^\top\left(\bm I-\eta \bM^{*}(t)\right) \bv_1 \\
&=\bD(t)^\top \left(\bm I-\eta \gamma^{*}(t) \bX^{\top} \bX\right) \bv_1-\eta \bD(t)^{\top} \bm\Gamma^{*}(t) \bv_1 \\
&=\left(1-\eta \gamma^*(t)\lambda_1\right) \bD(t)^\top \bv_1-\eta \bD(t)^\top \bm\Gamma^{*}(t) \bv_1
\end{aligned}
$$

First by Lemma~\ref{M* lemma} and Corollary~\ref{v_1 similarity coro}:
$$
\begin{aligned}
\gamma^*(t)\lambda_1 &= \bv_1^\top (k_s(\bM(t+1)-\bm\Gamma(t+1))+(1-k_s)(\bM(t) - \bm\Gamma(t)))\bv_1 \\&\geq k_s\sha^*(t+1) + (1-k_s)\sha^*(t) - \frac{2c_2}{m} \\&\geq \frac{2(1+\rho)}{\eta} - \frac{4c_2}{m}
\end{aligned}
$$

Also we have 
$
\left|\bD(t)^{\top} \bm\Gamma^{*}(t) \bv_1\right| \leq \frac{c_{2} +c_{6}}{m} \|\bD(t)\| \leq \frac{c_{2} +c_{6}}{m}|(\bD(t)^\top \bv_1|+\|\bR(t)\|)
$.

Hence 
$$
\begin{aligned}
|\bD(t+1)^{\top}\bv_1| 
&\geq |1-\eta \gamma^*(t)\lambda_1||\bD(t)^\top \bv_1|-\eta |\bD(t)^\top \bm\Gamma^{*}(t) \bv_1| \\
&\geq |-1-2\rho + \frac{4c_2\eta}{m} + \frac{(c_2+c_6)\eta}{m}|\\
&\geq (1+2\eta\rho^*)|\bD(t)^{\top}\bv_1| - \eta \frac{c_2+c_6}{m}c_7\sqrt{n} \\ &\geq (1+\epsilon_1)|\bD(t)^{\top}\bv_1|
\end{aligned}
$$

Here the third inequality holds by the definition of $\rho^*$.
\end{proof}

Now we state a lemma which proves that if $\bD(t)^\top \bv_1$ is large enough, then the sharpness will decrease in the next iteration.

\begin{lemma}
\label{EOS phase 3 lemma}
Assume $m= \Omega(\eta)$.
There exists constant $c_{8}, c_{9}>0$ such that if $\bD(t)^\top\bv_1>c_{8} \sqrt{n}$, then $\bv_1^\top \bM(t+1) \bv_1 - \bv_1^\top \bM(t) \bv_1 < -c_{9}/m$.
\end{lemma}

\begin{proof}
Recall Equation~\eqref{Key equation in EOS}. Also we have $\|\bR(t)\|\leq c_7\sqrt{n}$ by ~\eqref{R upper bound appendix}. Hence

$$
\begin{aligned}
&\sha^*(t+1) - \sha^*(t) \\
=&-\frac{8\eta \lambda_{1}}{mn^2}\left(\bF(t)^{\top}\bD(t)+(\bF(t)^{\top} \bv_1)(\bD(t)^{\top} \bv_1)-\frac{\eta}{2}(\bD(t)^{\top} \bv_1)^{2}  \sha^*(t)  \right. \\
&\left.-\frac{\eta}{2}\bR(t)^{\top} \bm \Gamma(t) \bR(t)  -\eta  \bR(t)^{\top} \bm \Gamma(t)\left(\bv_1 \bv_1^{\top}\bD(t)\right)-\bR(t)^{\top}\left(\frac{m}{d} \bX^{\top} \bX\right) \bR(t) \frac{\eta}{m n}\right) \\
\leq& -\frac{8\eta \lambda{1}}{mn^2}\left(2 \beta(\bD(t)^{\top} \bv_1)^{2}+\|\bR(t)\|^{2}-\frac{\eta c_2}{2m}\|\bR(t)\|^{2}    \right. \\
&\left.-\frac{c_{2}\eta }{m} \|\bR(t)\|   \cdot\|\bD(t)\|-\frac{\eta \lambda_{1}}{d n}\|\bR(t)\|^{2}  -2\|\bm Y\| \cdot\|\bD(t)\| \right) \\
\leq& -\frac{8\eta \lambda_{1}}{mn^2}\left(2 \beta(\bD(t)^{\top} \bv_1)^{2}+\|\bR(t)\|^{2}(1- \frac{c_{2}\eta}{2m} )\right. \\
&\left.-\frac{c_{2}\eta }{m} \|\bR(t)\|   \cdot(\|\bR(t)\| + |\bD(t)^\top \bv_1|)-\frac{\eta \lambda_{1}}{d n}\|\bR(t)\|^{2}  -2\|\bm Y\| \cdot(\|\bR(t)\| + |\bD(t)^\top \bv_1|)\right) \\ 
\leq&-\frac{8\eta \lambda_{1}}{mn^2}\left(2 \beta\left(\bD(t)^{\top}{\bv_1}\right)^{2}-(  \frac{\eta c_7c_{2}\sqrt{n} }{m}+2\sqrt{n}) \cdot|\bD(t)^\top \bv_1|\right. \\
&\left.+c_7^2n (1-\frac{c_{2}  \eta}{2 m}-\frac{c_{2} \eta}{m}-\frac{\eta \lambda_{1}}{d n}) - 2c_7n\right) \\
=&-\frac{8\eta \lambda_{1}}{mn^2}\left(2 \beta(\bD(t)^\top \bv_1)^{2}-c_{10}\left(\bD(t)^\top{\bv_1}\right)+c_{11}\right)
\end{aligned}
$$ 

Here $c_{10}:= \eta c_7\sqrt{n} \cdot \frac{c_{2}}{m}+2\sqrt{n} = \Theta(\sqrt{n})$, $c_{11} := c_7^2n \left(1-\frac{c_{2}  \eta}{2 m}-\frac{c_{2} \eta}{m}-\frac{\eta \lambda_{1}}{d n}\right) - 2c_7n = Nn)$, and the first inequality holds because of Assumption~\ref{sharpness upper bound appendix}, Assumption~\ref{Gamma assumption appendix} and $\|\bD(t)\|^2=(\bD(t)^{\top} \bv_1)^2 + \|\bR(t)\|^2$, the second inequality holds because of $\|\bD(t)\|\leq |\bD(t)^{\top} \bv_1| + \|\bR(t)\|$, and the third inequality follows from the upper bound \eqref{R upper bound appendix} of $\|\bR(t)\|$.

Note that it is a quadratic function of $\bD(t)^{\top}\bv_1$. Hence if $\bD(t)^\top \bv_1 > \frac{|c_{10}|}{2\beta} + \sqrt{\frac{|c_{11}|}{2\beta}}$, we have 
$$\sha^*(t+1) - \sha^*(t) < - \frac{8\eta \lambda_1}{n^2 m} \cdot \frac{|c_{10}|\sqrt{|c_{11}|}}{2\beta} = \Theta(\frac{\eta}{m})$$ Hence we finish the proof.
\end{proof}

Now we can prove the final part (the first conclusion) of 
Theorem~\ref{EOS main thm appendix}. 

\begin{lemma}
\label{first conclusion of EOS thm}
Let $\rho$ be the one defined in Theorem~\ref{EOS divergence thm}. If $\sha(t_0)>\frac{2}{\eta}(1+\rho)$ for some $t_0$, there exists some $t_1>t_0$ such that $\sha(t_1) < \frac{2}{\eta}(1+\rho)$.
\end{lemma}
\begin{proof}
Otherwise, for any $t>t_0$, $\sha(t) > \frac{2}{\eta}(1+\rho)$. Then by Theorem~\ref{EOS divergence thm}, $\bD(t)^\top \bv_1$ increases geometrically, hence there exists $t_2>t_0$, such that $\bD(t)^\top \bv_1 > c_8\sqrt{n}$. Now by Lemma~\ref{EOS phase 3 lemma}, each iteration $\sha^*(t)$ will decrease by at least a fixed amount. Hence there must exist a time $t_3>t_0$ such that $\sha^*(t_3)<\frac{2}{\eta} - 2c_2/m$. Then by Corollary~\ref{v_1 similarity coro}, we get a contradiction.
\end{proof}

\subsection{Our Results and the NTK Regime}
\label{our result and ntk}
In this subsection, we explain why our results (Theorem~\ref{Theorem C.1, ps}, Theorem~\ref{EOS main thm appendix}) are sufficiently different from 
the quadratic setting (e.g., linear regression) or the recent convergence analysis in
NTK setting.

A key requirement in the convergence analysis in the NTK regime is that 
the learning rate is very small and the GD trajectory almost tracks the gradient flow,
hence converges to the global minimum.
However, we consider typical learning rate used in practice, which can be much larger. In particular, $\eta>2/\Lambda$ can happen in our setting, which causes instability (i.e., such as the growth of loss in Lemma C.11) along the training trajectory.
Such instability cannot be captured by any existing convergence analysis in NTK regime at all. 
Hence, all existing NTK convergence results do not directly apply here. 

Equally importantly, we find that even when $\bW(t)$ changes slightly (several orders of magnitude smaller than its initialization), PS and EOS still happen with a not so small learning rate $\eta$. To support our claim, we include the experimental results in
Appendix~\ref{Appendix.D.2.3 Compare with NTK}. In Figure~\ref{ntk compare}, we can see that the initialization $\bW(0)$ is 
much larger than the change of $\bW(t)$ and the norm of $\bW(t)$ grows larger when $m$ becomes larger. However, we still observe that PS and EOS occur in this setting. Hence, the setting we study in this paper and our results are intrinsically different from the quadratic setting (in which case EOS cannot happen).

Last, 
in our proofs in Section~\ref{phase i and ps appen} and~\ref{edge of stab(phase ii-iv) appen}, our current bound requires that $m= \Omega(n^2)$ and we also assume $\lambda_r = \Omega(1)$. This may create an impression that we need
a very wide network which operates in the NTK regime.
However,
we remark that if our analysis can be tightened to $m=O(n)$, 
one can formally prove that
$\|\bA(t)\|$ can actually change significantly ($\|\bA(t)\|^2-\|\bA(0)\|^2$) has the same scale as the initialization $\|\bA(0)\|^2$), resulting a significant change of sharpness as well,
hence beyond the NTK regime. For example, in the proof of EOS (Section~\ref{edge of stab(phase ii-iv) appen}), we prove a loose $O(\sqrt{mn})$ upper bound of $\|\bD(t)\|$ (Lemma~\ref{D upper bound}). However by Lemma~\ref{EOS phase 3 lemma}, 
when $\|\bD(t)\|$ reaches $O(\sqrt{n})$, the sharpness starts dropping quickly. 
So if a better upper bound of $\|\bD(t)\| = O(\sqrt{n})$ can be proven (this is true empirically for all of our experiments), the width $m$ can be set to $\Theta(n)$, and this  suffices to implies a significant change of $\|\bA(t)\|^2$. We leave these improvements as future directions.

\section{Verification for Assumptions} \label{appendixD}
In this section, we first justify the assumptions we made in Section~\ref{Section 3} and Section~\ref{theoretical analysis} empirically. Then we present the experiment described in Section~\ref{5}. 

\subsection{Assumptions in Section~\ref{Section 3}}
In this subsection, we conduct experiments to verify 
the assumptions in Section~\ref{Section 3}. The detailed experiment settings  
can be found in Appendix~\ref{appendixA}.

\subsubsection{$\mathbf{M}_A$ has small eigenvalues}
In Section~\ref{subsection A norm}, we mentioned that the largest eigenvalue of $\bm {M_A}\ := (\frac{\partial \bm{F}}{\partial \bm{A}}) (\frac{\partial \bm{F}}{\partial \bm{A}})^\top$  is much smaller than the sharpness. 
We verify this assumption under different settings in Figure~\ref{M_A}, including a fully-connected linear network, a fully-connected network with tanh activation and a convolutional one. Observe that $\|\bm{M_A}\|$ (the blue curve) is very close to 0 and hardly increases during the training process along the whole trajectory.
\begin{figure}[H]
    \centering
    \subfigure[Fully-connected Linear Network]{
    \includegraphics[width=0.31\textwidth]{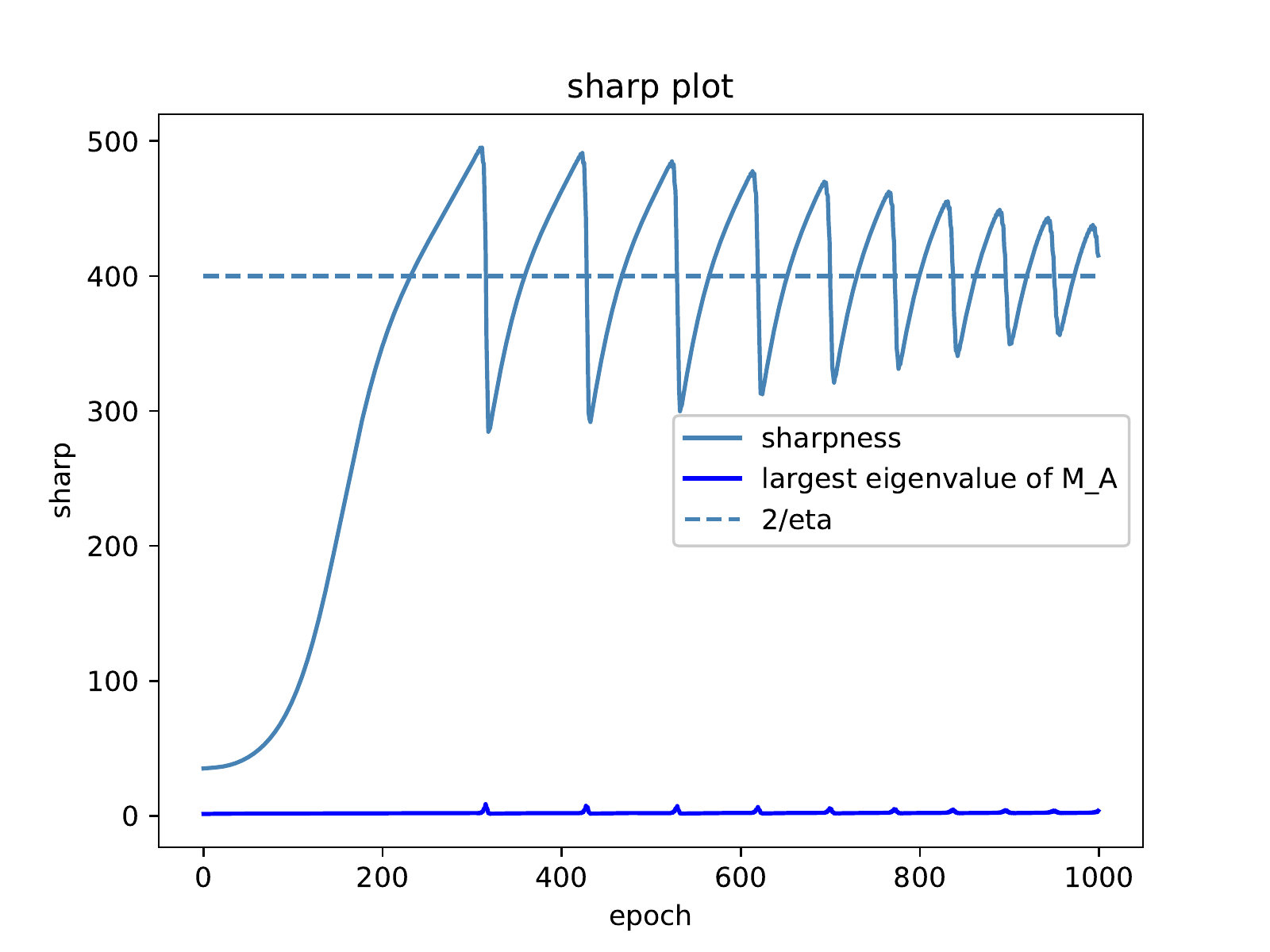}
    }
    \subfigure[Fully-connected tanh Network]{
    \includegraphics[width=0.31\textwidth]{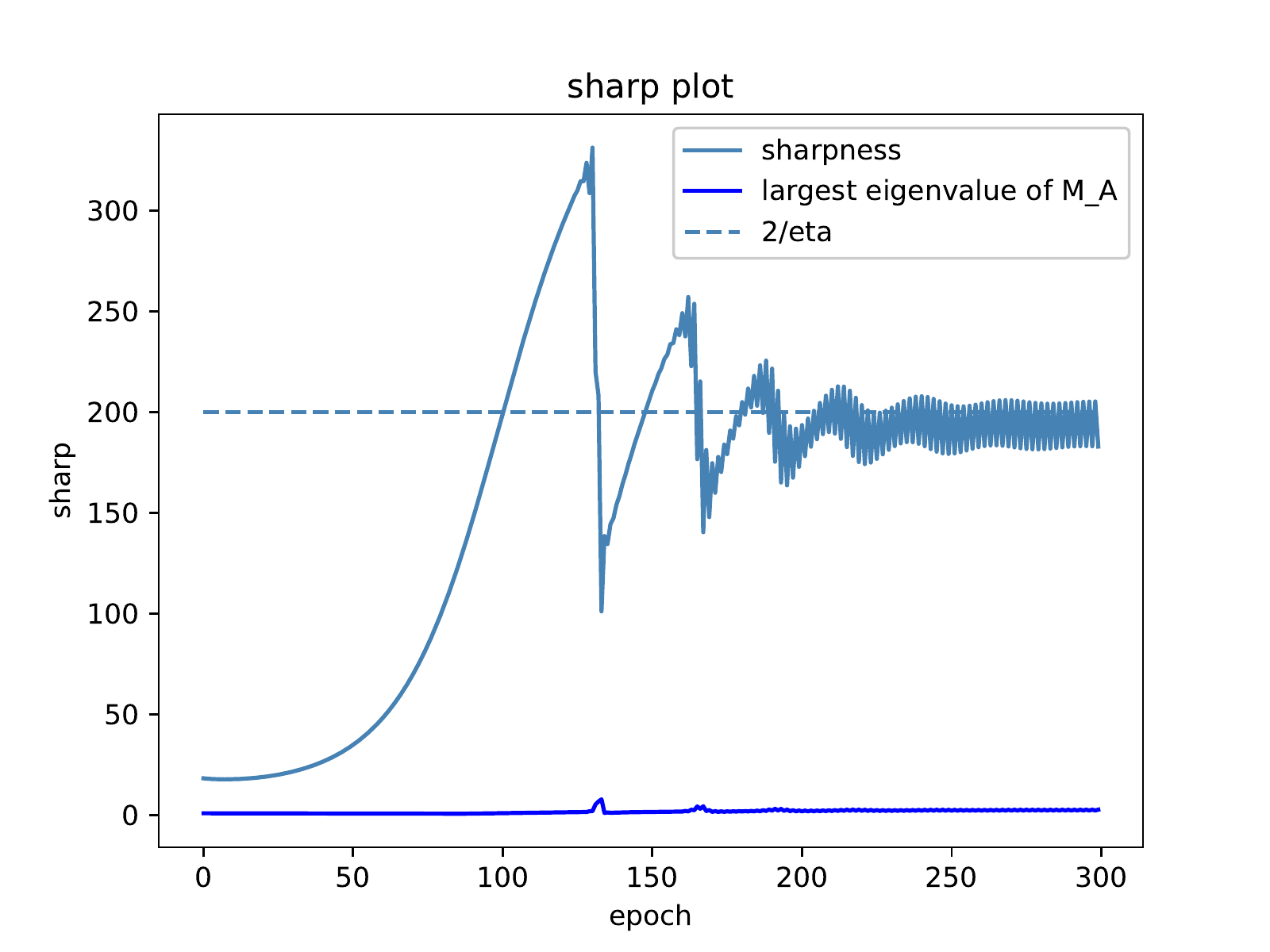}
    }
    \subfigure[Convolutional tanh Network]{
    \includegraphics[width=0.31\textwidth]{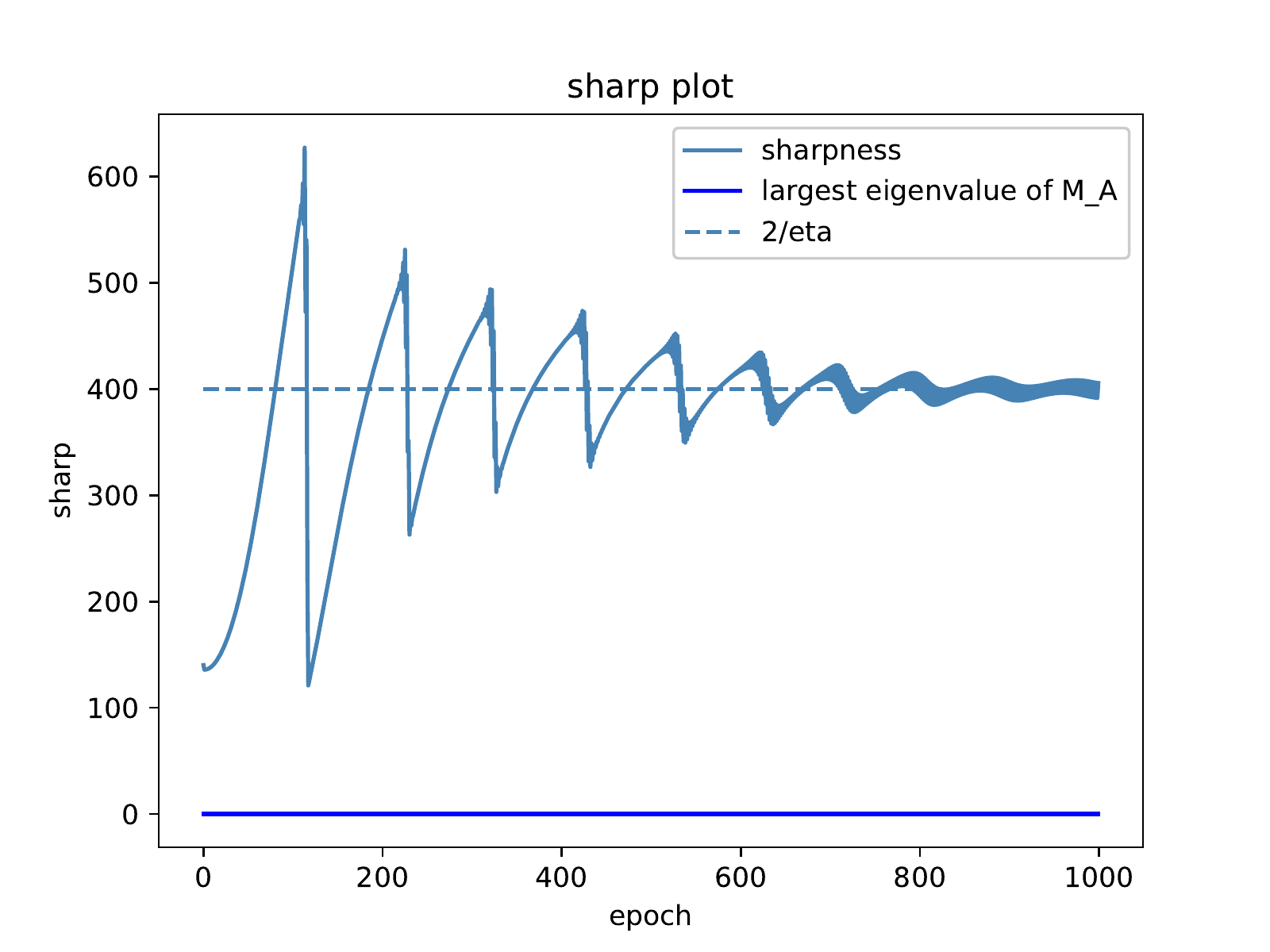}
    }
    \caption{In this figure, we show that $\|\bm{M_A}\|$ is much smaller than the sharpness. Note that $\|\bm{M_A}\|$ (the blue curve) is very close to 0 and hardly increases during the training process along the whole trajectory. The sharpness (the steel-blue curve) strictly dominates $\|\bm{M_A}\|$ for all time. The detailed experiment settings can be found in Appendix~\ref{appendixA}.
    }
    \label{M_A}
\end{figure}

\subsubsection{Assumption~\ref{outlier assp}}

In this assumption, we assume that there is a large gap between the largest and the second largest eigenvalue, and thus the second largest eigenvalue is always below $1/\eta$.
We verify the outlier assumption by calculating the largest and the second largest eigenvalue of $\bm{M}(t)$. In \citet{sagun2016eigenvalues,sagun2017empirical}, the sharpness is much larger than the largest eigenvalue in the bulk (the $(K+1)$-th largest eigenvalue of $\bM$ where $K$ is the number of classes). In our binary setting $K=1$. In Figure~\ref{Outlier Pics}, we show that the largest eigenvalue indeed dominates the second one, and the second one never reaches $1/\eta$, which verifies Assumption~\ref{outlier assp}.
\begin{figure}[H]
    \centering
    \subfigure[Fully-connected linear network]{
    \includegraphics[width=0.31\textwidth]{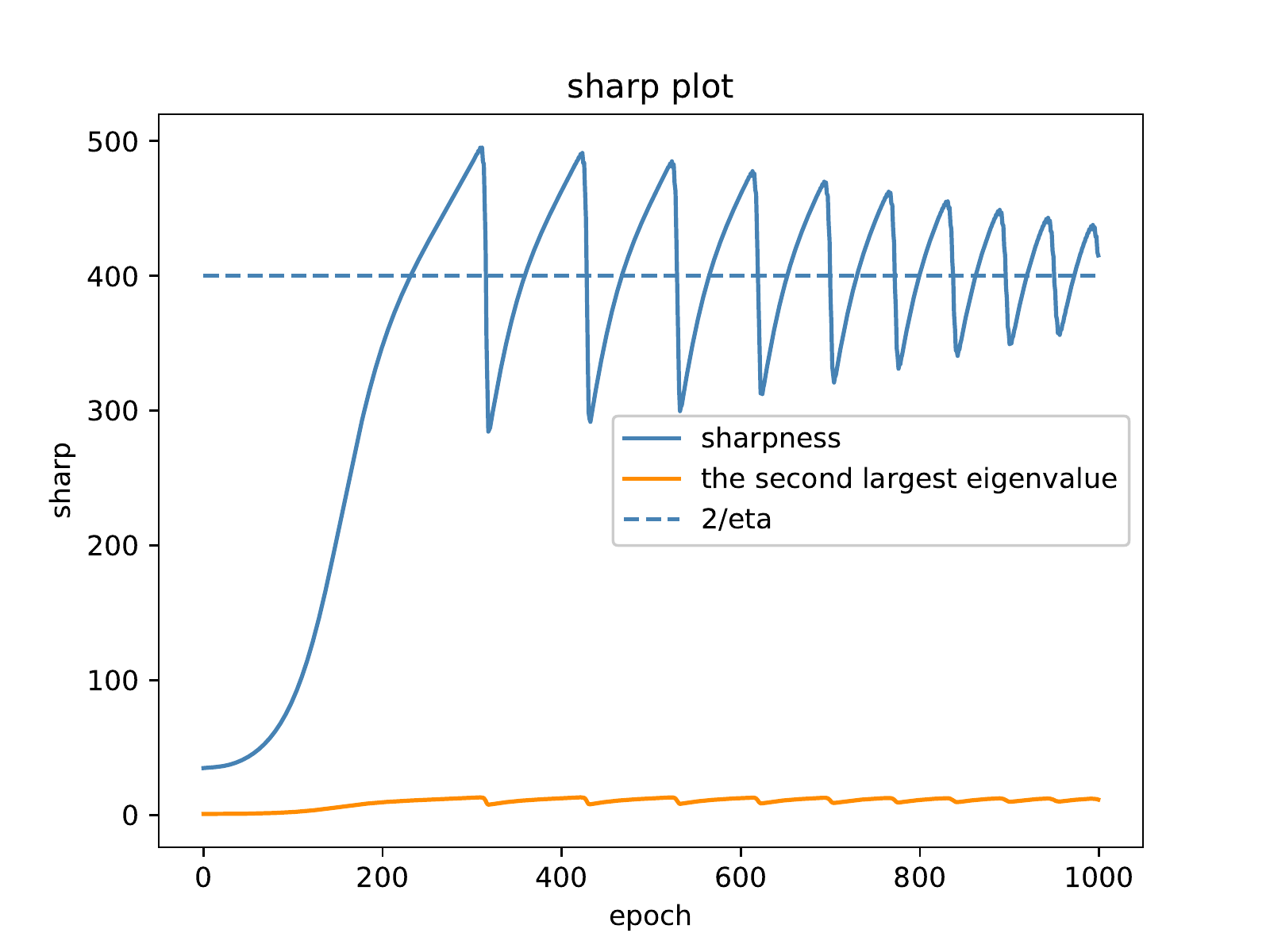}
    }
    \subfigure[Fully-connected tanh network]{
    \includegraphics[width=0.31\textwidth]{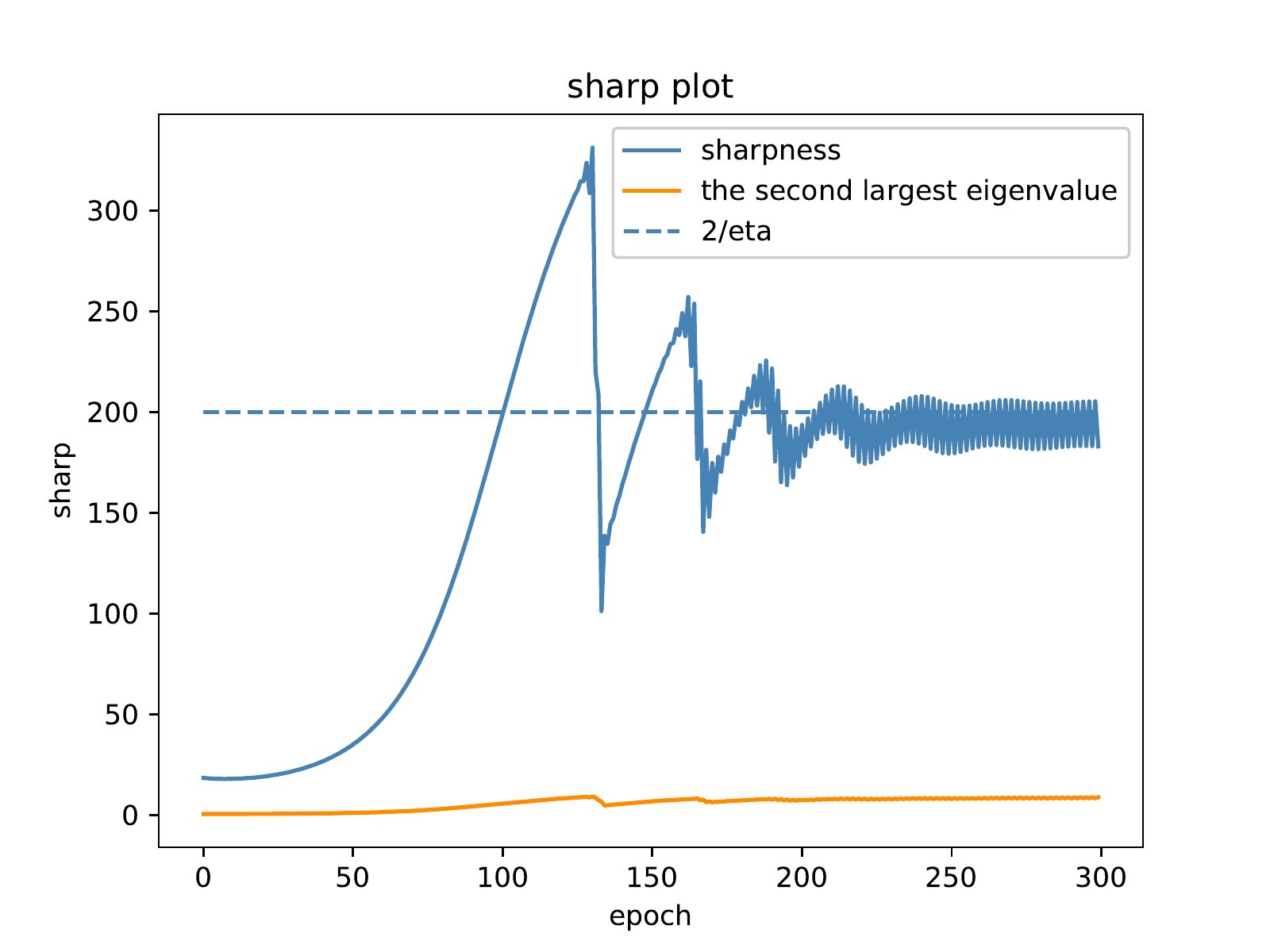}
    }
    \subfigure[Convolutional network with ELU activation]{
    \includegraphics[width=0.31\textwidth]{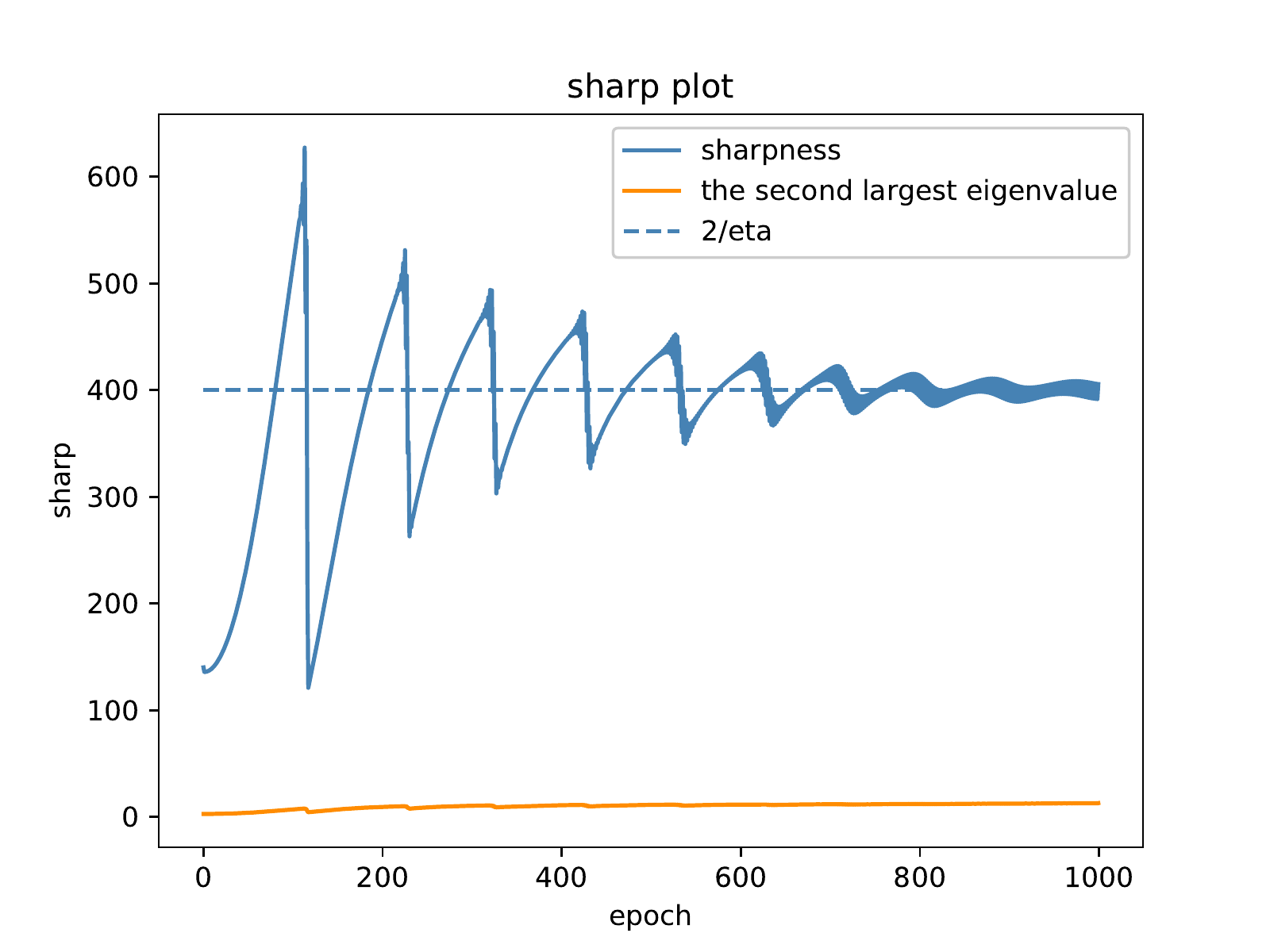}
    }
    \caption{The second largest eigenvalue (the orange curve) of $\bm{M}$ is much smaller than the sharpness (the steel-blue curve). Also, the second largest eigenvalue never reaches $2/\eta$.}
    \label{Outlier Pics appendix}
\end{figure}

\subsubsection{Assumption~\ref{Ass: Order one app} (First Order Approximation of GD)}
\label{Appen: order one apprx}
In Assumption~\ref{Ass: Order one app}, we assume the gradient descent trajectory is close to the first order approximation. To verify that the first order term is indeed dominant along the trajectory, for both the residual $\bD(t)$ and the output layer norm $\|\bA(t)\|^2$, we plot the norms of the actual GD update, the first order approximation order approximation and the higher order terms of the update rule in Figure~\ref{fig: first order approx}. Observe that in the progressive sharpening phase, the first order approximation is almost the same as the actual gradient update; while in the EOS phase, the first order approximation is still close to the actual gradient most of the time.
We can see that the norm of the higher order terms spikes occasionally, but when this happens the first order term spikes much higher. 

\begin{figure}[H]
    \centering
    \subfigure[Linear network]{
    \includegraphics[width=0.48\textwidth]{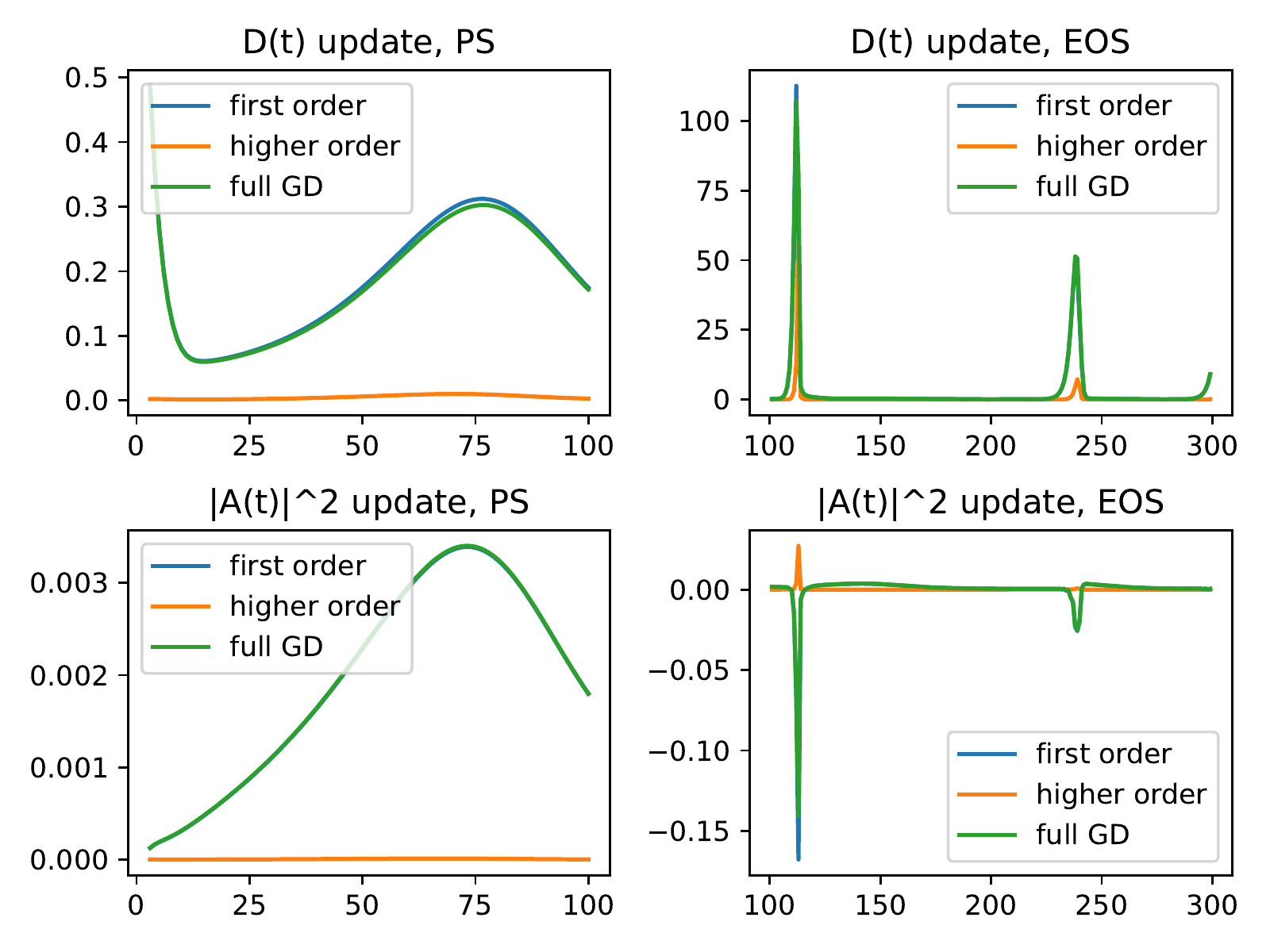}
    }
    \subfigure[Fully connected tanh network]{
    \includegraphics[width=0.48\textwidth]{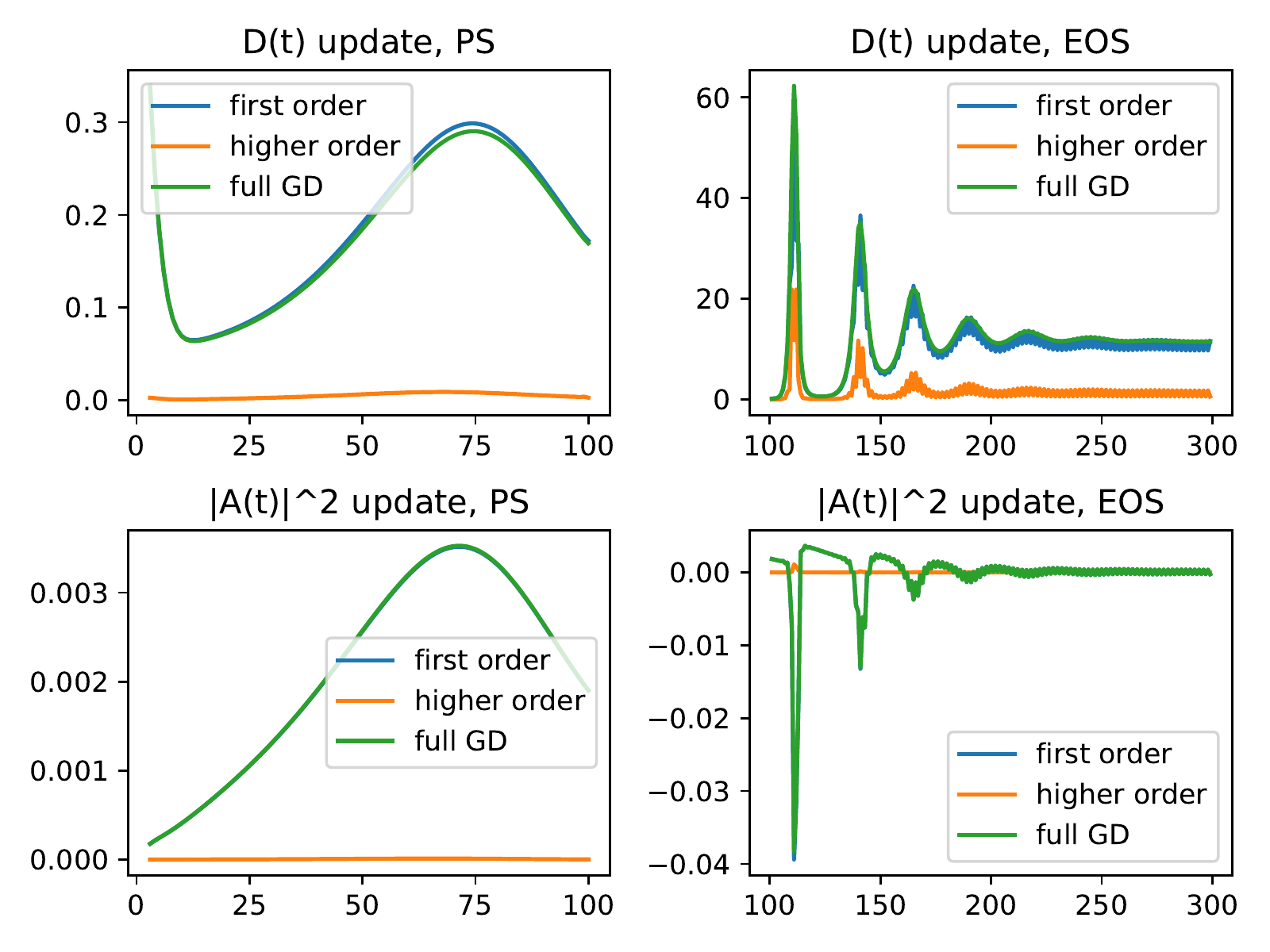}
    }
    
    \caption{This figure shows the norm of the full gradient $\|\bD(t+1)-\bD(t)\|$, the first order approximation $\|\eta\bM(t)\bD(t)\|$ and the higher order terms $\|\bD(t+1)-\bD(t)+\eta\bM(t)\bD(t) \|$ in the training process. Observe that in both the $\|\bA(t)\|^2$ and $\bD(t)$ dynamics, the first order approximation (the blue curve) almost overlaps with the actual update (the green curve) in the PS phase. While in the EOS phase, though the norm of the higher order terms (the orange curve) spikes, the first order approximation is still close to the actual update. }
    \label{fig: first order approx}
\end{figure}

\subsubsection{Assumption~\ref{gf assumption} (Gradient flow for the PS phase)}

In Assumption~\ref{gf assumption}, we assume that in the progressive sharpening phase, the gradient descent trajectory is close to the gradient flow trajectory. We refer to Appendix J  in \citet{cohen2021gradient} for more experiments about this claim. 
Here for readers' convenience, we duplicate their Figure 29 in Appendix J as Figure~\ref{Outlier Pics}.

\begin{figure}[H]
    \centering
    \includegraphics[width=0.9\textwidth]{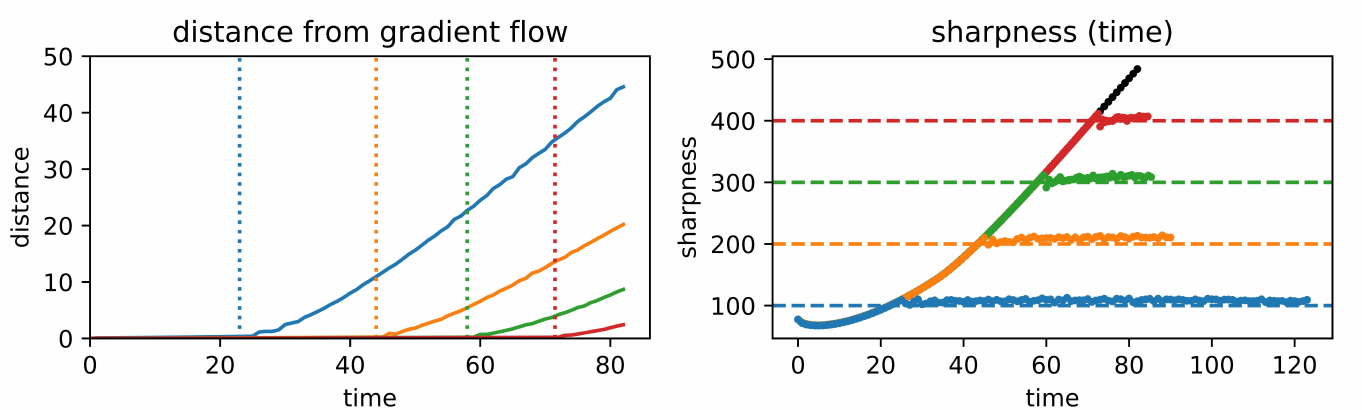}
    \caption{\textbf{Left}: the picture shows the $\ell_2$ distance between the gradient flow trajectory at $t$ iteration and the gradient descent trajectory at $t/\eta$ iteration. The vertical dotted line is the time when the sharpness reaches $2/\eta$. \textbf{Right: the plot of sharpness where (iteration $\times$ learning rate) is on the x-axis.}}
    \label{Outlier Pics}
\end{figure}

\subsubsection{Assumption~\ref{Ass: eigspace assumptions} (iii) (principal directions moves slowly)}
In Figure~\ref{Largest direction} we verify Assumption~\ref{Ass: eigspace assumptions} (iii) and as well as the discussion after Assumption~\ref{Ass: eigspace assumptions} (i). 
In general models, we find that the eigenvectors corresponding to small
eigenvalues may change drastically. But for the largest eigenvector, it indeed changes slowly from the initialized direction. In Figure~\ref{Largest direction}, we can see that over a long training time the similarity of $\bv_1(t)$ with its initialization is still larger than at least 0.98.  \citet{mulayoff2021implicit} also proved that near the minima, the top eigenvectors of the Hessian matrices tend to align. That is, the directions of these top eigenvectors are approximately parallel. This fact also corroborates our assumption near the minima.

\begin{figure}[H]
    \vspace{-0.2cm}
    \centering
    \subfigure[Fully-connected linear network]{
    \includegraphics[width=0.315\textwidth]{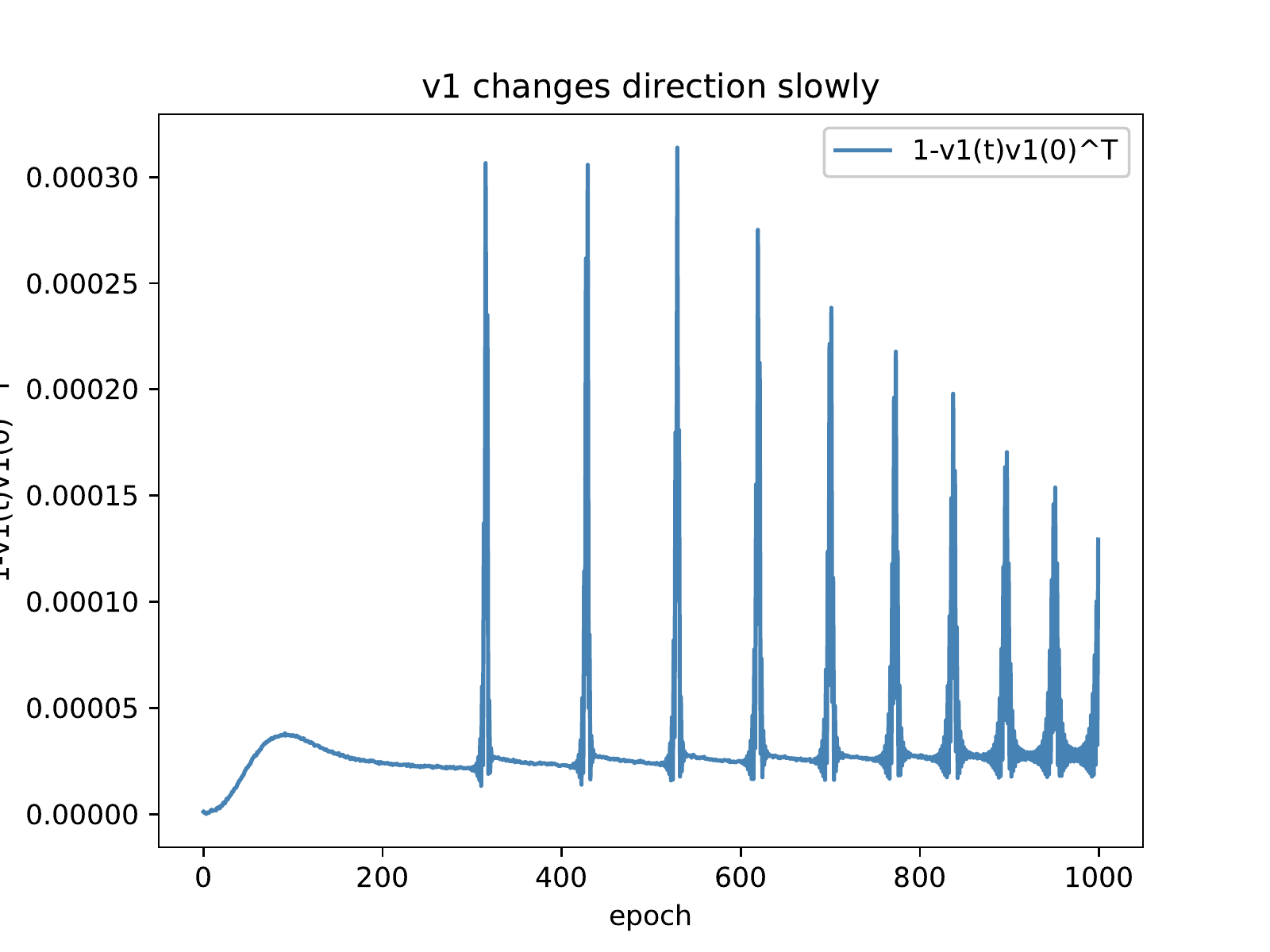}
    }
    \subfigure[Fully-connected tanh network]{
    \includegraphics[width=0.315\textwidth]{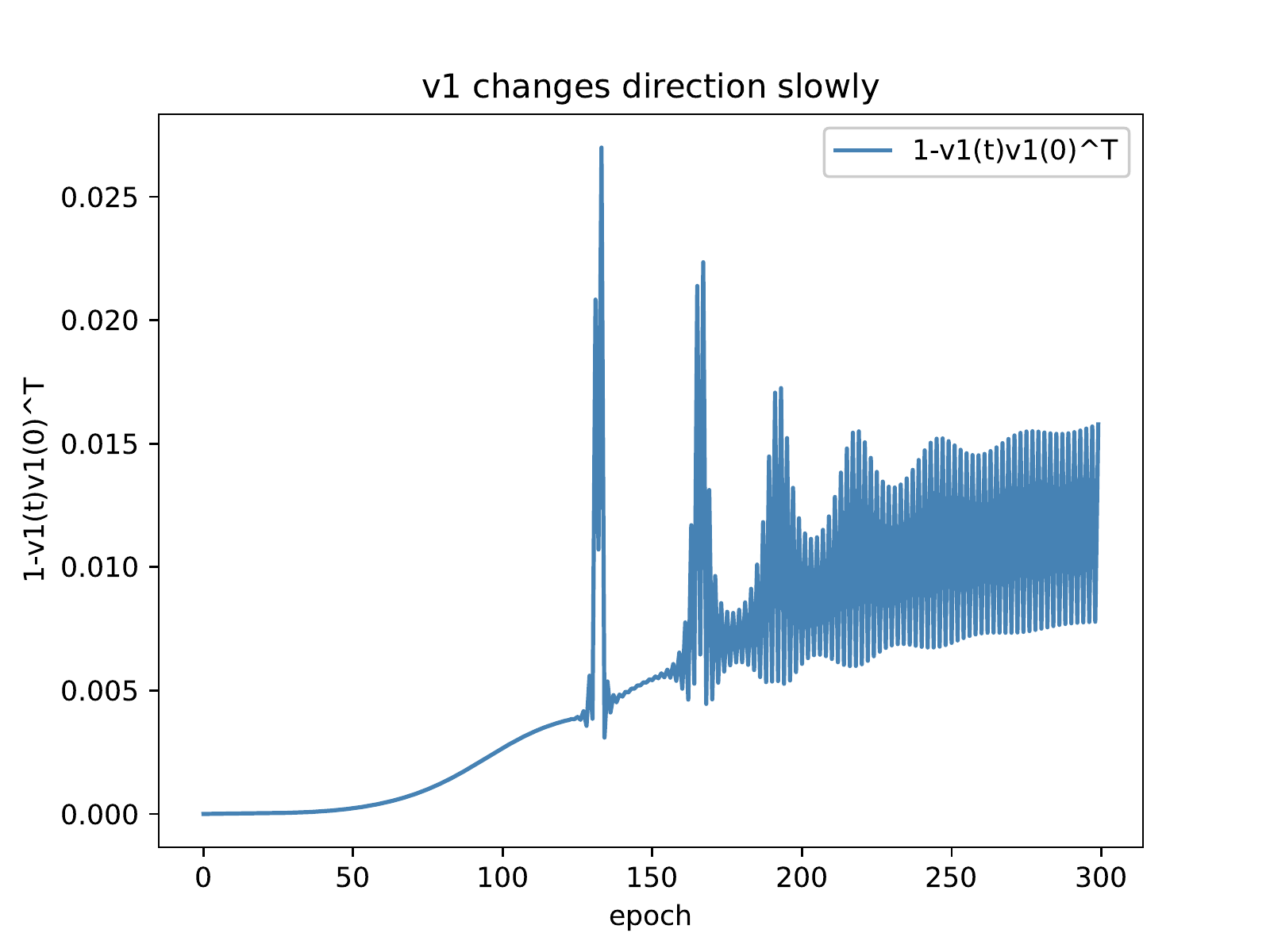}
    }
    \subfigure[Convolutional network with ELU activation]{
    \includegraphics[width=0.315\textwidth]{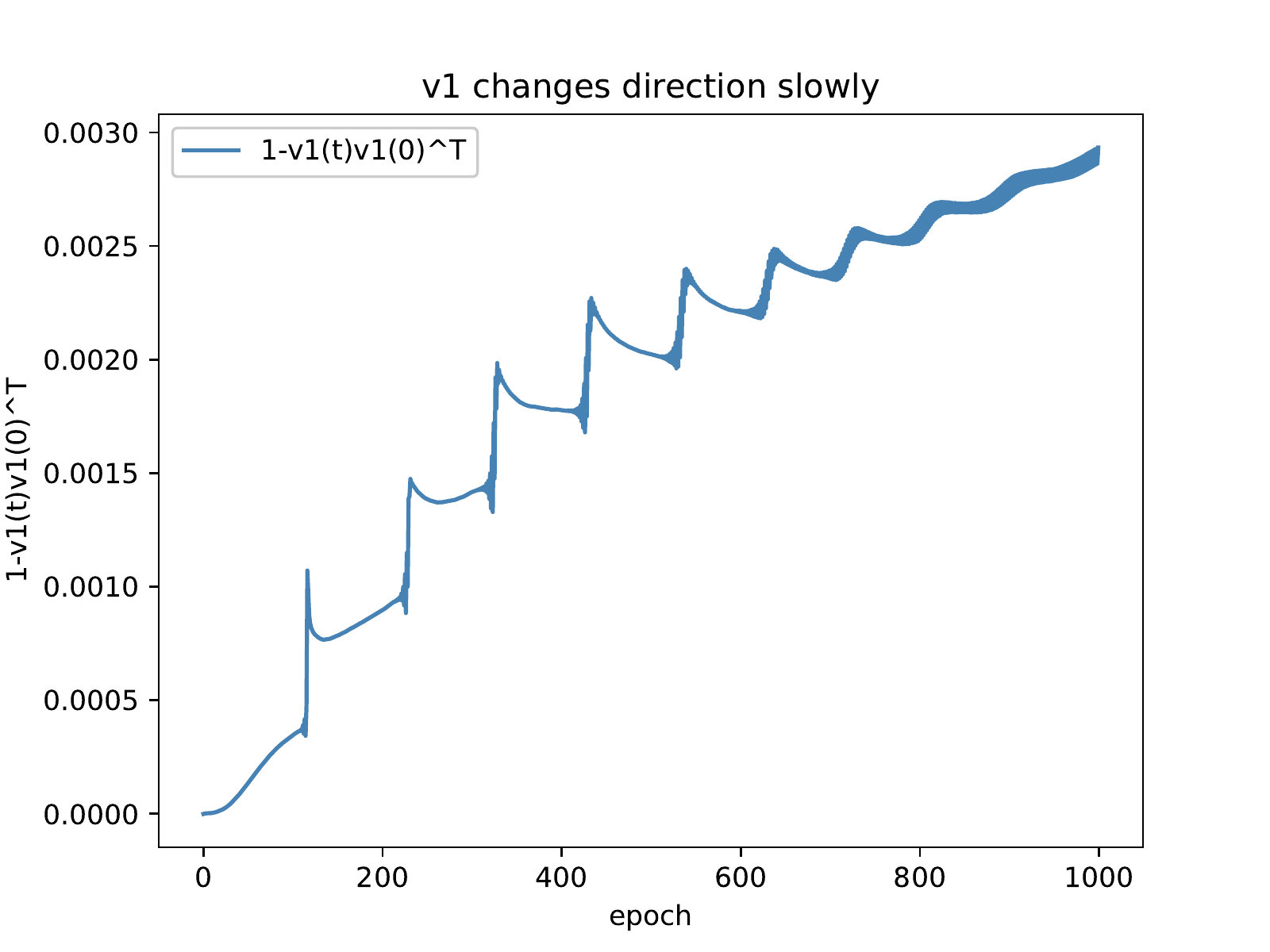}
    }
    \caption{In this figure, we show that under various architectures, the  first eigenvector of $\bm{M}$ changes slowly throughout the training process, even when Edge of Stability phenomenon occurs.}
    \label{Largest direction}
\end{figure}

\subsection{Assumptions in Section~\ref{theoretical analysis}}\label{D.2 appen}
In this subsection, we present the assumptions in Section~\ref{theoretical analysis}, and conduct experiments to verify them. We add a scale coefficient $\frac{1}{\sqrt m}$ in the linear network to be consistent with the settings of theoretical analysis.
\subsubsection{Assumption~\ref{eigenspectrum assumption}}\label{Assumption 4.1}

In Assumption~\ref{eigenspectrum assumption}, we assume the ratio $\lambda_i/\lambda_{i+1}$ between the two adjacent eigenvalues of $\bX^\top\bX$ is bounded by a small constant. In the 1000-example subset of CIFAR-10, we verify this assumption by experiments. We plot the eigenvalues and their ratio in Figure~\ref{eigenspec assumption fig nonz}. It shows that
Assumption~\ref{eigenspectrum assumption} holds and the constant $\chi\approx 38$, since almost all the ratios are close to 1, and the largest ratio is $\lambda_1/\lambda_{2}\approx 38$.

\begin{figure}[H]
    \centering
    \subfigure[The eigenvalue spectrum (log scale)]{
    \includegraphics[width=0.48\textwidth]{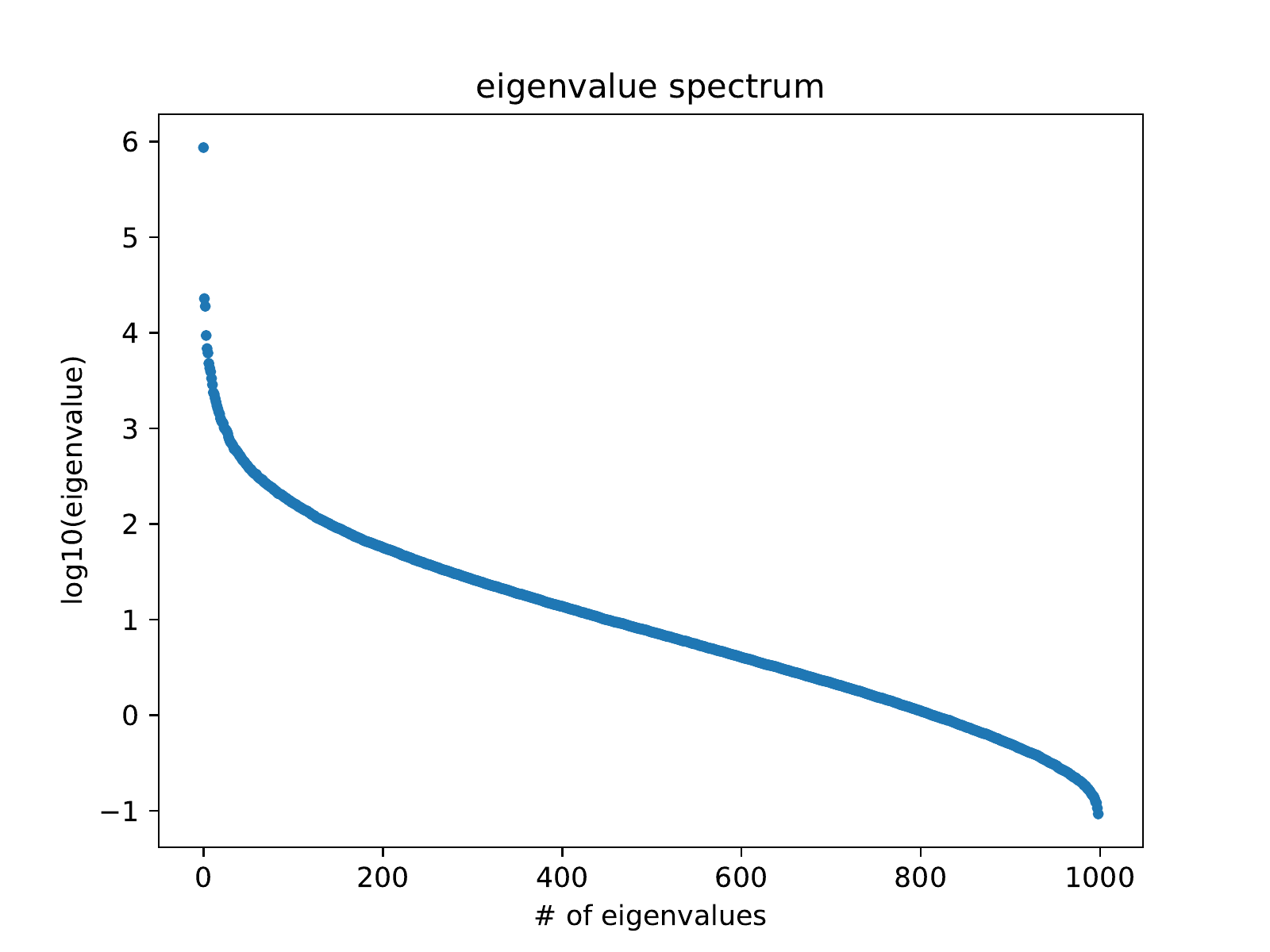}
    }
    \subfigure[The ratio between $\lambda_i$ and $\lambda_{i+1}$]{
    \includegraphics[width=0.48\textwidth]{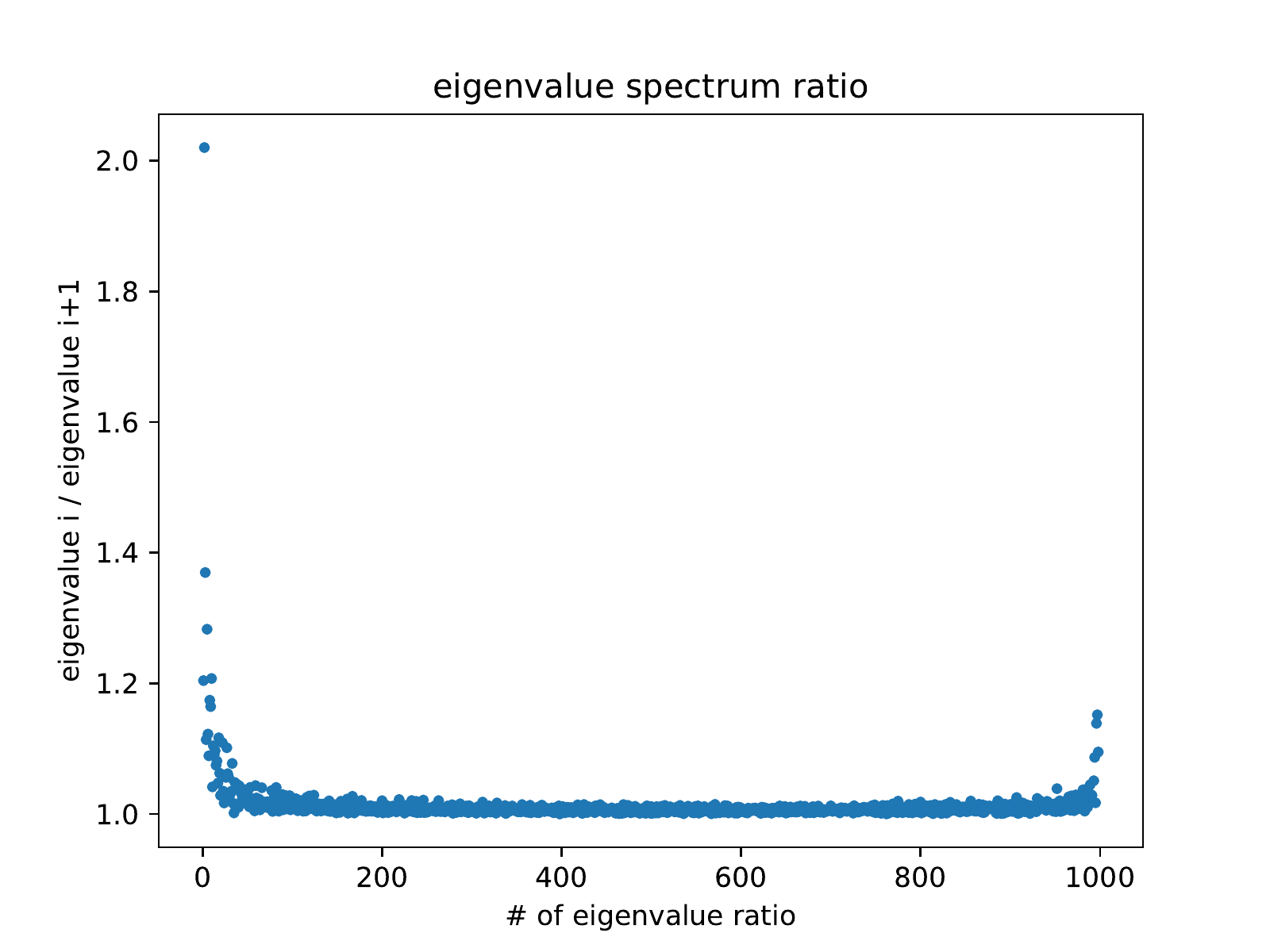}
    }
    \caption{(a): The full eigenvalue spectrum of of $\bX^\top\bX$. Observe that the maximal ratio $\max\{\lambda_i/\lambda_{i+1}\}=\lambda_1/\lambda_2$. (b): The ratio between two adjacent eigenvalues $\lambda_i/\lambda_{i+1}$, $i\geq 2$. We exclude the largest eigenvalue to make the figure clearer. }
    \label{eigenspec assumption fig nonz}
\end{figure}

If we further consider a mean-subtracted version of this subset of CIFAR-10 (See Figure~\ref{eigenspec assumption fig}), we can reduce the ratio to $\chi\approx 3$.
\begin{figure}[H]
    \centering
    \subfigure[The eigenvalue spectrum (log scale)]{
    \includegraphics[width=0.48\textwidth]{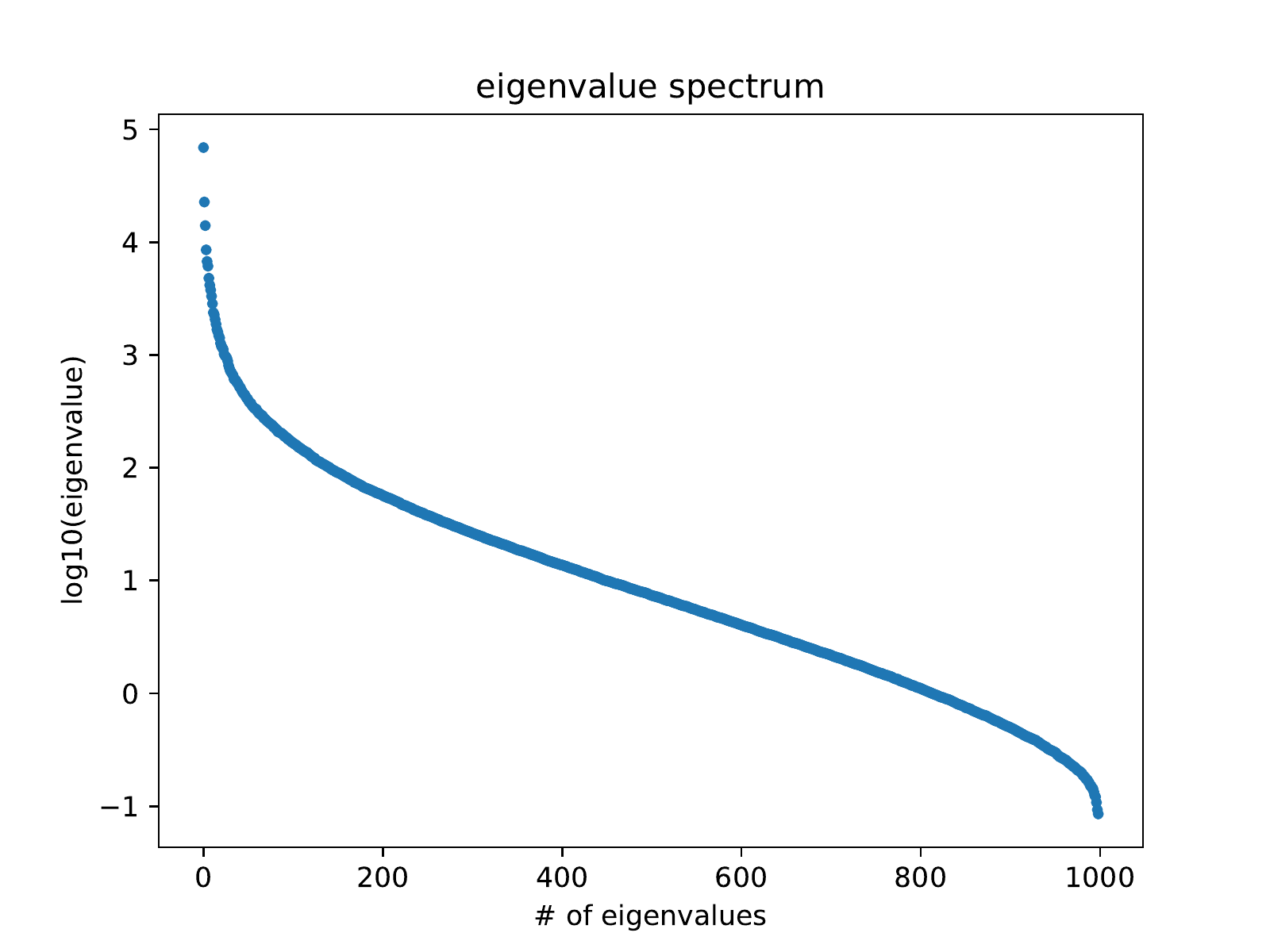}
    }
    \subfigure[The ratio between $\lambda_i$ and $\lambda_{i+1}$]{
    \includegraphics[width=0.48\textwidth]{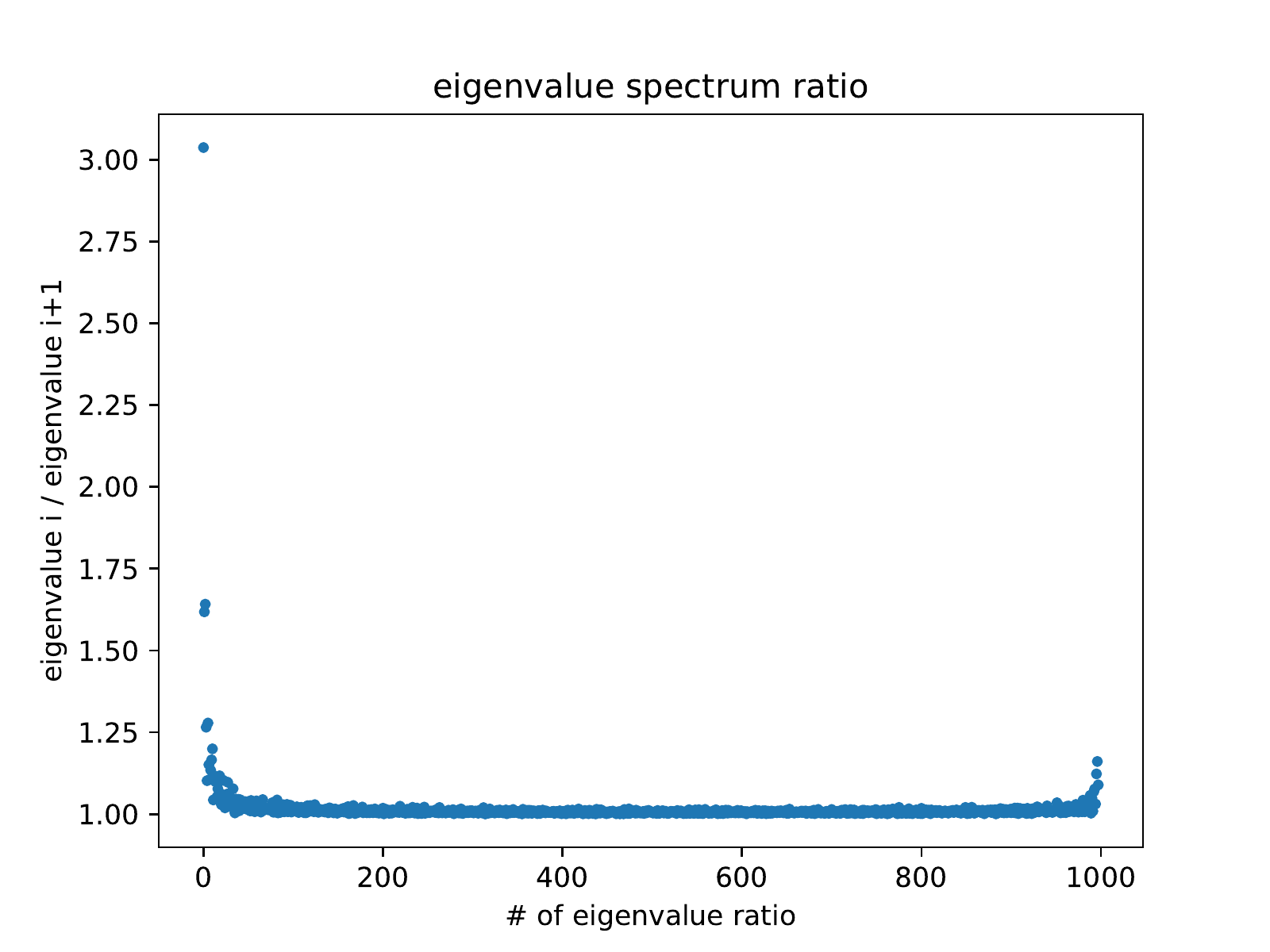}
    }
    \caption{(a): The full eigenvalue spectrum of of $\bX^\top\bX$ when samples' mean is subtracted. (b): the ratio between two adjacent eigenvalues $\lambda_i/\lambda_{i+1}$. In Figure (b), observe that the maximal ratio $\max\{\lambda_i/\lambda_{i+1}\}=\lambda_1/\lambda_2\approx 3$. \textbf{Remark.} Note we should exclude the minimal eigenvalue if the mean of examples is subtracted, because when we subtract the mean from the dataset, it cannot be full rank. The minimal eigenvalue of $\bX^\top\bX$ is 0, thus it is unnecessary to check it for our assumption.}
    \label{eigenspec assumption fig}
\end{figure}

\subsubsection{Assumption~\ref{Lambda assumption}}
\label{D.2.2 gamma assumption verification}
In Appendix~\ref{phase i and ps appen}, by Theorem~\ref{Theorem C.1, ps} we prove that $\|\bm\Gamma(t)\|$ is bounded by $R_w=O(1/m)$ in the progressive sharpening phase. When gradient descent enters EOS, the proof does not hold and we make this bound into an assumption (Assumption~\ref{Lambda assumption}). We empirically verify that the bound can only increase by a constant factor (despite some spikes in EOS). See Figure~\ref{gamma assumption}.
In this figure $\|\bm \Gamma(t)\|\leq \frac{24}{m}$ with $m=40, 80, 160, 200$.

\begin{figure}[H]
    \centering
    \subfigure[The sharpness, $m=40, 80, 160, 200$]{
    \includegraphics[width=0.245\textwidth]{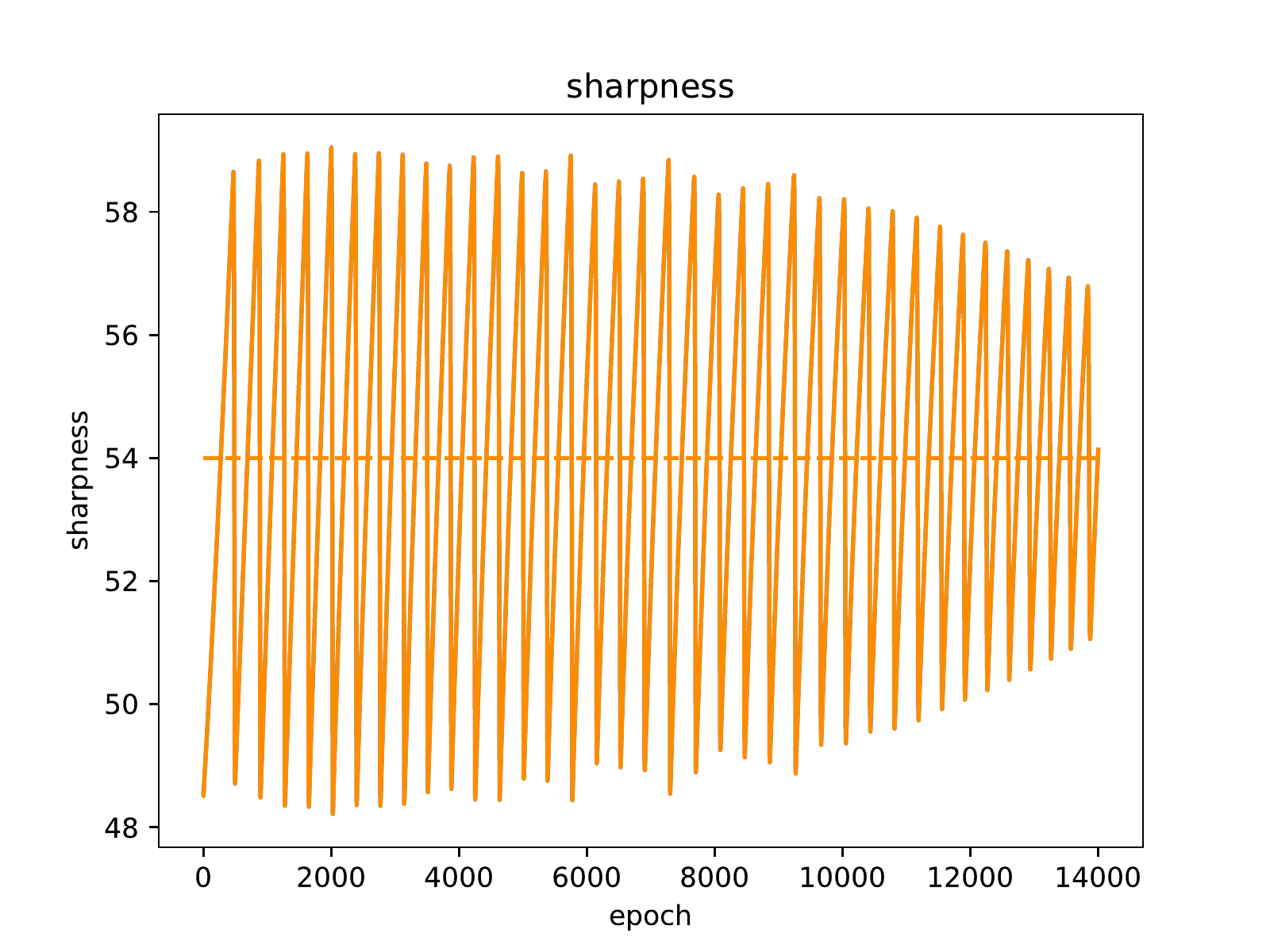}
    
    \includegraphics[width=0.245\textwidth]{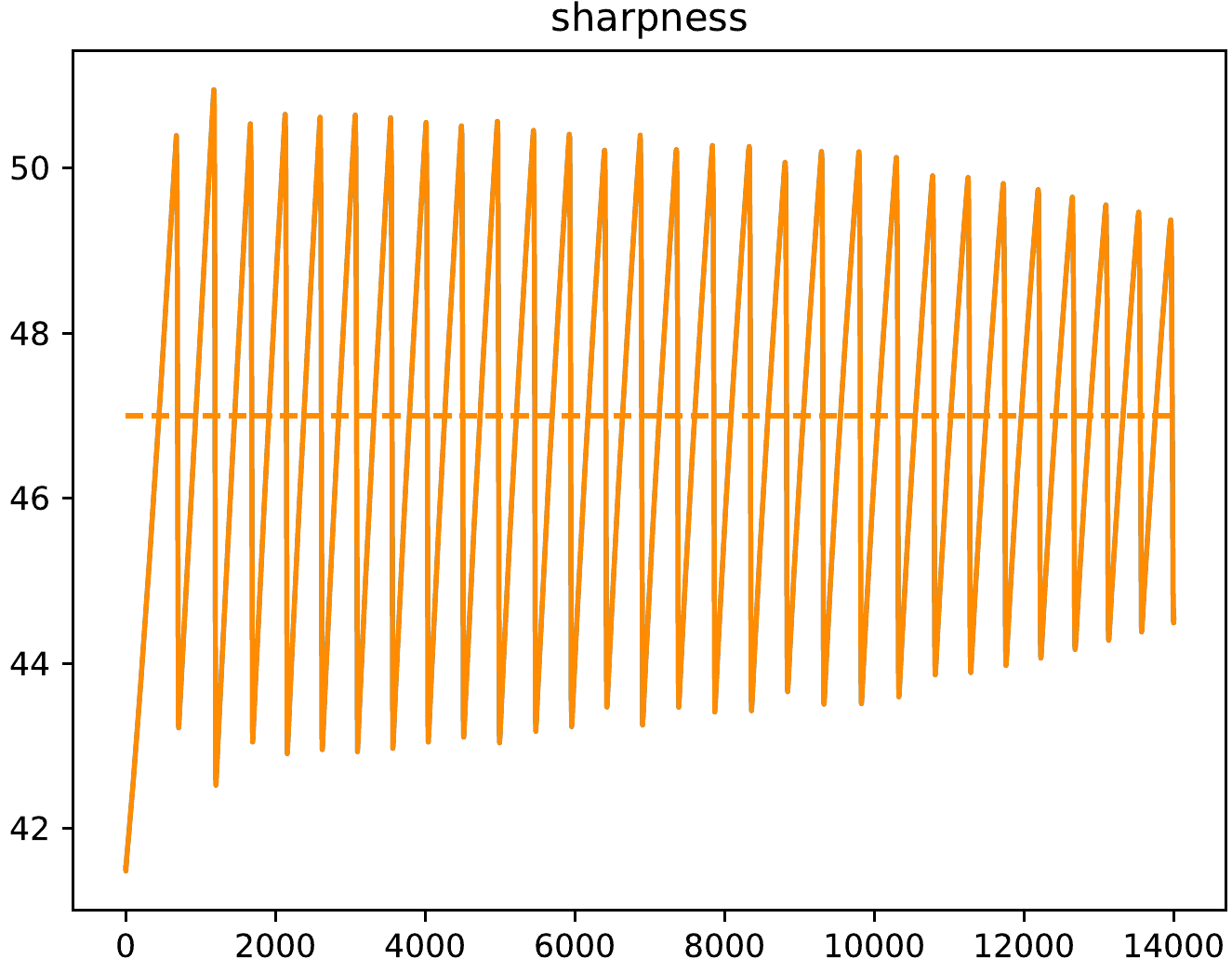}
    
    \includegraphics[width=0.245\textwidth]{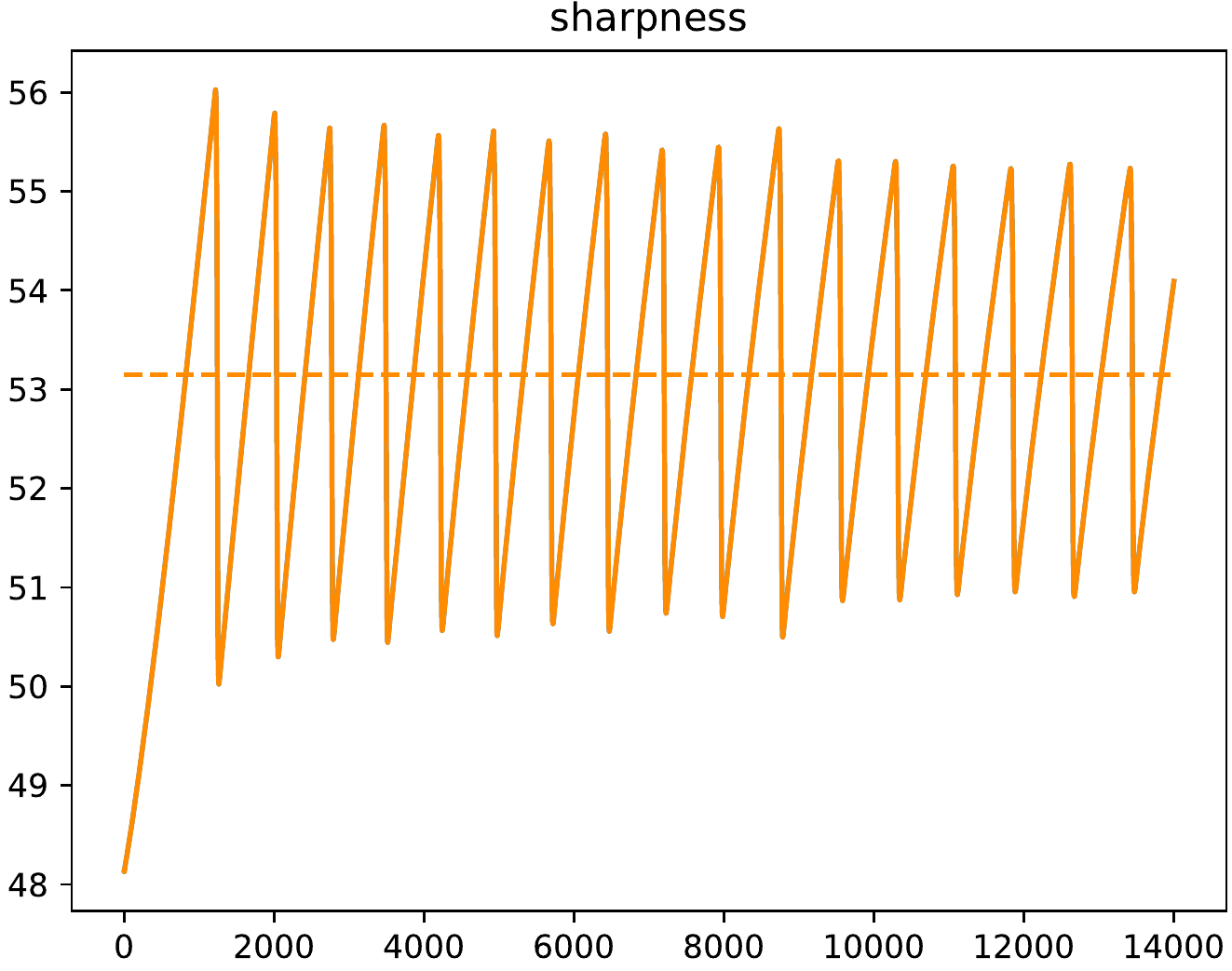}

    \includegraphics[width=0.245\textwidth]{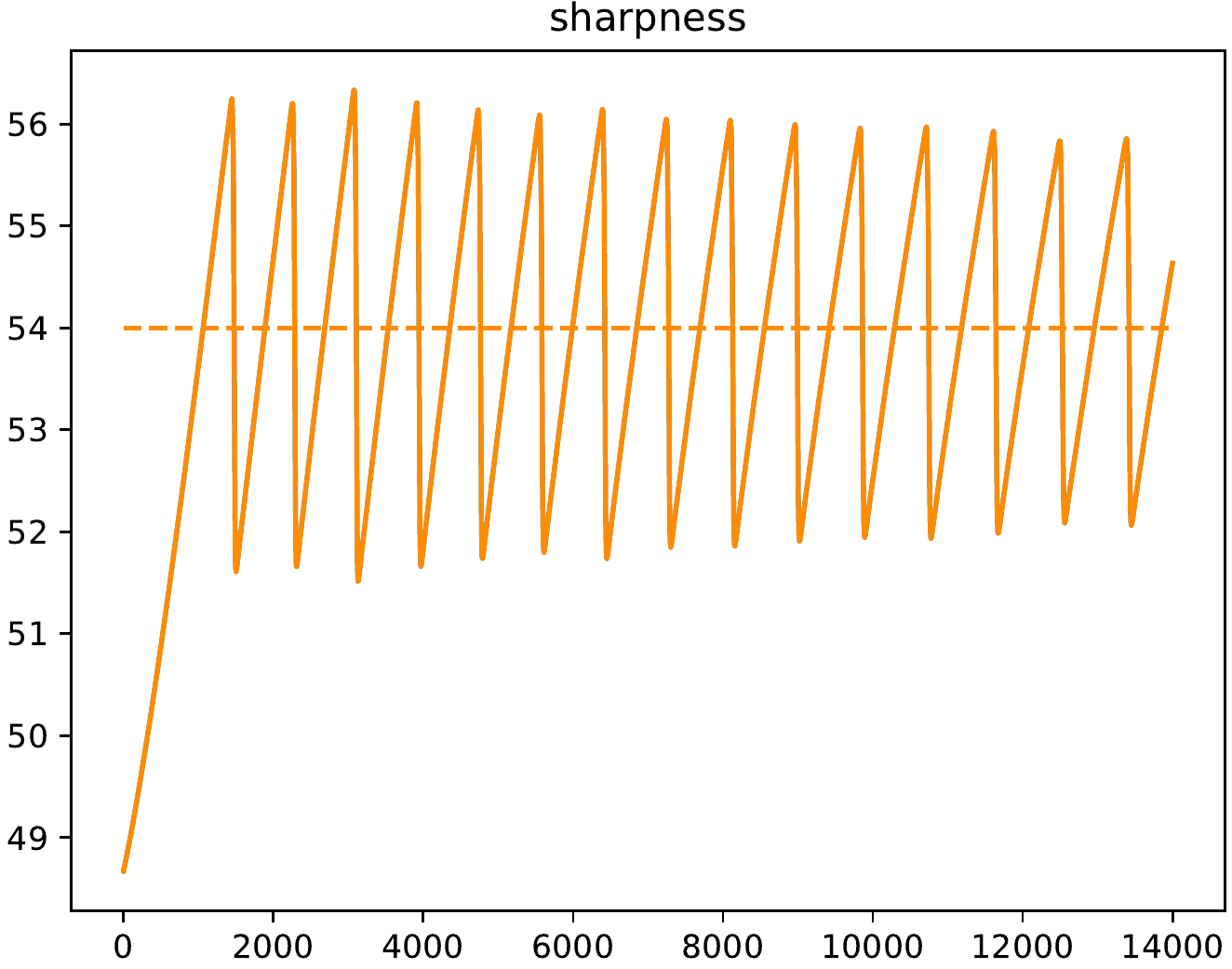}
    }
    \subfigure[$\|\bm\Gamma(t)\|$, $m=40, 80, 160, 200$]{
    \includegraphics[width=0.245\textwidth]{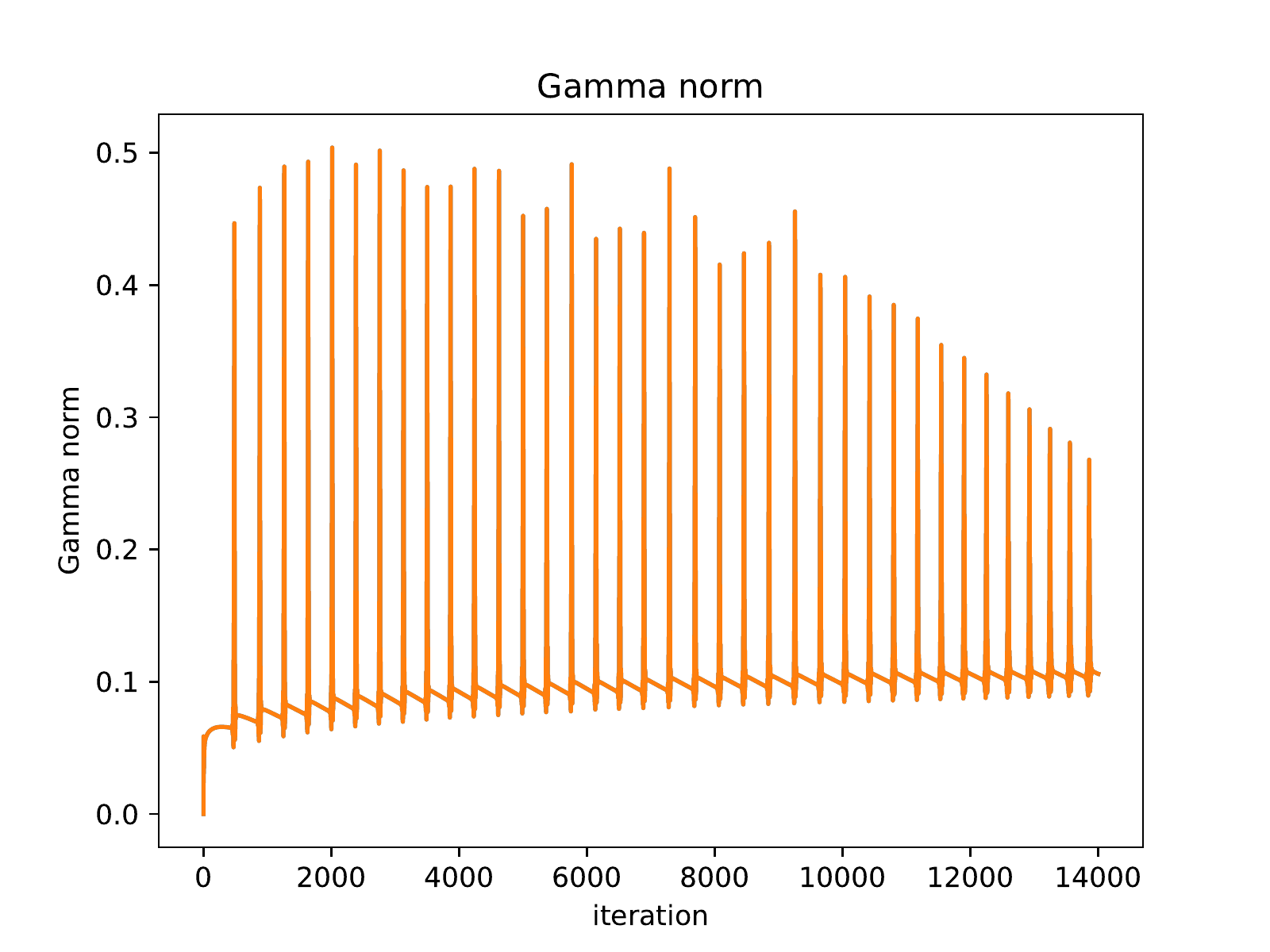}
    \includegraphics[width=0.245\textwidth]{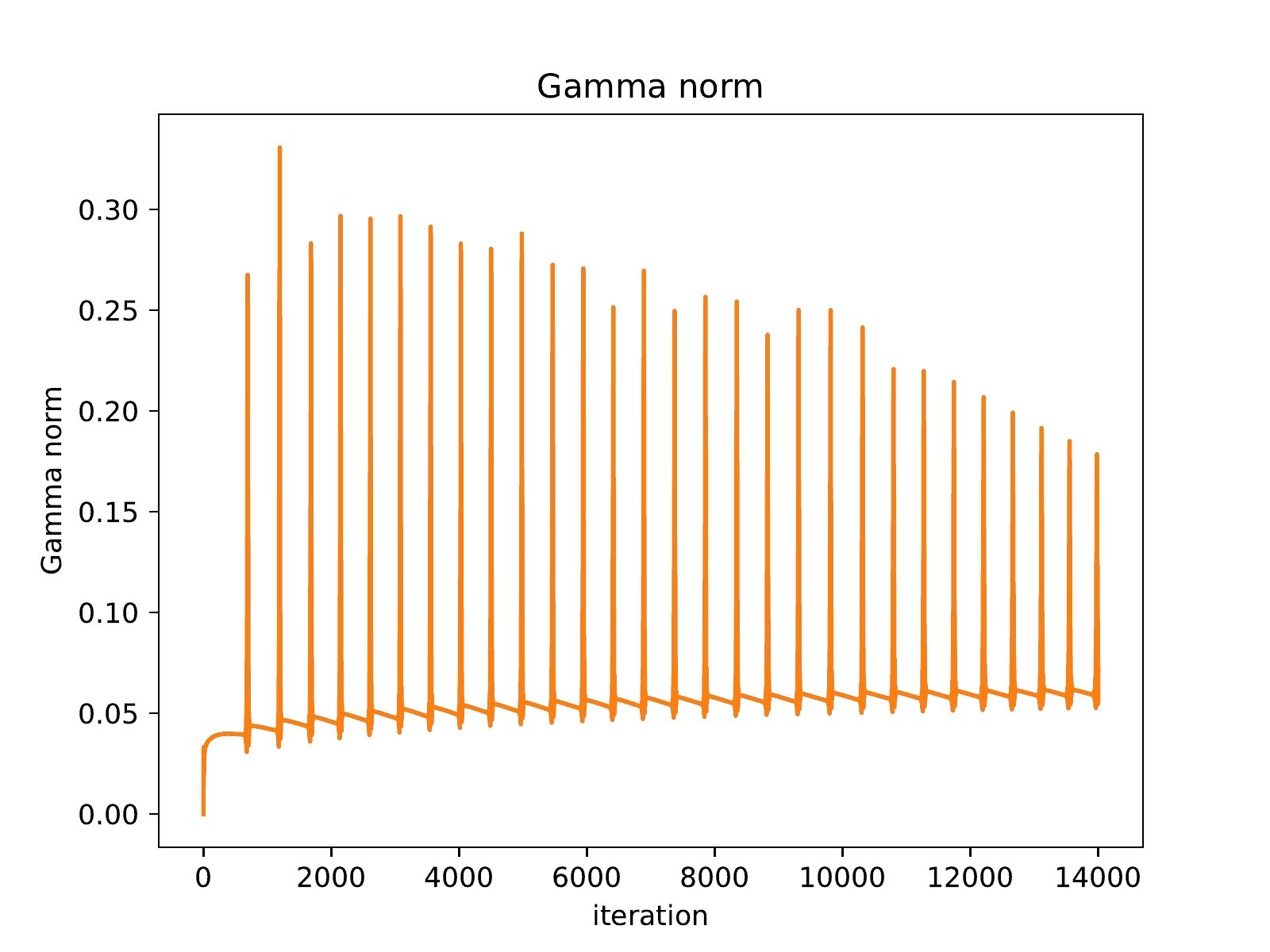}
    
    \includegraphics[width=0.245\textwidth]{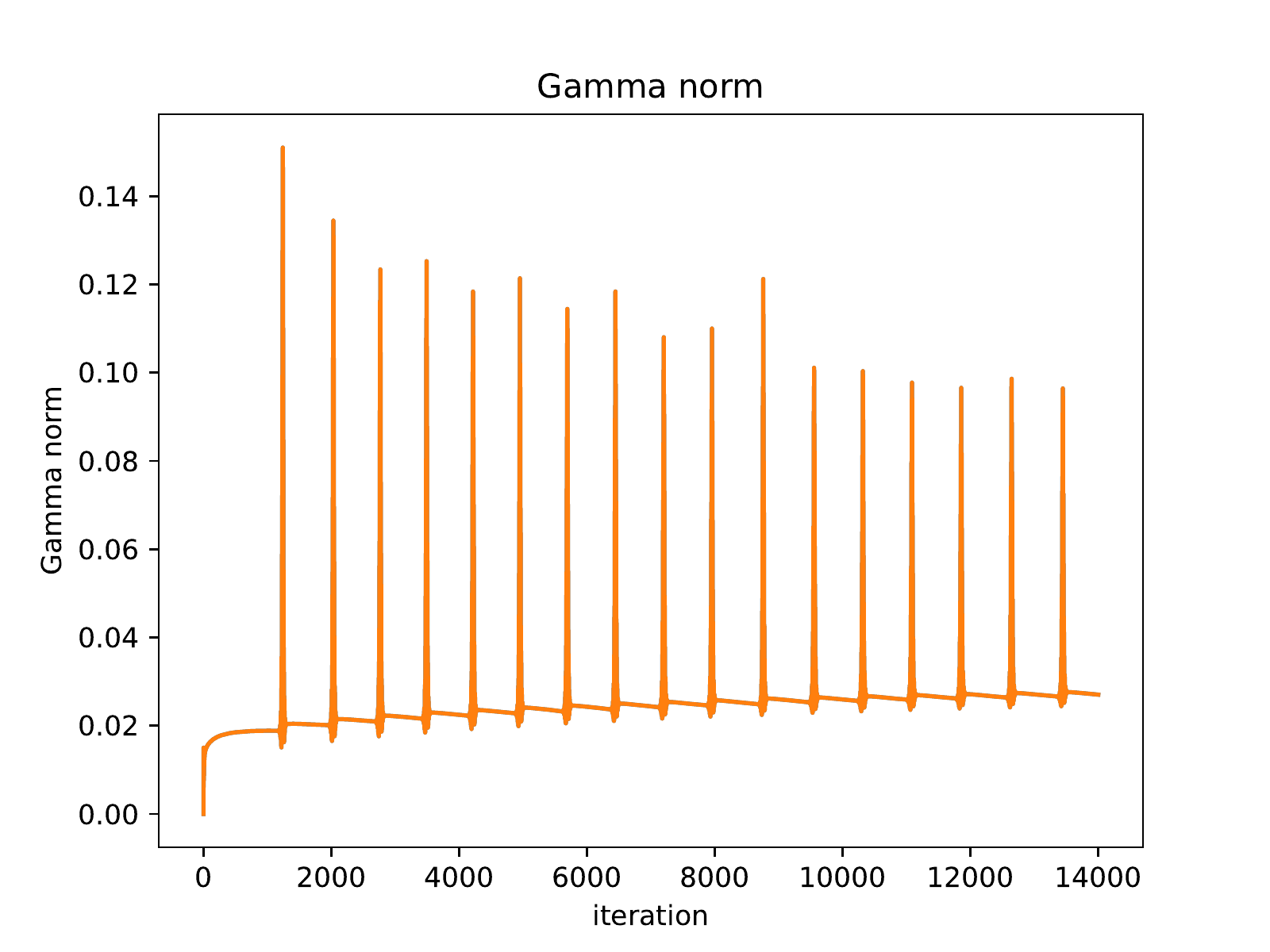}
    
    \includegraphics[width=0.245\textwidth]{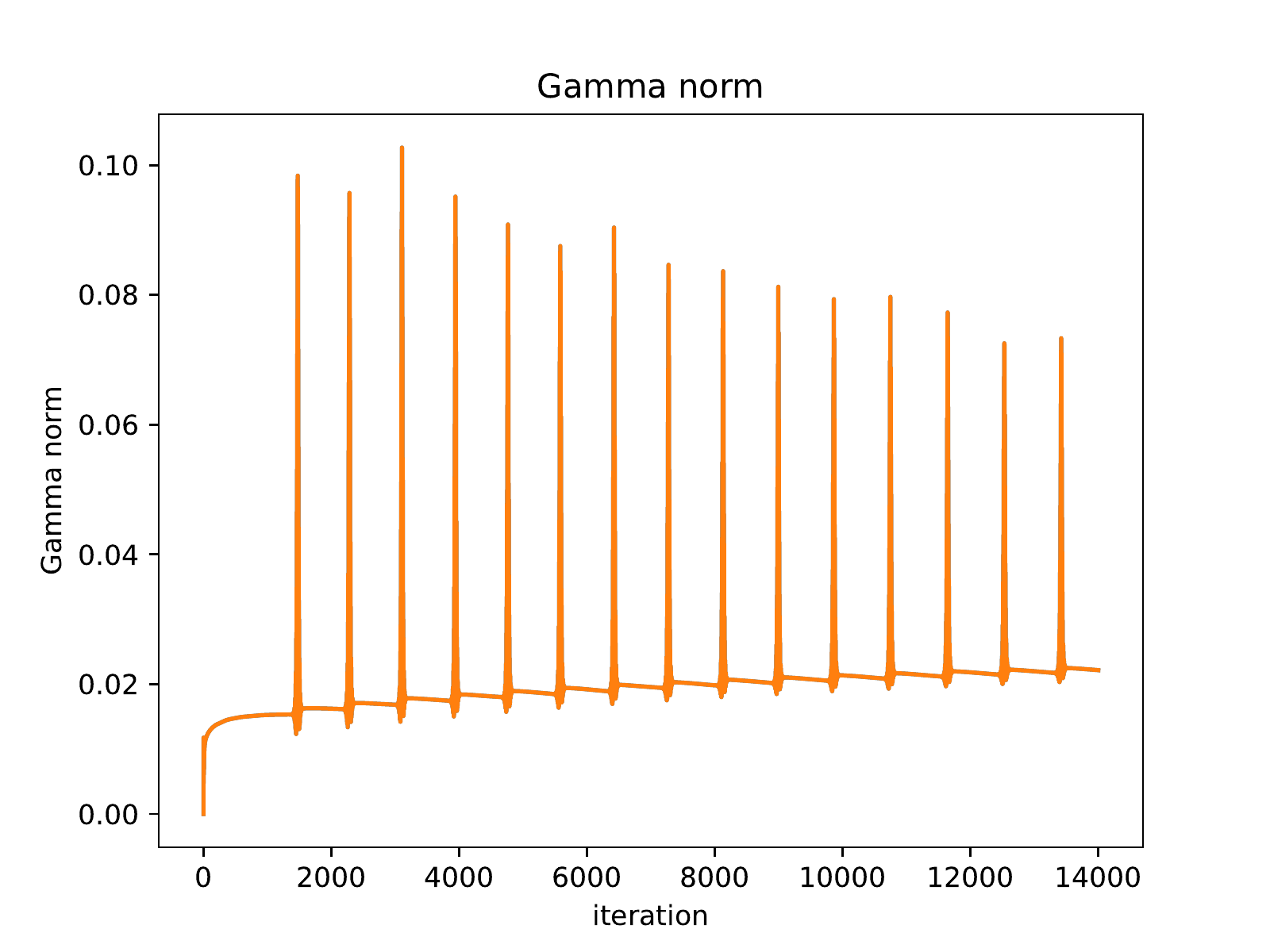}
    }
    \caption{In this figure, we show that even when GD enters EOS, the bound of $\|\bm\Gamma(t)\|$ is still $\Theta(1/m)$ along the whole trajectory. The corresponding constant $c_2\approx 24$.
    }
    \label{gamma assumption}
\end{figure}

Actually there are two interesting empirical facts related to $\|\bm\Gamma(t)\|$. One is the noticeable fact that $\|\bm\Gamma(t)\|$ spikes, the other is that the values of 
$\|\bm\Gamma(t)\|$ are overall quite small (despite the fact that 
$\bm W$
 changes non-trivially in our experiments) and decreases as 
$m$
becomes larger. Our assumption tries to model the second fact (see Figure \ref{gamma assumption}, 
$\|\bm\Gamma(t)\|$
 (despite the spikes) decreases as the width 
$m$
 grows, and the largest 
$\|\bm\Gamma(t)\|$
 is almost 
$24/m$ in all these experiments).
However, we admit that our results do not reflect the first fact (the spikes of 
$\|\bm\Gamma(t)\|$
), and it is an interesting fact that is worth investigating. We have strong intuition that $\|\bm \Gamma(t)\|$
 only grows by at most a constant factor, but currently do not have a formal proof yet. Nevertheless, the spiking behavior does not directly contradict our assumption.

\textbf{Remark: }Note that Assumption~\ref{Lambda assumption} is not equivalent to a small movement of $\bW(t)$. Actually, in the experiments above ($m=40, 80, 160, 200$) the movement of $\bW(t)$ is quite significant compared to the norm of $\bW(0)$ at initialization. See Figure~\ref{W(t) actually changes much}.

\begin{figure}[H]
    \centering
    \subfigure[$\bW(t)\&\|\Delta \bW(t)\|$, $m=40$]{
    \includegraphics[width=0.4\textwidth]{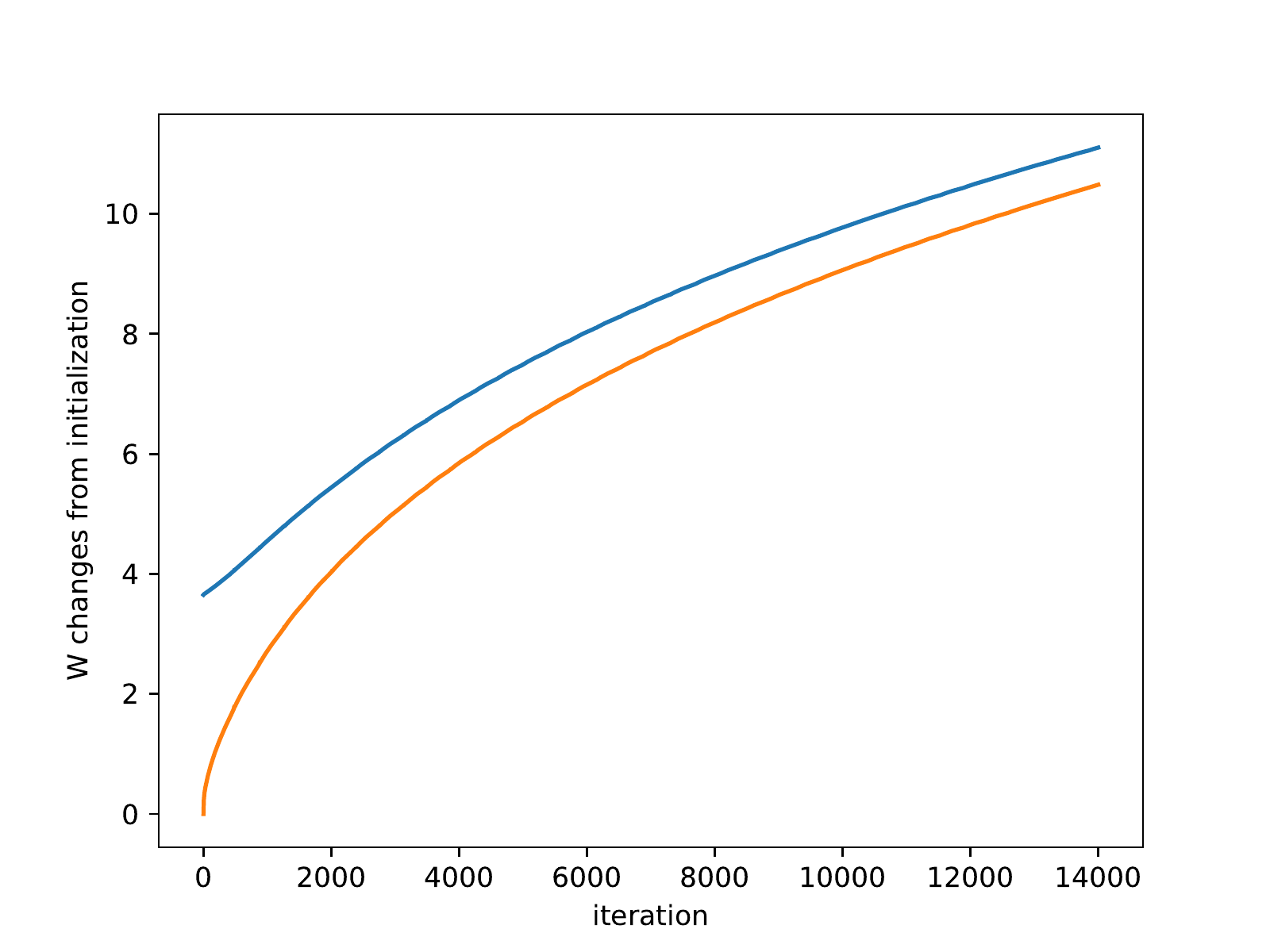}
    }
    \subfigure[$\bW(t)\&\|\Delta \bW(t)\|$, $m=80$]{
    \includegraphics[width=0.4\textwidth]{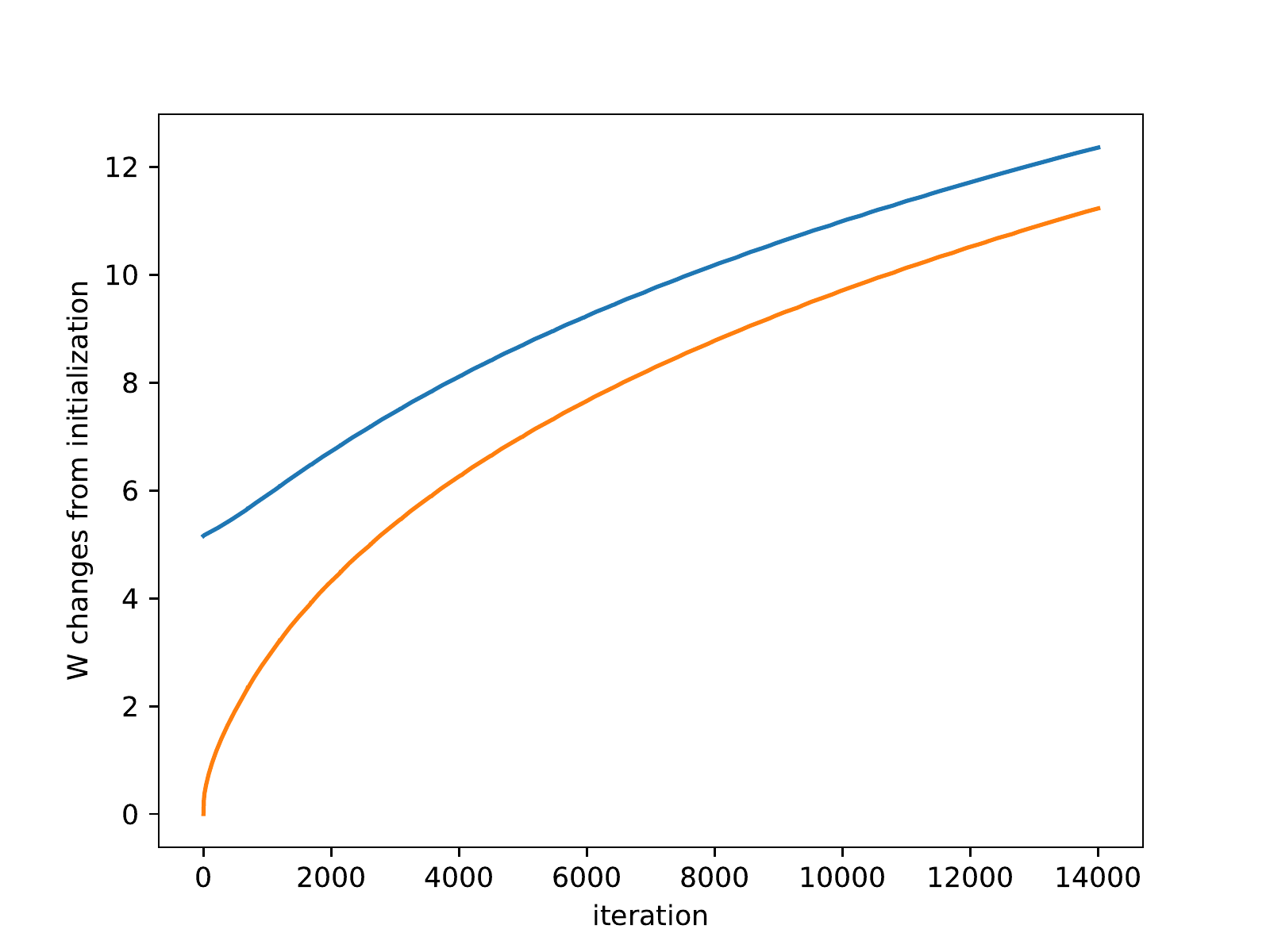}
    }
    \subfigure[$\bW(t)\&\|\Delta \bW(t)\|$, $m=160$]{
    \includegraphics[width=0.4\textwidth]{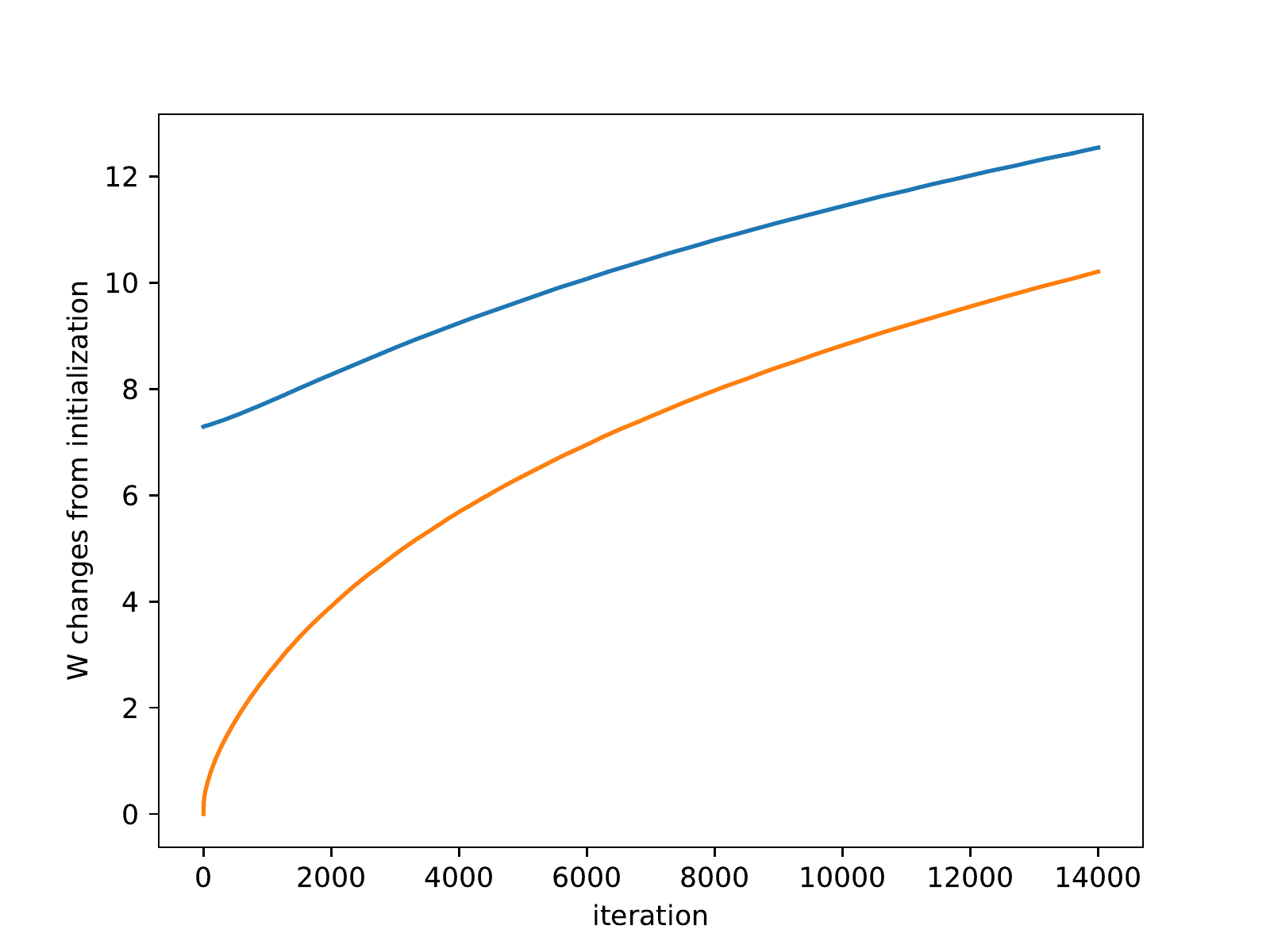}
    }
    \subfigure[$\bW(t)\&\|\Delta \bW(t)\|$, $m=200$]{
    \includegraphics[width=0.4\textwidth]{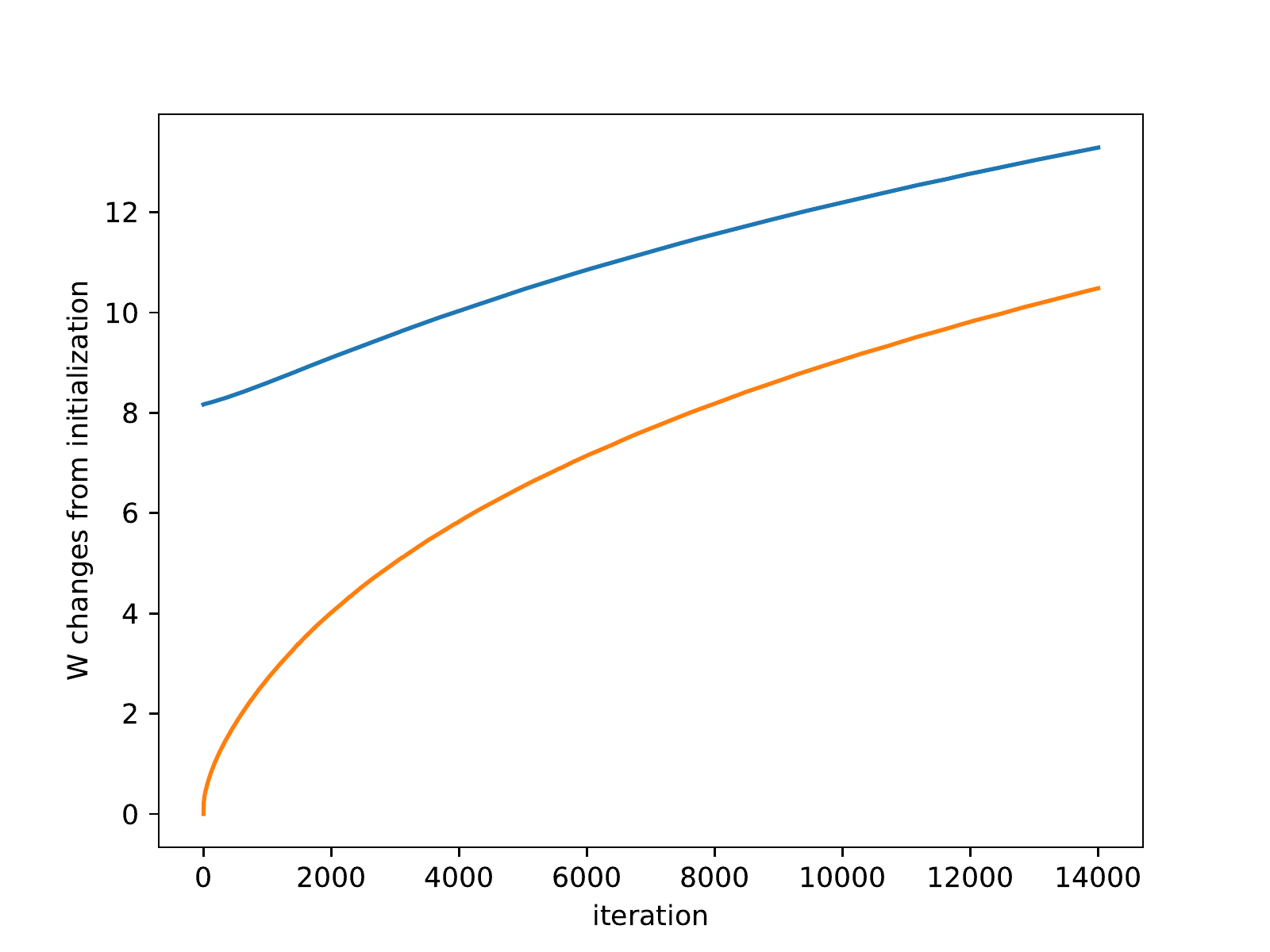}
    }
    \caption{In this figure, we show that even the $\bm \Gamma (t)$ is small and bounded, the movement $\|\Delta\bW(t)\|:=\|\bW(t)-\bW(0)\|$ (the dark orange curve) is close to $\|\bW(t)\|$ (the blue curve), implying that $\bW(t)$ moves considerably.
    }
    \label{W(t) actually changes much}
\end{figure}

\subsubsection{Comparison with the NTK regime: the non-quadratic property}\label{Appendix.D.2.3 Compare with NTK}
Here we illustrate why our setting and results are sufficiently different from 
the quadratic setting (e.g., linear regression) or the recent convergence analysis in
NTK setting. In particular, we show that even when the movement of $\bW(t)$ is comparably negligible compared to the initialization $\bW(0)$, EOS can still happen. Here we take a larger initialization of $\bW(0)$, which is ten times of the standard initialization in order to dwarf the movement of $\bW(t)$. The widths are $m=1000, 2000, 4000, 8000$. We can see that the initialization $\bW(0)$ is 
much larger than the change of $\bW(t)$ and the norm of $\bW(t)$ grows larger when $m$ becomes larger. See Figure~\ref{ntk compare}. Detailed comparison is in Appendix~\ref{our result and ntk}.
\begin{figure}[ht]
    \centering
    \subfigure[The sharpness, $m=1000 (\eta = 2/54), 2000 (\eta = 2/52), 4000 (\eta = 2/53), 8000 (\eta = 2/52)$]{
    \includegraphics[width=0.245\textwidth]{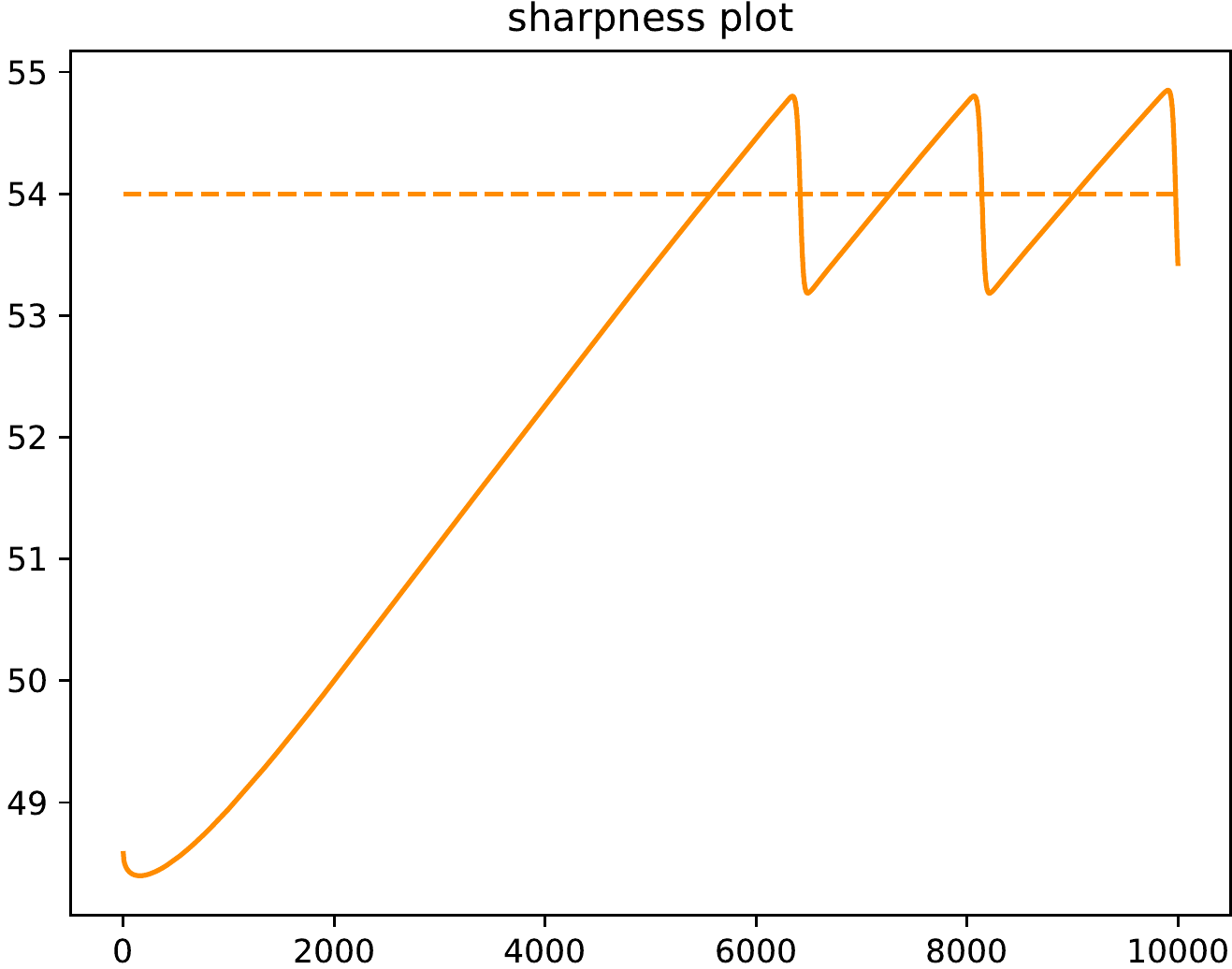}
    
    \includegraphics[width=0.245\textwidth]{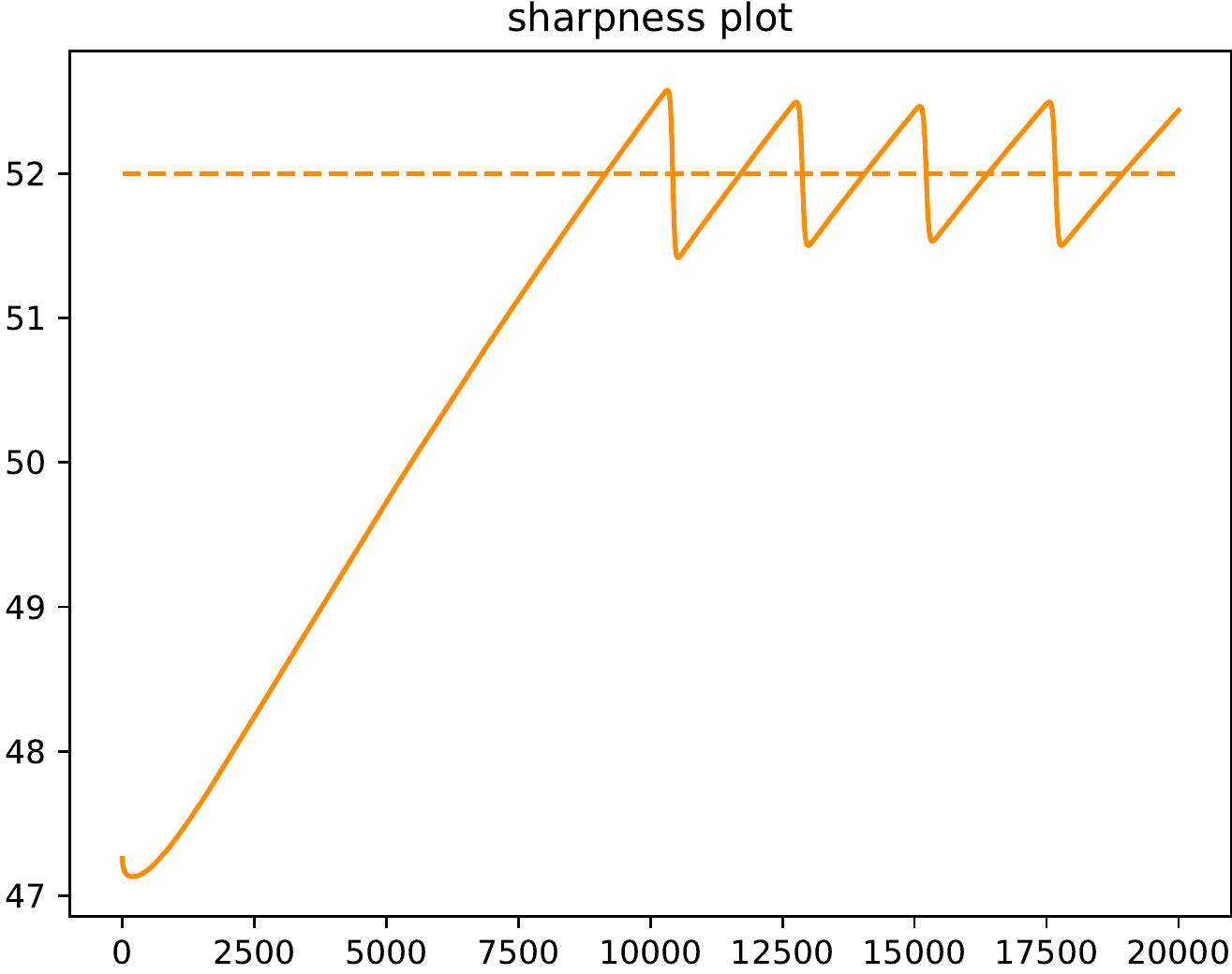}
    
    \includegraphics[width=0.245\textwidth]{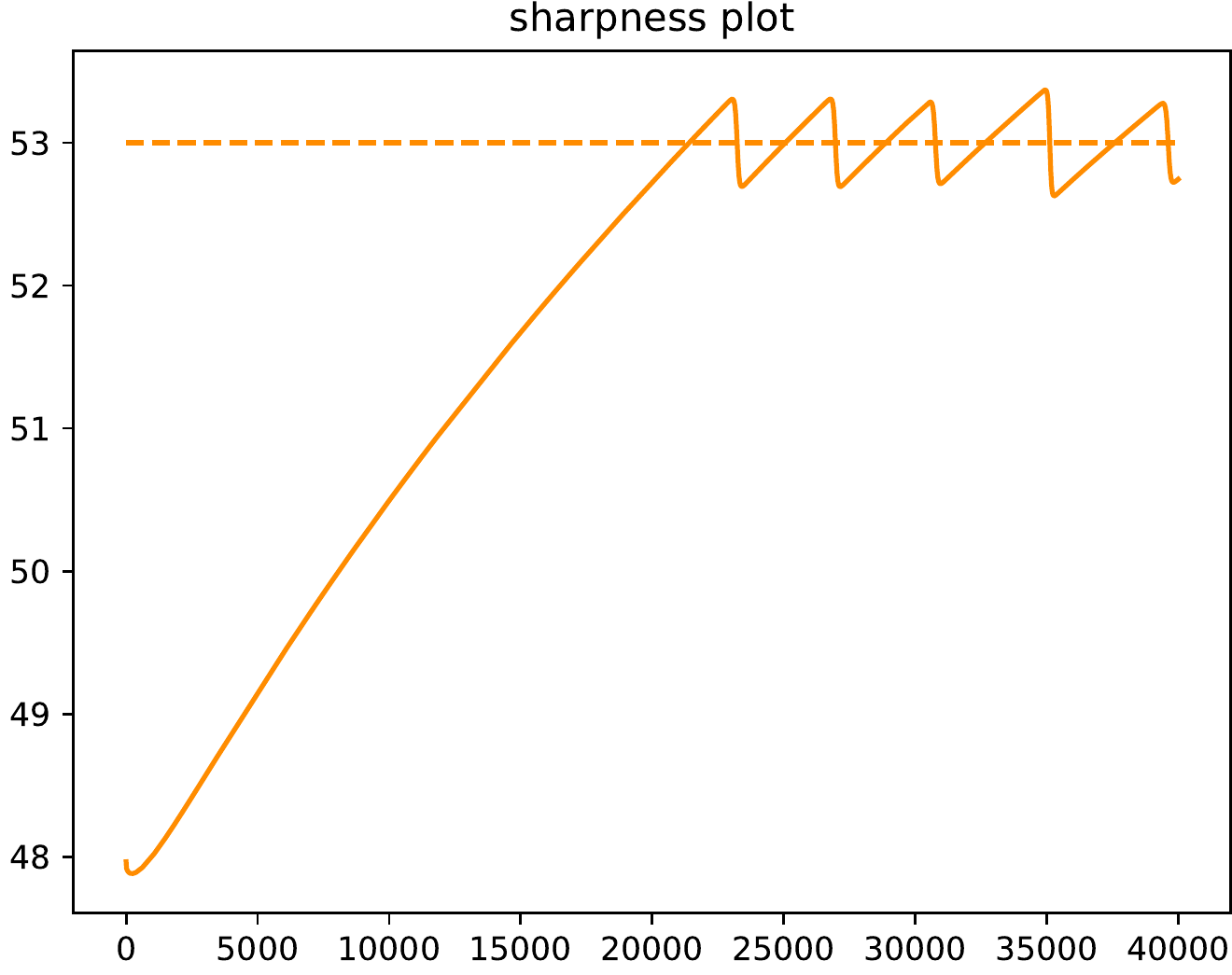}
    
    \includegraphics[width=0.245\textwidth]{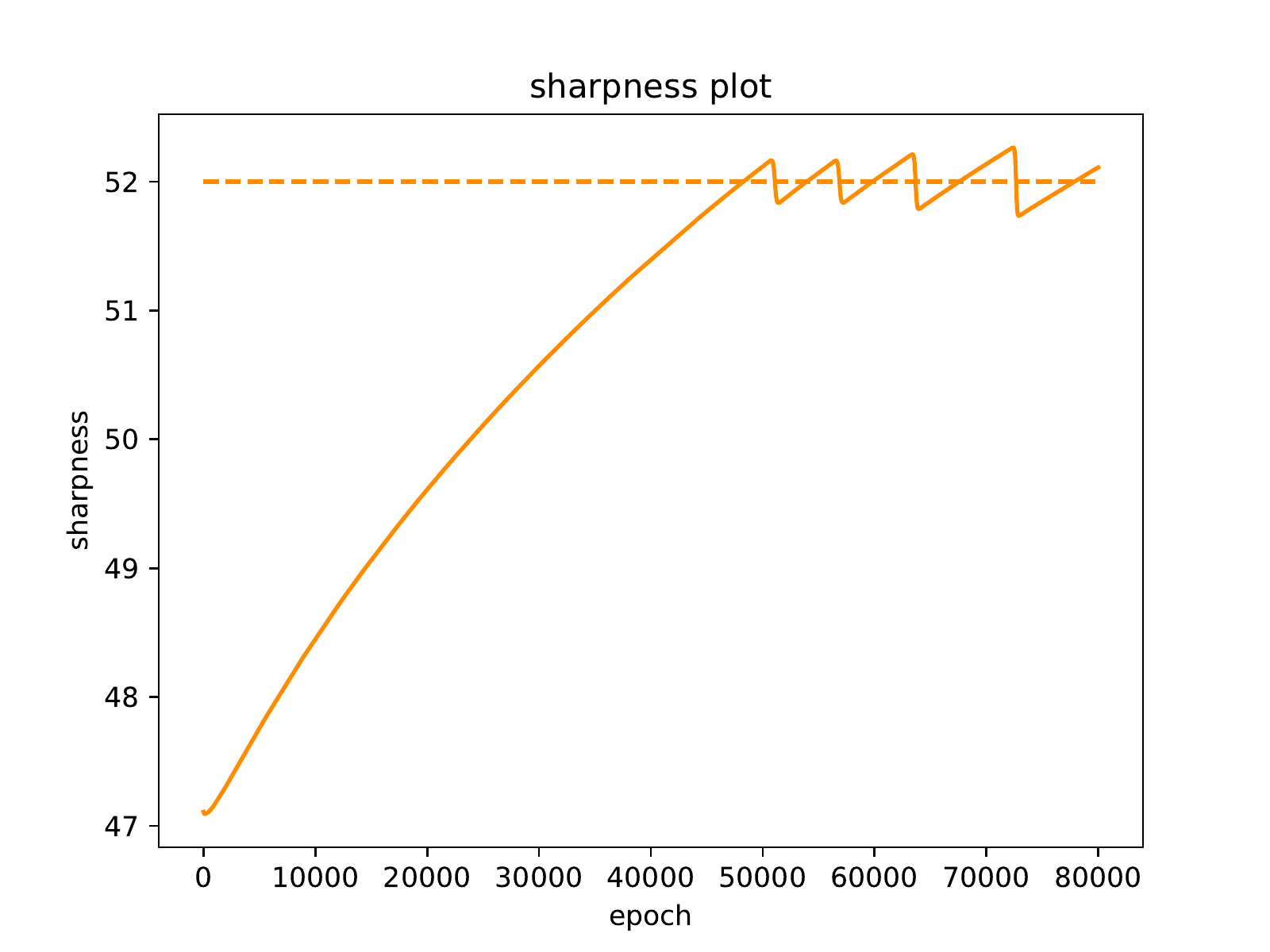}
    }
    \subfigure[The loss, $m=1000, 2000, 4000, 8000$]{
    \includegraphics[width=0.245\textwidth]{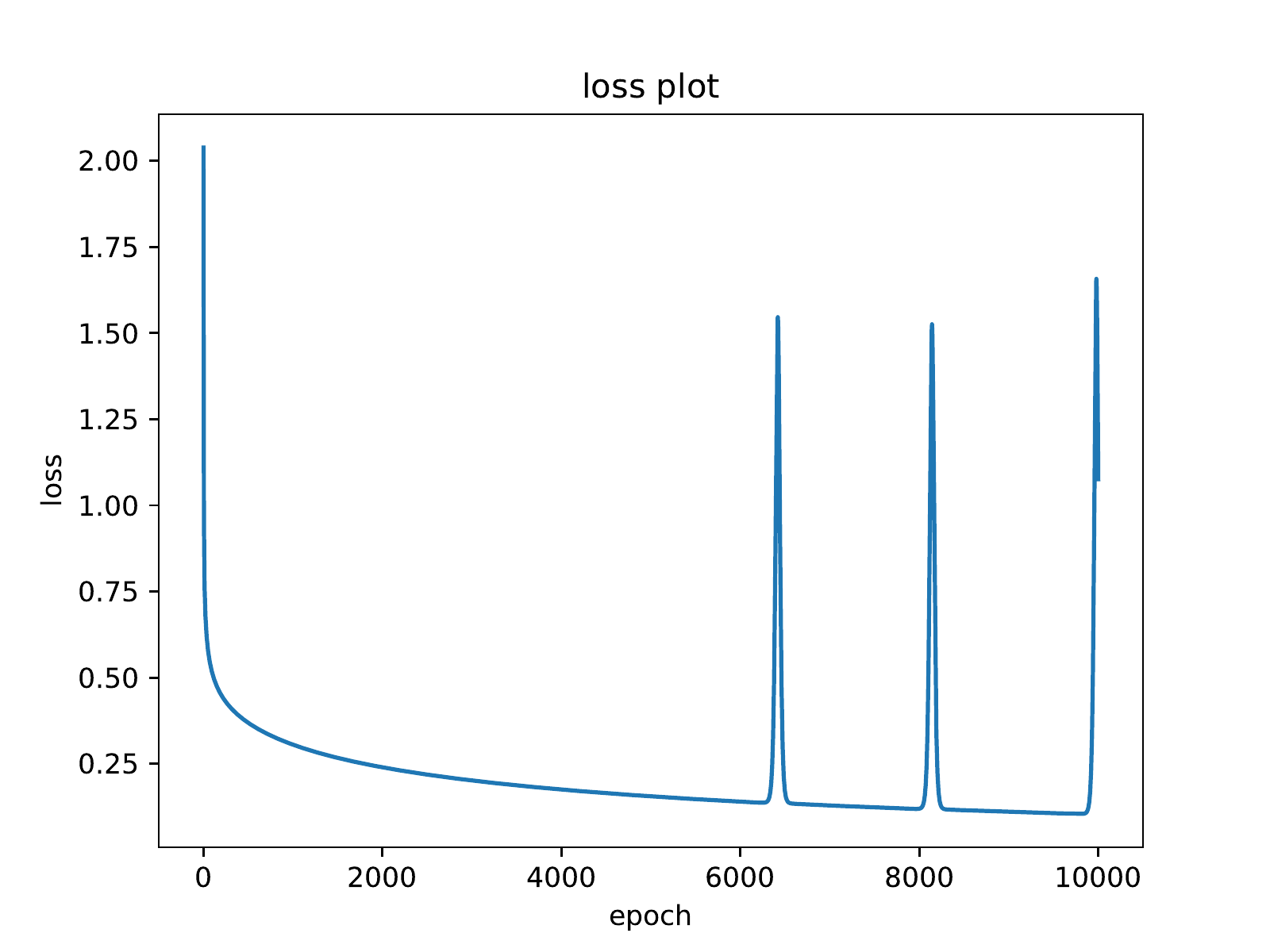}
    
    \includegraphics[width=0.245\textwidth]{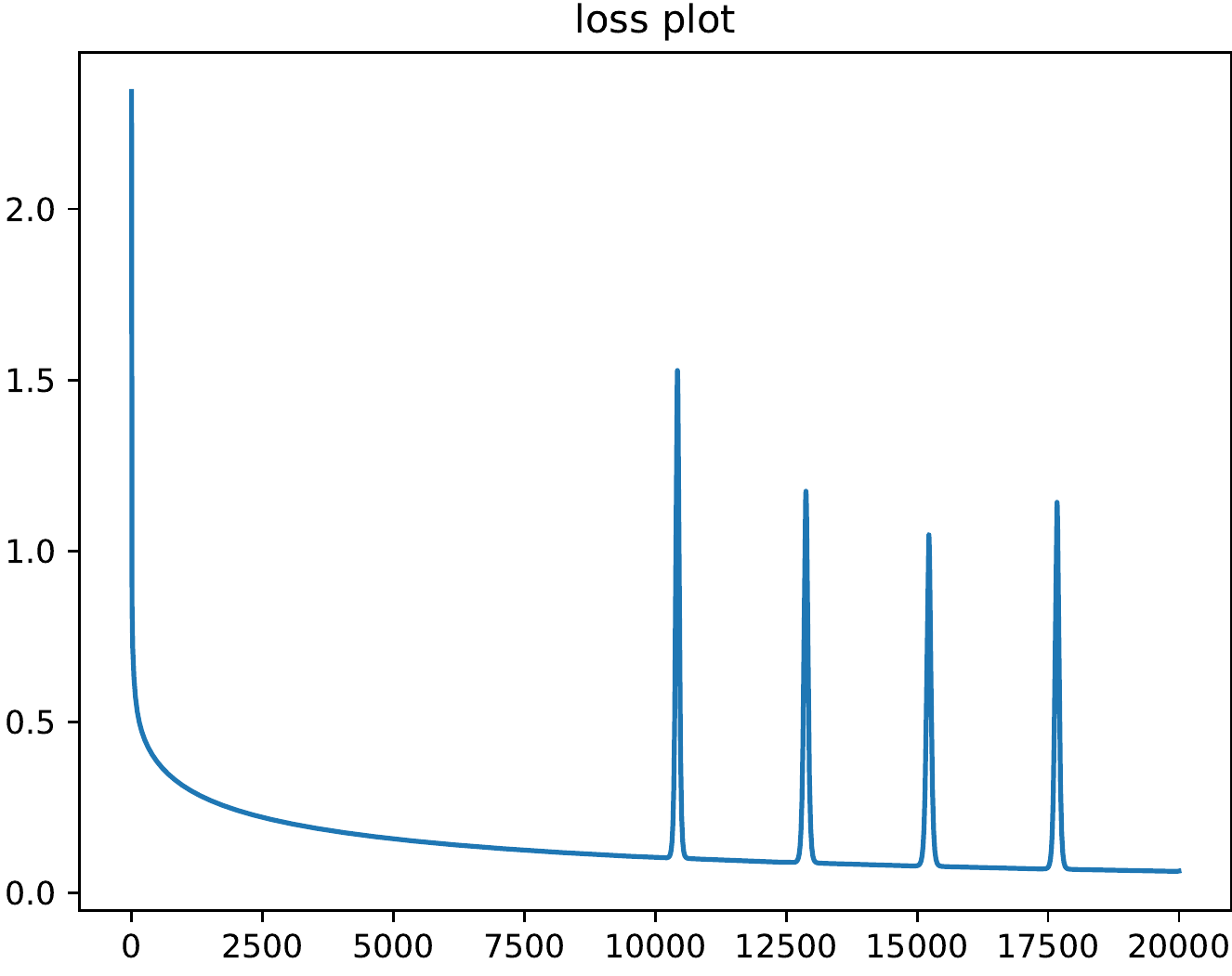}
    
    \includegraphics[width=0.245\textwidth]{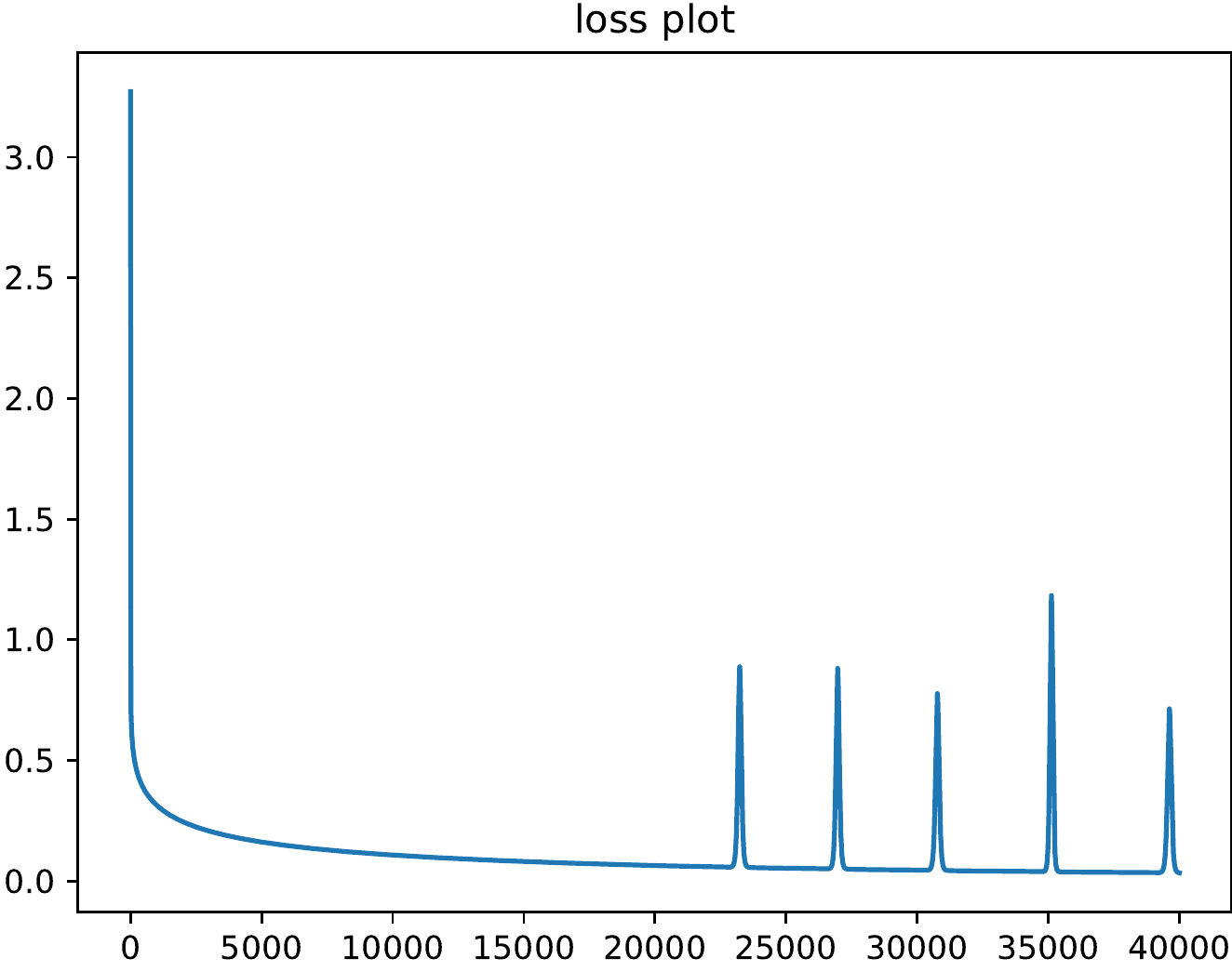}
    
    \includegraphics[width=0.245\textwidth]{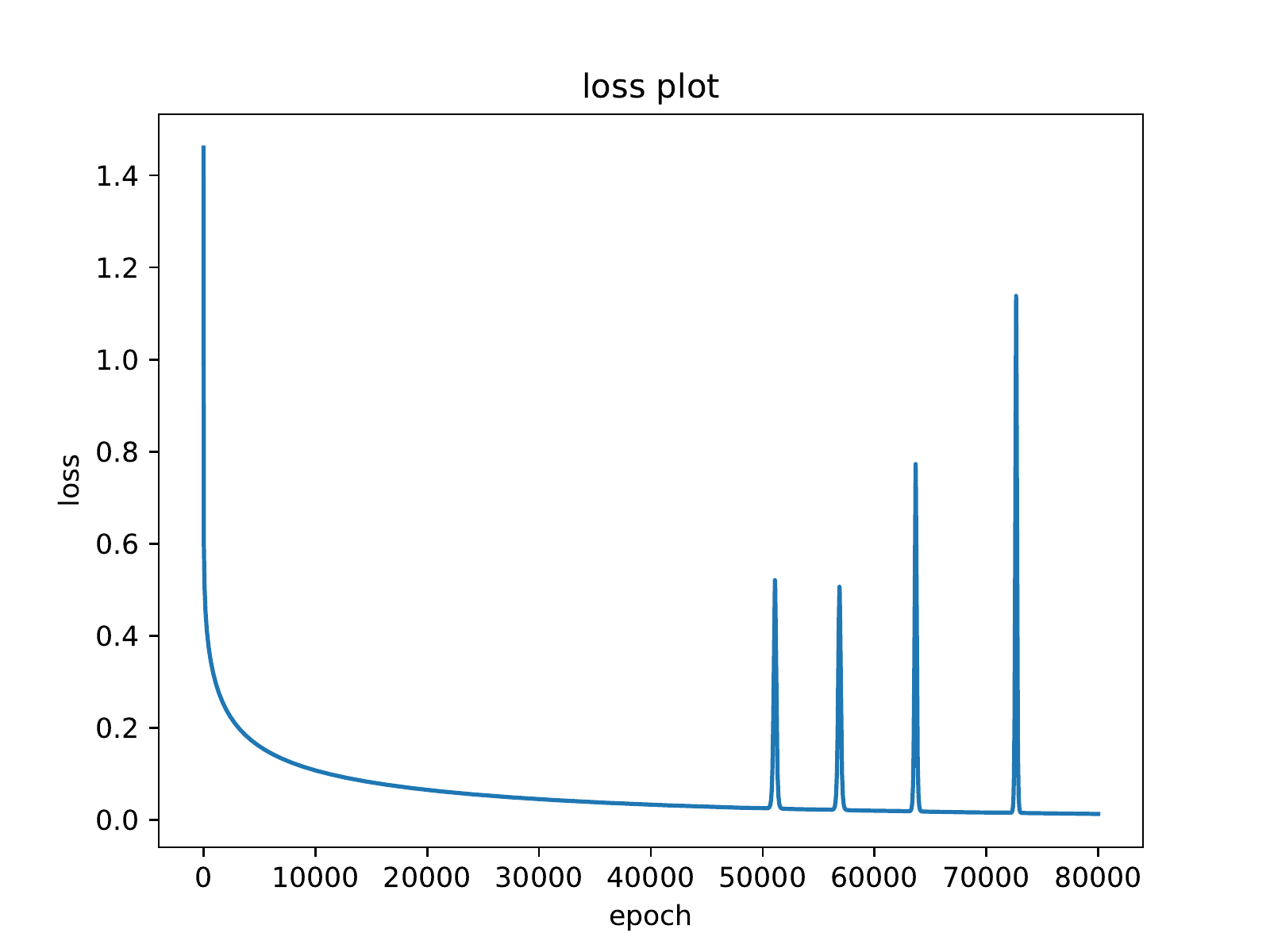}
    }
    \subfigure[$\bW(t)\&\|\Delta \bW(t)\|$, $m=1000, 2000, 4000, 8000$]{
    \includegraphics[width=0.245\textwidth]{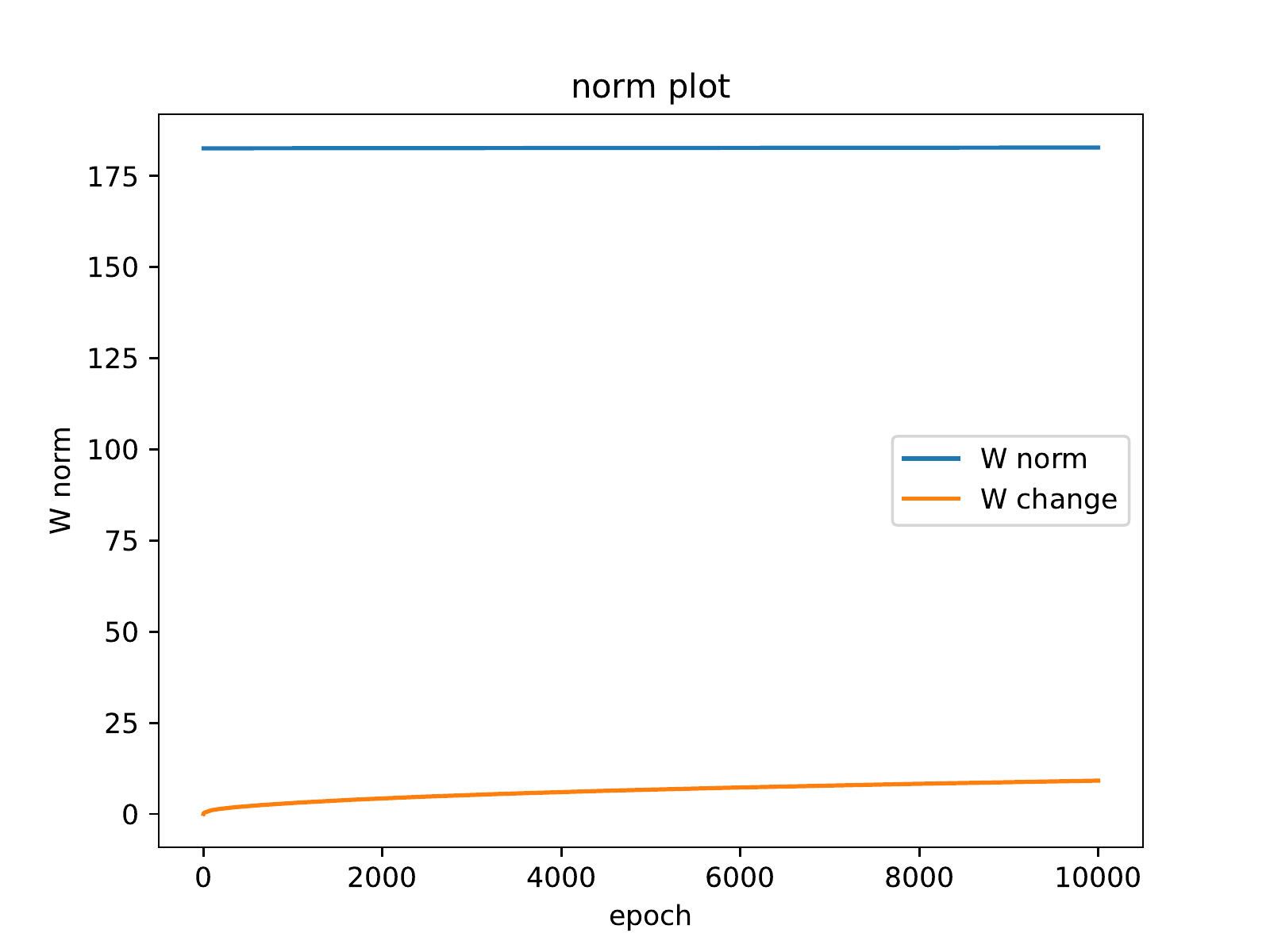}
   
    \includegraphics[width=0.245\textwidth]{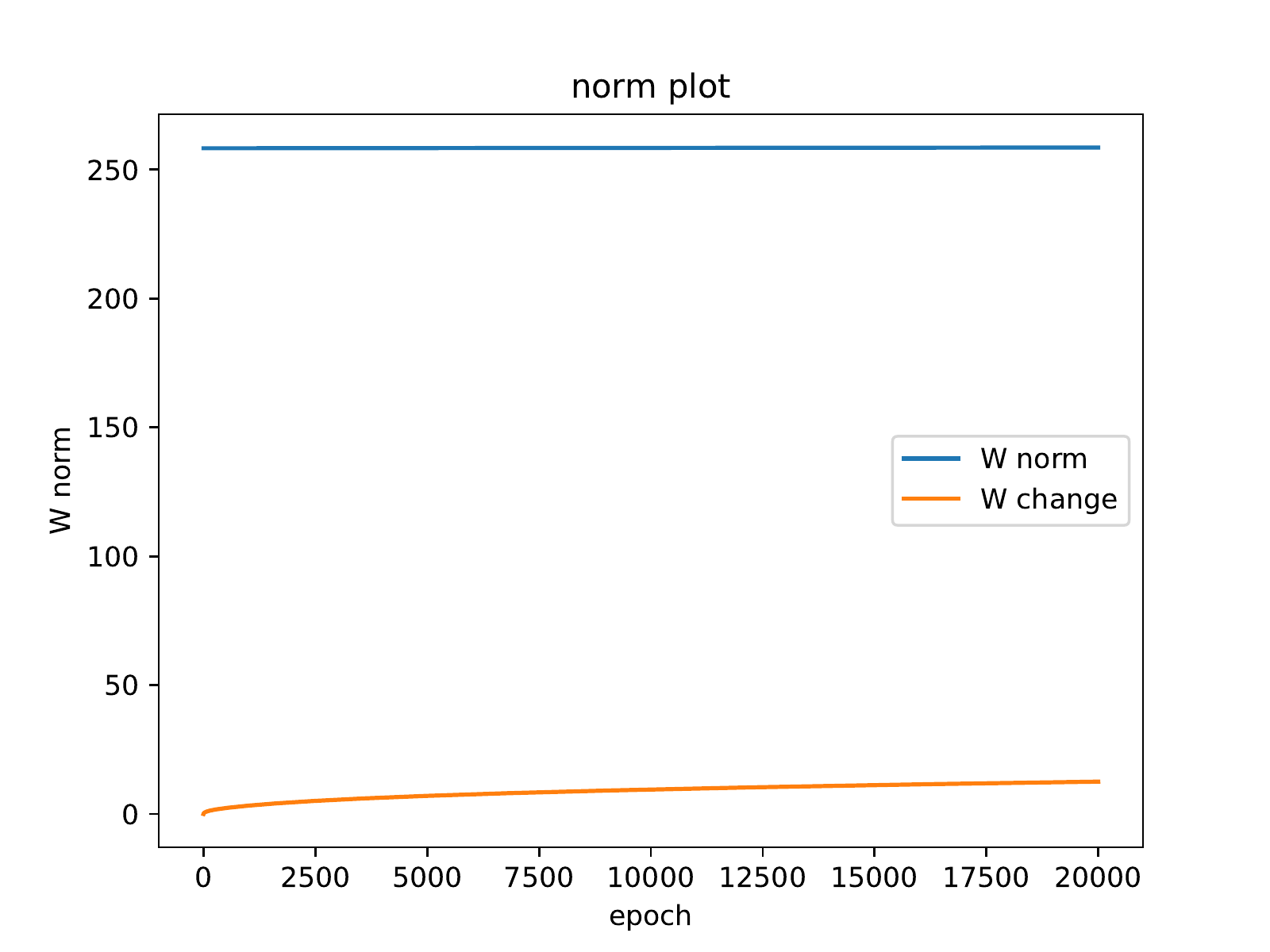}
    
    \includegraphics[width=0.245\textwidth]{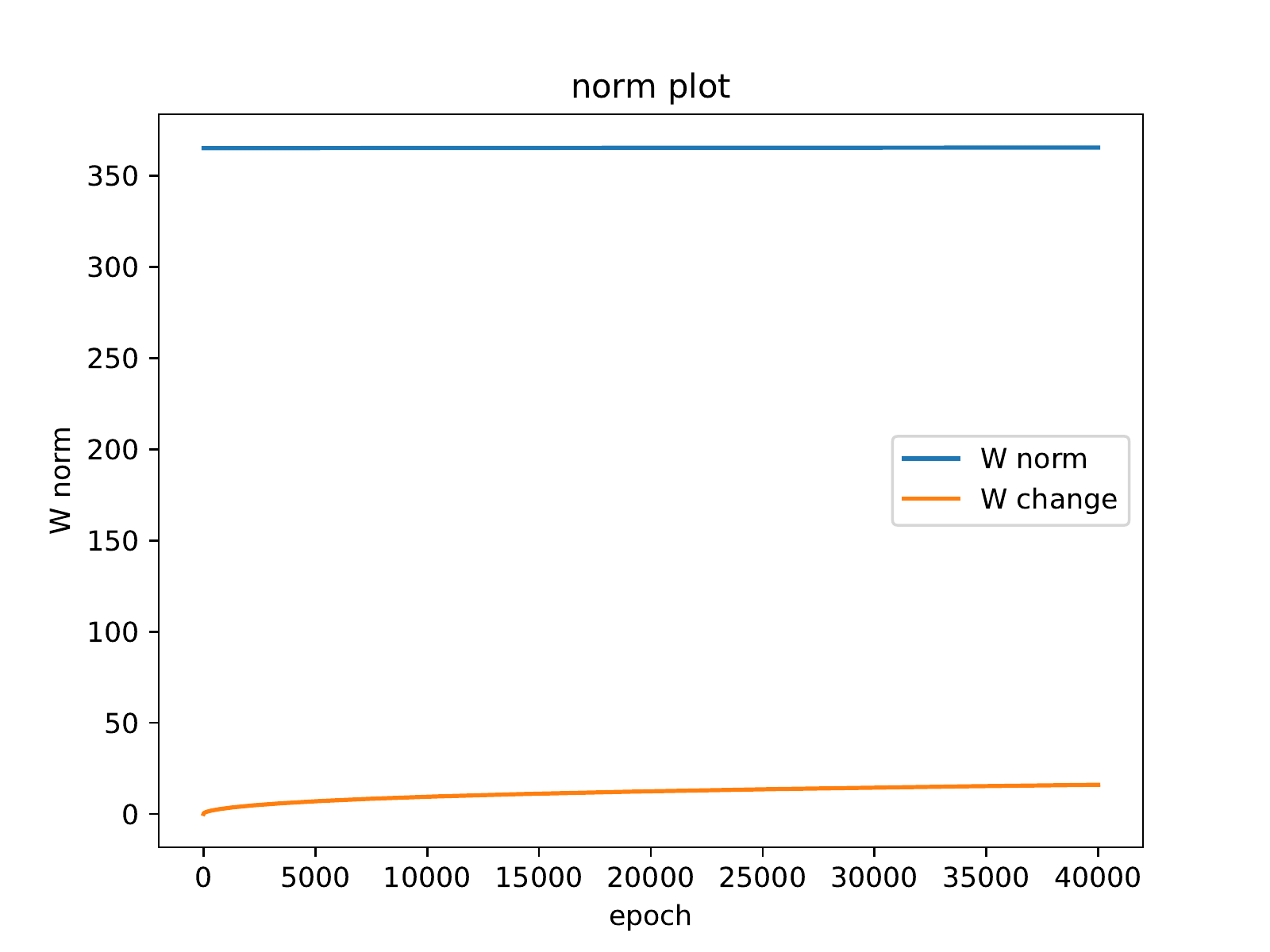}
   
    \includegraphics[width=0.245\textwidth]{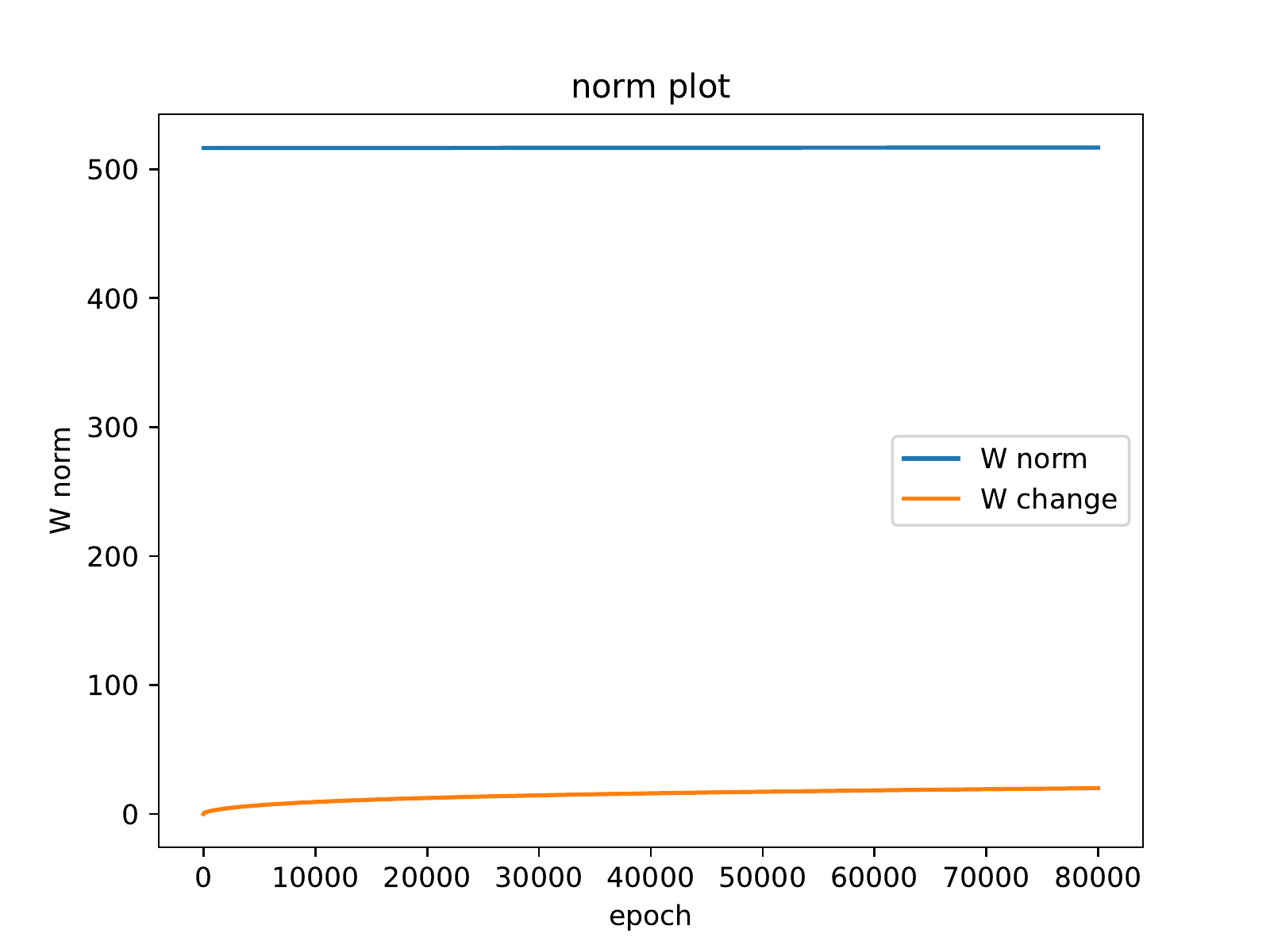}
    }
    \caption{In this figure, we show that even when $\bW(t)$ stays close to its initialization, EOS can still happen. In particular, Figure (c) shows the $\|\bW(t)\|$ (the blue curve) and $\|\bW(t)-\bW(0)\|$ (the orange curve).
    Here for $m=1000, 2000, 4000, 8000$, observe that the EOS still happens when $\|\bW(t)\|/\|\bW(t)-\bW(0)\|$ > 10. That indicates the intrinsic non-quadratic property of neural networks in the EOS regime.}
    \label{ntk compare}
\end{figure}

\subsection{Experiments in Section~\ref{5}}
\label{appendix section: discussion experiment}
In this subsection we show some empirical results that may reveal some interesting relation between the inner layers and the sharpness, which is not yet reflected in our theory.
Our experimental results show that all layers seem to work together to influence the sharpness, and contribute to the progressive sharpening and edge of stability phenomena. 
In our experiment, while the sharpness is still calculated by the gradient of all parameters, we freeze some of the layers in the training process. 
The experimental results show that the more layers we freeze, the slower progressive sharpening happens and the weaker the oscillation of the sharpness is (See Figure~\ref{freeze tanh Pics}, Figure~\ref{Freeze linear pics}). 
This indicates that layers other than the output layer has nontrivial influence on the sharpness,
but since different layers seem to work in the same direction, further justifying our assumption that $\|\bA\|$ is positively related with the sharpness.

\textbf{Fully-connected tanh Network:}
\begin{figure}[H]
    \centering
    \includegraphics[width=0.5\textwidth]{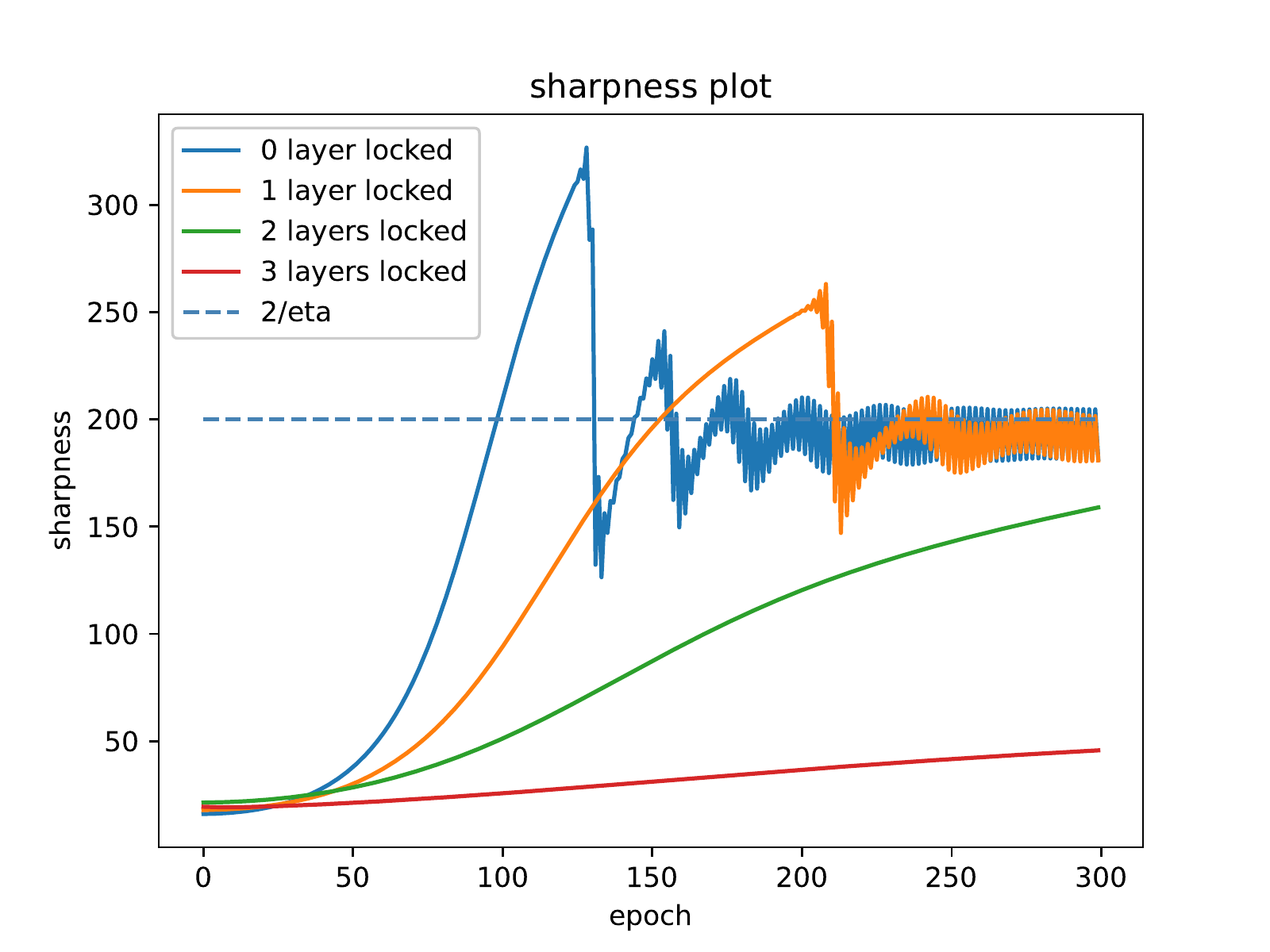}
    \caption{We conduct the experiment on a 5-layer tanh fully-connected network. 
    Four lines in the plot show the sharpness for four independent training processes but with 0,1,2,3 outer layers locked. All other hyper-parameters are the same. The result shows that all layers have cooperative effect on the sharpness.}
    \label{freeze tanh Pics}
\end{figure}

\textbf{Fully-connected Linear Network:}
\begin{figure}[H]
    \vspace{-0.5cm}
    \centering
    \subfigure[Fully-connected linear network 0 layer frozen]{
    \includegraphics[width=0.48\textwidth]{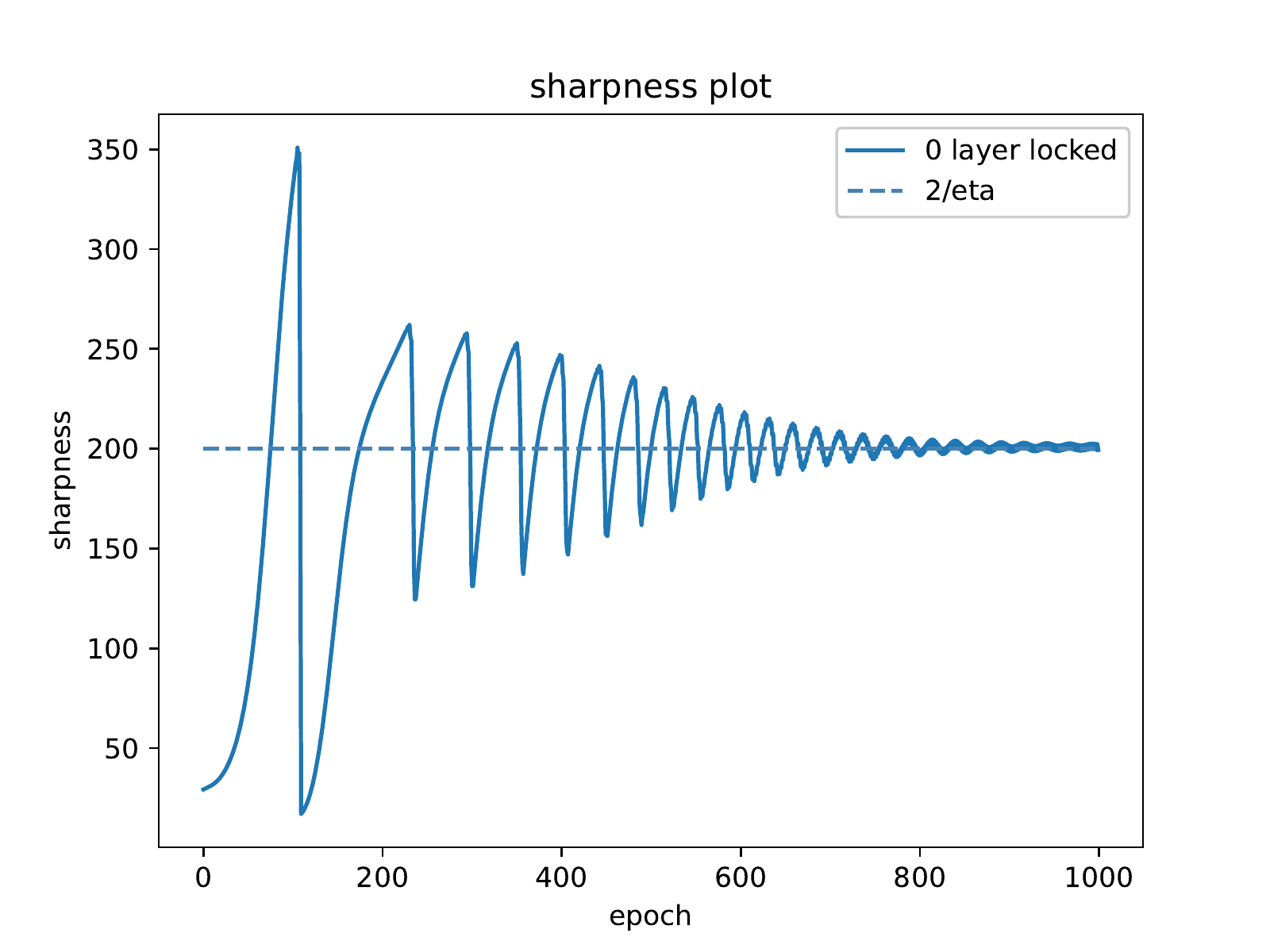}
    }
    \subfigure[Fully-connected linear network 1 layer frozen]{
    \includegraphics[width=0.48\textwidth]{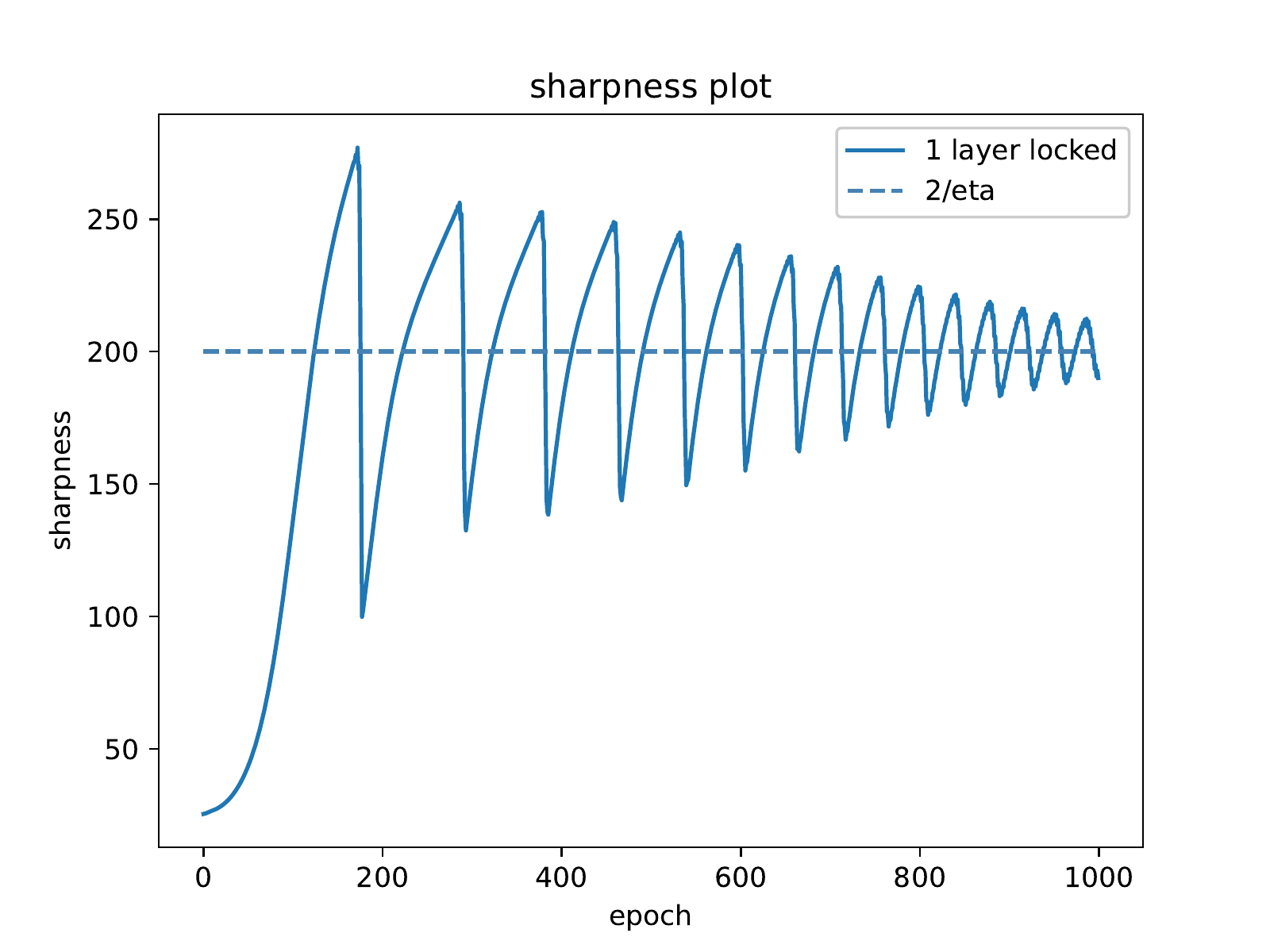}
    }
    \subfigure[Fully-connected linear network 2 layers frozen]{
    \includegraphics[width=0.48\textwidth]{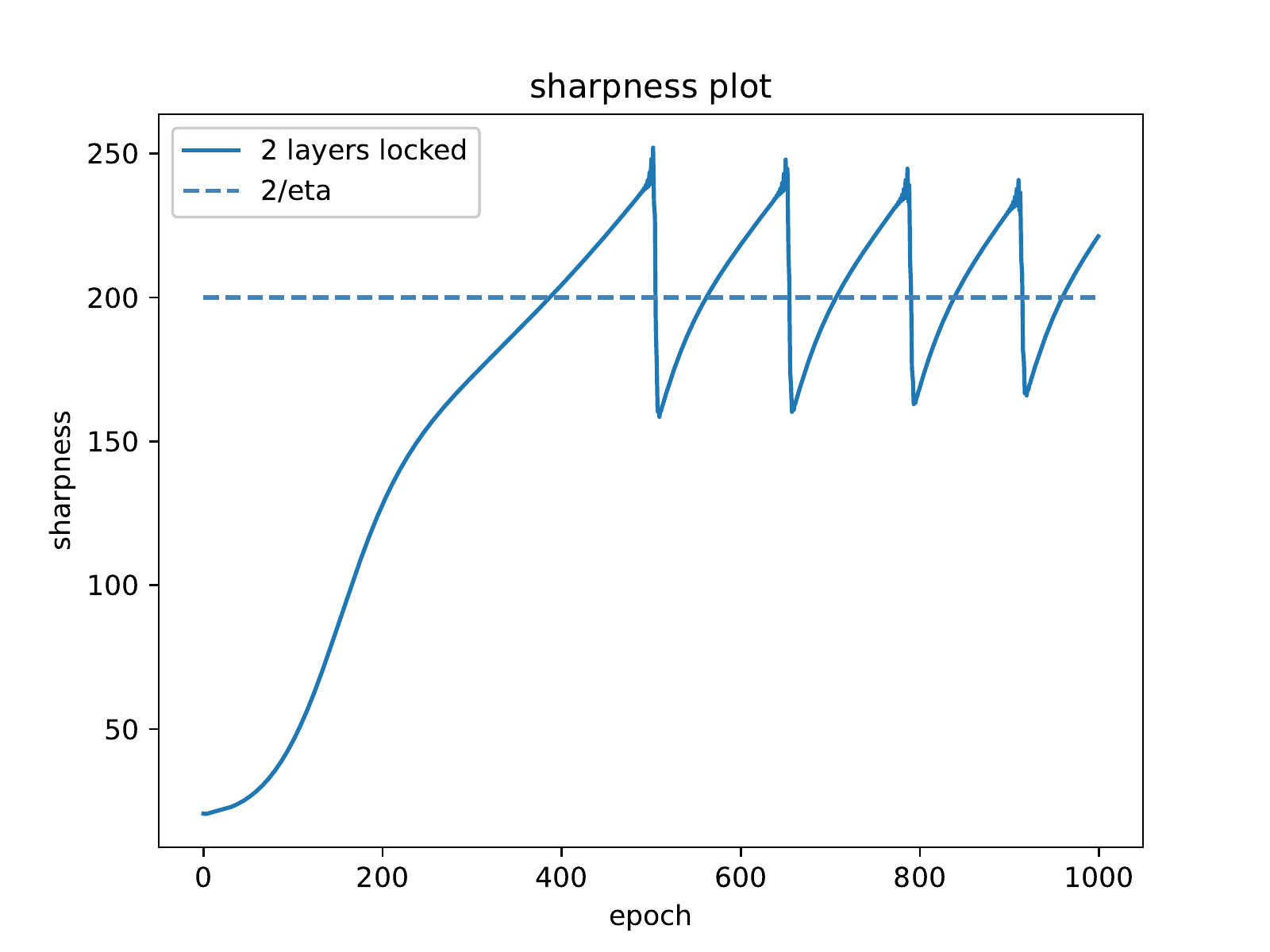}
    }
    \subfigure[Fully-connected linear network 3 layers frozen]{
    \includegraphics[width=0.48\textwidth]{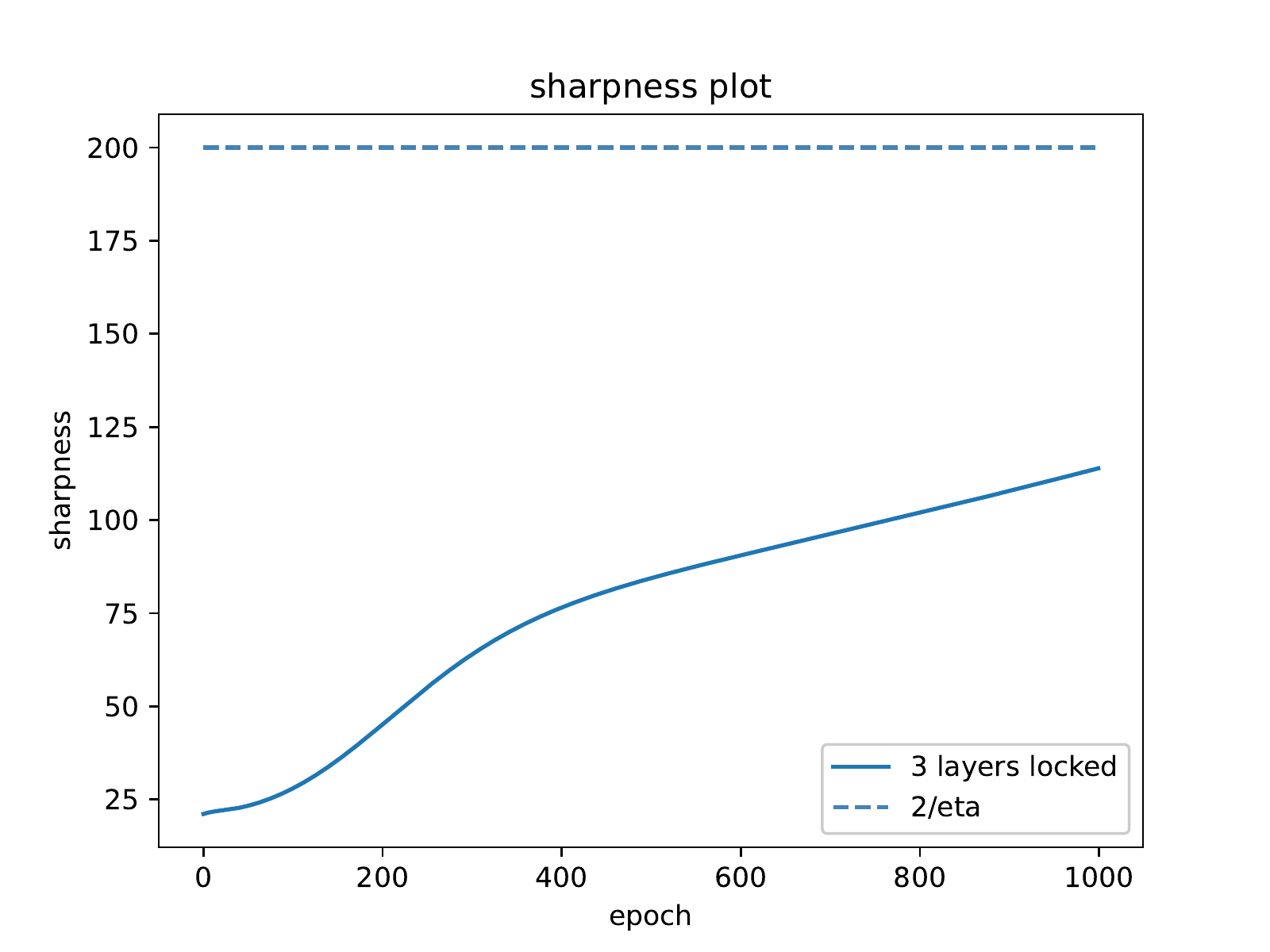}
    }
    \caption{The sharpness after 0,1,2,3 outer layers are locked. Details refer to Figure~\ref{freeze tanh Pics}.}
    \label{Freeze linear pics}
\end{figure}

\section{Missing proofs}\label{appendixE}

\subsection{Proofs in Section~\ref{Section 3}}
\label{appendix: proofs in sec3}
In this section, we provide the missing proofs in Section~\ref{Section 3}.

\begin{lemma}(Lemma~\ref{ps lemma})
For all $t$ in Phase I, 
under Assumption~\ref{Ass: eigspace assumptions} (i) and \ref{gf assumption},
it holds that $\bD(t)^\top\bF(t) < 0$.\label{lemma: appendix_ps lemma}
\end{lemma}

\begin{proof}
Consider the inner product $\bD(t)^\top\bm v_i$. 
Recall $\bm v_i$ is the $i$-th eigenvector of $\bM$.
By Assumption~\ref{Ass: eigspace assumptions}(i), $\bm v_i$ does not change with $t$.
By Assumption~\ref{gf assumption}, we have
$$
\frac{\mathrm d\bD(t)^\top\bm v_i}{\mathrm d t} = - \bD(t)^\top \bM(t)\bm v_i = - \lambda_i(t) \bD(t)^\top\bm v_i
$$
where $\lambda_i(t)\geq 0$ is the corresponding eigenvalue of $\bm v_i$.

Since $\bF(0) =0$, $\bD(0)=\bF(0)-\bY = -\bY.$ Then the differential equation has the solution
$$\bD(t)^\top\bv_i = e^{- \int_0^{t}\lambda_i(t)\mathrm dt}\cdot \bD(0)^\top\bv_i = -e^{- \int_0^{t}\lambda_i(t)\mathrm dt}\cdot \bY^\top\bv_i$$

Plug it into the expression $\bD(t)^\top\bF(t) = \bD(t)^\top(\bD(t)+\bY)$ and we have:
$$
\bD(t)^\top\bF(t) = \sum_{i=1}^n \bD(t)^\top\bv_i(\bD(t)^\top\bv_i + \bY^\top \bv_i) 
= -\sum_{i=1}^n (\bY^\top\bv_i)^2e^{- \int_0^{t}\lambda_i(t)\mathrm dt}(1-e^{- \int_0^{t}\lambda_i(t)\mathrm dt})
$$
since $(\bY^\top\bv_i)^2e^{- \int_0^{t}\lambda_i(t)\mathrm dt}(1-e^{- \int_0^{t}\lambda_i(t)\mathrm dt})>0$ for all $i\in [n]$, $\bD(t)^\top\bF(t)<0$.
\end{proof}

\begin{lemma} (Lemma~\ref{divergence proposition})
\label{lemma: appendix_Dincrease_exponentially}
Suppose Assumption~\ref{Ass: eigspace assumptions} (iii) and \ref{noise assumption}
hold during this phase (with constants $\epsilon_2>0$ and $c>1$).
If $\sha(t) = (2 + \tau)/\eta$ and $\tau > \frac{1}{1-\epsilon_2 - 1/c}-1$, then $\bm{D}(t)^\top\bm{v}_1(t)$ increases geometrically with factor $(1+\tau)(1-\epsilon_2 - 1/c) > 1$
for $t\geq t_0$ in this phase.
\label{lemma: appendix_divergence proposition}
\end{lemma}

\begin{proof}
First by \eqref{Order one dynamics}, we have $\bm{D}(t+1)^\top\bm{v}_1(t) = \bm{D}(t)^\top(\bm I-\eta \bm{M}(t))\bm{v}_1(t) = (1-\eta \sha(t))\bm{D}(t)^\top\bm{v}_1(t)$. 

Then we have 
\begin{align*}
    |\bm{D}(t+1)^\top\bm{v}_1(t+1)|&\geq |\bm{D}(t+1)^\top\bm{v}_1(t)(1-\epsilon_2)| - \epsilon_2  \|\bm{D}(t+1)\|\\ 
    &\geq |\bm{D}(t)^\top\bm{v}_1(t)|(1-\epsilon_2)(\eta\sha(t) - 1) - (\eta\sha(t) - 1)\epsilon_2\|\bm{D}(t)\| \\ 
    & \geq |\bm{D}(t)^\top\bm{v}_1(t)|(1-\epsilon_2)(\eta\sha(t) - 1) - (\eta\sha(t) - 1)|\bm{D}(t)^\top\bm{v}_1(t)|/c \\
    &= |\bm{D}(t)^\top\bm{v}_1(t)|(\eta \sha(t) - 1)(1-\epsilon_2 - 1/c) \\
    &= |\bm{D}(t)^\top\bm{v}_1(t)|(1+\tau)(1-\epsilon_2 - 1/c)
\end{align*}
where first inequality holds due to Assumption~\ref{Ass: eigspace assumptions} (iii)
and the third follows from Assumption~\ref{noise assumption}.
\end{proof}

\begin{proposition} (Proposition~\ref{prop:DlargerthanY})
\label{prop:appendix_DlargerthanY}
If $\|\bm{D}(t)\| > \|\bm{Y}\|$, then $\bm{D}(t) ^\top\bm{F}(t)>0$.
\end{proposition}
\begin{proof}
This proposition can be proved by simply noticing that
$\bm{D}(t) ^\top\bm{F}(t)=\bm{D}(t) ^\top(\bm{D}(t)-\bm{Y}(t)) \geq \|\bm{D}\|^2 - \|\bm{D}\|\|\bm{Y}\| >0$
when $\|\bm{D}(t)\| \geq \|\bm{Y}\|$.
\end{proof}

\begin{proposition}(Proposition~\ref{prop:Adrop})
\label{prop:appendix_Adrop}
Under Assumption~\ref{Ass: Order one app},
if $\|\bD(t)\| > \|\bY\|$, then $\|\bm A(t+1)\|^2 - \|\bm A(t)\|^2 < -\frac{4\eta}{n}(\|\bD(t)\| - \|\bY\|)^2$.
\end{proposition}

\begin{proof}
Use Assumption~\ref{Ass: Order one app} and notice that $\bD(t)^\top\bF(t) \geq \|\bD(t)\|^2 - \|\bD(t)\|\|\bY\| > (\|\bD(t)\| - \|\bY\|)^2$. 
\end{proof}

\begin{lemma} (Lemma~\ref{Convergence lemma})
If $\sha(t)<2/\eta$, then for any vector $\bm u\in\mathbb R^n$, $\|\bm u^\top (\bm I - \eta\bM(t))\|\leq (1-\eta\alpha)\|\bm u\|$, where $\alpha = \min\{2/\eta - \sha(t), \lambda_{\min}(\bM(t))\}$.
In particular, replacing $\bm u$ with $\bD(t)$, we can see 
$\|\bD(t+1)\|\leq (1-\eta \alpha)^2\|\bD(t)\|$.
\label{lemma:appendix_Convergence lemma}
\end{lemma}
\begin{proof} For any vector $\bm u\in\mathbb R^n$, 
\begin{align*}
\|\bm u^\top(\bm I- \eta\bM(t))\|^2 &= \sum_{i=1}^n \|(\bm u^\top \bv_i(t))\bv_i(t)^\top (\bm I- \eta\bM(t))\|^2 \\&= 
\sum_{i=1}^n (1-\eta\lambda_i)^2 (\bm u^\top \bv_i(t))^2 \leq (1-\eta \alpha)^2\sum_{i=1}^n(\bm u^\top \bv_i(t))^2 = (1-\eta \alpha)^2\|\bm u\|^2.
\end{align*}
Then the second statement is due to Assumption~\ref{Ass: Order one app}.
\end{proof}

\begin{lemma} (Lemma~\ref{lm:Rt_update})
\label{lm:appendix_Rt_update}
Suppose Assumption~\ref{Ass: eigspace assumptions} (iii) holds.
$\bm R(t)$ satisfies the following:
$$
 \bm{R}(t+1) = (I - \eta \bm{M}(t))\bm{R}(t) + \bm e_1(t), \quad \text{where}\quad
 \|\bm e_1(t)\|\leq 6\sqrt{\epsilon_2}\|\bm D(t)\|(B_\sha -1)
$$
\end{lemma}

\begin{proof}
We consider the update rule for $\bm R(t)$:
\begin{align*}
    \bm{R}(t+1) ={}& (\bm I - \bm{v}_1(t+1)\bm{v}_1(t+1)^\top)(\bm {I} -  \eta \bm{M}(t))\bm{D}(t) \\
    ={}& (\bm {I} -  \eta \bm{M}(t))\bm R(t) - (\bm{v}_1(t+1)\bm{v}_1(t+1)^\top )(\bm {I} -  \eta \bm{M}(t)) (\bm I -\bm v_1(t) \bm v_1(t)^\top)\bm D(t) \\
    & +  (\bm I - \bm{v}_1(t+1)\bm{v}_1(t+1)^\top)(\bm {I} -  \eta \bm{M}(t))(\bm v_1(t) \bm v_1(t)^\top)\bm D(t)\\
    ={}&(\bm {I} -  \eta \bm{M}(t))\bm R(t) - (1-\eta\sha(t)) (\bm{v}_1(t+1)\bm{v}_1(t+1)^\top) (\bm I - \bm v_1(t) \bm v_1(t)^\top) \bm D(t) \\
    &+ (1-\eta\sha(t)) (\bm I - \bm{v}_1(t+1)\bm{v}_1(t+1)^\top) (\bm v_1(t) \bm v_1(t)^\top) \bm D(t) 
\end{align*}

Now if we decompose $\bm v_1(t+1) = (1-\delta(t))\bm v_1(t) + \sqrt{1-(1-\delta)^2}\bm v_{\perp}(t)$, where $\langle\bv_1(t),\bv_\perp\rangle=0$. Then we have 
\begin{align*}
    &\|(\bm I - \bm{v}_1(t+1)\bm{v}_1(t+1)^\top)\bm{v}_1(t)\bm{v}_1(t)^\top \| \\
    ={}& \|\bm{v}_1(t) - (\bm{v}_1(t)^\top\bm{v}_1(t+1)) \bm{v}_1(t+1)\| \\
    ={}&\|(1-(1-\delta(t))^2)\bm v_1(t) + \sqrt{1-(1-\delta(t)^2)}\bm v_{\perp}(t)\| \\
    \leq{}& \sqrt{(2\delta(t)-\delta(t)^2)(1-\delta(t)) + (2\delta(t)-\delta(t)^2)^2} \\
   \leq{}& \sqrt{2\delta(t) + 4\delta(t)^2} \\
   \leq{}& 3\sqrt{\delta(t)} \\
    \leq{}& 3\sqrt{\epsilon_2}.
\end{align*}

Similar result can be obtained for the term $(\bm{v}_1(t+1)\bm{v}_1(t+1)^\top) (\bm I - \bm v_1(t) \bm v_1(t)^\top)$.

Therefore, $
\| (\eta \Lambda(t) - 1)  
(I - \bm{v}_1(t+1)\bm{v}_1(t+1)^\top)(\bm{v}_1(t)\bm{v}_1(t)^\top)\bm{D}(t)\| 
\leq 6\sqrt{\epsilon_2}(B_\sha -1)B_D. 
$
\end{proof}

\begin{lemma} (Lemma~\ref{RT pert prop})
Define an auxiliary sequence $\bm R'(t)$ by $\bm R'(0) = \bm R(0)$, and $\bm R'(t+1) = (\bm I-\eta\bm M(t)(\bm I - \bv_1(t)\bv_1(t)^\top))\bm R'(t)$.
If Assumption~\ref{Ass: sha upper bound section 3},
Assumption~\ref{Ass: eigspace assumptions} (iii),
Assumption~\ref{outlier assp} hold, and for any time $t$ there exists a quantity $\lambda_r>0$, such that 
the smallest eigenvalue of $\bM(t)$, i.e. $\lambda_{\min}(\bM(t)) > \lambda_r$, 
then there exists a constant $c_r>0$ such that $\|\bm R(t) - \bm R'(t)\| \leq c_r \frac{B_D(B_\sha -1)\sqrt{\epsilon_2}}{\eta \lambda_r}$.
\label{lemma: appendix_RT pert prop}
\end{lemma}
\begin{proof}
First we can see that the eigenvalues of $\bM(t) (\bm I - \bv_1(t)\bv_1(t)^\top)$ are
all the eigenvalues of $\bM$ except the largest one. 
Hence with Assumption~\ref{outlier assp}, all eigenvalues of $\bM(t) (\bm I - \bv_1(t)\bv_1(t)^\top)$ are in $(\lambda_r, 1/\eta)$. Thus 
$\|\bm I - \eta \bM(t) (\bm I - \bv_1(t)\bv_1(t)^\top)\|\leq 1-\eta \lambda_r$.

Hence by Assumption~\ref{lm:Rt_update} we can get
\begin{align*}
    &\|\bm R(t+1) - \bm R'(t+1)\| \\={}& \|(\bm I - \eta \bM(t))\bm R(t) - (\bm I-\eta\bm M(t)(\bm I - \bv_1(t)\bv_1(t)^\top))\bm R'(t) + \bm e_1(t)\| \\
    ={}&  \|(\bm I - \eta \bM(t) (\bm I - \bv_1(t)\bv_1(t)^\top))(\bm R(t) - \bm R'(t)) + \bm e_1(t)\| \\
    \leq{}&  \|(\bm I - \eta \bM(t) (\bm I - \bv_1(t)\bv_1(t)^\top))(\bm R(t) - \bm R'(t))\| + \|\bm e_1(t)\| \\
    \leq{}& \|\bm I - \eta \bM(t)(\bm I - \bv_1(t)\bv_1(t)^\top)\| \|\bm R(t) - \bm R'(t)\| + \|\bm e_1(t)\| \\
    \leq{}& \|\bm R(t) - \bm R'(t)\| (1 - \eta \lambda_r) + 3B_D(B_\sha - 1)\sqrt{\epsilon_2}.
\end{align*}

Thus if we denote $\|\bm R(t) - \bm R'(t)\| - \frac{3B_D(B_\sha -1)\sqrt{\epsilon_2}}{\eta \lambda_r}$ by $p(t)$, and replace $\|\bm R(t) - \bm R'(t)\|$ in the inequality above by $p(t) +  \frac{3B_D(B_\sha -1)\sqrt{\epsilon_2}}{\eta \lambda_r}$, we can get $|p(t+1)| \leq |p(t)|(1-\eta \lambda_r)$.
Hence, we can see that $|p(t)| < |p(0)| = \frac{3B_D(B_\sha - 1)\sqrt{\epsilon_2}}{\eta \lambda_r}$ for any time $t$. 
Therefore, we obtain that
$$
\|\bm R(t) - \bm R'(t)\| < \frac{3B_D(B_\Lambda -1)\sqrt{\epsilon_2}}{{\eta \lambda_r}} + |p(0)| < \frac{6B_D(B_\Lambda -1)\sqrt{\epsilon_2}}{\eta \lambda_r}.
$$
Taking $c_r = 6$, we finish the proof.
\end{proof}

\begin{lemma}
\label{RF lemma appendix}
For all $t$ in Phase I, 
under Assumption~\ref{Ass: eigspace assumptions} (i) and Assumption~\ref{outlier assp}, 
it holds that $\bR'(t)^\top(\bR'(t)+\bm Y) < 0$ where $\bR'(t)$ is defined in Lemma~\ref{RT pert prop}.
\end{lemma}

\begin{proof}
Consider the inner product $\bR'(t)^\top\bm v_i$. 
Recall $\bm v_i$ is the $i$-th eigenvector of $\bM$.
By Assumption~\ref{Ass: eigspace assumptions} (i), $\bm v_i$ does not change with $t$.

By Assumption~\ref{outlier assp}, $\lambda_i(t)<1/\eta$ for $i>1$,
where $\lambda_i(t)\geq 0$ is the corresponding eigenvalue of $\bm v_i$.

By definition of $\bR'(t)$ and Assumption~\ref{Ass: eigspace assumptions} (i), $\bR'(t+1) = (\bm I - \eta \bM(t)(\bm I - \bv_1\bv_1^\top))\bR'(t)$, hence $\bR'(t+1)^\top\bm v_i = (1-\eta \lambda_i(t))\bR'(t)^\top\bm v_i$ for $i>1$ and $\bR'(t)^\top \bv_1 = 0$. Hence for any $i>1, t>0$, $\frac{\bR'(t)^\top\bm v_i}{\bR'(0)^\top\bm v_i} = \prod_{j = 0}^{t-1}(1-\eta\lambda_i(j)) \in (0,1)$.

Since $\bR'(0) = \bR(0) = (\bm I - \bv_1\bv_1^\top)\bm Y$, hence $\bR'(0)^\top\bm v_i = -\bY^\top \bv_i$ for any $i>1$.

Plug it into the expression $\bR'(t)^\top(\bR'(t)+\bY)$ and we have:
$$
\bD(t)^\top\bF(t) = \sum_{i=2}^n \bR'(t)^\top\bv_i(\bR'(t)^\top\bv_i + \bY^\top \bv_i) 
= \sum_{i=2}^n (\bR'(t)^\top\bv_i)^2(1 - \frac{\bR'(0)^\top \bv_i}{\bR'(t)^\top\bm v_i})<0
$$

\end{proof}

\subsection{Progressive Sharpening under Weaker Assumptions}
\label{Subsection: relaxed PS}

We prove Lemma~\ref{ps lemma} in Section~\ref{Section 3}, i.e. progressive sharpening happens under Assumption~\ref{Ass: eigspace assumptions} (i) that the set of eigenvectors of the Gram matrix $\bM(t)$ is fixed throughout. 
However, empirically, the eigenvectors of $\bM(t)$ change non-negligibly
(see Figure~\ref{Angular velocity Pics}).
In this section, we show that it is possible to relax Assumption~\ref{Ass: eigspace assumptions} (i)
to more realistic assumption on the eigenspace of $\bM(t)$. 
To this end, we first present a dynamical system view of the dynamics of $\bD(t)$,
then prove Lemma~\ref{relaxed ps lemma} which is analogue of 
Lemma~\ref{ps lemma} but under weaker assumptions.

\subsubsection{A dynamical system view}
Recall that we assume that 
$\bD(t)$ follows the gradient flow trajectory:
$\frac{\mathrm d\bD(t)}{\mathrm d t} = - \bM(t) \bD(t)$.
This is a linear dynamical system with changing coefficients $- \bM(t)$.
In order to understand this dynamics, we first consider
the linear dynamical system with constant coefficients $- \bM$, 
$\frac{\mathrm d\bD(t)}{\mathrm d t} = - \bM \bD(t)$,
which is much better understood.
The corresponding phase portrait is shown in Figure~\ref{Figure: PS Dynamics pics}(a).
Here, $O$ is only fixed point.
In this figure, we consider the ball $\sphere\subset \mathbb{R}^n$
such that the south pole of $\sphere$ is the origin $O$ and the north pole $N$ of $\sphere$ is fixed to be $-\bY$, which is the initial point $-\bD(t)$.
Recall that $\bF(t) = \bD(t) + \bY$ and $\bD(0)=-\bY$. We denote the tip of the vector $\bD(t)$ as $D$. Hence, $\overrightarrow{DO} = \bD(t)$ and $\overrightarrow{ND} = \bF(t)$. A useful geometric observation is 
the following:

\begin{observation}
\label{ob:sphere}
$\bD(t) \in \sphere$ if and only if
$\bD(t)^\top \bF(t)=\bD(t)^\top (\bD(t)-(-\bY))<0$
(equivalently, the angle $\measuredangle (N D O )\geq \pi/2$).
\end{observation}

Let $\{\bv_i(t)\}_i$ be the set of eigenvectors of $\bM(t)$.
By symmetry, we can assume the direction of all $\{\bv_i(t)\}_i$ 
are pointing inside $\sphere$ (or equivalently, $-\bv_i(t)^\top \bY\geq 0$).
We also define the hypercube $\bH(t)$, which is defined by $\{\bv_i(t)\}_i$ as the edges and $ON$ as the diagonal (see the dashed rectangle in Figure~\ref{Figure: PS Dynamics pics}(a). All vertices of $\bH(t)$ are on the sphere $\partial\sphere$.
In our proof of Lemma~\ref{ps lemma} (in which we assume
the eigenvectors do not change),
we have shown the entire trajectory of $\bD(t)$ is contained in $\bH(t)$
which is always contained in $\sphere$.
In particular, we prove that the projection of $\bD(t)$ onto each $\bv_i$ decreases monotonically).

Now, we consider the more general case when the eigenvectors $\{\bv_i(t)\}_i$ may change.
In Lemma~\ref{relaxed ps lemma} presented below, we show that if the eigenvectors change directions relatively slowly (precisely the change rate satisfies \eqref{C1 condition}), $\bD(t)$ is still contained in $\bH(t)$, hence also in $\sphere$.
Then by Observation~\ref{ob:sphere} and \eqref{Order one dynamics},
the norm $\|A\|$ increases (hence the sharpness increases).

We also note that it is impossible to prove $\bD(t)\in \sphere$ for all $t>0$
without any assumption on the change of eigenvector directions.
In particular, if the eigenvector changes fast enough so that $\bD(t)$ is out of
the hypercube $\bH(t)$, the trajectory may eventually go out of the ball $\sphere$
(e.g., the region inside the dashed red box in Figure~\ref{Figure: PS Dynamics pics}(b)). However, it seems that this region is fairly close to the origin $O$,
and the trajectory does not go very far away from $\sphere$ (implying 
$\bD(t)^\top \bF(t)$ is not a large positive number).  
Hence, the above geometric intuition suggests that it is possible to 
further relax the assumption (inequality \eqref{C1 condition}) 
made in ~\ref{relaxed ps lemma} and prove 
that A-norm increases (hence sharpness increases) until the sharpness 
reaches close to $2/\eta$, after which gradient flow is not a good approximation
of gradient descent anymore. We leave it as an interesting future direction.

\begin{figure}[t]
    \vspace{-0.5cm}
    \centering
    \subfigure[The dynamics of $\bD(t)$ with constant coefficients.]{
    \includegraphics[width=0.48\textwidth]{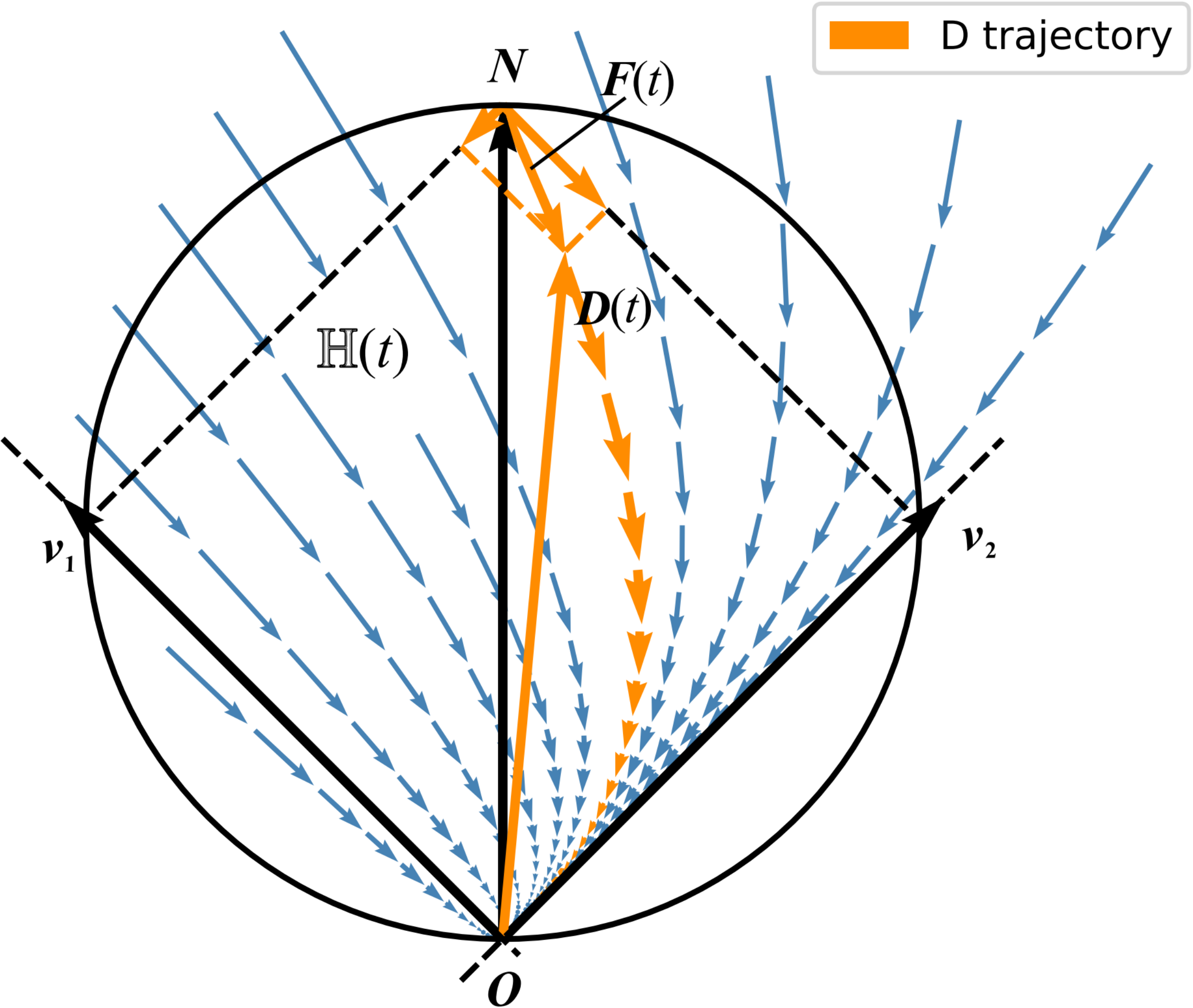}
    }
    \subfigure[The eigenvectors of $\bM(t)$ change]{
    \includegraphics[width=0.48\textwidth]{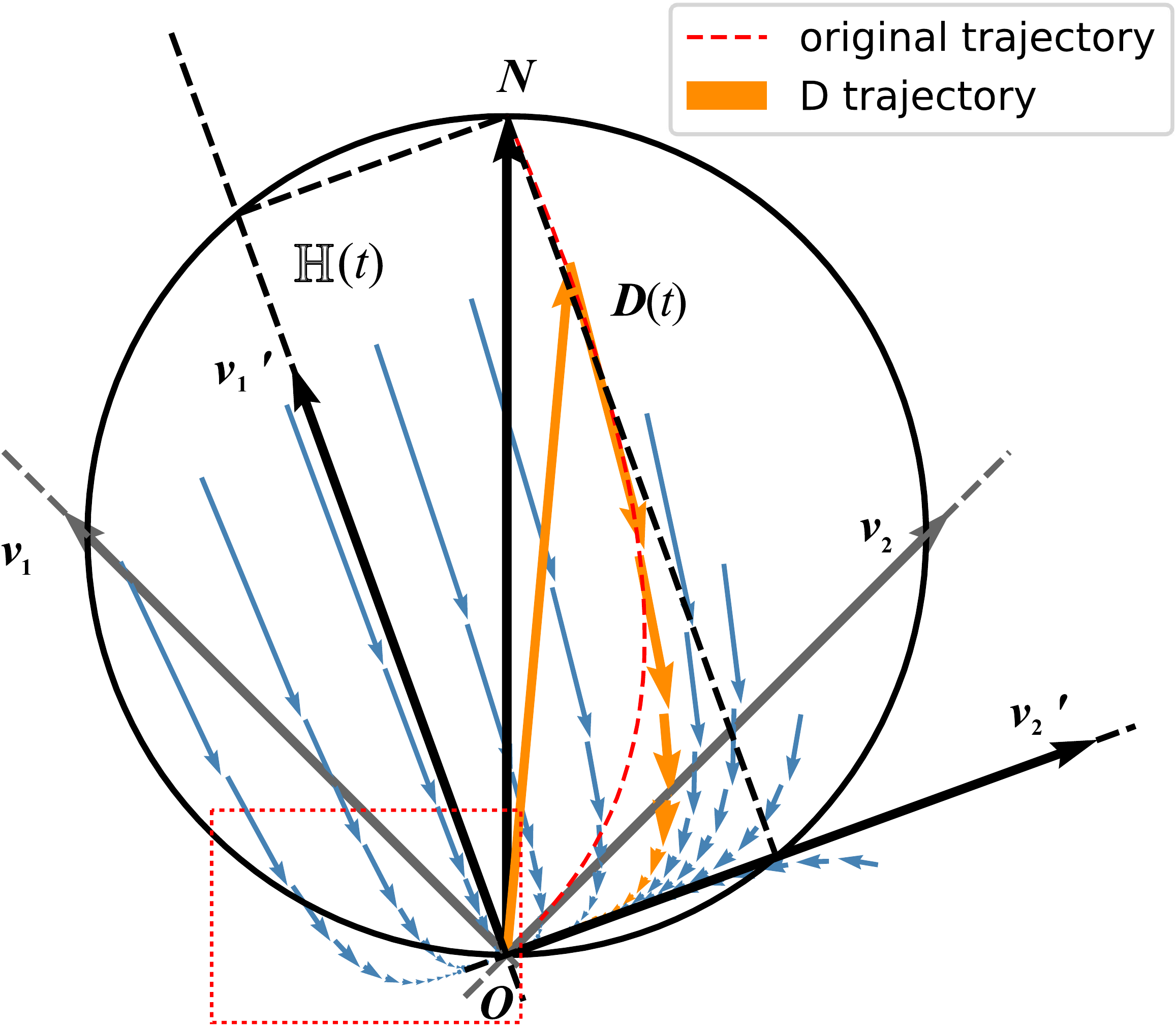}
    }

    \caption{The phase portraits of the dynamics of $\bD(t)$ 
    (projected to the first two principle directions).
    We rotate the coordinate system so that south pole of $\sphere$ is the origin $O$ and the north pole is $-\bY$ which is the initial point $\bD(t)$.
    $\bv_1$ and $\bv_2$ are the first two eigenvectors of $\bM$.
    In (b), $\bv_1$ and $\bv_2$ change directions to $\bv_1'$ and $\bv_2'$.
    The trajectory may go out of $\sphere$ from certain regions (the dashed red box). 
    }
    \label{Figure: PS Dynamics pics}
\end{figure}




\subsubsection{Progressive sharpening when eigendirections change slowly}

Suppose at initialization $\bF(0) = 0$. By $\bD(t) = \bF(t) - Y$, we know $\bD(0)= -Y$. With loss of generality, we assume $\bD(0)^\top \bv_i(0) > 0$ at initialization for all $i$ (Otherwise, we define $\bv_i(0) = -\bv_i(0)$ and this condition holds). 

\begin{lemma}
Suppose Assumption~\ref{gf assumption} holds, and the set of (unit) eigenvectors of $\bM(t)$ is $\{\bv_i(t)\}$. Assume at all time $t$ and direction $\bv_i(t)$, the following condition holds:
\begin{equation}
 \bF(t)^\top \frac{\mathrm d \bv_i(t)}{\mathrm d t} < \lambda_i(t)\bD(t)^\top \bv_i(t).  \label{C1 condition}
\end{equation}
Then we have $\bD(t)^\top\bF(t) < 0$ for all $t<t_1$, where $t_1$ is the first time when there exists some $i$, $\bD(t)^\top\bv_i(t)\leq 0$.
\label{relaxed ps lemma}
\end{lemma}

In the above lemma, we can see that inequality~\eqref{C1 condition} holds if $\bv_i(t)$ changes relatively slowly. In particular, if $\frac{\mathrm d \bv_i(t)}{\mathrm d t} = 0$ for all $i$, the lemma reduces to Lemma~\ref{ps lemma}. Note that in contrast to the fixed eigenvector assumption, we cannot show that $\bD(t)^\top \bv_i(t)$ decreases in all directions all the time (since $\bM(t)$ may change). 
Instead, we show that $\bD(t)$ is always in the hypercube $\bH(t)$ under the condition of the lemma.

\begin{proof} 
We first notice the decomposition 
$$
\bD(t)^\top\bF(t)= \sum_{i=1}^r\bD(t)^\top\bv_i(t)\bF(t)^\top\bv_i(t).
$$
We show for each $\bv_i(t)$, 
$ \bD(t)^\top\bv_i(t)\bF(t)^\top\bv_i(t) <0$ under the conditions of the lemma.

We first consider the dynamics of inner product $\bD(t)^\top \bv_i(t)$. 
According to Assumption~\ref{gf assumption}, when $t<t_1$, $\bD(t)^\top \bv_i(t)>0$ due to the continuity.
Moreover, by the dynamics of $\bD(t)$ in Assumption~\ref{gf assumption}, 
we can see
\begin{align*}
    \frac{\mathrm d\bD(t)^\top\bm v_i(t)}{\mathrm d t} &= - \bD(t)^\top \bM(t)\bm v_i(t) + \bD(t)^\top \frac{\mathrm d \bv_i(t)}{\mathrm d t} \\&= - \lambda_i(t) \bD(t)^\top\bm v_i(t) +\bD(t)^\top \frac{\mathrm d \bv_i(t)}{\mathrm d t}
\end{align*}

With this dynamics above and $\bF(t) = \bD(t) + \bY$, we can get $\bF$'s dynamics:
\begin{align*}
    \frac{\mathrm d\bF(t)^\top\bm v_i(t)}{\mathrm d t} &= - \bD(t)^\top \bM(t)\bm v_i(t) + \bF(t)^\top \frac{\mathrm d \bv_i(t)}{\mathrm d t} \\&= - \lambda_i(t) \bD(t)^\top\bm v_i(t) +\bF(t)^\top \frac{\mathrm d \bv_i(t)}{\mathrm d t}\tag{D1}\label{dynamics of f}
\end{align*}

With \eqref{dynamics of f} and $\bF(t)^\top \frac{\mathrm d \bv_i(t)}{\mathrm d t} \leq \lambda_i(t)\bD(t)^\top\bm v_i(t)$, we know
\begin{align*}
    \frac{\mathrm d\bF(t)^\top\bm v_i(t)}{\mathrm d t} &= - \lambda_i(t) \bD(t)^\top\bm v_i(t) +\bF(t)^\top \frac{\mathrm d \bv_i(t)}{\mathrm d t} < 0.
\end{align*}

In this way, we know $\bF(t)^\top\bm v_i(t)$ always decreases for $t<t_1$. Recall that $\bF(0) =\bD(0)+Y = 0$, which makes $\bF(0)^\top\bv_i(0) = 0$. So $\bF(t)^\top\bv_i(t)<0$ according to the dynamics.
Geometrically, this implies $\bD(t)$ does not cross the facets of $\bH(t)$ incident on the north pole $N$ 
\footnote{
The hypercube $\bH(t)$ has $2n$ facets, $n$ of them incident on the south pole $O$
and the others the north pole $N$.
}
(otherwise $\bF(t)^\top\bv_i(t)$ needs to change sign). See Figure~\ref{Figure: PS Dynamics pics}).
On the other hand, $\bD(t)^\top\bv_i(t)>0$ for $t<t_1$,
by the definition of $t_1$, 
which is equivalent to say that $\bD(t)^\top\bv_i(t)$ does not change sign or $\bD(t)$ does not cross the facets of $\bH(t)$ incident on the south pole $O$. 

Therefore we have $\bD(t)^\top\bv_i(t)\bF(t)^\top\bv_i(t) <0$ and $\bD(t)\in \bH(t)\subset \sphere$ for all $t<t_1$, where $t_1$ is the first time when $\bD(t_1)^\top \bv_i(t_1)\leq 0.$
Summing over all the components $\bD(t)^\top\bv_i(t)\bF(t)^\top\bv_i(t)$ for all $i$,
we prove the lemma.
\end{proof}




\subsubsection{Empirical verification of inequality~\eqref{C1 condition}}
\begin{figure}[H]
    \vspace{-0.5cm}
    \centering
    \includegraphics[width=0.48\textwidth]{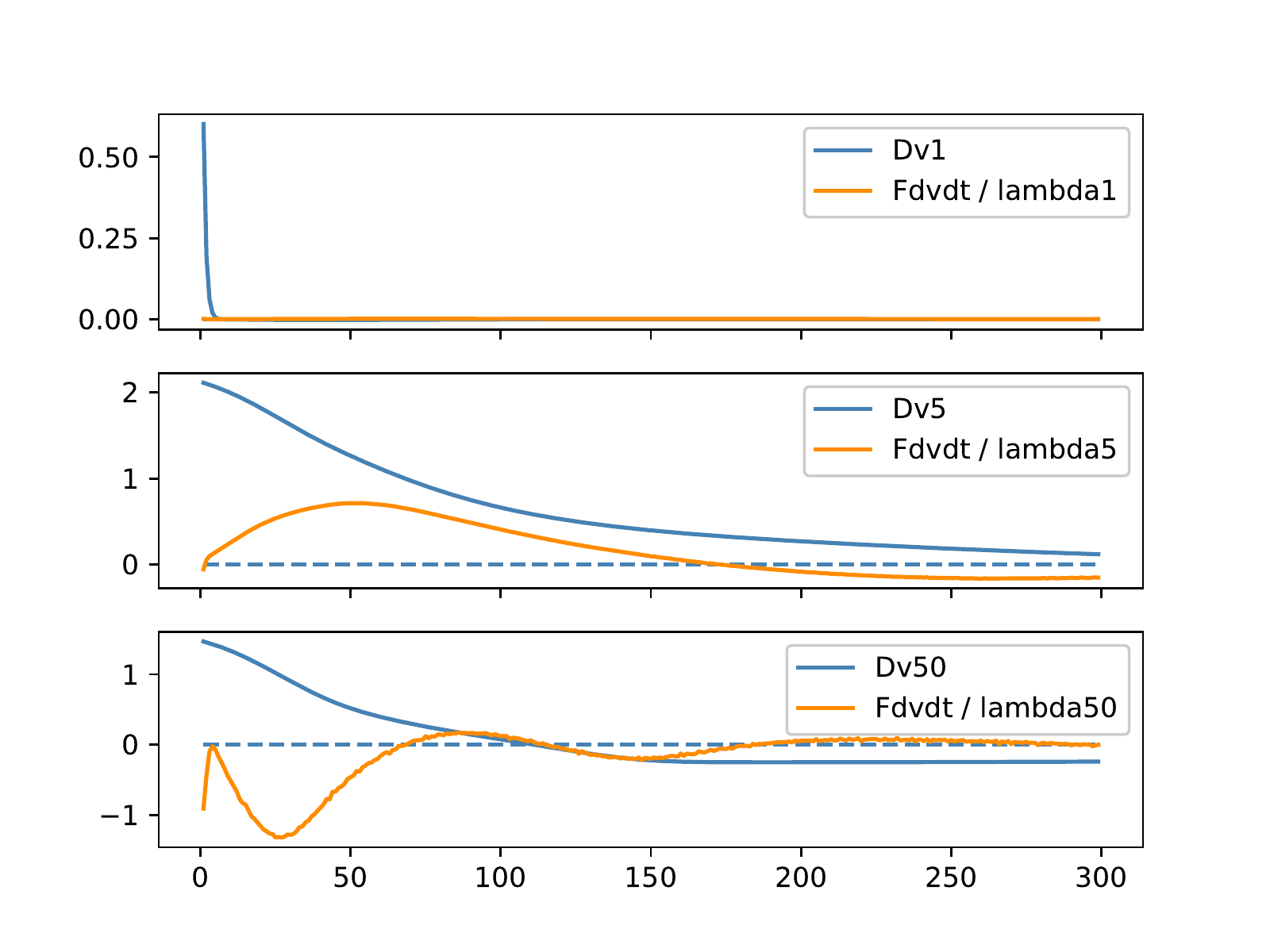}
    \caption{Verification for inequality~\eqref{C1 condition} when $i=1,5,50$.}
    \label{Angular velocity Pics}
\end{figure}

We verify the condition \eqref{C1 condition} in Lemma~\ref{relaxed ps lemma} 
via experiments. See Figure~\ref{Angular velocity Pics}.
We can see that \eqref{C1 condition} holds for $i=1,5,50$ during the first hundred iterations of the training process. 
However, for very large $i$, the eigenvectors corresponding to very small eigenvalues may change quickly
and \eqref{C1 condition} does not hold. The contributions (i.e., 
$\bD(t)^\top\bv_i(t)\bF(t)^\top\bv_i(t)$) from those 
corresponding to very small eigenvalues tend to cancel each other empirically, hence does not affect the sign of the total
sum $\sum_{i=1}^r\bD(t)^\top\bv_i(t)\bF(t)^\top\bv_i(t)$. Hence, it is possible to further relax the condition \eqref{C1 condition} and we leave it as a future direction.

\subsection{Proofs of equations in Appendix~\ref{appen c}}
The proof of some equations in Appendix~\ref{appen c} is omitted, and we list them in this subsection.
\begin{lemma}
    (Lemma~\ref{update rule of d appenc}) The update rule of the residual vector of $\bD(t)$:
   \begin{equation}
        \bD(t+1)^\top-\bD(t)^\top=-\eta\bD(t)^\top \bM(t)+\frac{4\eta^2}{n^2m}\bD(t)^\top(\bF(t)^\top\bD(t))\bX^\top \bX
    \end{equation}
    \label{update rule of d appenE}
\end{lemma}
 
 \begin{proof}
        \begin{align*}
        \bD(t+1)^\top-\bD(t)^\top={}&\frac{1}{\sqrt m}(\bA(t+1)^\top\bW(t+1)-\bA(t)^\top\bW(t))\bX\\
        ={}&\frac{1}{\sqrt m}\left((\bA(t+1)^\top-\bA(t)^\top)\bW(t+1)+\bA(t)^\top(\bW(t+1)-\bW(t))\right)\bX\\
        ={}&\frac{1}{\sqrt m}\left((\bA(t+1)^\top-\bA(t)^\top)\bW(t)+\bA(t)^\top(\bW(t+1)-\bW(t))\right)\bX\\
        &+\frac{1}{\sqrt m}(\bA(t+1)^\top-\bA(t)^\top)(\bW(t+1)-\bW(t))\bX\\
        ={}&\frac{2\eta}{n m}\left((-\bm W (t)\bX \bD(t))^\top\bW(t)-\bA(t)^\top\bA(t) \bD(t)^\top\bX^\top\right)\bX\\
        &+\frac{1}{\sqrt m}(\frac{2\eta}{n\sqrt m}\bm W (t)\bX \bD(t))^\top\frac{2\eta}{n\sqrt m}\bA(t) \bD(t)^\top\bX^\top\bX\\
        ={}&-\bD(t)^\top \frac{2\eta}{nm}(\bX^\top\bW(t)^\top \bW(t)\bX+\|\bA(t)\|^2\bX^\top \bX)\\
         &+\frac{4\eta^2}{n^2m}\bD(t)^\top(\bF(t)^\top\bD(t))\bX^\top \bX\\
        ={}&-\eta\bD(t)^\top \bM(t) + \frac{4\eta^2}{n^2m}\bD(t)^\top(\bF(t)^\top\bD(t))\bX^\top \bX
        \end{align*}
 \end{proof}
        
\begin{lemma} (Lemma~\ref{update rule of M}) The update rule of the Gram matrix $\bM(t)$ is,
    $$\begin{aligned}
  \bM(t+1) - \bM(t)  
=&-\frac{4\eta}{n^{2}m} \left( 2(\bF(t)^{\top} \bD(t)) \bX^\top \bX+\bF(t) \bD(t)^{\top}\bX^\top \bX +\bX^\top \bX\bD(t)\bF(t)^\top \right) \\
&+\frac{8\eta^2}{n^3m^2} (\bD(t)^{\top} \bX^{\top} \bm W(t)^{\top}\bm W(t) \bX \bD(t)) \bX^\top \bX\\ &+\frac{8\eta^2}{n^3m^2}\|\bA(t)\|^{2}\ \bX^\top\bX \bD(t)\bD(t)^\top\bX^\top\bX
\end{aligned}$$
\label{update rule of M appenE}
    \end{lemma}
    
    \begin{proof} By the update rule of $\bA(t)$ and $\bW(t)$ in equation~(\ref{C1}) and (\ref{C2}), we have
    \begin{align*}
        &\bM(t+1)-\bM(t)\\ ={}&\frac{2}{nm}\left((\|\bA(t+1)\|^2-\|\bA(t)\|^2)\bX^\top \bX+\ \bX^\top (\bm W(t+1)^{\top}\bm W(t+1)-\bm W(t)^{\top}\bm W(t))
        \bX\right)\tag{\text{Use equation~(\ref{C2})}}\\
        ={}&\frac{2}{nm}(-\frac{4\eta}{n}\bD(t)^\top\bF(t) +\eta^2 \|\frac{\partial \mathcal{L}}{\partial \bA}\|^2)\bX^\top \bX \\
        &{}+ \frac{2}{nm}\left(\bX^\top ((\bm W(t+1)^{\top}-\bm W(t)^{\top} )\bm W(t) +\bW(t)^\top (\bm W(t+1)-\bm W(t)))
        \bX\right.\\
        &\left.+\ \bX^\top (\bm W(t+1)^{\top}-\bm W(t)^{\top} )(\bm W(t+1)-\bm W(t))\bX\right) \tag{\text{Use equation~(\ref{C1})}}\\
        ={}&-\frac{8\eta}{n^2m}\bD(t)^\top\bF(t)\bX^\top\bX+\frac{8}{n^3m^2}\eta^2(\bD(t)^{\top} \bX^{\top} \bm W(t)^{\top}\bm W(t) \bX \bD(t))\bX^\top\bX\\
        &-\frac{4\eta}{n^2m}\left(\bF(t)\bD(t)^\top\bX^\top\bX +\bX^\top\bX\bD(t)\bF(t)^\top \right)\\
        &+\frac{8\eta^2}{n^3m^2}\|\bA(t)\|^2\bX^\top\bX\bD(t)\bD(t)^\top\bX^\top\bX
    \end{align*}
    Rearrange the terms and we complete the proof.
    \end{proof}

\begin{lemma}\label{colspace}
   For all iteration $t$, if any vector $\bm u\in \mathbb{R}^n$ satisfies $\bX^\top \bX\bm u =0$, then $$\bY^\top\bm u=\bD(t)^\top\bm u=0$$
   \end{lemma}  
    \begin{proof}
    If $\bX^\top \bX\bm u =0$, then 
    
    $$ \bm u^\top \bX^\top \bX\bm u =0 \Rightarrow \bX \bm u = 0.$$
    
    So we have 
    $$
    \bY^\top\bm u = -\bD(0)^\top\bm u = \bA^{*\top}\bW^*\bX\bm u = 0
    $$
    
    With the dynamics of $\bD(t)$, we have
    $$
    \bD(t+1)^\top\bm u = \bD(t)^\top(\bm I_n-\eta \bM^*(t))\bm u
    $$
    While 
    \begin{align*}
       \bM^*(t)\bm u &= \bM(t) \bm u - \frac{4\eta}{n^2m}(\bD(t)^\top\bF (t))\bX^\top \bX \bm u = \bM(t)\bm u \\
        &= \frac{2}{mn}(\|\bA(t)\|^2\bX^\top \bX +\bX^\top \bW^\top (t) \bW(t)\bX)\bm u = 0
    \end{align*}
    
    Thus we have $$\bD(t+1)^\top\bm u = \bD(t)^\top(\bm I_n-\eta \bM^*(t))\bm u  =\bD(t)^\top \bm u =...= \bD(0)^\top \bm u = 0$$
    \end{proof} 

\begin{lemma}
\label{calculations for equations}

(Lemma~\ref{Key equation lemma}) Here we give the proof of the equation~\eqref{Key equation in EOS}:
$$
\begin{aligned}
\sha^*(t+1)-&\sha^*(t) 
=- \frac{8\eta \lambda_{1}}{mn^2} \left(\bF(t)^{\top} \bD(t)+(\bF(t)^\top \bv_1)(\bD(t)^\top \bv_1)-\frac{\eta}{2}(\bD(t)^\top \bv_1)^{2}  \sha^*(t)  \right.\\
&\left.-{}\frac{\eta}{2}\bR(t)^\top \bm\Gamma(t) \bR(t)  -\ \eta \bR(t)^\top \bm\Gamma(t)\left(\bv_1 \bv_1^\top\bD(t)\right)-\frac{\eta}{m n}\bR(t)^{\top}\left(\frac{m}{d} \bX^{\top} \bX\right) \bR(t)  \right).
\end{aligned}
$$
\end{lemma}
\begin{proof}

First, since $\bm W^\top \bm W$ is initialized as $\frac{m}{d}\bX^\top\bX$, by definition, we have

$\bX^\top \bm W^\top \bm W\bX = \frac{mn}{2}\bm\Gamma(t) + \frac{m}{d} \bX^\top \bX$.

Then 
$$
\begin{aligned}
&\sha^*(t+1) - \sha^*(t)\\
=& \bv_1^{\top} \bM(t+1) \bv_1-\bv_1^{\top} \bM(t) \bv_1 \\
=& \bv_1^{\top}(\bM(t+1)-\bM(t)) \bv_1 \\
=&-\frac{8\eta}{n^{2}m} \left(\lambda_{1} \bF(t)^{\top} \bD(t)+\lambda_{1}(\bF(t)^{\top} \bv_1)(\bD(t)^{\top} \bv_1)\right) \\
&+\frac{8\eta^{2}}{n^{3}m^2} \left(\lambda_{1} \bD(t)^{\top} \bX^{\top} \bm W^{\top}\bm W \bX \bD(t)+\|\bA\|^{2} \lambda_{1}^{2}(\bD(t)^{\top} \bv_1)^{2}\right)
\end{aligned}
$$

Because 
$$
\begin{aligned}
&\bD(t) ^{\top}\bX^{\top} \bm W^{\top} \bm W \bX \bD(t) \\={}&\bD(t)^{\top}\left(\frac{m}{d} \bX^{\top} \bX+\frac{m n}{2} \bm\Gamma(t)\right) \bD(t) \\
={}&\frac{m}{d} \bD(t)^{\top}\left(\bv_1 \bv_1^{\top}+\bm I-\bv_1 \bv_1^{\top}\right) \bX^{\top} \bX\left(\bD(t)^{\top}\left(\bv_1 \bv_1^{\top}+\bm I-\bv_1 \bv_1^{\top}\right)\right)^{\top}\\&+\frac{m n}{2}\bD(t)^{\top} \bm\Gamma(t) \bD(t)   \\
={}&\frac{m}{d} \lambda_{1}(\bD(t) ^{\top}\bv_1)^{2}+ \frac{m}{d} \bR(t)^{\top} \bX^{\top} \bX \bR(t)+\frac{m n}{2}\bD(t)^{\top} \bm\Gamma(t) \bD(t) 
\end{aligned}
$$
and $$
\lambda_{1}\left(\frac{m}{d}+\|\bA\|^{2}\right)=\frac{m n}{2}\bv_1^{\top}(\bM(t)-\bm\Gamma(t)) \bv_1  =\frac{m n}{2}  \sha^*(t)-\frac{m n}{2} \bv_1^{\top}\bm\Gamma(t) \bv_1
$$
we can see that 
$$
\begin{aligned}
& \sha^*(t+1) - \sha^*(t) \\
=&- \frac{8\eta \lambda_{1}}{n^{2}m}\left(\bF(t)^{\top} \bD(t)+(\bF(t)^{\top} \bv_1)(\bD(t)^{\top} \bv_1)-\frac{\eta}{m n}(\bD(t)^{\top} \bv_1)^{2} \left(\frac{m}{d}+\|\bA\|^{2}\right)\lambda_{1} \cdot   \right.\\
&\left.-\frac{\eta}{m n}\cdot\frac{m}{d} \bR(t)^{\top} \bX^{\top} \bX \bR(t)  -\frac{\eta}{2}\bD(t)^{\top} \bm\Gamma(t) \bD(t)  \right) \\
=&-\frac{8\eta \lambda_{1}}{mn^2}\left(\bF(t)^{\top} \bD(t)+(\bF(t)^{\top} \bv_1)(\bD(t)^{\top} \bv_1)-\frac{\eta}{2}(\bD(t)^{\top} \bv_1)^{2}  \sha^*(t)  \right.\\
&+\frac{\eta}{2}\left(\bv_1^{\top} \bm\Gamma(t) \bv_1 \cdot(\bD(t) ^{\top}\bv_1)^{2}-\bD(t)^{\top} \bm\Gamma(t) \bD(t)\right)  -\left.\bR(t)^{\top}\left(\frac{m}{d} \bX^{\top} \bX\right) \bR(t)\right)
\end{aligned}
$$

Note that 
$$
\begin{aligned}
& \bv_1 \bm\Gamma(t) \bv_1^{\top}(\bD(t)^{\top} \bv_1)^{2}-\bD(t)^{\top} \bm\Gamma(t) \bD(t) \\
=& \bD(t)^{\top} \bv_1 \bv_1^{\top} \bm\Gamma(t) \bv_1 \bv_1^{\top} \bD(t)-\left(\bR(t)^{\top}+\bD(t)^{\top} \bv_1 \bv_1^{\top}\right) \bm\Gamma(t)\left(\bR(t)+\bv_1 \bv_1^{\top} \bD(t)\right) \\
=&-\bR(t)^{\top} \bm\Gamma(t) \bR(t)-2 \bR(t)^{\top} \bm\Gamma(t)\left(\bD(t)^{\top} \bv_1 \bv_1^{\top}\right)^{\top}
\end{aligned}
$$

Now we finish the proof.
\end{proof}

\begin{lemma}(Equation~\ref{RT update rile})
\label{RT update rule appendix E}
$\bR(t+1) = \left(I-\eta \bM^{*}(t)\right)\bR(t)+\eta\left(\bv_1 \bv_1^{\top}\bm\Gamma(t) - \bm\Gamma(t)\bv_1 \bv_1^{\top}\right) \bD(t)$
\end{lemma}
\begin{proof}
\begin{align*}
    \bR(t+1)&= \left(\bm I-\bv_1 \bv_1^{\top}\right)\bD(t+1) \\
=& \left(\bm I-\bv_1 \bv_1^{\top}\right)\left(\bm I-\eta \bM^{*}(t)\right) \bD(t)\\
=& \left(\bm I-\bv_1 \bv_1^{\top}\right)\left(\bm I-\eta \bM^{*}(t)\right)\bv_1 \bv_1^{\top}\bD(t)  \\
&+\left(\bm I-\bv_1 \bv_1^{\top}\right)\left(\bm I-\eta \bM^{*}(t)\right)\left(\bm I-\bv_1 \bv_1^{\top}\right) \bD(t)\\
=& \left(\bm I-\eta \bM^{*}(t)\right)\bR(t)- \bv_1 \bv_1^{\top}\left(\bm I-\eta \bM^{*}(t)\right)\left(\bm I-\bv_1 \bv_1^{\top}\right)  \bD(t)\\
&+ \left(\bm I-\bv_1 \bv_1^{\top}\right)\left(\bm I-\eta \bM^{*}(t)\right) \bv_1 \bv_1^{\top}\bD(t)\\
=& \left(\bm I-\eta \bM^{*}(t)\right)\bR(t)-  \bv_1 \bv_1^{\top}(-\eta \bm\Gamma(t))\left(\bm I-\bv_1 \bv_1^{\top}\right)\bD(t) \\
&+\left(\bm I-\bv_1 \bv_1^{\top}\right)  (-\eta \bm\Gamma(t)) \bv_1 \bv_1^{\top}\bD(t)\\
=& \left(I-\eta \bM^{*}(t)\right)\bR(t)+\eta\left(\bv_1 \bv_1^{\top}\bm\Gamma(t) - \bm\Gamma(t)\bv_1 \bv_1^{\top}\right) \bD(t)
\end{align*}
\end{proof}

\section{Missing Preliminaries}
\label{appendix: prelim}

\noindent\textbf{Hessian, Fisher information matrix and NTK:}
The Hessian of the objective $\mathcal L(f(\bm\theta))$ is:
$$
    \nabla_{\bm\theta}^2\mathcal L(f(\bm\theta))=  \frac{1}{n}\sum_{i=1}^n\nabla_{\bm\theta} ^2\ell(f(\bm\theta, \mathbf{x}_i))
$$
We need the relationship between the Hessian matrix, the Fisher information matrix (FIM) and the neural tangent kernel (NTK).
As shown in \citet{papyan2019measurements, martens2016second, bottou2018optimization}, the Hessian can be
decomposed into two components:
$$\nabla_{\bm\theta}^2\mathcal L(f(\bm\theta))=\frac{1}{n}\sum_{i=1}^n\left(\frac{\partial f(\mathbf{x}_i, \bm\theta)}{\partial \bm\theta}^\top \left. \frac{\partial^2 \ell(z,y_i)}{\partial z^2}\right|_{ f(\mathbf{x}_i)}\frac{\partial f(\mathbf{x}_i, \bm\theta)}{\partial \bm\theta}+\left.\frac{\partial \ell(z,y_i)}{\partial z}\right|_{f(\mathbf{x}_i)}\nabla_{\bm\theta}^2f(\mathbf{x}_i,\bm\theta)\right)
$$
Here $\frac{\partial f(\mathbf{\bx_i, {\bm\theta}})}{\partial {\bm\theta}}\in \mathbb R^{1\times p}$ is the gradient of $f$. 
\citet{papyan2019measurements} also demonstrate empirically that the first term, known as ``Gauss-Newton matrix'', G-term or Fisher information matrix (FIM), dominates the second term in terms of the largest eigenvalue.
Thus when considering the sharpness, we can analyze the largest eigenvalue of FIM which is a close proxy of the largest eigenvalue of Hessian. In 
the binary MSE loss case, $\left. \frac{\partial^2 \ell(z,y_i)}{\partial z^2}\right|_{f(\mathbf{x}_i)}=2$, which implies that
$$    
\left\|\nabla_{\bm\theta}^2\mathcal L(f(\bm\theta))\right\|_2\approx  \bm \left\|G\right\|_2:= \frac{2}{n}\left\|\sum_{i=1}^n\frac{\partial f(\mathbf{x}_i, \bm\theta)}{\partial {\bm\theta}}^\top\frac{\partial f(\mathbf{x}_i, \bm\theta)}{\partial {\bm\theta}} \right\|_2= \frac{2}{n}\left\|\frac{\partial \bF(\bm\theta)}{\partial {\bm\theta}}^\top \frac{\partial \bF(\bm\theta)}{\partial {\bm\theta}}\right\|_2 
$$
where $\frac{\partial \bF(\bm\theta)}{\partial {\bm\theta}} := \left(\frac{\partial f(\mathbf{\bx_1,{\bm\theta}})}{\partial {\bm\theta}}^\top,...,\frac{\partial f(\mathbf{\bx_n,{\bm\theta}})}{\partial {\bm\theta}}^\top\right)^\top\in \mathbb R^{n\times p}$. 
Meanwhile, \citet{karakida2019pathological} pointed out the duality between the FIM and a Gram matrix $\bM$, defined as  
\begin{equation}
\label{eq:M appendix}
    \bm{M} = \frac{2}{n}\frac{\partial \bF(\bm\theta)}{\partial {\bm\theta}} \frac{\partial \bF(\bm\theta)}{\partial {\bm\theta}}^\top
\end{equation}
It also known as the neural tangent kernel NTK
 (\citet{karakida2019pathological,karakida2019normalization}), which has been studied extensively in recent years (see e.g., \cite{jacot2018neural},\cite{du2018gradient},\cite{arora2019fine},\cite{chizat2019lazy}).
Note that in this paper, we do not assume the training is in NTK regime,
in which the Hessian does not change much during training.
It is not hard to see that $\bM$ and FIM share the same non-zero eigenvalues: if $\bm G \bm u = \lambda \bm u$ for some eigenvector $\bm u\in \mathbb R^p$, $$\bM\frac{\partial \bF(\bm\theta)}{\partial {\bm\theta}} \bm u = \frac{\partial \bF(\bm\theta)}{\partial {\bm\theta}}\bm G  \bm u = \lambda \frac{\partial \bF(\bm\theta)}{\partial {\bm\theta}}\bm u,$$
in other words, $\lambda$ is also an eigenvalue of $\bM$.

\noindent\textbf{Sharpness:}
There are various definitions of sharpness in the literature (\cite{li2021happens,wu2018sgd, foret2020sharpness}).
In particular, a popular definition of the sharpness is the largest eigenvalue of the Hessian (\cite{wu2018sgd}).
Based on the above discussion, in this paper, we adopt the largest eigenvalue of $\bM$, $\sha(\bm\theta)=\lambda_{\max}(\bM)$ as the definition of the sharpness, which is a close
approximation of the largest eigenvalue of the Hessian empirically. 

\noindent\textbf{Gradient Descent:}
In this paper, we study the trajectory of gradient descent
and gradient flow \footnote{
Gradient flow is a good approximation of gradient descent in the beginning of the training. See \citet{cohen2021gradient} for more experiments and \citet{ahn2022understanding} for theoretical justification. 
We assume this fact during the progressive sharpening phase. See Assumption~\ref{3.3}. 
}.
We use $\bm\theta(t)$ to denote the parameter at iteration $t$ (or time $t$) and the sharpness at time $t$ as $\sha(t)=\sha(\bm\theta(t))$.
We similarly define $\bM(t), \bF(t), \bD(t),\mathcal{L}(t)$.

Along GD trajectory, the weight vector $\bm\theta(t)$ is updated in the following way:
$$
\bm\theta(t+1) = \bm\theta(t) -\eta\frac{\partial \mathcal L(\bm\theta(t))}{\partial\bm\theta(t)}
$$

When the learning rate $\eta$ is infinitesimal, the GD trajectory above is equivalent to gradient flow trajectory. Here we show the gradient flow dynamics of the residual vector $\bm D(t)$:
\begin{equation}
    \frac{\mathrm d\bm{D}(t)}{\mathrm dt} = \frac{\partial \bD(t)}{\partial {\bm\theta}}\frac{\mathrm d\bm{\theta}(t)}{\mathrm d t}  = - \frac{\partial \bF(t)}{\partial {\bm\theta}}\frac{\partial \mathcal L(t)}{\partial{{\bm\theta}}} = -\frac{2}{n}  \frac{\partial \bF(t)}{\partial {\bm\theta}} \frac{\partial \bF(t)}{\partial {\bm\theta}}^\top\bm{D}(t) = - \bm{M}(t)\bm{D}(t)
    \label{appendix residual dynamics}
\end{equation}

In light of \eqref{appendix residual dynamics}, we have the following approximate update rule of $\bD(t)$ under gradient descent:

$$\bD(t+1)-\bD(t)\approx-\eta\bM(t)\bD(t)$$

\end{document}